\documentclass[11pt]{article}

\usepackage[letterpaper, left=1in, right=1in, top=1in,bottom=1in]{geometry}

\usepackage{enumerate}
\usepackage{paralist}
\usepackage{amsmath,amsfonts,graphpap,amscd,mathtools,mathrsfs,graphicx,lscape,enumitem,dsfont,bm,url,subfigure}
\usepackage{graphpap,amscd,mathrsfs,graphicx,lscape,enumitem,dsfont,bm,url,subfigure}
\usepackage{epsfig,amstext,xspace}
\usepackage{enumerate}
\usepackage{auxiliary}
\usepackage{setspace}

\usepackage{color}              
\usepackage[suppress]{color-edits}
\addauthor{tl}{cyan}
\addauthor{ws}{red}
\addauthor{ms}{blue}
\addauthor{as}{magenta}

\begin{document}
\title{Corruption-robust exploration in episodic reinforcement learning}

 \author{
 Thodoris Lykouris\thanks{Massachusetts Institute of Technology, \texttt{lykouris@mit.edu}. Research conducted while the author was a postdoctoral researcher at Microsoft Research NYC.}
 \and Max Simchowitz\thanks{UC Berkeley, \texttt{msimchow@berkeley.edu}. Research conducted while author was intern at Microsoft Research NYC.} \and Aleksandrs Slivkins\thanks{Microsoft Research NYC, \texttt{slivkins@microsoft.com}} \and
 Wen Sun\thanks{Cornell University, \texttt{ws455@cornell.edu}. Research conducted while the author was a postdoctoral researcher at Microsoft Research NYC.}
 }
\date{First version: November 2019\\Current version: October 2023%
\footnote{The first version of this paper focused on tabular RL (Theorem~\ref{thm:main_tabular}). The extension to linear MDPs (Theorem~\ref{thm:main_linear}) was added in April 2020. A $1$-page abstract appeared in the 34th Annual Conference on Learning Theory (COLT 2021). The full version will appear in \emph{Mathematics of Operations Research}. The current version reflects the content that was accepted in this journal.}}
\maketitle
\begin{abstract}
We initiate the study of multi-stage episodic reinforcement learning under adversarial corruptions in both the rewards and the transition probabilities of the underlying system extending recent results for the special case of stochastic bandits \cite{LykourisMiPa18,gupta2019better}. We provide a framework which modifies the aggressive exploration enjoyed by existing reinforcement learning approaches based on ``optimism in the face of uncertainty'', by complementing them with principles from ``action elimination''. Importantly, our framework circumvents the major challenges posed by naively applying action elimination in the RL setting, as formalized by a lower bound we demonstrate. Our framework yields efficient algorithms which (a) attain near-optimal regret in the absence of corruptions and (b) adapt to unknown levels corruption, enjoying regret guarantees which degrade gracefully in the total corruption encountered. To showcase the generality of our approach, we derive results for both tabular settings (where states and actions are finite) as well as linear-function-approximation settings (where the dynamics and rewards admit a linear underlying representation). Notably, our work provides the first sublinear regret guarantee which accommodates any deviation from purely i.i.d. transitions in the bandit-feedback model for episodic reinforcement learning.
\end{abstract}

\addtocounter{page}{-1}
\thispagestyle{empty}
\newpage


\section{Introduction}
\label{sec:intro}


\newcommand{\tO}{\widetilde{O}}
\newcommand{\ie}{{\em i.e.,~\xspace}}
\newcommand{\eg}{{\em e.g.,~\xspace}}

\newcommand{\rbr}[1]{\left(\,#1\,\right)}
\newcommand{\sbr}[1]{\left[\,#1\,\right]}
\newcommand{\cbr}[1]{\left\{\,#1\,\right\}}



\newcommand{\gapcomplexity}{\mathtt{GapComplexity}}


\newcommand{\thetast}{\theta^{\star}}
\newcommand{\supervisedc}{\textsc{SuperVIsed.C}}
\newcommand{\supervised}{\textsc{SuperVIsed}}
\newcommand{\bellmanlu}{\overline{\underline{\bellman}}}
\newcommand{\model}{\mathfrak{m}}

\newcommand{\kglhpl}{_{k;\mathrm{gl};h+1}}
\newcommand{\supervisedcband}{\textsc{SuperVIsed.Band.C}}
\newcommand{\supervisedtab}{\textsc{SuperVIsed.Tab}}

\newcommand{\mustar}{\mu^{\star}}

\newcommand{\pibase}{\pi^{\textsc{base}}}

\newcommand{\psample}{p_{\mathrm{smp}}}
\newcommand{\Nsp}{N_{\mathrm{smp}}}

\newcommand{\Nksp}{N_{k,\mathrm{sp}}}
\newcommand{\Nkminsp}{N_{k-1,\mathrm{sp}}}
\newcommand{\lemref}[1]{Lemma~\ref{#1}}
\newcommand{\bellman}{\mathbf{Bell}}
\newcommand{\pimaster}{\boldpi^{\textsc{master}}}
\newcommand{\pimasterk}{\pimaster_{k}}

\newcommand{\calX}{\mathcal{X}}
\newcommand{\calA}{\mathcal{A}}
\newcommand{\belhatk}{\belhat_k}
\newcommand{\belhat}{\widehat{\bel}}
\newcommand{\bel}{\mathcal{T}}

\newcommand{\Vset}{\mathscr{V}}

\newcommand{\Thetabar}{\overline{\Theta}}

\newcommand{\eventcorcross}{\mathcal{E}^{\corRL,\mathrm{cross}}}
\newcommand{\Rtil}{\widetilde{R}}

\newcommand{\catelim}{\mathsf{CAT}\text{-}\mathsf{ELIM}}
\newcommand{\catvi}{\mathsf{CAT}}
\newcommand{\cat}{\mathsf{CAT}}

\newcommand{\eventReg}{\calE^{\mathrm{reg}}}

\newcommand{\boldm}{\bm{m}}
\newcommand{\roundevent}{\mathcal{E}}
\newcommand{\qgreedy}{\qup\text{-greedy}}
\newcommand{\qgreedyk}{\qupk\text{-greedy}}
\newcommand{\dcoin}{\mathcal{D}^{\mathrm{coin}}}
\newcommand{\dcoink}{\dcoin_k}
\newcommand{\boldpiki}{\boldpik^{(i)}}
\newcommand{\boldpikboldik}{\boldpik^{(\boldi_k)}}
\newcommand{\pigreedksupi}{\pi^{\mathrm{grd},(i)}}
\newcommand{\boundbar}{\overline{\bound}}
\newcommand{\eventbonus}{\mathcal{E}^{\mathrm{bonus}}}
\newcommand{\eventratio}{\mathcal{E}^{\mathrm{rat}}}
\newcommand{\regretnom}{\mathrm{NominalRegret}}
\newcommand{\eventactiveset}{\mathcal{E}^{\activeset}}
\newcommand{\eventvi}{\mathcal{E}^{\mathrm{vi}}}
\newcommand{\rtilk}{\rtil_k}
\newcommand{\ptilk}{\ptil_k}
\newcommand{\eventcomp}{\mathcal{E}^{\mathrm{comp}}}
\newcommand{\boldi}{\boldsymbol{i}}
\newcommand{\supi}{^{(i)}}

\newcommand{\BigOh}[1]{{\mathcal{O}}\left(#1\right)}
\newcommand{\epSet}{\mathcal{K}}
\newcommand{\type}{\mathsf{type}}
\newcommand{\calH}{\mathcal{H}}
\newcommand{\nSet}{N_{\epSet}}

\newcommand{\BigOhTil}[1]{\widetilde{\mathcal{O}}\left(#1\right)}
\newcommand{\LR}{\mathrm{LR}}
\newcommand{\calB}{\mathcal{B}}
\newcommand{\calZ}{\mathcal{Z}}
\newcommand{\pl}{(1)}
\newcommand{\kl}{_{k,\ell}}
\newcommand{\klh}{_{k,\ell;h}}
\newcommand{\klpr}{_{k,\ell'}}
\newcommand{\klprh}{_{k,\ell';h}}
\numberwithin{equation}{section}
\newcommand{\vboldpi}{V^{\boldpi}}
\newcommand{\gapclip}{\ddot{\gap}}
\newcommand{\ck}{c_k}
\newcommand{\vpi}{\valf^{\pi}}
\newcommand{\nklsb}{N_{k,\ell;sb}}
\newcommand{\eventlocal}{\calE^{\mathrm{loc}}}
\newcommand{\bolddel}{\boldsymbol{\delta}}
\newcommand{\clsb}{C_{\ell;sb}}
\newcommand{\kH}{_{k;H}}
\newcommand{\xtil}{\widetilde{x}}
\newcommand{\eventcross}{\mathcal{E}^{\mathrm{err\,bound}}}
\newcommand{\eventsub}{\mathcal{E}^{\mathrm{sb}}}
\newcommand{\err}{\mathrm{Err}}
\newcommand{\errk}{\err_k}

\newcommand{\lsb}{_{\ell;sb}}

\newcommand{\tlsb}{_{t,\ell;\mathrm{sb}}}
\newcommand{\klgl}{_{k,\ell;\mathrm{gl}}}
\newcommand{\klsb}{_{k,\ell;\mathrm{sb}}}
\newcommand{\wkl}{\weight\kl}
\newcommand{\wklh}{\weight\klh}
\newcommand{\nbar}{\overline{N}}
\newcommand{\nbark}{\nbar_k}
\newcommand{\chat}{\widehat{C}}
\newcommand{\chatl}{\chat_{\ell,sb}}
\newcommand{\eventproblelst}{\mathcal{E}^{\mathrm{lead},\le\lst}}
\newcommand{\eventprobl}{\mathcal{E}^{\mathrm{lead},\ell}}
\newcommand{\Torus}{\mathbb{T}}
\newcommand{\astar}{a_{\star}}
\newcommand{\calI}{\mathcal{I}}
\newcommand{\klsttau}{_{k,\lst;\tau}}
\newcommand{\eventcorrlbonus}{\mathcal{E}^{\mathrm{bonus,CRANE}}}
\newcommand{\kappalelst}{\kappa_{\le\lst}}
\newcommand{\kappal}{\kappa_{\ell}}
\newcommand{\lbar}{\overline{\ell}}
\newcommand{\cbarl}{\cbar_{\ell}}
\newcommand{\eventvalconc}{\mathcal{E}^{\mathrm{val,conc}}}
\newcommand{\eventprobconc}{\mathcal{E}^{\mathrm{prb,conc}}}
\newcommand{\xaxpr}{(x,a,x')}
\newcommand{\goodeps}{\mathcal{K}_{\mathrm{good}}}
\newcommand{\advrange}{\adv^{\mathrm{range}}}
\newcommand{\corrupt}{\mathbf{C}}
\newcommand{\Errlead}{\mathrm{Err}^{\mathrm{val}}}
\newcommand{\klglb}{_{k,\ell,\mathrm{gl}}}
\newcommand{\klsub}{_{k,\ell,\mathrm{sb}}}
\newcommand{\klelst}{_{k,\le \lst}}
		\newcommand{\klst}{_{k,\lst}}

	\newcommand{\Pikl}{\Pi_{k}^{\ell}}
	\newcommand{\Piklelst}{\Pi_{k}^{\le\lst}}
\newcommand{\Pik}{\Pi_k}
\newcommand{\kli}{_{k,l,i}}
\newcommand{\klpl}{_{k,\ell+1}}
\newcommand{\klminh}{_{k,\ell-1;h}}
\newcommand{\piucb}[1]{\pi^{\mathrm{ucb},#1}}
\newcommand{\boldlk}{\boldl_k}
\newcommand{\Finite}[1]{\left[#1\right]_{<\infty}}
\newcommand{\nk}{N_k}
\newcommand{\nkl}{N_{k,\ell}}
\newcommand{\robustLfactor}{\Lfactor^{\mathrm{rbst}}}
\newcommand{\klih}{_{k,\ell,i;h}}
\newcommand{\Index}{\mathcal{I}}
		\newcommand{\klhpl}{_{k,\ell,h+1}}
\newcommand{\lfin}{\boldl^{\mathrm{fin}}}
\newcommand{\pik}{\pi_k}
\newcommand{\envir}{\mathscr{E}}
\newcommand{\boldpielim}{\boldpi^{\mathrm{unif}}}
\newcommand{\boldpielimk}{\boldpielim_k}
\newcommand{\lucbelim}{\mathsf{LUCBVI}\text{-}\mathsf{Elim}}
\newcommand{\Bernoulli}{\mathrm{Bernoulli}}
\newcommand{\False}{\mathsf{False}}
\newcommand{\True}{\mathsf{True}}
\newcommand{\isgreedy}{\mathsf{greedyAction}}
\newcommand{\Uniform}{\mathrm{Uniform}}
\newcommand{\plauset}{\activeset^{\mathrm{plaus}}}
\newcommand{\kHpl}{_{k;H+1}}
\newcommand{\kplh}{_{k+1;h}}
\newcommand{\fulelim}{\mathsf{FullElim}}
\newcommand{\xapr}{(x,a')}
\newcommand{\bonusk}{\bonus_k}
\newcommand{\lmax}{\ell_{\mathrm{max}}}
\newcommand{\klsth}{_{k,\lst;h}}
	\newcommand{\lst}{\ell_{\star}}
		\newcommand{\lkh}{\boldl\kh}
	
\newcommand{\Regret}{\mathrm{Regret}}

\newcommand{\pigreedboldl}{\pi^{\mathrm{grd},\,(\boldl)}}
	\newcommand{\pigreedboldlk}{\pi^{\mathrm{ucb},\,(\boldl)}_k}
	\newcommand{\pigreedboldlkk}{\pi^{\mathrm{ucb},\,(\boldl_k)}_k}
\newcommand{\distlayer}{\bm{\mathsf{P}}}

\newcommand{\boldpik}{\boldsymbol{\pi}_k}
\newcommand{\boldpil}{\boldsymbol{\pi}^{\ell}}
\newcommand{\boldpilk}{\boldsymbol{\pi}^{\ell}_k}

\newcommand{\boldpilel}{\boldsymbol{\pi}^{\le\ell)}}
\newcommand{\boldpilelst}{\boldsymbol{\pi}^{\le\lst}}

\newcommand{\boldpilelk}{\boldsymbol{\pi}^{\le\ell}_k}
\newcommand{\boldpilelstk}{\boldsymbol{\pi}^{\le\lst}_k}
\newcommand{\ellset}{\mathscr{L}}
\newcommand{\ellsetl}{\ellset_{\ell}}
\newcommand{\xkhakh}{(x\kh,a\kh)}

\newcommand{\klplh}{_{k,\ell+1;h}}
\newcommand{\oplustau}{\oplus_{\tau}}
\newcommand{\oplush}{\oplus_{h}}
\newcommand{\oplushbar}{\oplus_{\hbar}}
\newcommand{\vsthpl}{\vst_{h+1}}
\newcommand{\xah}{(x,a,h)}
\newcommand{\gap}{\mathrm{gap}}
\newcommand{\gapmin}{\gap_{\min}}
\newcommand{\clip}[2]{\operatorname{clip}\left[#2|#1\right]}
 \newcommand{\advhalfclip}{\dot{\bellmanup}}
 \newcommand{\phatk}{\phat_k}
\newcommand{\rhatk}{\rhat_k}
\newcommand{\matE}{\mathbf{E}}

 \newcommand{\advlucbfullclip}{\ddot{\bellmanlu}}
 \newcommand{\advlucbhalfclip}{\dot{\bellmanlu}}
\newcommand{\advlucb}{\bellmanlu}
\newcommand{\advfullclip}{\ddot{\bellmanup}}

\newcommand{\gaph}{\mathrm{gap}_h}
\newcommand{\epsclip}{\epsilon_\mathrm{clip}}
\newcommand{\vsth}{\vst_h}
\newcommand{\mup}{\overline{\calM}}
\newcommand{\mlow}{\underline{\calM}}
\newcommand{\mlowk}{\mlow_k}
\newcommand{\mupk}{\mup_k}
\newcommand{\vupk}{\vup_k}
\newcommand{\vlowk}{\vlow_k}
\newcommand{\qlowk}{\qlow_k}
\newcommand{\qupk}{\qup_k}
\newcommand{\pigreedk}{\pigreed_k}
\newcommand{\ofx}{(x)}
\newcommand{\ofxh}{(x_h)}
\newcommand{\pisth}{\pist_{;h}}
\newcommand{\xhah}{(x_h,a_h)}
\newcommand{\optacts}{\mathcal{Z}_{opt}}
\newcommand{\subacts}{\mathcal{Z}_{sub}}
\newcommand{\nipsvspace}[1]{\iftoggle{nips}{\vspace{1}}{}}

\newcommand{\onespace}{\vspace{.1in}}

\renewcommand{\ast}{a^\star}
\newcommand{\xast}{(x,\ast)}
\newcommand{\qsth}{\qst_h}
\newcommand{\bolda}{\mathbf{a}}
\newcommand{\poly}{\mathrm{poly}}
\newcommand{\polylog}{\textrm{polylog}}
\newcommand{\activeset}{\mathscr{A}}
\newcommand{\activesetk}{\activeset_k}
\newcommand{\tuple}{\mathscr{T}}
\newcommand{\tuplek}{\tuple_k}
\newcommand{\Kpl}{K^{\pl}}
\newcommand{\matz}{\mathbf{z}}
\newcommand{\matxpl}{\matx^{\pl}}
\newcommand{\matypl}{\maty^{\pl}}
\newcommand{\calEtil}{\widetilde{\mathcal{E}}}
\newcommand{\kclass}{\mathscr{K}}
\newcommand{\lossid}[1]{\mathcal{L}_{\mathrm{id},#1}}

\newcommand{\advlowgreed}{\advlow^{\mathrm{lcb}}}

\newcommand{\boldl}{\boldsymbol{\ell}}

\newcommand{\qf}{Q}
\newcommand{\vup}{\bar{\valf}}
\newcommand{\qup}{\bar{\qf}}
\newcommand{\qlow}{\makelow{\qf}}
\newcommand{\vlow}{\makelow{\valf}}

\newcommand{\valf}{V}
\newcommand{\vst}{\valf^{\star}}
\newcommand{\qst}{\qf^{\star}}
\newcommand{\pist}{\pi^{\star}}

\newcommand{\qcirc}{\mathring{\qf}}
\newcommand{\vcirc}{\mathring{\valf}}
\newcommand{\pibar}{\overline{\pi}}
\newcommand{\adv}{\bellmanup}
\newcommand{\advcirc}{\mathring{\adv}}

\newcommand{\rst}{r^{\star}}
\newcommand{\pst}{p^{\star}}

\newcommand{\bellmanup}{\overline{\mathbf{Bell}}}
\newcommand{\bellmanlow}{\underline{\mathbf{Bell}}}

\newcommand{\occ}{\boldsymbol{\omega}}

\newcommand{\phat}{\widehat{p}}
\newcommand{\rhat}{\widehat{r}}

\newcommand{\rhatkhl}{\rhat_{k,h,\ell}}
\newcommand{\khl}{_{k;h;l}}
\newcommand{\khbar}{_{k;\hbar}}

\newcommand{\khboldl}{_{k;h,\boldl}}

\newcommand{\khlpr}{_{k;h,l'}}
\newcommand{\khpl}{_{k;h+1}}
\newcommand{\khpr}{_{k;h'}}

\newcommand{\khlh}{_{k;h,\boldl_h}}
\newcommand{\xa}{(x,a)}

\newcommand{\boundclipsah}{\tilde{\bound}}
\newcommand{\boundclip}{\check{\bound}}

\newcommand{\makelow}[1]{\mkern2mu\underline{\mkern-2mu#1\mkern-4mu}\mkern4mu }

\newcommand{\bonus}{b}
\newcommand{\bonuskhl}{\bonus_{k,h,\ell}}
\newcommand{\bonusbar}{\overline{\bonus}}
\newcommand{\bonuscirc}{\mathring{\bonus}}

\newcommand{\nkhl}{n_{k,h,\ell}}
\newcommand{\Lfactor}{\mathbf{L}}
\newcommand{\Lcorrupt}{\mathbf{L}_{\mathrm{crpt}}}
\newcommand{\Lf}{\Lfactor}

\newcommand{\cbon}{c_{\bonus}}
\renewcommand{\hbar}{\tau}
\newcommand{\hboldbar}{\overline{\boldsymbol{\tau}}}

\newcommand{\Otilde}{\widetilde{O}}

\renewcommand{\ast}{a^\star}

\newcommand{\kh}{_{k;h}}
\newcommand{\ktau}{_{k;\tau}}

\newcommand{\kone}{_{k;1}}

\newcommand{\boldpi}{\boldsymbol{\pi}}
\newcommand{\pigreed}{\pi^{\mathrm{ucb}}}
\newcommand{\piexp}{\pi^{\mathrm{exp}}}
\newcommand{\pimix}{\pi}
\newcommand{\pilcb}{\pi^{\mathrm{lcb}}}
\newcommand{\calMlow}{\makelow{\calM}}
\newcommand{\advlow}{\bellmanlow}
\newcommand{\bonuslow}{\makelow{\bonus}\,}

\newcommand{\weight}{\boldsymbol{\omega}}
\newcommand{\weighthbar}{\weight^{\hbar}}

\newcommand{\khhbar}{_{k,h;\hbar}}
\newcommand{\khhboldbar}{_{k,h;\hboldbar}}
\newcommand{\dens}{\rho}
\renewcommand{\colon}{{\,:\,}}
\newcommand{\calM}{\mathcal{M}}

\newcommand{\calMbar}{\overline{\calM}}
\newcommand{\pbar}{\overline{p}}
\newcommand{\rbar}{\overline{r}}

\newcommand{\states}{\mathcal{X}}
\newcommand{\actions}{\mathcal{A}}
\newcommand{\outputs}{\mathcal{Y}}
\newcommand{\matx}{\mathbf{x}}
\newcommand{\mata}{\mathbf{a}}

\newcommand{\Gclass}{\mathscr{G}}
\newcommand{\Ust}{U_{\star}}
\newcommand{\Vst}{V_{\star}}
\newcommand{\sfP}{\mathsf{P}}

\newcommand{\Ast}{A^{\star}}
\newcommand{\Bst}{B_{\star}}
\newcommand{\Kst}{K_{\star}}
\newcommand{\calN}{\mathcal{N}}
\newcommand{\Khat}{\widehat{K}}
\newcommand{\matU}{\mathbf{U}}
\newcommand{\matDel}{\boldsymbol{\Delta}}
\newcommand{\Rhat}{\widehat{R}}

\newcommand{\Proj}{\mathrm{Proj}}
\newcommand{\ProjK}{\mathrm{Proj}_{\Kst}}
\newcommand{\ProjKzero}{\mathrm{Proj}_{\mathbf{0} \oplus \Kst}}
\newcommand{\veczero}{\mathbf{0}}

\newcommand{\Stief}{\mathrm{Stief}}
\newcommand{\nablabar}{\overline{\nabla}}
\newcommand{\nablacheck}{\check{\nabla}}
\newcommand{\fcheck}{\check{f}}
\newcommand{\Gcheck}{\check{G}}

\newcommand{\calK}{\mathcal{K}}
\newcommand{\calE}{\mathcal{E}}

\newcommand{\Rcheck}{\check{R}}

\newcommand{\xbar}{\overline{x}}
\newcommand{\xcheck}{\check{x}}







\renewcommand{\Pr}{\mathbb{P}}
\newcommand{\Exp}{\mathbb{E}}
\newcommand{\Var}{\mathrm{Var}}
\newcommand{\Cov}{\mathrm{Cov}}

\newcommand{\Law}{\mathrm{Law}}
\newcommand{\Leb}{\mathrm{Lebesgue}}

\newcommand{\unifsim}{\overset{\mathrm{unif}}{\sim}}
\newcommand{\iidsim}{\overset{\mathrm{i.i.d.}}{\sim}}
\newcommand{\probto}{\overset{\mathrm{prob}}{\to}}
\newcommand{\ltwoto}{\overset{L_2}{\to}}

\newcommand{\info}{\mathbf{i}}
\newcommand{\KL}{\mathrm{KL}}
\newcommand{\Ent}{\mathrm{Ent}}
\newcommand{\TV}{\mathrm{TV}}
\newcommand{\rmd}{\mathrm{d}}

\newcommand{\median}{\mathrm{Median}}
\newcommand{\range}{\mathrm{range}}
\newcommand{\im}{\mathrm{im }}
\newcommand{\tr}{\mathrm{tr}}
\newcommand{\rank}{\mathrm{rank}}
\newcommand{\spec}{\mathrm{spec}}
\newcommand{\diag}{\mathrm{diag}}
\newcommand{\Diag}{\mathrm{Diag}}
\newcommand{\sign}{\mathrm{sign\ }}

\newcommand{\op}{\mathrm{op}}
\newcommand{\F}{\mathrm{F}}
\newcommand{\dist}{\mathrm{dist}}

\renewcommand{\implies}{\text{ implies }}
\renewcommand{\iff}{\text{ iff }}
\newcommand{\matX}{\mathbf{X}}
\newcommand{\matbeta}{\boldsymbol{\beta}}
\newcommand{\mattheta}{\boldsymbol{\theta}}

\newcommand{\betast}{\matbeta^*}
\newcommand{\betahat}{\widehat{\matbeta}}

\newcommand{\thetahat}{\widehat{\theta}}
\newcommand{\supp}{\mathrm{supp}}

\newcommand{\matw}{\mathbf{w}}
\newcommand{\maty}{\mathbf{y}}

\renewcommand{\Im}{\mathfrak{Im }}
\renewcommand{\Re}{\mathfrak{Re }}

\newcommand{\Z}{\mathbb{Z}}

\newcommand{\I}{\mathbf{1}}
\newcommand{\Q}{\mathbb{Q}}

\newcommand{\sphered}{\mathcal{S}^{d-1}}


\newcommand{\BigThetaTil}[1]{\widetilde{\BigTm}\left({#1}\right)}
\newcommand{\BigOmega}[1]{\BigWm\left({#1}\right)}

\newcommand{\Rsimple}{\mathcal{R}_{\mathrm{simp}}}
\newcommand{\calF}{\mathcal{F}}
\newcommand{\actst}{a_*}
\newcommand{\acthat}{\widehat{a}}

\newcommand{\lasso}{\mathsf{lasso}}
\newcommand{\four}{\mathscr{F}}
\newcommand{\LS}{\mathsf{LS}}
\newcommand{\Alg}{\mathsf{Alg}}

\newcommand{\vi}{\matv^{(i)}}

\newcommand{\SAT}{\mathit{SAT}}
\renewcommand{\P}{\mathbf{P}}
\newcommand{\NP}{\mathbf{NP}}
\newcommand{\coNP}{\co{NP}}
\newcommand{\co}[1]{\mathbf{co#1}}

\newcommand{\Drew}{\mathcal{D}_{\mathrm{rew}}}

\newcommand{\Drewtil}{\widetilde{\mathcal{D}}_{\mathrm{rew}}}
\newcommand{\ptil}{\widetilde{p}}
\newcommand{\rtil}{\widetilde{r}}
\newcommand{\calMtil}{\widetilde{\calM}}
\newcommand{\nomMDP}{{\calM}}
\newcommand{\didcorrupt}{c}
\newcommand{\matR}{\mathbf{R}}

\newcommand{\gapclipmin}{\ddot{\gap}_{\min}}

\newcommand{\ucb}{\mathrm{ucb}}
\newcommand{\lucb}{\mathrm{lucb}}
\newcommand{\lcb}{\mathrm{lcb}}
\newcommand{\aae}{\mathrm{aae}}

\newcommand{\boundlead}{\bound^{\mathrm{local}}}
\newcommand{\boundcross}{\bound^{\mathrm{cross}}}
\newcommand{\boundtotal}{\bound^{\mathrm{total}}}
\newcommand{\xau}{(x,a;u)}
\newcommand{\bound}{\mathbf{B}}

\newcommand{\cbar}{\overline{C}}
\newcommand{\scale}{scl}

\newcommand{\maxgap}{\check{\gap}}

\newcommand{\corRL}{\mathsf{CRANE}\text{-}\mathsf{RL}}

\newcommand{\pigreedboldllelst}{\boldpi^{\mathrm{grd},\,(\boldl_{\leq\lst})}}
  \newcommand{\pigreedboldlell}{\boldpi^{\mathrm{grd},\,(\boldl_{k,H} = \ell)}}
  \newcommand{\problesslst}{\mu(\mathcal{L}_{\leq \lst})}
  \newcommand{\probell}{\mu(\mathcal{L}_{\ell})}
  \newcommand{\klstgl}{_{k, \lst; \mathrm{gl}}}
  \newcommand{\klsthpr}{_{k, \lst; h'}}
  \newcommand{\klhpr}{_{k,\ell; h'}}

 \newcommand{\wkih}{\weight_{k,i;h}}
  \newcommand{\wki}{\weight_{k,i}}
  \newcommand{\wti}{\weight_{t,i}}
  \newcommand{\Nsample}{N_{\mathrm{sample}}}
  \newcommand{\nbarki}{\nbar_{k,i}}
  \newcommand{\ki}{_{k,i}}
  \newcommand{\kih}{_{k,i,h}}
  \newcommand{\boundbari}{\boundbar\supi}
  \newcommand{\boundbarik}{\boundbari_k}
    \newcommand{\constc}{\alpha}
    \newcommand{\constcbar}{\bar{\constc}}
    \newcommand{\Mbar}{\bar{\constM}}
    \newcommand{\constM}{\beta}

  \newcommand{\nbari}{\nbar_{K,i}}

\newcommand{\nki}{N_{k,i}}

\newcommand{\boundk}{\bound_k}
\newcommand{\xtauatau}{(x_{\tau},a_{\tau})}
\newcommand{\wkh}{\weight\kh}
\newcommand{\wk}{\weight_k}
\newcommand{\eventsample}{\mathcal{E}^{\mathrm{sample}}}
\newcommand{\nbarK}{\nbar_K}

\newcommand{\eventktype}{\calE_{k,i}}

\newcommand{\coin}{\mathrm{coin}}
\newcommand{\Hkalg}{\calH_{k,\coin}}
\newcommand{\Hcoinatk}{\calH_{[k],\coin}}

\newcommand{\Hk}{\calH_{k}}
\newcommand{\Dcoink}{\mathcal{D}_{\coin,k}}
\newcommand{\episodeset}{\mathcal{S}}
\newcommand{\eptype}{\mathrm{type}}

\newcommand{\pitil}{\widetilde{\pi}}

\newcommand{\eventadm}{\calE^{\mathrm{adm}}}

\newcommand{\boldPi}{\mathbf{\Pi}}

\newcommand{\layerVI}{\mathsf{LayerVI}}
\newcommand{\klmaxh}{_{k,\lmax;h}}
\newcommand{\layerShare}{\mathsf{LayerShare}}
\newcommand{\layerSched}{\mathsf{LayerSchedule}}
\newcommand{\veclam}{\vec{\lambda}}
 \newcommand{\evklelst}{\calE_{k,\le \lst}}
  \newcommand{\evkl}{\calE_{k,\ell}}
\newcommand{\tH}{_{t;H}}
\newcommand{\Lawk}{\mathrm{Law}_k}
\newcommand{\klelsthpl}{_{k,\le \lst;h+1}}
\newcommand{\klsthpl}{_{k,\lst;h+1}}
\newcommand{\klelsth}{_{k,\le \lst;h}}
\newcommand{\errcross}{\err^{\mathrm{cross}}}
\newcommand{\boundik}{\bound\supi_k}

\newcommand{\algcomment}[1]{\texttt{\color{blue} #1}}
\newcommand{\getModel}{\mathsf{getModel}}
\newcommand{\getBonus}{\mathsf{getBonus}}
\newcommand{\getModelBonus}{\mathsf{getModelBonus}}

\newcommand{\datacorrupt}{\mathcal{T}^{\,\mathrm{corrupt}}}
\newcommand{\Vsthpl}{V^{\star}_{h+1}}
\newcommand{\calS}{\mathcal{S}}

\newcommand{\datacorruptsbl}{\calD_{\ell}^{\,\mathrm{corrupt}}}
\newcommand{\kgl}{_{k;\mathrm{gl}}}
\newcommand{\xstoch}{x^{\mathrm{stoch}}}
\newcommand{\astoch}{a^{\mathrm{stoch}}}
\newcommand{\rstoch}{R^{\,\mathrm{stoch}}}
\newcommand{\datastoch}{\mathcal{T}^{\,\mathrm{stoch}}}
\newcommand{\phatstoch}{\phat^{\mathrm{\,stoch}}}
\newcommand{\rhatstoch}{\rhat^{\mathrm{\,stoch}}}
\newcommand{\modelstoch}{\mathfrak{m}^{\mathrm{stoch}}}

\newcommand{\eventsubest}{\mathcal{E}^{\mathrm{sub,est}}}
\newcommand{\eventsubsmp}{\mathcal{E}^{\mathrm{sub}}}
\newcommand{\eventglest}{\mathcal{E}^{\mathrm{gl,est}}}

\newcommand{\jh}{_{j;h}}
\newcommand{\jhpl}{_{j;h+1}}

\newcommand{\kglell}{_{k,\ell;\mathrm{gl}}}
\newcommand{\ksbell}{_{k,\ell;\mathrm{sb}}}
\newcommand{\stoch}{\text{stoch}}
\newcommand{\muhat}{\widehat{\mu}}

\newcommand{\deltaeff}{\delta_{\mathrm{eff}}}

\newcommand{\regret}{\emph{Reg}}
 \newcommand{\R}{\mathbb{R}}
\newcommand{\eventconc}{\mathcal{E}^{\mathrm{conc}}}

\newcommand{\xhdr}[1]{\vspace{1mm} \noindent{\bf #1}}

In reinforcement learning (RL), an agent encounters a particular state and decides which action to select; as a result, the agent transitions to a new state and collects some reward. Standard RL approaches assume that rewards and transition dynamics are drawn identically and independently from fixed (yet unknown) distributions that depend on the current state and the selected action. However, these techniques tend to be vulnerable to even a small amount of outliers from such i.i.d. patterns. Such outliers are prevalent in most RL applications, \eg click fraud in online advertising, patients not following prescriptions in clinical trials, attacks against RL agents in computer gaming.

We focus on \emph{episodic RL}, a basic paradigm in which time is partitioned into \emph{episodes} of fixed length $H$, and the agent's
state is reinitialized in each episode. 
The algorithm only observes the outcomes of their chosen actions, that is, past trajectories consisting of states visited and rewards received.. This is referred to as \emph{bandit feedback}. We consider one model that captures outliers, that of \emph{adversarial corruptions}. This model posits that while most episodes share stochastic, i.i.d. reward and state-transition distributions, some episodes are \emph{corrupted} such that rewards and transitions are selected by an adversary. The number of episodes that are corrupted, denoted by $C$, is not known to the agent. The adversary can choose an arbitrary sequence of episodes to be corrupted, \eg all corruptions may happen in the initial episodes, which causes irrevocable damage to standard algorithms that rely on these rounds for exploration. The goal is to design \emph{corruption-robust} algorithms, i.e., algorithms whose performance gracefully degrades as $C$ becomes larger while retaining the near-optimal bounds for the i.i.d. setting (when $C=0$).

This model is well-understood in the special case of multi-armed bandits, where there is only one state \cite{LykourisMiPa18,gupta2019better,ZimmertSeldin21}. However, the main challenge in RL, absent in multi-armed bandits, lies in effectively learning the transition dynamics. Moreover, the effective ``decision space'' is exponential in the episode duration.

The algorithms which have been successful either for efficient RL or corruption-robust multi-armed bandits can be divided into three categories, each of which appears problematic for corruption-robust RL. The most popular approach to RL (and a popular approach for numerous  stochastic decision-making problems) has been the ``optimism under uncertainty'' paradigm. Here, a learner selects the action which has the greatest potential for resulting in high reward. Unfortunately, as demonstrated in  \cite{LykourisMiPa18}, these algorithms are easily manipulated, even in simpler decision-making domains such as multi-armed bandits. A second class of algorithms, initially proposed for multi-armed bandits and which has proven particularly effective in corruption-robust settings \cite{LykourisMiPa18,gupta2019better}, follow action elimination: here, the learner identifies a set of ``active actions and samples uniformly from this set \cite{Even-DarManMan06}. Unfortunately, this strategy is far too passive for episodic RL, as we demonstrate in Section~\ref{sec:lower_bound_active_arm_elimination} that naive elimination suffers exponential regret in the episode length $H$. Finally, algorithms based on mirror descent, initially studied for adversarial bandit decision-making  (e.g. the celebrated EXP3 \cite{auer2002nonstochastic} algorithm), have recently proven successful for corruption-robust bandits \cite{ZimmertSeldin21}. Extensions to RL remain elusive, however, because it is unclear how to incorporate corrupted-transition feedback into these approaches.~\footnote{This is why the concurrent work of \cite{jin2019learning}, which tackles RL with adversarial rewards, requires transitions to be i.i.d.; and the follow-up work of \cite{JinLuo20} which simultaneously handles RL with i.i.d. and adversarial rewards makes an even stronger assumption of known transitions.}

\subsection{Our contribution}
We present a new approach to corruption-robust exploration. We instantiate it to the \emph{tabular} and \emph{linear MDP} variants of episodic RL (the main two variants that have been studied), achieving the aforementioned desiderata. Our results achieve some ``firsts'' for the broader problem of RL with non-i.i.d. transition dynamics: the first non-trivial statistical guarantees in the bandit feedback model, and the first computationally efficient guarantees in any feedback model.

In tabular episodic RL, the guarantees are expressed in terms of $S$ and $A$: the number of states and actions, resp., which are finite. When there are $K$ episodes of length $H$ (hence, $T=KH$ total steps), we obtain the following regret bound:
\begin{align}\label{eq:intro-regret-tabular}
\Regret =  (C+1)\cdot \tO
 (\; \min\{\;  \sqrt{SAT},\; \gapcomplexity \;\} + AS(C+S) \;)\cdot\text{poly}(H).
 \end{align}
Specializing to constant corruption level $C$,
we achieve dependence on $S,A,T$ which is worst-case optimal even for the i.i.d. case
\cite{AzarOsMu17,dann2017unifying,auer2002nonstochastic}. Following most prior work, we achieve $\poly(H)$ dependence on $H$, and do not optimize it further.~\footnote{That said, (only) one extra factor of $H$ appears inherent in our technique, compared to the state-of-art.}
The dependence on $C$ is multiplicative, whereas it is additive for bandits \cite{gupta2019better}.

$\gapcomplexity$ measures the complexity of a problem instance in terms of the ``gaps" between the actions' $Q$-function values; see \eqref{eq:gap_complexity}. It can be as low as $SAH^2$ for instances with sharp separation between optimal and ``bad'' actions,\footnote{This happens when rewards are only collected at end-states (\eg when some tasks are completed or the game ends), and each action either allows reaching the best end-state deterministically, or disallows it entirely. In particular, this arises in \emph{combination lock} instances.}
yielding regret bounds with only $\log(T)$ dependence on $T$. $\gapcomplexity$ and similar notions have been studied in the i.i.d. case \cite{SimchowitzJamieson19,du2019provably,JinLuo20, he2020logarithmic}.
For constant $C$, our dependence on $\gapcomplexity$ is optimal even for the i.i.d. case \cite{SimchowitzJamieson19}.

We then turn to the linear-MDP setting, where the state space may be infinite, but rewards and transitions admit a linear underlying representation with a $d$-dimensional feature embedding. We obtain the following regret bound:
\begin{align}\label{eq:intro-regret-linear}
\Regret =   (C+1)\cdot \sqrt{dT}\cdot
    \tO(\; C+ \sqrt{A}  + d \;) \cdot \poly(H).
\end{align}
If $A<\tO(d^2)$, our regret is $d^{1.5} \sqrt{T}$ which is state-of-the-art for the i.i.d. case \cite{jin2019provably}.~\footnote{Yang and Wang~\cite{YangWang2019kernel} achieve linear dependence on $d$ for the i.i.d. environment, under additional assumptions.}

\xhdr{Our techniques.} On the algorithmic front, we extend the multi-armed bandit algorithm of Lykouris, Mirrokni, and Paes Leme \cite{LykourisMiPa18} by introducing two key algorithmic ideas to overcome fundamental roadblocks in incorporating active sets. The multi-armed bandit algorithm runs in parallel multiple \emph{base} bandit algorithms each with increasing tolerance to corruption that successively eliminate actions. Some algorithms are faster in recovering the optimal actions ensuring the efficient regret when the input is (mostly) i.i.d. but are prone to corruptions; more robust ones correct these mistakes and adapt to the corruption level. The first challenge that arises in RL is similar to the exponential lower bound of Section~\ref{sec:lower_bound_active_arm_elimination} and arises due to a trajectory mismatch between base learning algorithms. To deal with it, we design a careful episode schedule across them. The second challenge is that classical RL algorithms do not provide certificates of suboptimality that can enable more robust base algorithms to supervise less robust ones. To deal with it, we increase the interaction across base algorithms; in particular each base algorithm operates \emph{not only} with its local estimates \emph{but also} with global estimates. We expand on the algorithmic contribution in Section~\ref{sec:algorithm}.

On the analysis front, we provide a flexible framework that goes beyond optimism under uncertainty, the dominant paradigm in prior work. Our framework,  termed \textsc{SuperVIsed},  enables us to design RL algorithms that deviate from ``optimism" and compare their performance to that of the ``optimistic'' RL algorithms. In particular, the framework allows us to meaningfully combine ideas from ``optimism'' and ``action elimination''. Since the latter has proven useful in many bandit problems beyond corruption-robustness, we hope that our framework will find applications beyond corruption-robustness. We describe this framework in Section~\ref{sec:our_framework_RL}. Notably, our framework captures tabular and linear MDPs in a unified, model-based way. The latter point is of independent interest  even in the i.i.d. case.~\footnote{Such a unification was also achieved concurrently to our work by \cite{NeuPikeBurke20}.}

\xhdr{Other results.}
We also consider simpler solutions that do not fully meet our desiderata. If one is only interested in the worst-case bound in \eqref{eq:intro-regret-tabular}, we propose a simple adaptation of algorithm UCBVI \cite{AzarOsMu17}. This adaptation achieves corruption-robustness against any $C<\sqrt{T/SA}$, with regret bound that is state-of-art for the i.i.d. case. However, this approach does not yield any results in terms of gap-complexity. A similar adaptation for the linear-MDP variant builds on the algorithm from \cite{jin2019provably} and infuses it with a technique from robust statistics. Given $\overline{c}$, it achieves corruption-robustness against any $C\leq \overline{c}$
with regret $\poly(H)\cdot \tO(\; \sqrt{dT}(d + \overline{c}) \;)$.
Note that \eqref{eq:intro-regret-linear} essentially replaces the upper bound $\overline{c}$ with the actual corruption level $C$. Both algorithms, as well as their analyses, are building blocks for the respective main results and are simple in that they do not require substantially new ideas. We defer further discussion to Appendix~\ref{app:discussion}. 

\subsection{Related work} We study an RL model that features bandit feedback, non-i.i.d. rewards and non-i.i.d. transitions. In the \emph{bandit feedback} setting of reinforcement learning, each step the learner only observes the outcomes (reward and new state) of the chosen actions. This is in contrast to the \emph{generative} setting, where the learner has access to a simulator that allows her to sample transitions and rewards from any state-action pairs she chooses (\eg \cite{Even-DarManMan06}). 
Prior work contains no non-trivial statistical guarantees for such models.

\xhdr{Reinforcement learning with adversarial intervention.} 
``Easier'' RL models with adversarial rewards and i.i.d. transitions have been studied previously.
These models assume either known transition probabilities \cite{neu2010online,zimin_neu_2013online} 
or full feedback on rewards~\footnote{That is, rewards for all states and all actions are observed at the end of each episode}
\cite{neu2012adversarial,rosenberg2019online};
the initial paper \cite{even2009online} assumes both. Abbasi-Yadkori, Bartlett, Kanade, Seldin, and Szepesvari \cite{NIPS2013_4975} posit adversarial rewards \emph{and} transitions, but assume full feedback on both. Their algorithm is statistically efficient against a given policy class $\Pi$. Its computational complexity scales $\poly(|\Pi|)$, whereas, in our setting, there are $A^S$ possible policies. 

Closer to our model is the bandit-feedback setting, where rewards are adversarial, and are only observed for state-action pairs visited. Neu, Antos, Gy\"orgy, and Szepesv\'ari \cite{neu2010online} give regret bounds for this setting, under known transition distributions. Concurrently to our work, Jin, Jin, Luo, Sra, and Yu \cite{jin2019learning} provide regret bounds under adversarial rewards \emph{and} unknown transition dynamics in the tabular setting. Interpreted in the lens of corruptions, they assume that the adversary can only corrupt rewards and not transition dynamics. In this setting, their result is better than ours when the corruption level is not too small. On the other hand, our technique allows corruption to also occur in the transition dynamics, extends to linear-function-approximation settings, enjoys optimal dependence on the state in the tabular setting, and also provides logarithmic regret when the gaps between suboptimal actions is large.

\xhdr{Optimism in reinforcement learning.} `Optimism under uncertainty'' has been the dominant paradigm for
(the i.i.d. bandit-feedback setting of) episodic reinforcement learning. According to this paradigm, the learner greedily selects a policy that maximizes an upper confidence bound on the Q-function (a standard notion which measures the largest possible reward associated with selecting a given action at a certain state). This paradigm has been applied broadly in numerous variants of multi-armed bandits, starting from \cite{auer2002finite}. Our algorithms and analyses build on ``optimistic value iteration'' technique, a version of optimism under uncertainty which has been applied extensively in the tabular setting \cite{jaksch2010near,dann2015sample,AzarOsMu17,dann2017unifying,zanette2019tighter,SimchowitzJamieson19}, as well as in the linear-function-approximation setting \cite{YangWang2019kernel,jin2019provably,wang2019optimism}.

For the tabular setting with $S$ states and $A$ actions, a careful instantiation of this technique yields the minimax-optimal regret rate of $\sqrt{HSAT}$, up to lower-order terms \cite{AzarOsMu17,dann2017unifying}. Improvements due to Zanette \& Brunskill \cite{zanette2019tighter} achieve adaptive, problem-dependent regret guarantees, and Simchowitz \& Jamieson \cite{SimchowitzJamieson19} demonstrate that a certain class of optimistic value iteration algorithms attains $\log(T)$ instance-dependent regret bounds. For constant corruption levels, our bound matches the minimax and gap dependent rates up to logarithmic and $H$ factors. Two recent algorithms depart from optimistic value iteration: the Q-learning approach from Jin, Allen-Zhou, Bubeck, and Jordan \cite{jin2018q}, and an algorithm augmenting the observations with random perturbations by Russo
\cite{russo2019worst}. The former algorithm also attains $\sqrt{SAT \poly(H)}$
regret for worst-case i.i.d. distributions, whereas the latter algorithm suffers suboptimal dependence on $S$ in its regret bound. In both papers, the learner's policy still maximizes an estimate of the Q-function.

Up to logarithmic factors, the state-of-the-art guarantee for the linear-function-approximation setting we consider here is $\poly(H)\sqrt{d^3 T}$, where $d$ is the feature dimension \cite{jin2019provably}. Although this result does not have statistical dependence on the number of actions $A$, its runtime is linear in $A$.  Similar guarantees have been recently established for generalized linear models \cite{wang2019optimism}, and sharper guarantees are attainable with modified assumptions \cite{YangWang2019kernel}.

\xhdr{Adversarial corruptions.} The model of adversarial corruptions studied in this paper was introduced by Lykouris, Mirrokni, and Paes Leme \cite{LykourisMiPa18} in the bandit setting. They provide a high-probability regret bound that scales similarly to our $\log(T)$ regret bound. This regret bound has been improved by Gupta, Koren, and Talwar \cite{gupta2019better}, moving the dependence on the corruption level $C$ to an additive term. Zimmert and Seldin \cite{zimmert2018tsallisinf}  further optimize the dependence on the number of actions via an elegant modification of mirror descent, for the weaker notion of pseudoregret and assuming that the best action is unique. All these results crucially depend on selecting actions randomly according to their empirical gaps. Such techniques are not known to work in reinforcement learning and, in some cases there are fundamental limitations in applying them (as our lower bound on uniform elimination suggests). The model of adversarial corruptions has subsequently been extended to several other stateless problems such as linear optimization \cite{LiLouShan2019}, assortment optimization \cite{ChenKrishWang2019}, Gaussian process bandit optimization \cite{bogunovic}, dueling bandits \cite{AgarwalAgarwalPatil21}, dynamic pricing \cite{KrishnamurthyLykPod2020corruptedbinarysearch, Chen2020RobustDP}, learning product rankings \cite{GolrezaeiManSchSek21} and prediction with expert advice \cite{AmirAttiasKorenLivniMansour2020predictioncorrupted}. 

A somewhat more general theme is \emph{stochastic adaptivity} in bandits: adapting to stochastic inputs in adversarial environments. Bubeck and Slivkins \cite{BestofBoth-colt12} obtained the first \emph{``best of both words"} result: an algorithm with simutaneous near-optimal guarantees in both adversarial and stochastic settings. Subsequent work in this theme \cite{BestofBoth-icml14, Auer-colt16,Seldin-colt17,WeiL18,Haipeng-BoB019} and the corruption-robustness papers mentioned above,
has investigated more natural algorithms, more refined regret bounds, and ``intermediate regimes" between stochastic and adversarial. However, unlike the bandit setting where typically there exist effective regret guarantees for the ``adversarial regime'', in reinforcement learning, there is no such guarantee when transitions are adversarial. Our work is therefore a first step towards effective RL  algorithms in the bandit feedback setting when the transitions are not i.i.d.

\subsection{Subsequent work on our model}
After our technical report appeared on \texttt{arxiv.org}, several papers studied the corruption-robust RL model that we introduced. Chen, Du, and Jamieson \cite{ChenDuJamieson21} improve the dependence on the number of corrupted episodes $C$ in the tabular MDP variant; in particular their worst-case bound is $\tilde{\bigO}(\sqrt{T}+C^2)$ improving upon our guarantee of $\tilde{\bigO}((1+C)\sqrt{T}+C^2)$. However, their approach is computationally inefficient, suboptimal in the number of states $S$, and does not seem to extend to linear MDPs. Subsequently, Wei, Dann, and Zimmert \cite{WeiDannZimmert22} offer a model selection approach that uses our algorithm as a building block to provide the optimal regret guarantee of $\tilde{\bigO}(\sqrt{T}+C)$. Their algorithm is computationally efficient for the tabular MDP variant but is computationally inefficient for the linear MDP variant. Both of these results also provide gap-dependent refined guarantees for a weaker notion of gap (implying a larger gap complexity); we refer the reader to \cite{WeiDannZimmert22} for a detailed comparison. On an orthogonal note, Zhang, Chen, Zhu, and Sun \cite{zhang2021robust} tackle a setting where $C=\Theta(T)$ but their guarantee is PAC-based and thus results to a suboptimal dependence on $T$ when transformed into regret when $C$ is relatively small.

\subsection{Concurrent and subsequent work on non-stationary RL}

We now discuss other concurrent and subsequent developments in episodic RL with bandit feedback in a non-stationary environment. We provide direct comparison with our results when appropriate.

\xhdr{Adversarial rewards and i.i.d. transitions.}
One line of work studies online RL with adversarial rewards but 
i.i.d. transitions. The concurrent work of Jin, Jin, Luo, Sra, and
Yu \cite{jin2019learning} achieve  near-optimal regret. Subsequent to our work, Jin and Luo \cite{JinLuo20} designed an algorithm that achieved similar guarantees while simultaneously obtaining optimal gap-dependent guarantees when the rewards are i.i.d.. This result assumes known transition probabilities, but was later  extended to unknown but i.i.d. transitions \cite{JinHuangLuo21}. All these algorithms rely on importance-weighted estimates for the Q-function, which does not easily extend to ``corrupted'' transition probabilities. Similar ``best-of-both-worlds'' results have been previously obtained for multi-armed bandits, as discussed above.
Moreover, Neu and Olkhovskaya~\cite{neu_olkhovskaya_2020} handle adversarial MDPs in the linear-MDP variant, given access to an exploratory policy.

\xhdr{Switching RL.}
Another line of work posits that the underlying reward and transition distributions switch at most $L$ times.
Cheung, Simchi-Levi, and Zhu~\cite{cheung2023nonstationary} initiated the study of this model in the tabular-MDP setting, achieving $L^{1/4}T^{3/4}$ rate for dynamic regret, i.e., comparing to the best sequence of policies. Subsequently,  Domingues, M\'enard, Pirotta,
Kaufmann, and Valko \cite{domingues2020kernel} improved this to $L^{1/3}T^{2/3}$ and Wei and Luo \cite{WeiLuoCOLT21} provided the optimal $\sqrt{LT}$ guarantee; see \cite{WeiLuoCOLT21} for detailed related work in that direction. The first two works only apply to the tabular-MDP variant, while the latter also extends to the linear-MDP variant.

Our model of adversarial corruptions can be interpreted as a special case of Switching RL with $L=2C$ switches: at each corrupted round, we have one switch from the nominal MDP to the corrupted one, and another switch back to the nominal MDP. Let us compare our guarantees with those in \cite{cheung2023nonstationary,domingues2020kernel,WeiLuoCOLT21}.\footnote{\cite{cheung2023nonstationary} is essentially concurrent with our work, and \cite{domingues2020kernel,WeiLuoCOLT21} are subsequent to our work.}
Letting $L=2C=T^a$ to enable a direct comparison, our gap-dependent bound is of the order $T^a$ and our worst-case bound is of the order $T^{1/2+a}$.  With respect to worst-case guarantees, our bound compares favorably to the one in \cite{cheung2023nonstationary} if and only if $a<\frac{1}{3}$, to \cite{domingues2020kernel} if and only if $a<\frac{1}{6}$, and is always worse than the one in \cite{WeiLuoCOLT21}. When the gap complexity is independent of $T$, our gap-dependent result strictly dominates all the aforementioned bounds.

We emphasize that gap-dependent guarantees are unattainable for Switching RL. This is because they consider with dynamic regret, and the optimal rate thereof is $\Theta(\sqrt{T})$ even if $L=1$ even with a single switch.\footnote{To see this, consider a simple example: a bandit problem with two arms where the first arm has mean $0.5$ and the second arm has mean $0$ for the first $T_0$ rounds and then $1$ (for an unknown $T_0$).  Note that the gap  complexity in both cases is $0.5$. Divide the time interval into blocks of length $\sqrt{T}$. Let $q_j$ be the probability of selecting the second arm at least once in block $j$. If $q_j\geq 0.2$ for all blocks then setting $T_0=T$ leads to a dynamic regret of $0.1\sqrt{T}$. Otherwise, setting $T_0$ as the start of the first block $j'$ where $q_{j'}<0.2$, we have $0.8$ probability of not detecting the switch until the block is over, which leads to a contribution to the dynamic regret of $0.4\sqrt{T}$.} 

While our approach of enlarging confidence intervals to account for corruption originates from \cite{LykourisMiPa18}, a similar idea arises in the work on Switching RL \cite{cheung2023nonstationary,WeiLuoCOLT21}. The initial incarnation of this idea in Switching RL, in \cite{cheung2023nonstationary}, employs an EXP3 algorithm on top of the base learners, and thus cannot attain gap-dependent guarantees when applied to our problem.

\section{Preliminaries and main results}
\label{sec:prelims}
In this section, we introduce  basic background from episodic RL as well as our model and our main results.

\xhdr{Markov Decision Processes (MDPs) and Policies.}
A finite-horizon \emph{MDP} consists of \emph{state space} $\states$, \emph{action space} $\actions$, and horizon $H$. In the tabular setting, the state and action spaces are finite $|\calX| = S < \infty$ and $|\calA| = A < \infty$. At each \emph{step} (or \emph{stage}) $h \in [H]$, the agent observes the current state $x_h\in\states$, selects an action $a_h\in\actions$, earns reward $R_h\in [0,1]$, and transitions to the new state $x_{h+1}\in\states$. $R_h$ and $x_{h+1}$ are drawn independently at random from \emph{reward distribution} $\Drew(x_h,a_h)$ and \emph{transition distribution} $p(x_h,a_h)$, respectively. Both distributions are parameterized by a state-action pair and fixed for all stages. We assume that $x_1$ is drawn from an initial distribution~$p_0$.

A policy $\boldpi$ for selecting actions is a (possibly randomized) mapping from the observable history
    $(x_1,a_1; \,\ldots\, ; x_{h-1},a_{h-1};x_h)$,
for every given step $h$, to an action~$a_h$.~\footnote{Importantly, we allow the randomness in $\boldpi$ to be shared across episodes. In other words, $\boldpi$ can equivalent to a selecting a policy $\pi$ according to a given distribution of deterministic policies $\{\pi\}$. }  Operators
    $\Pr^{\calM,\boldpi}$ and $\Exp^{\calM,\boldpi}$
denote probabilities and expectations induced by policy $\boldpi$ and MDP $\calM$. The \emph{value} of the policy under $\calM$ is defined as its total expected reward:
$\textstyle V^{\calM,\boldpi} := \Exp^{\calM,\boldpi}\left[\sum_{h\in [H]} R_h\right]$. 
It is known that, for a given MDP $\calM$, the policy value $V^{\calM,\boldpi}$ is maximized (possibly not uniquely) by a policy $\pi$ which is deterministic and \emph{Markovian}, in the sense that it depends only on the current step $h$ and the current state $x_h$.
A \emph{randomized} Markovian policy $\boldpi$ is as distribution over deterministic Markovian policies. Henceforth, all policies are assumed to be Markovian unless specified otherwise.

In our setting, there is a nominal MDP denoted $\calM = (\pst_0,\pst,\Drew)$. The mean reward for state $x$ and action $a$ is denoted
 $\rst(x,a) = \Exp_{R\sim \Drew(x,a)}[R]$.
We often omit the dependence on $\calM$ from notation: for a randomized policy $\boldpi$, we write expectations, probabilities and policy value under $\calM$ as
    $\Pr^{\boldpi}  = \Pr^{\calM,\boldpi}$,
    $\Exp^{\boldpi} = \Exp^{\calM,\boldpi}$,
and
    $V^{\boldpi} := V^{\calM,\boldpi}$. The optimal value and the optimal policy of the nominal MDP are denoted by
    $\vst := \max_{\boldpi\in\boldPi}\vboldpi$ and $\pist:=\arg\max_{\boldpi\in\boldPi}V^{\boldpi}$ respectively.

\xhdr{Model with adversarial corruptions.} We have $K$  \emph{episodes}, each operating with a finite-horizon MDP of horizon $H$, thus there are $T=HK$ steps in total. Throughout we assume the \emph{MDP environment} $(\states,\actions,H)$ is known to the learner. The learner faces the \emph{nominal MDP} $\calM = (\pst_0,\pst,\Drew)$
in all episodes that are not corrupted. Before episode $k$, the adversary decides whether to corrupt the episode, in which case the corresponding 
MDP  $\calM_k = (p_0^{(k)}, p^{(k)},\Drew^{(k)})$
can be arbitrary.\footnote{We allow the adversary to corrupt different steps in different ways, so transition and reward distributions are parameterized by step $h$. In contrast, these distributions are not step-dependent for the nominal MDP.} 
Crucially, the learner observes neither the nominal MDP nor the choices of the adversary. The adversary can be \emph{adaptive}, in the following sense. Before each episode $k$, the learner commits to a randomized Markovian policy $\boldpi_k$.\footnote{This is restrictive, in principle; however, all our policies are of this shape.} The adversary observes $\boldpi_k$ and the observable history of the previous episodes, and then decides whether and how to corrupt the episode. 

\xhdr{Conventions and notation.}
We consistently use $k$ for episodes, $h$ for stages, $x$ for states, and $a$ for actions. We use $\pi$ for deterministic policies and $\boldpi$ for randomized ones. We use subscript $(k;h)$ for various quantities in step $h$ of episode $k$: thus, we have
state $x\kh$, action $a\kh$, reward $R\kh$, and $\pi\kh$ for the deterministic policy $\pi_k$ that learner draws from $\boldpi_k$ at the beginning of episode~$k$. We write $\didcorrupt_k=1$ if episode $k$ is corrupted, and $\didcorrupt_k=0$ otherwise. We denote total corruption level $C := \sum_{k=1}^K \didcorrupt_k$;
by adversary's adaptivity, 
this is a random variable not determined until episode~$K$. We assume $C \ge 1$ without loss of generality.~\footnote{We recover the stochastic setting by taking $C = 1$, and assuming that the adversary selects i.i.d. transitions and rewards on the corrupted episode.}

For a given transition distribution $p$, we interpret $p(x,a)$, for a given state-action pair, as a vector in the simplex over states: $p(x,a)\in \Delta(S)
$. By a slight abuse of notation, its $x'$-th component, i.e., the probability of transition to state $x'$, is sometimes denoted by $p(x'\mid x,a)$.

\xhdr{Regret.} We evaluate the algorithm using the ``all-knowing benchmark'' which knows
the MDP $\calM_k$ before each episode $k$ (but not the realized rewards or state transitions). We define \emph{regret} as
\begin{align}\label{eq:prelims-regret}
    \Regret &:=\textstyle 
    \sum_{k\in [K]}\left( \max_{\boldpi\in\boldPi} V^{\calM_k,\boldpi}-V^{\calM_k,\boldpi_k}\right),
\end{align}
where $\boldPi$ is the set of all policies, possibly randomized and non-Markovian. Recall that $V^{\calM_k,\boldpi_k}$ is the expected reward at episode $k$ conditioned on the choice of $(\calM_k,\boldpi_k)$. Hence, $\Regret$ is a random variable, where the randomness comes both from the learner and from the adversary (and, implicitly, from the realization of rewards and state transitions in the history). Our algorithms bound this quantity with high probability $1-\delta$ for $\delta>0$. To ease presentation, we upper bound regret as
\begin{align}
 \Regret &\le \textstyle \sum_{k\in [K]} \vst - V^{\boldpik} + CH, \label{eq:corrupted_regret}
\end{align}
which penalizes the suboptimality of policies $\boldpik$ only on the nominal MDP. Note that even though the policies $\boldpik$ are evaluated on the nominal MDP, they are constructed using data collected from both corrupted and uncorrupted episodes. The dependence on $CH$ is unavoidable: $\Regret \geq CH$ in the worst case.
\footnote{Consider the setting with $H$ multi-armed bandit instances and transitions that are independent from the selected actions (essentially run $H$ disjoint bandit settings). An adversary that sets reward of $0$ to all actions other than the one selected with the lowest probability in the first $C$ episodes causes regret at least $CH$ while corrupting $C$ episodes.}
In the remainder of the paper, we directly compare to $\vst$.

\xhdr{Background: episodic RL without corruptions.} The optimal value function $\vsth(x)$ captures the reward of the optimal policy $\pist$ starting at stage $h$ and state $x$. This policy maximizes the optimal Q-function  $\qsth\xa := r^{\star}\xa + p^{\star}\xa^\top \vsthpl$,  where $p^{\star}(x,a)^{\top} \vsthpl=\Exp_{x'\sim p^{\star}(x,a)}\Big[\vst_{h+1}(x')\Big]$ is the payoff obtained by playing optimally after transitioning to a  state $x'$ drawn from $p^{\star}\xa$. In the tabular setting, $\pst\xa$ and $\rst\xa$ are arbitrary elements of the simplex $\Delta(\calX)$ and $[0,1]$ respectively. Together with the initial condition $\vst_{H+1} \equiv 0$, $\vsth$ is defined recursively via $\vsth(x) := \max_{a} \qsth\xa$. These recursions for $\vsth$ and $\qsth$ are known as \emph{Bellman updates}.

The performance guarantees in the tabular setting often scale with the difference between the optimal value function and the Q-functions of other actions. This is captured by the notion of \emph{gap}:
\begin{align*}
\gaph\xa := \vsth(x) - \qsth\xa 
    = \textstyle \max_{a'}\;\qsth(x,a') - \qsth\xa,
\end{align*}
which represents the suboptimality gap of playing action $a$ at state $x$ at stage $h$, and continuing to play optimally for for all $h' > h$.  For simplicity, we define the minimal gap over all stages as $\gap\xa := \min_{h} \gap_h\xa$ and the minimum non-zero gap across all state-action pairs as $\gapmin=\min_{(x,a)\in\states\times\actions}\gap\xa$. 
We also define $\optacts := \{(x,a): \exists h\in [H], a\in \pist_h(x)\}$, which denotes the set of state-action pairs whose actions belong to the set of the optimal actions of the state at some stage $h\in [H]$. We define $\subacts := \states\times\actions - \optacts$ as the complement of $\optacts$. Note that for any $(x,a)\in \subacts$, we must have $\gap_h\xa > 0$ for all $h\in [H]$. The \emph{gap complexity} is:
\begin{align}
\gapcomplexity := 
\textstyle 
\sum_{\xa \in \subacts} 
H/\gap\xa + 
H^2\,|\optacts|/\gapmin. 
\label{eq:gap_complexity}
\end{align}
This definition reflects the typical dependence on  state-action gaps given in standard gap-dependent guarantees for episodic RL (without corruptions); for example, see Corollary 2.1 in \cite{SimchowitzJamieson19} for a more detailed discussion. We note that the dependence on $|\optacts|/\gapmin$ is unavoidable for optimism-based algorithms (Theorem 2.3 in \cite{SimchowitzJamieson19}).  If we neglect this $|\optacts|/\gapmin$ term required for the RL setting, and consider S = H = 1, the gap complexity reduces to the standard gap complexity in the bandit setting \cite{audibert2009exploration}.

The above quantities are not known to the learner as the nominal MDP is latent. As a result, model-based RL approaches approximate the nominal MDP from data. For the tabular setting, we use unweighted sample averages over a given subset of episodes $\mathcal{K}\in[K]$. Given state-action pair $(x,a)$, consider the set of relevant samples, i.e., episode-step pairs where this state-action pair occurs:
\begin{align}\label{eq:subsampled_data}\Phi_{\epSet}(x,a)=\left\{(k,h)\in [K]\times [H]: x\kh=x \text{ and } a\kh=a\right\}.\end{align}
Denote its cardinality as $N_{\mathcal{K}}(x,a)$. We estimate expected rewards and transition probabilities:
\begin{align}\label{eq:subsampled_estimates}
    \widehat{r}_{\epSet}(x,a)=\frac{\sum_{(k,h)\in\Phi_{\epSet}(x,a)}R_{k;h}}{N_{\epSet}(x,a)} \quad and \quad \widehat{p}_{\epSet}(x'|x,a)=\frac{\sum_{(k,h)\in\Phi_{\epSet}(x,a)} \mathbf{1}\left\{x_{k;h+1}=x'\right\}}{N_{\epSet}(x,a)}
\end{align}
We define these quantities for general $\mathcal{K}$ because our algorithm combines different base algorithms and, in some occasions, these base algorithms use only the data from a subset of the episodes. Classical episodic RL approaches use all data globally, i.e., $\mathcal{K}=[k-1]$; then the above quantities average over all episodes $t<k$ and we use the shorthands $N\kgl(x,a)$, $\widehat{p}\kgl(x,a)$ and $\widehat{r}\kgl(x,a)$.

A canonical episodic RL algorithm is \emph{Upper Confidence Bound Value Iteration} or \emph{UCBVI} \cite{AzarOsMu17} which mimics Bellman updates using the empirical model adjusted by bonuses $\bonus_k$:
\begin{equation}
\begin{split}
 &\vup_{k,H+1}(x)=0, \quad \qup_{k,h}(x,a)=\widehat{r}\kgl(x,a)+\widehat{p}\kgl(x,a)^{\top} \vup_{h+1}+\bonus_k(x,a) \\ 
   & \pi\kh(x)\in\arg\max_a \bar{Q}_{k,h}(x,a),\quad \vup\kh\ofx = \qup\kh(x,\pi\kh\ofx)
\end{split}\label{eq:VI_def}
\end{equation}
where $\bar{V}_{k,h}$ and $\bar{Q}_{k,h}$ are bonus-enhanced upper bounds of the value and Q-functions respectively. $\bar{Q}_{k,h}(x,a)$ is therefore similar to $\textrm{UCB}_t(a)$ in bandits. The UCBVI policy $\pi_k^{\textrm{ucb}}$ selects the action that maximizes this upper bound for each state-action pair $(x,a)$, i.e. $\pi_{k,h}^{\textrm{ucb}}(x,a)\in\arg\max_a \bar{Q}_{k,h}(x,a).$. The bonus used by \cite{AzarOsMu17} and on which we are building is:
\begin{align}
\bonus_k\xa&\approx H \min\{\;1,\sqrt{\ln (SAHT/\delta)\;/\;N\kgl\xa}
\;\}.
\end{align}

\xhdr{Main results.} Before presenting the algorithm in Section~\ref{sec:algorithm}, we present the main results it achieves. For the tabular setting, the regret guarantee comes from the following theorem.~\footnote{Here and throughout, $\lesssim$ hides constant factors.} Recall that we assume without loss of generality that the total number of corruptions is $C \ge 1$.

\vspace{0.1in}
\begin{theorem}[Main Guarantee for Tabular Setting]\label{thm:main_tabular} \hspace{0.1em}  For $\delta \in (0,1/2)$, and define $\deltaeff :=  \frac{\delta}{HSAT}$, with probability $1 - \delta$, $\supervisedc$ obtains regret of:
\begin{align*}
 \Regret &\lesssim \min\left\{ CH^3 \ln^2(T) \cdot \sqrt{ SAT\ln \tfrac{1}{\deltaeff} }, \quad  C H^4 \cdot \gapcomplexity \cdot  \ln^3 \tfrac{1}{\deltaeff}   \right\} \\
 & \qquad +  H^6 A \left(S^2 C \ln^3 \tfrac{1}{\deltaeff} + SC^2\ln^2(T)\right).
 \end{align*}
\end{theorem}
Although to ease presentation, we initially explain the arguments for the tabular setting, our results also extend to the linear MDP setting  \cite{jin2019provably} where rewards and transitions admit a linear underlying representation with a $d$-dimensional feature embedding. The algorithm in the two settings only differs by the way that model estimates and bonuses are created; we explain the exact model and the algorithmic differences in Sections~\ref{ssec:linear_model} and \ref{ssec:linear_algorithm}. 

\vspace{0.1in}
\begin{theorem}[Main Guarantee for Linear Setting]\label{thm:main_linear} \hspace{0.1em} Fix confidence $\delta \in (0,1/2)$. Then, with probability $1 - \delta$, $\supervisedc$ obtains regret of:
\begin{align*}
\Regret_T \lesssim  \sqrt{(d^3 + dA) T} \cdot  CH^4 \cdot \ln^{3}\tfrac{TAd}{\delta} + \sqrt{dT} \cdot C^2H^5 \cdot \ln^{2}\tfrac{T}{\delta},
\end{align*}
\end{theorem}

\section{Our algorithm: \textsc{SuperVIsed.C}}
\label{sec:algorithm}

\newcommand{\pimasternonboldk}{\pimasternonbold_k}
\newcommand{\UCBVIbase}{\textsc{UCBVI-base}}
\newcommand{\pimasternonbold}{\pi^{\textsc{master}}}

At a high level, our algorithm (Algorithm~\ref{alg:corruption_robust_rl}) coordinates $\lmax = \lceil{\log T}\rceil$ RL base algorithms which we refer to as \emph{base learners} (or \emph{layers}). Each of the base learners looks at the data and recommends a policy. At the beginning of each episode, a random seed determines whose base learner's recommendation our algorithm will follow at any particular step in that episode. Note that our algorithm follows the recommendation of different base learners in different time steps of the same episode.

We formalize 
our algorithm below and, to ease exposition, we instantiate it to the tabular RL setting. That said, Algorithms~\ref{alg:corruption_robust_rl} and \ref{alg:UCBVI_base} apply beyond the tabular setting; what changes is only the definition 
of the confidence bonuses, and reward and transition estimates (for linear RL, these are defined in Section \ref{ssec:regret_guarantee_linear_mdp}).

\xhdr{Robustness via subsampling.} Similar to past algorithms for multi-armed bandits \cite{LykourisMiPa18}, each base learner $\ell$ keeps a set of \emph{local} (or \emph{subsampled}) data; these correspond to data collected in a subset of the episodes which we denote by $\epSet\kl$. We note that $\epSet\kl$ includes all the episodes where we follow the recommendation of base learner $\ell$ in the final step $h=H$ of the episode. That said, in previous steps of the same episode, the algorithm may follow recommendation of other base learners $\ell'<\ell$ (as described below).

Base learners differ by the probability with which a particular episode ends up in their episode set; this probability is approximately $2^{-\ell}$ for base learner $\ell$. Hence
base learners $\ell\geq\lceil C\rceil$ have in expectation $C\cdot 2^{-\ell}\leq 1$ corrupted episodes in their episode sets and with high probability at most $\ln(16\ell^2/\delta)$ since the corrupted episodes are subsampled. The \emph{critically tolerant} layer $\lst=\lceil \log C\rceil$ is responsible to correct the mistakes of the layers $\ell<\lst$ that may be misled by corruptions. A simple way to ensure that robust base learners $\ell\geq \lceil C\rceil$ keep valid confidence intervals despite the corrupted episodes is to augment the bonus by $\ln(16\ell^2/\delta)$; this gives the first set of bonuses:
\begin{align}
    \bonus\klsb\xa&=\min\left\{H,\left( 2H\sqrt{\frac{ 2\ln (64 SAHT^3/\delta)}{N\klsb\xa}}  + \frac{2H^2   \ln\frac{16\ell^2}{\delta}}{N\klsb\xa}\right)\right\}.
    \label{eq:subsample_bonus}
\end{align}
where $N\klsb\xa = N_{\epSet\kl}\xa$ is the number of visitations for state-action $(x,a)$ restricted to episode set $\epSet\kl$. The subsampled estimates for the reward and transition functions are the corresponding estimates restricted to episodes in $\epSet\kl$, via  \eqref{eq:subsampled_estimates}:
\begin{align}
\rhat\klsb := \rhat_{\epSet\kl}, \quad \phat\klsb := \phat_{\epSet\kl}, \quad \rhat_{\epSet},\phat_{\epSet} \text{ as in  \eqref{eq:subsampled_estimates}} \label{eq:sub_empirical_estimates}.
\end{align}
To ensure that robust base learners supervise less robust ones, we create nested active sets. Each base learner receives its active
set $\activeset_{k,\ell;h}$ from learner $\ell+1$ and computes the policy that maximizes an upper bound on the corresponding Q-function when restricted only on actions of the active set. The active $\activeset_{k,\ell}$ for learner $\ell$ is a
collection of active sets $\activeset_{k,\ell;h}\ofx \subseteq \actions$, indexed by $h \in [H]$ and $x \in \calX$, and UCBVI computes a deterministic policy $\pi^{\textsc{base}}\klh\ofx$ by maximizing the upper bound on the Q-function $\qup\klh\ofx$, restricted to the active set $\activeset_{k,\ell;h}\ofx$ provided by learner $\ell+1$.

\xhdr{Challenges unique to the RL setting and trajectory mismatch.} In multi-armed bandits, by selecting the critically tolerant base learner $\lst$ with probability $2^{-\lst}$, we ensure that the lower confidence bound of the optimal action eventually becomes larger than the upper confidence bound of any suboptimal action $a$. At this point, we can eliminate the suboptimal action $a$ and transfer this information to less robust base learners (by also removing $a$ from their active set). As we already discussed, in episodic RL, action-eliminating algorithms have regret that is exponential in the episode horizon $H$ (see Section~\ref{sec:lower_bound_active_arm_elimination}) so we cannot directly use an action-eliminating algorithm as a base learner. Moreover, algorithms such as \textsc{UCBVI} do not come with tight confidence intervals for the value functions of the deployed policies and hence it is unclear how they can lead to elimination of suboptimal actions in the less robust base learners.

To make things worse, the states visited by the different base learners may be significantly different, which does not allow us to easily use the data collected from the critically-tolerant base learner to help identify errors in the Q-functions of not selected actions in the states visited by the deployed policy in the less robust base learners. We refer to this phenomenon as \emph{trajectory mismatch}. Note that this state-wise trajectory mismatch does not happen in multi-armed bandits because all base learners share a common unique state.

\newcommand{\LDOTS}{\, ,\ \ldots\ ,}     
\newcommand{\algTAB}{\hspace{1.5em}}
\newcommand{\mD}{\mathcal{D}}

\xhdr{Episode schedule across base learners.} To ensure efficient regret, we need to allow tolerant base learners $\ell$ to supervise less tolerant $\ell'<\ell$ on the trajectories induced by policies $\pi_{k,\ell'}$ recommended by the latter. We accomplish this via traversing base learners in increasing order with a careful episode scheduling procedure based on the following resampling distribution   $\mD_\ell$ for every $\ell \in[\lmax]$; this distribution is supported
    $\{\ell \LDOTS \lmax\}$:
\begin{align}\label{eq:resamplingD}
\begin{cases}
\Pr[\; \mD_\ell = n\; ] = \tfrac{1}{2n H}\cdot 2^{-(n-\ell)}
     & \text{for each } n\in \{\ell+1 \LDOTS \lmax\}, \\
\mD_\ell = \ell
    &\text{with the remaining probability}.
\end{cases}
\end{align}
When episode $k$ ends, we add the episode's data to the episode set of the last base learner in this episode schedule $\epSet\kl$. Our scheduling distribution ensures that any particular episode is charged to base learner $\ell$ with probability less than $2^{-\ell}$, making base learners $\ell\geq \log C$ robust to corruption $C$. It addresses the problem of trajectory mismatch by ensuring that for any $\ell > \ell'$ and any stage $h$, there is roughly $2^{-(\ell - \ell')}/H$ probability of switching from base learner $\ell'$ to a trajectory which executes the more tolerant learner's policy $\pi\kl$ for all stages $h,h+1,\dots,H$. This allows the more tolerant learners to develop an estimate of the value function $\vsth(x)$ on states that are frequented by the less tolerant policy $\pi_{k,\ell'}$, thereby ensuring adequate supervision.\footnote{More formally, our decomposition measures the performance of the non-tolerant base learner on a fictitious alternative trajectory where, at some step $h$, the trajectory instead switched to following $\pi_{k,\lst}$, denoted by $\pi_k\oplus_h \pi_{k,\lst}$. Our layer schedule ensures that the densities of the actual and fictitious trajectories are close; formally, via a bounded visitation ratio (Definition~\ref{defn:visitation}) which is central to the analysis (Section~\ref{sec:our_framework_RL}).} This episode schedule is the random seed discussed in the beginning of the section.

\begin{algorithm}\caption{\textsc{SuperVIsed.C} (Supervised Value Iteration with Corruptions)}\label{alg:corruption_robust_rl}
{\bf Initialization:}
    $\lmax= \lceil{\log T}\rceil$ \algcomment{(max \#layers)} \;
    \algTAB
    episode sets $\epSet_\ell= \emptyset$, for each $\ell\in[\lmax]$ \;
	\algTAB
        $\activeset_{k,\lmax;h}(x)= \actions$
    for all $k\in[K]$, $x\in\states$, $h \in [H]$
    \algcomment{(active sets for layer $\lmax$)}\\
\For{episode $k=1 \LDOTS K$}{
    \algTAB
        \algcomment{(\emph{Notation for analysis:} $\epSet_{k,\ell} \leftarrow \epSet_\ell$ for each $\ell$.)}\\
  \algTAB \algcomment{(Compute policy $\&$ active set for each base learner)}\\
  \For{base learner $\ell=\lmax \LDOTS 1$}{
      $(\pi_{k,\ell}^{\textsc{base}},\activeset_{k,\ell-1})\leftarrow \UCBVIbase(k,\ell,\activeset_{k,\ell},\epSet_\ell)$
    \label{line:get_active_set}
    }
    $\ell\leftarrow 1$  \algcomment{(current layer)}\\
    \For{stage $h = 1 \LDOTS H$}{
        Update the current layer $\ell$: resample it from distribution $\mD_\ell$, as per \eqref{eq:resamplingD}.\\
        \algTAB
        \algcomment{(\emph{Notation for analysis:}
            $f(k,h) \leftarrow \ell$
            and
            $\pimasternonbold_{k;h}\leftarrow \pibase_{k,\ell;h}$)}
        \label{line:alg_layerSched}\\
	  Select action $a\kh$ using policy $\pibase_{k,\ell;h}$\\
	 }
	 $\epSet_\ell\leftarrow \epSet_\ell \cup\{k\}$
    \algcomment{(add $k$ to the episode set for last layer in the episode)}
  \label{line:alg_epSet}
}
\end{algorithm}

\xhdr{Supervision with global data.} We now turn towards \textsc{UCBVI-BASE} (Algorithm~\ref{alg:UCBVI_base}) that provides the recommendation of each base learner. Unlike elimination algorithms that quantify \emph{how fast} suboptimal actions are eliminated, algorithms such as UCBVI never deem an action suboptimal (they stop selecting it when the upper confidence bound is at most the mean of the optimal action). We therefore combine the local (subsampled) model estimates, already present in active arm elimination with global estimates that limit how many times suboptimal actions are selected. Since base learner $\ell$ is supposed to be robust to corruption up to $2^{\ell}$ and we cannot rely on subsampling when using all data, we augment its global bonus with $2^\ell$:
\begin{align}
\label{eq:global_bonus}
\bonus\klgl\xa&=\min\left\{H,\left( 2H\sqrt{\frac{ 2\ln (64 SAHT^2/\delta)}{N\kgl\xa}}  + \frac{2^{\ell}H^2}{N\kgl\xa}\right)\right\},
\end{align}
where $N\kgl\xa = N_{[k-1]}\xa$ is the total number of visitations for state-action $(x,a)$ up to episode $k-1$. Parallel to  the subsampled case, the global estimates for the reward and transition functions are the corresponding estimates restricted to episodes in $[k-1]$, via  \eqref{eq:subsampled_estimates}:
\begin{align}
\rhat\kgl := \rhat_{[k-1]}, \quad \phat\kgl := \phat_{[k-1]}, \quad \rhat_{\epSet},\phat_{\epSet} \text{ as in \eqref{eq:subsampled_estimates}} \label{eq:global_empirical_estimates}.
\end{align}

\xhdr{Double-supervised value iteration.}
In order to compute $\qup\klh\ofx$, we use both global and subsampled bonuses defined above as well as global and subsampled reward-transition estimates $(\rhat\kgl,\phat\kgl)$ and $(\rhat\klsb,\phat\klsb)$ defined based on \eqref{eq:subsampled_estimates} applied to $[k-1]$ and $\epSet\kl$ respectively: $\mathfrak{m}\klgl=(\rhat\kgl,\phat\kgl,\bonus\klgl)$ based on all data and one $\mathfrak{m}\klsb=(\rhat\klsb,\phat\klsb,\bonus\klsb)$ based only on data in episode set $\epSet\kl$. The upper confidence bound for the Q-function of each state-action pair is determined based on the tightest between the corresponding upper confidence bounds generated by global and subsampled models. Since instantaneous rewards are in $[0,1]$ we also upper bound it by $H$ which allows to more accurately estimate the effect of future steps. As a result, for any state-action pair $(x,a)\in\states\times \actions$, the upper bound is $\qup_{k,\ell;h}(x,a)=\min\{H,\rhat\kgl(x,a)+\phat\kgl(x,a)^{\top}\vup_{k,\ell;h+1}+\bonus\klgl,\rhat\klsb(x,a)+\phat\klsb(x,a)^{\top}\vup_{k,\ell;h+1}+\bonus\klsb\}$
where $\vup_{k,\ell;h}(x)=\max_{a\in\activeset\klh\ofx}\qup_{k,\ell;h}(x,a)$.
Base learner $\ell$ then provides the active sets to $\ell-1$ by removing (only for episode $k$) actions $a$ whose upper bound $\qup\klh\xa$ is less  than the maximum lower bound $\max_{a'}\qlow_{k,\ell;h}(x,a')$. The lower bounds on the Q-functions are also constructed via value iteration, with $\vlow\klh\ofx=\max_{a\in\activeset\klh\ofx}\qlow\klh(x,a)$.\footnote{By construction, all active sets are non-empty so the base learner's recommendation is always well-defined.}

\newcommand{\GL}{\mathtt{gl}}
\newcommand{\SB}{\mathtt{sb}}
\begin{algorithm}\caption{$\textsc{UCBVI-BASE}
(\text{ episode $k$; layer $\ell$; active sets $\activeset_{k,\ell}$; episode set $\epSet_\ell\subset [k]$ })$ \label{alg:UCBVI_base}}
{
    $(\rhat_{\GL},\phat_{\GL};\; \bonus_{\GL})$
        is the model/bonus for the \emph{global} episode set $[k]$ (\eqref{eq:global_bonus} and \eqref{eq:global_empirical_estimates}).\\
        $(\rhat_{\SB},\phat_{\SB};\; \bonus_{\SB})$
     is the model/bonus for the \emph{subsampled} episode set $\epSet_\ell$ (
\eqref{eq:subsample_bonus} and \eqref{eq:sub_empirical_estimates}).\\
$\vup_{H+1}\ofx = \vlow_{H+1}\ofx = 0 \qquad   \forall x \in \states$\\
\algTAB \algcomment{Shorthand: $\activeset_h = \activeset_{k,\ell;\,h}$ for each stage $h$.}\\
\algTAB
    \algcomment{Notation for analysis: $\qup_{k,\ell;\,h} \leftarrow \qup_h$, same for $\qlow_h,\; \vup_h,\; \vlow_h$.}\\
\For{each stage $h = H,\ldots,1$}{
         $\qup_h=\min\big{\{}H, \qquad \rhat_{\GL}+\phat_{\GL}^{\top}\cdot\vup_{h+1} +  \bonus_{\GL}, \qquad \rhat_{\SB}+\phat_{\SB}^{\top}\cdot\vup_{h+1} + \bonus_{\SB}\big{\}}$\\
        $\qlow_h=\max\big{\{}0, \qquad \rhat_{\GL}+\phat_{\GL}^{\top}\cdot\vlow_{h+1} - \bonus_{\GL}, \qquad \rhat_{\SB}+\phat_{\SB}^{\top}\cdot \vlow_{h+1} - \bonus_{\SB}\big{\}}$\\
        $\pibase_h\ofx \leftarrow \arg\max_{a \in \activeset_h\ofx} \qup_h(x,a) \qquad \qquad \forall x\in\states$\label{line:pibase}\\
       $\vup_h\ofx \leftarrow \qup_h(x,\pibase_h\ofx) \quad \text{and} \quad \vlow_h\ofx \leftarrow \max_{a\in\activeset_h\ofx}\qlow_h(x,a), \;\;\forall x\in\states$ \label{line:vup}\\
      $\activeset_{k,\ell-1;h}\ofx \leftarrow \{a \in \activeset_h\ofx: \qup_h(x,a) \ge \vlow_h\ofx \qquad \forall x\in\states\}$
}
}
\textbf{Return: }
$(\; \text{base policy $\pibase$, active sets $\activeset_{k,\ell-1}$} \;)$
\end{algorithm}

\section{Our analytical framework: \textsc{SuperVIsed}}
\label{sec:our_framework_RL}

\newcommand{\Dsched}{\mathcal{D}^{\textsc{sched}}}
\newcommand{\klel}{_{k,\le \ell}}
\newcommand{\ql}{q_{\ell}}
\newcommand{\qlelst}{q_{\le \lst}}

We now introduce our analytical framework that extends prior analyses based on ``optimism under uncertainty'' to incorporate active sets. Our $\supervised$ framework allows for more flexible algorithm design as it does not require the algorithm to select the action that has the highest upper confidence bound on the corresponding Q-function. This section is not tied to a particular episodic RL setting; in Sections~\ref{ssec:tabular_initial} and \ref{ssec:regret_guarantee_linear_mdp}, we show how it can be applied to tabular and linear MDPs.

\subsection{Key principles.}
\label{ssec:key_principles_supervised}
To present our framework, we first introduce its main principles.
 
\xhdr{Confidence-admissibility.} Our analytical framework centers around (a) confidence intervals for the optimal Q-function and (b) active sets that are restricted to actions which are plausibly optimal, given these confidence interval. We refer to the object which contains this data as a Q-supervisor $\tuple_k=\big(\qup_k,\qlow_k,\activeset_k,\pigreed_k\big)$, where
\begin{enumerate}
\item $\qup_{k;h}(x,a),\qlow_{k;h}(x,a)\in [0,H]$ are upper and lower estimates for the optimal Q-function.
\item For all $(x,a,h)\in\states\times\actions\times [H]$, $\activeset_k=\{\activeset_{k;h}(x)\}_{x\in\states, h\in[H]}$ consists of subsets of actions.
\item The UCB policy $\pigreed_k$ is a deterministic policy maximizing $\qup_k$: $\pigreed_{k,h}(x)=\argmax_{a}\qup_{k;h}(x,a)$.
\end{enumerate}
The main statistical condition of our framework is that a desired Q-supervisor is confidence-admissible:
\vspace{0.1in}
\begin{definition}\label{defn:conf_admissible}
A Q-supervisor $\tuple_k=\left(\qup_k,\qlow_k,\activeset_k,\pigreed_k\right)$ is \emph{confidence-admissible} if:
\begin{itemize}
    \item  $\activeset\kh\ofx$ contains all optimal actions for state $x$ and step $h$, i.e.,  $ \pist_h\ofx \subseteq \activeset\kh\ofx$.
    \item $\qup\kh\xa,\qlow\kh\xa$ are upper $\&$ lower bounds of $\qst_h$: $\qlow\kh\xa\leq \qst_h\xa\leq \qup\kh\xa$.
    \end{itemize}
\end{definition}
For corruption-robustness, we use confidence intervals and active set of the critically tolerant base learner $\lst=\lceil \log C \rceil$ to analyze regret from data charged to under-robust base learners $\ell<\lst$. Base learner $\lst$ has both global and subsampled confidence intervals valid, which ensures that her Q-supervisor $\tuple_{k,\lst}$ is confidence-admissible and can correctly supervise under-robust base learners. We analyze corruption-robust base learners $\ell>\lst$ using their own estimates as Q-supervisors.

\xhdr{Q-supervised policies.}
The main flexibility that our framework provides is that it does not impose to select the UCB policy $\pigreed_k$ of a confidence-admissible Q-supervisor. Instead, the Q-supervisor induces a \emph{plausible set} designating which actions from its active set are plausibly optimal:
\begin{align*}
    \plauset\kh\ofx := \{a \in \activeset\kh\ofx: \qup\kh\xa \ge \qlow\kh(x,a')
        \quad \forall a'\in \activeset\kh\ofx \},
        \qquad \forall h\in[H],\, x\in \states.
\end{align*}
A policy is called \emph{Q-supervised} by this Q-supervisor if it only selects actions from the plausible set:
\vspace{0.1in}
\begin{definition}[Q-supervised policies]\label{defn:compatible}\hspace{0.1em}
A randomized Markovian policy $\boldpi_k$ is called \emph{Q-supervised} by a given Q-supervisor $\tuple_k=(\qup_k,\qlow_k,\activeset_k, \pigreedk)$ if the following holds:
\[
\Pr\left[\boldpi\kh(x)=a\right]>0 \Rightarrow a \in \plauset\kh\ofx
\qquad\forall h\in[H],\,a\in\actions,\, x\in\states.
\]
As shorthand, we shall also say that $\boldpi_k$ is $\tuplek$-supervised.
\end{definition}
For corruption-robustness, nested active sets ensure that less robust base learners $\ell$ always select policies that are Q-supervised from the Q-supervisors $\tuple_{k,\ell'}$ of all more robust base learners $\ell'>\ell$. In particular, under-robust base learners $\ell<\lst$ select policies Q-supervised from  Q-supervisor~$\tuple_{k,\lst}$.

\xhdr{UCB visitation ratio.}
The above two principles are rather minimal: impose correct confidence intervals and disallow selecting provably suboptimal actions. Unfortunately, as shown in Appendix~\ref{app:lower_bound_active_arm_elimination}, and as is well understood in the RL folklore, selecting actions uniformly from confidence intervals does not suffice for bounded regret. In particular, this seems to rule out elimination-based approaches for corruption-robust learning as the one proposed in \cite{LykourisMiPa18}. 

The reason why this elimination is insufficient is because the confidence intervals developed for the Q-function require accurately estimating the reward of future actions. Without corruptions, this accuracy is typically ensured by selecting policies which execute the principle of optimism uncertainly, selecting policies $\pigreedk$ which maximized the upper estimate of the Q-function $\qupk$.

Generalizing this idea, we show that the regret of a supervised policy $\boldpik$ can be controlled by the extent to which that policy's trajectories deviate from the UCB policy $\pigreedk$. Formally,
\vspace{0.1in}
\begin{definition}\label{defn:concatenated_policy}\hspace{0.1em} We define the \emph{concatenated policy} $\boldpi_k\oplus_{h}\pi'$ as the policy which selects actions according to $\boldpi_k$ for steps $\tau = 1,\dots,h-1$, and then actions according to $\pi'$ for steps $\tau = h,\dots,H$.\footnote{Formally, this policy draws a determinisitic Markovian policy $\pi$ from the distribution induced by $\boldpi_k$, selects actions according to $\pi$ for steps $\tau = 1,\dots,h-1$, and then actions according to $\pi'$ for steps $\tau = h,\dots,H$.}
\end{definition}
We require that the measures on states and actions induced by concatenated policies have bounded density ratios with respect to the UCB policy. We term this property \emph{UCB visitation ratio} which is ensured in $\supervisedc$ via the carefully designed episode schedule.
\vspace{0.1in}
\begin{definition}\label{defn:visitation}\hspace{0.1em}
We say that a randomized policy $\boldpi_k$ has \emph{UCB visitation ratio}  $\rho \ge 1$ with respect to a  Q-supervisor $\tuple_k=(\qup_k,\qlow_k,\activeset_k, \pigreedk)$ if the following holds:~\footnote{In linear MDPs, these ratios may be understood as a ratio of densities. When we instantiate UCB visitation ratios for $\supervisedc$, the layer schedule ensures that the visitation measure of the concatenated policy is absolutely continuous with respect to the UCB policy. For simplicity, we avoid measure-theoretic formalities.}
\begin{align*}
\max_{1\leq h \leq \tau \leq H}\,\max_{x\in \states, \,a \in \actions}
\frac{\Pr^{\boldpi_k\oplus_{h}\pigreed_k}[(x_\tau,a_\tau)  = \xa]}{\Pr^{\boldpi_k}[(x_\tau,a_\tau)  = \xa]}\leq \rho.
\end{align*}
\end{definition}

\subsection{Per-episode Bellman-error regret decomposition}
The quality of a confidence-admissible Q-supervisor $\tuple_k=(\qup_k,\qlow_k,\activeset_k,\pigreed_k)$ is measured via the notion of \emph{Bellman error} which intuitively posits that $\qup_k$ should approximately satisfy Bellman updates with $\pi = \pigreed_k$, in the sense that
\begin{align}\label{eq:prelims-Bellman-errors-intuition}
\qup\kh\xa \approx r^{\star}\xa+
\textstyle \sum_{x'} \; p^{\star}(x'|x,a)  \;\qup\khpl(x',\pigreed\khpl(x')).
\end{align}
Hence, the \emph{upper} Bellman error measures the extent to which this equality is violated. In particular, the Bellman error localizes the mistakes across steps by considering the effect of miscalculating the true parameters $(r^{\star},p^{\star})$ at step $h$ assuming that the model is correct in subsequent steps.  
The \emph{lower} Bellman error captures a similar intuition with respect to $\qlow_k$.
\vspace{0.1in}
\begin{definition}[Bellman Errors]\label{defn:Bellman_error} \hspace{0.1em}
Given Q-supervisor
    $\tuple_k=(\qup_k,\qlow_k,\activeset_k, \pigreedk)$ and letting $\vup_{k;H+1}(\cdot) = \vlow_{k;H+1}(\cdot ) = 0$,
for all $h\in[H]$, $x\in\states$ and $a\in\actions$, upper and lower Bellman errors are defined as:
\begin{align*}
\bellmanup\kh\xa
    &=\qup\kh\xa-\left(r^{\star}\xa+p^{\star}\xa\cdot
        \vup_{k;h+1}\right),
    &\vup\kh(x)=\qup\kh(x,\pigreed\kh(x)),
    \\
\bellmanlow\kh\xa
    &=\qlow\kh\xa-\left(r^{\star}\xa+p^{\star}\xa\cdot
       \vlow_{k;h+1}\right),
&\vlow\kh(x)=\max_{a\in\activeset\kh\ofx}\qlow\kh(x,a).
\end{align*}
\end{definition}

\xhdr{Bellman error regret decomposition for Q-supervised policies.} The regret analysis of UCBVI \cite{AzarOsMu17} starts from a decomposition based on Bellman errors (Lemma~\ref{lm:LUCB-Bellman} in Appendix~\ref{app:bellman_error_decomposition_supervised}). We extend this decomposition to Q-supervised policies that may differ from the UCB policy. Proposition~\ref{prop:lucb_regret} differs from decomposition for UCB policies in the expectations $\Exp^{\boldpi_k\oplus_{h}\pigreed_k}$ over \emph{fictitious trajectories} arising from following $\boldpi_k$ until a stage $h$, and switching to $\pigreed_k$. 

\vspace{0.1in}
\begin{proposition}
\label{prop:lucb_regret}\hspace{0.1em} 
For episode	$k$, if a Q-supervisor
    $\tuplek := (\qup_k,\qlow_k,\activeset_k, \pigreed_k)$
is confidence-admissible then any randomized Markovian policy $\boldpi_k$ that is Q-supervised by $\tuplek$ satisfies:
	\begin{align*}
	\vst - \valf^{\boldpi_k} &\le \sum_{h=1}^H \Exp^{\boldpi_k}\left[\bellmanup\kh(x_h,a_h)\right] + \sum_{h=1}^H \Exp^{\boldpi_k\oplus_{h}\pigreed_k}
    \left[\sum_{\tau=h}^{H} \bellmanup_{k;\tau}(x_{\tau},a_{\tau})-\bellmanlow_{k;\tau}(x_{\tau},a_{\tau})\right],
	\end{align*}
where $\boldpi_k\oplus_{h}\pigreed_k$ is the concatenated policy defined in Definition~\ref{defn:concatenated_policy}.
\end{proposition}
	\begin{proof}[Proof sketch.]
If the policy $\boldpi_k$ coincides with $\pigreed_k$ (i.e. greedily optimizes w.r.t. $\qup$) then the proposition holds without the second term. To address the mismatch between $\boldpi_k$ and $\pigreed_k$, we relate the performance of policy $\boldpi_k$ to the one of a fictitious policy that follows $\boldpi_k$ until some pre-defined step and then switches to $\pigreed_k$. The second term in the proposition captures the expected confidence intervals of future cumulative rewards that such a policy would have had, balancing the aforementioned mismatch. The proof then  follows similar arguments as the ones of optimistic policies by working recursively across steps and is formalized in Appendix~\ref{app:bellman_error_decomposition_supervised}.
\end{proof}
	
\xhdr{Combining with UCB visitation ratio.} Proposition~\ref{prop:lucb_regret} bounds the regret at episode $k$ as an expectation of upper and lower Bellman errors. These quantities behave like confidence intervals for the Q-functions under the corresponding policies, which shrink when state-action pairs are selected.

While the first term reflects visitations under the actual policy $\boldpi_k$ selected, the Bellman errors in the second term regard state-action pairs visited by fictitious trajectories $\boldpi_k\oplus_h \pigreed_k$ (Definition~\ref{defn:concatenated_policy}). The UCB visitation ratio addresses this mismatch as it requires that the policy $\boldpi_k$ (effectively) switches to rolling out policy $\pigreed_k$ with probability at least $1/\rho$. This implies the next proposition.

\vspace{0.1in}
\begin{proposition}
\label{prop:lucb_regret_rho} \hspace{0.1em} 
For episode	$k$, if a Q-supervisor
    $\tuplek := (\qup_k,\qlow_k,\activeset_k, \pigreed_k)$
is confidence-admissible then any randomized Markovian policy $\boldpi_k$ that is Q-supervised by $\tuplek$ and has UCB visitation ratio with respect to $\tuple_k$ at most $\rho$ satisfies: 
	\begin{align*}
	\vst - \valf^{\boldpi_k} &\le \rho\cdot \Big(1+H\Big)\cdot \Exp^{\boldpi_k}\Big[\max\big(\sum_{h=1}^H\bellmanup\kh(x_h,a_h),\sum_{h=1}^H \bellmanup_{k;h}(x_{h},a_h)-\bellmanlow_{k;h}(x_h,a_h)\big) \Big],
	\end{align*}
\end{proposition}
Note that the second term is not always smaller than the first as Bellman errors can be negative.

\subsection{Generic regret guarantee for $\supervisedc$.}
To turn the above proposition to a final regret guarantee, we introduce the following notation for each base learner $\ell$:
\begin{enumerate}
    \item The Q-supervisors $\tuple\kl := (\qup\kl,\qlow\kl,\activeset\kl,\pibase\kl)$
    \item The associated Bellman errors
    \begin{align*}
\bellmanup\klh\xa
    &=\qup\klh\xa-\left(r^{\star}\xa+p^{\star}\xa^\top \vup\klhpl\right),
    &\vup\klh(x)=\qup\klh(x,\pigreed\klh(x)),
    \\
\bellmanlow\klh\xa
    &=\qlow\klh\xa-\left(r^{\star}\xa+p^{\star}\xa^\top
       \vlow\klhpl\right),
&\vlow\klh(x)=\max_{a\in\activeset\klh\ofx}\qlow\klh(x,a)
\end{align*}
\end{enumerate}
We now introduce the randomized policies to which we apply the above \supervised{} framework. Consider the episode schedule distribution $\Dsched$ across sequences $(\ell_1,\dots,\ell_H) \in [\lmax]^H$ in Algorithm~\ref{alg:corruption_robust_rl} based on Eq.~\eqref{eq:resamplingD} that determines which base learner's policy $\boldpi_{k,\ell}^{\textsc{base}}$, the master policy $\pimasterk$ follows at each step $h$. We decompose $\pimasterk$ into a distribution over randomized policies.
\vspace{0.1in}
\begin{definition}[Policy decomposition]\label{defn:policy_decomp}
\hspace{0.1em} Let $\Dsched_{\ell}$ be the distribution  $\Dsched$ conditioned on $\ell_{H}=\ell$ and $\Dsched_{\le l}$ be the distribution  $\Dsched$ conditioned on $\ell_{H}\leq\ell$. We define:
\begin{enumerate}
    \item $\pimaster\kl$ as the randomized Markovian policy which selects a deterministic policy according to $\pimasternonbold_{k;h}\leftarrow \pibase_{k,\ell_h;h}$, where $\ell_1,\dots,\ell_H \sim \Dsched_{\ell}$.
    \item $\pimaster\klel$ as the randomized Markovian policy which selects a deterministic policy according to $\pimasternonbold_{k;h}\leftarrow \pibase_{k,\ell_h;h}$, where $\ell_1,\dots,\ell_H \sim \Dsched_{\le \ell}$.
\end{enumerate}
Lastly, we define the mixture probabilities
\begin{align*}
\ql := \Pr[\ell_H = \ell \mid (\ell_1,\dots,\ell_H)\sim \Dsched]\quad \text{ and }\quad
q_{\le \ell} := \Pr[\ell_H \le \ell \mid (\ell_1,\dots,\ell_H)\sim \Dsched].
\end{align*}
\end{definition}
We also define \emph{model-estimates} $\mathfrak{m}_k=(\rtil_k,\ptil_k,\bonus_k)$ where $\rtil_k:\states\times \actions\rightarrow [0,1]$ corresponds to estimated rewards, $\ptil_k:\states\times\actions\rightarrow [0,1]^{\states}$
corresponds to estimated transitions, and $\bonus_k:\states\times \actions \rightarrow [0,H]$ corresponds to an associated bonus. An important notion regarding these model-estimates is that they are valid for the Q-supervisor in the sense that the bonuses (used in its Q-estimates) correctly upper and lower bound the corresponding Q-functions.

\vspace{0.1in}
\begin{definition}
\label{defn:valid}
\hspace{0.1em}  A model-estimate $\mathfrak{m}_k=(\rtil_k,\ptil_k,\bonus_k)$ is \emph{valid} with respect to a Q-supervisor $\tuple_k=(\qup_k,\qlow_k,\activeset_k,\pigreed_k)$ if for all $h\in[H]$, $x\in\states$, $a\in\actions$, it holds that
\begin{align*}
|\rtil_k\xa - r^{\star}\xa + (\ptil_k\xa - p^{\star}\xa)^\top V_{h+1}| \le \bonus_k\xa.
\end{align*}
either (a) only for $V_{h+1}=\vst_{h+1}$ if $\ptil_k(x,a)\geq 0$ or (b) for $V_{h+1}\in\{\vst_{h+1}, \vup_{k;h+1},\vlow_{k;h+1}\}$ where $\vup_{k;h+1}\ofx=\qup_{k;h+1}(x,\pigreed_k(x))$ and $\vlow_{k;h+1}\ofx=\max_{a\in\activeset_{k;h+1}(x)}\qlow_{k;h+1}(x,a)$.
\end{definition}
\vspace{0.1in}
\begin{remark}
\hspace{0.1em} Condition (a) is mostly applicable in settings where the transition estimates are computed point-wise such as in tabular settings and enables tighter bonus leading to enhanced guarantees. Condition (b) is targeted towards function-approximation settings where the transition estimate is the output of a linear regression and is therefore not guaranteed to be nonnegative.
\end{remark}
In \supervisedc, there are two important model-estimates for each base learner $\ell$; these are instantiated to the tabular and linear setting  in Section~\ref{sec:algorithm} and Appendix \ref{ssec:linear_algorithm} respectively. 

\vspace{0.1in}
\begin{definition}[Global and subsampled model-estimates]\label{defn:true_model} \hspace{0.1em} We define the following model-estimates: \emph{global} $\mathfrak{m}\klgl=(\rhat_{k,gl},\phat_{k,gl},\bonus_{k,gl,\ell})$ and \emph{subsampled} $\mathfrak{m}\klsb=(\rhat_{k,sb,\ell},\phat_{k,sb,\ell},\bonus_{k,sb,\ell})$.
\end{definition}
With the above definitions, we can now provide the final regret guarantee of our framework. The only assumption of the theorem is that, for the critically tolerant base learner $\lst=\lceil\log C\rceil$, the model-estimates $\mathfrak{m}_{k,gl,\ell}$ and $\mathfrak{m}_{k,sb,\ell}$ of all $\ell\geq \lst$ are valid for all episodes $k$. Under this assumption, the theorem upper bounds the total regret by sum of the global bonuses of the critically robust base learner, the sum of the subsampled bonuses for all robust base learners $\ell>\lst$, as well as term that can be bounded by uniform convergence arguments. Note that this theorem does not assume that we are in any particular episodic RL setting; in the next section, we apply it to tabular and linear MDPs.
\vspace{0.1in}
\begin{theorem}\label{thm:final_reg_guarantee}
\hspace{0.1em} Let $\lst=\min\{\ell: 2^{\ell} \ge C\}$ and $\Vset\klhpl = \{\vup\klhpl -
    \vsthpl,\vup\khpl -
    \vlow\klhpl\}$. Define the per-episode suboptimality terms for layers $\ell$ and all layers $\ell \le \lst$:
    \begin{align*}
        \Delta_{k,\ell} &:= \Exp^{\boldpi_{k,\ell}^{\textsc{master}}}\left[\sum_{h=1}^H\min\left(3H,2\bonus_{k,sb,\ell}\xhah+ \max_{V\in\Vset_{k,\ell;h+1}}\left(\phat_{k,sb}\xhah-p^{\star}\xhah\right)^{\top}V\right)\right],\\
         \Delta_{k,\le \lst} &:=  \Exp^{\boldpi_{k,\le\lst}^{\textsc{master}}}\left[\sum_{h=1}^H\min\left(3H,2\bonus_{k,gl,\lst}\xhah+ \max_{V\in\Vset_{k,\lst;h+1}}\left(\phat_{k,gl}\xhah-p^{\star}\xhah\right)^{\top}V\right)\right].
    \end{align*}
If both global and subsampled model-estimates $\mathfrak{m}_{k,gl,\ell}$ and $\mathfrak{m}_{k,sb,\ell}$ are valid for all $\ell\geq \lst ,k\in[K]$, then the regret may be bounded as follows
\begin{align*}
&\Regret\leq CH+4eH^2 \left(\sum_{\ell>\lst}\ell q_{\ell} cdot \sum_{k} \Delta_{k,\ell}\right) + \left(8 e CH^2\log (2C)\right) \cdot q_{\le \lst}\sum_k \Delta_{k,\le \lst} \,.
\end{align*}
\end{theorem}

\subsection{Proof of Theorem~\ref{thm:final_reg_guarantee}}
We now describe the main ideas that help us prove the above theorem by using Proposition~\ref{prop:lucb_regret_rho} for different base learners. Base learners $\ell\geq \lceil \log C \rceil$ have both their global and subsampled bonuses account for the corruption they encounter; this \emph{validity} of the bonuses implies that their respective Q-supervisor $\tuple\kl$ is confidence-admissible. Since the active sets are nested and the episode schedule is increasing, in the event that an episode ends in the episode set $\epSet\kl$, all the actions selected throughout the episode are Q-supervised by $\tuple\kl$ and have visitation ratio with respect to the policy $\pi_{k,\ell}^{\textsc{base}}$ at most $\rho\kl\approx H$. As a result, by Proposition~\ref{prop:lucb_regret_rho}, the expected regret contribution from these events is equal to $H$ times the one of the i.i.d. case. 
Under-robust base learners $\ell\leq \lst$ are Q-supervised by the Q-supervisor of the critically tolerant base learner $\lst$ (which behaves as in the i.i.d. case as it has valid bonuses). For events ending in base learners $\ell\leq \lst$, the visitation ratio with respect to $\pi_{k,\lst}^{\textsc{base}}$ is at most $\rho\klelst \approx CH$ (thanks to our episode schedule). This ensures that the expected regret contribution from those events is at most $CH$ times the one of the i.i.d. case. We formalize this sketch below.

From the law of total probability, we have the following regret decomposition:
\vspace{0.1in}
\begin{lemma}\label{lem:policy_decomp}
\hspace{0.1em} For any $\lst \in [\lmax]$, the per-episode regret can be separated in the cumulative regret of base learners $\ell\leq \lst$ and the regret of base learners $\ell>\lst$:
\begin{align*}
    \vst - \valf^{\pimasterk}=  \qlelst\left(\vst-\valf^{\pimaster\klelst}\right)+ \sum_{\ell>\lst}\ql\left(\vst-\valf^{\pimaster\kl}\right).
\end{align*}
\end{lemma}
Moreover, since the active sets are nested, the policies $\pibase_{k,\ell'}$ of any base learner $\ell'<\ell$ always select actions inside the active set of base-learner $\ell$. As a result, $\pimaster_{k,\leq\ell}$ always selects actions which are in the active set $\activeset\kl$, and similarly, $\pimaster\kl$ in the active set $\activeset\kl$. Thus,
\vspace{0.1in}
\begin{lemma}[Q-Supervision in \supervisedc]\label{lem:supervision_supervisedc} \hspace{0.1em}  $\pimaster\klelst$ is supervised by the Q-supervisor $\tuple\klst$, and $\pimaster\kl$ is supervised by the Q-supervisor $\tuple\kl$.
\end{lemma}
In what follows, we let $\lst := \min\{\ell:2^{\ell} \ge C\}$. We shall then show the following:
\begin{enumerate}
    \item The Q-supervisor $\tuple\klst$ which supervises $\pimaster\klelst$ is confidence- admissible. Moreover, the associated visitation ratio is (approximately) at most $\rho_{k,\leq\lst} \lessapprox  CH$, and its Bellman errors are controlled by the global estimates.
    \item For $\ell > \lst$, the Q-supervisor $\tuple\kl$ which supervises $\pimaster\kl$ is confidence-admissible.  Moreover, the associated visitation ratio is (approximately) at most $\rho_{k,\ell} \lessapprox  H$ and its Bellman errors are controlled by the subsampled estimates. 
\end{enumerate}

\xhdr{Confidence-admissibility via valid bonuses.}
First, we provide easy-to-verify conditions which ensure that the Q-supervisors that arise in our algorithm are indeed confidence-admissible. Since the Q-supervisor associated to each base learner $\tuple_{k,\ell}=(\qup_{\ell},\qlow_{\ell},\activeset_{k,\ell},\pigreed_{k,\ell})$ has Q-estimates determined by the tightest of global and subsampled model-estimates, confidence-admissibility of $\tuple_{k,\ell}$ holds if both global and subsampled model-estimates $\mathfrak{m}_{k,gl,\ell'}$ and $\mathfrak{m}_{k,sb,\ell'}$ are valid for all $\ell'\geq \ell$.

\vspace{0.1in}
\begin{lemma}\label{lem:valid_implies_admissibility} \hspace{0.1em} 
Fix $\lst\in[\log K]$. If global and subsampled model-estimates $\mathfrak{m}_{k,gl,\ell}$ and $\mathfrak{m}_{k,sb,\ell}$ are valid w.r.t. Q-supervisors $\tuple_{k,\ell}=(\qup_{k,\ell},\qlow_{k,\ell},\activeset_{k,\ell},\pigreed_{k,\ell})$ for all $\ell\geq\lst$ then $\tuple_{k,\ell}$ are confidence-admissible.
\end{lemma}
\begin{proof}
We apply double induction, an outer on base learners $\ell=\lmax,\ldots,\lst$ and an inner on steps $h=H,\ldots,1$. The induction hypothesis for the outer induction is that $\activeset_{k,\ell;h}(x)$ contains all optimal actions for every state $x$ and step $h$; this is the first condition of confidence-admissibility for $\tuple_{k,\ell}$. The induction hypothesis for the inner induction is that $\vup_{k,\ell;h+1}(x)\geq \vst_{h+1}(x)\geq \vlow_{k,\ell;h+1}(x)$ for all states~$x$. The remainder of the proof follows standard dynamic programming arguments and is provided in Appendix~\ref{app:valid_implies_admissibility}.
\end{proof}

\xhdr{Bounding Bellman errors.} We now show how to bound Bellman errors stemming from base learners with valid model-estimates as a function of the corresponding bonuses.
\vspace{0.1in}
\begin{lemma}\label{lem:belmman_bound}
For any base learner $\ell$ such that global and subsampled model-estimates $\mathfrak{m}\klgl$ and $\mathfrak{m}\klsb$ are valid, its Bellman errors with respect to Q-supervisor $\tuple_{k,\ell}$ are bounded by:
\begin{align*}
    \max\Big(\bellmanup\klh\xa,\bellmanlu\klh\xa\Big)
    &\leq \min\left(3H,2\bonus_{k,gl,\ell}\xa+ \max_{V\in\Vset\klhpl}\left(\phat_{k,gl}\xa-p^{\star}\xa\right)^{\top}V\right)\\   \max\Big(\bellmanup\klh\xa,\bellmanlu\klh\xa\Big)
    &\leq \min\left(3H,2\bonus_{k,sb,\ell}\xa+ \max_{V\in\Vset\klhpl}\left(\phat_{k,sb}\xa-p^{\star}\xa\right)^{\top}V\right)
\end{align*}
where $\bellmanlu\klh=\bellmanup\klh-\bellmanlow\klh$ and $\Vset\klhpl = \{\vup\klhpl -
    \vsthpl,\vup\klhpl -
    \vlow\klhpl\}$.
\end{lemma}
\begin{proof}
We show the guarantee with respect to the global estimates; the proof for the subsampled estimates is identical. All estimates for the Q-functions and the value functions are in $[0,H]$ and the rewards are in $[0,1]$; this means that the LHS is at most $2H+1\leq 3H$ so always less than the first term of the RHS. The bound for the upper Bellman error comes by upper bounding the upper estimate $\qup\klh$ by the global estimates (Algorithm~\ref{alg:UCBVI_base}) and using the fact that $\mathfrak{m}_{k,gl,\ell}$ is valid.
\begin{align*}
    \bellmanup\klh\xa
    &=\qup\klh\xa-\left(r^{\star}\xa+p^{\star}\xa^{\top}\vup_{k,\ell;h+1}\right)\\
    &\leq \left(\rhat_{k,gl}\xa+\phat_{k,gl}\xa^{\top}\vup\klhpl +  \bonus_{k,gl,\ell}\xa\right)-\left(r^{\star}\xa+p^{\star}\xa^{\top}\vup\klhpl\right)\\
    &\leq 2\bonus_{k,gl,\ell}\xa+ \left(\phat_{k,gl}\xa-p^{\star}\xa\right)^{\top}\left(\vup\klhpl -\vst_{h+1}\right).
\end{align*}
The bound for the difference between Bellman errors comes by upper and lower bounding $\qup\klh$ and $\qlow\klh$ with respect to the global estimates (Algorithm~\ref{alg:UCBVI_base}) and rearranging terms.
\begin{align*}
    \bellmanup\klh\xa-\bellmanlow\klh\xa
    &=\left(\qup\klh\xa-p^{\star}\xa^{\top}\vup_{k,\ell;h+1}\right)-\left(\qlow\klh\xa-p^{\star}\xa^{\top}\vlow\klhpl\right)\\
    &= 2\bonus_{k,gl,\ell}\xa+ \left(\phat_{k,gl}-p^{\star}\xa\right)^{\top}\left(\vup\klhpl -\vlow\klhpl\right).
\end{align*}
\end{proof}

\xhdr{Bounding the UCB visitation ratios.} We now bound the UCB visitation ratios of $\pimaster\klelst$ with respect to Q-supervisor $\tuple\klelst$, and of $\pimaster\kl$ with respect to Q-supervisor  $\tuple\kl$. Let $(\ell_1,\dots,\ell_H) \sim \Dsched$ denote a random vector of learner indices from the scheduling distribution; all probabilities are with respect to $\Dsched$. Recall the episode schedule's probability of moving to learner $\ell \in [\lmax]$ from learner $f \in [\lmax]$ is $\frac{1}{2\ell H}\cdot 2^{-(\ell-f)}$ if $\ell > f$. We begin with the following claim:
\vspace{0.1in}
\begin{claim}\label{claim:sched_ub} \hspace{0.1em} Let $\ell_1,\dots,\ell_H$ be drawn from $\Dsched$. Then, for any $h$, $\Pr[\ell_h = f] \le 2^{1 - f}$ for all $f \in [\lmax]$ and $h \in [H]$. Moreover, for $f > 1$, $\Pr[\ell_h = f] \le 2^{-f}$.
\end{claim}
\begin{proof} 
The base case is $f=1$; the probability is upper bounded by $1$. As inductive hypothesis, assume that the probability of $\ell_h=f$ for some step h is at most $2^{-f+1}$. Then, the probability of moving to learner $\ell$ by any learner $f<\ell$ at any step $h$ is, by union bound, at most $\sum_{f=1}^{\ell}2^{-f+1}\sum_{h=1}^H \frac{1}{2\ell H} 2^{-(\ell-f)}\leq 2^{-\ell}$. 
\endproof
\end{proof}
Next, we check whenever we switch to a learner $f$, we  with constant proability we remain with $f$ for an entire rollout:
\vspace{0.1in}
\begin{claim}\label{claim:stay_with_layer} \hspace{0.1em} For any $h \in [H]$ and $f \in [\lmax]$, $\Pr[ \ell_H = f \mid \ell_h = f] \ge \frac{1}{e}$.
\end{claim}
\begin{proof} 
For any $h$, the probability of switching from $\ell_h = f$ to $\ell_{h+1} > f$ is at most $\sum_{\ell > f}\frac{1}{2\ell H}\cdot 2^{-(\ell-f)} \le \frac{1}{H}$. Hence, $\Pr[ \ell_H = f \mid \ell_h = f] \ge (1 - \frac{1}{H})^{H - h} \ge (1 - \frac{1}{H})^H \ge \frac{1}{e}$.
\end{proof}
We can now establish the main visitation ratio bound:
\vspace{0.1in}
\begin{lemma}[UCB Visitation Ratios for \supervisedc] \label{lem:visitation_bound}  The  UCB visitation ratio of $\pimaster\kl$ w.r.t. Q-supervisor $\tuple\kl$ is at most $\rho\kl \le 4e H \ell$. Moreover, the UCB visitation ratio of $\pimaster\klel$ w.r.t. Q-supervisor $\tuple\kl$ is at most $\rho\klel \le 2 \cdot 2^\ell e H \ell$. In particular, we have $\rho\klelst \le 4 e C H \log (2C)$.
\end{lemma}
\begin{proof}
We first bound the UCB visitation ratio of $\pimaster\kl$ with respect to the Q-supervisor $\tuple\kl$. Since the layer schedule is nondecreasing, $\rho\kl^{-1}$ is at least the minimum over $h \in [H]$ of:
\begin{align*}
\Pr[ \ell_h = \dots  = \ell_H = \ell \mid \ell_H = \ell] 
&=\Pr[ \ell_h = \ell \mid \ell_H = \ell] = \frac{\Pr[\ell_h = \ell, \ell_H = \ell]}{\Pr[\ell_H = \ell]} \\ &= \Pr[\ell_H = \ell \mid \ell_h = \ell] \cdot \frac{\Pr[\ell_h = \ell]}{\Pr[\ell_H = \ell]}.
\end{align*}
By Claim~\ref{claim:sched_ub} and Claim~\ref{claim:stay_with_layer}, the above is at least $\frac{\Pr[\ell_h = \ell]}{e2^{1 -\ell}}$. Moreover, we have that
\begin{align*}
\Pr[\ell_H = \ell] &\ge \Pr[\ell_{h-1} \le \ell] \min_{\ell' \le \ell}\Pr[\ell_h = \ell \mid \ell_{h-1} = \ell'].
\end{align*}
Furthermore, $\Pr[\ell_h = \ell \mid \ell_{h-1} = \ell'] =  \frac{2^{\ell'- \ell}}{2\ell H}$ for $\ell' < \ell$, and $1 - 1/H$ for $\ell' = \ell$ (see the proof of Claim~\ref{claim:stay_with_layer}). In either case, $\Pr[\ell_h = \ell \mid \ell_{h-1} = \ell'] \ge \frac{2^{-\ell}}{\ell H}$. 
On the other hand, $ \Pr[\ell_{h-1} \le \ell] \ge 1 - \sum_{\ell' > } \Pr[\ell'_{h-1} \ge \ell] \ge 1 - \sum_{\ell' > \ell} 2^{-\ell} \ge \frac{1}{2}$.
Thus, $\Pr[\ell_H = \ell] \ge \frac{2^{-\ell}}{2H \ell}$ and we conclude
\begin{align*}
\Pr[ \ell_h = \dots  = \ell_H = \ell \mid \ell_H = \ell] \ge \frac{1}{4e H \ell},
\end{align*}
which implies $\rho\kl \le 4e H \ell$. Next, we bound $\rho\klel$. Analogously, we have
\begin{align*}
\Pr[ \ell_h = \dots  = \ell_H = \ell \mid \ell_H \le \ell] &= \frac{\Pr[\ell_H = \ell \text{ and } \ell_h = \ell]}{\Pr[\ell_H \le \ell]} \ge \frac{\Pr[\ell_h = \ell]}{e\Pr[\ell_H \le \ell]} \ge \frac{1}{e}\Pr[\ell_h = \ell] \ge \frac{2^{-\ell}}{2e H\ell},
\end{align*}
where again we use Claim~\ref{claim:stay_with_layer} and the above computation for lower bounding $\Pr[\ell_h = \ell] \ge \frac{2^{-\ell}}{2 H \ell},$. This yields $\rho \klel \le 2\cdot 2^{\ell} e H \ell $. The specialization to $\ell = \lst$ uses $2^{\lst} \le 2C$.
\end{proof}


\begin{proof}[Proof of Theorem~\ref{thm:final_reg_guarantee}]
Since the rewards lie in $[0,1]$ and episodes are of length $H$, which can bound our total regret in terms of regret relative to the optimal policy:
\begin{align*}
 \Regret &\le \sum_{k\in [K]} \vst - V^{\boldpik} + CH, 
\end{align*}
Next, recall the value decomposition from Lemma~\ref{lem:policy_decomp}:
\begin{align*}
    \vst - \valf^{\pimasterk}=  \qlelst\left(\vst-\valf^{\pimaster\klelst}\right)+ \sum_{\ell>\lst}\ql\left(\vst-\valf^{\pimaster\kl}\right).
\end{align*}
Moreover, since model-estimates $\model_{k,gl,\ell}$ and $\model_{k,sb,\ell}$ are valid for all $\ell\geq \lst$, Lemma~\ref{lem:valid_implies_admissibility} implies that Q-supervisors $\tuple_{k,\ell}$ is confidence-admissible for all $\ell\geq \lst$. Moreover, from Lemma~\ref{lem:supervision_supervisedc}, the $\tuple\klst$ supervises $\pimaster\klelst$, and for $\ell > \lst$, $\tuple\kl$ supservises $\tuple\kl$. Hence, letting $\rho\klelst$ be an upper bound on the UCB visitation ratio of $\pimaster\klelst$ with respect to the Q-supervisor $\tuple\klelst$, combining Proposition~\ref{prop:lucb_regret_rho} and Lemma~\ref{lem:belmman_bound}, and using $1 + H \le 2H$, yields
\begin{align*}
    &\vst-\valf^{\boldpi_{k,\leq \lst}^{\textsc{master}}} \\
    &\leq 2H\rho_{k,\leq \lst}\sum_k \underbrace{\Exp^{\boldpi_{k,\leq\lst}^{\textsc{master}}}\left[\sum_{h=1}^H \min\left(3H,2\bonus_{k,gl,\lst}\xhah+ \max_{V\in\Vset_{k,\lst;h+1}}\left(\phat_{k,gl}\xhah-p^{\star}\xhah\right)^{\top}V\right)\right]}_{:= \Delta_{k,\leq \lst}}.
\end{align*}

Similarly, letting $\rho_{k,\ell}$ be the UCB visitation ratio of $\boldpi_{k,\ell}^{\textsc{master}}$ with respect to Q-supervisor $\tuple_{k,\ell}$, 
\begin{align*}
    &\vst-\valf^{\boldpi_{k,\ell}^{\textsc{master}}}\leq\\
    &\quad 2H\rho_{k,\ell} \sum_k \underbrace{\Exp^{\boldpi_{k,\ell}^{\textsc{master}}}\left[\sum_{h=1}^H\min\left(3H,2\bonus_{k,sb,\ell}\xhah+ \max_{V\in\Vset_{k,\ell;h+1}}\left(\phat_{k,gl}\xhah-p^{\star}\xhah\right)^{\top}V\right)\right]}_{:= \Delta_{k,\le \ell}}.
\end{align*}
Finally, bounding $\rho \kl, \rho\klelst$ as in Lemma~\ref{lem:visitation_bound} concludes the proof.
\end{proof}

\subsection{Recipe to turn Theorem~\ref{thm:final_reg_guarantee} into regret guarantees} 
Before specializing to the tabular and linear settings, we describe a common recipe to address estimation errors induced by corruptions. To orient ourselves, let us introduce a filtration which makes the notion of ``conditioning on past history'' precise:
\vspace{0.1in}
\begin{definition}\hspace{0.1em}  We let $(\calF_{k;h})$ denote the filtration generated by all random coins of the algorithm from episodes $1,\dots,k$, actions $(x_{j,h'},a_{j,h'})$ for $j < k$ and for $j = k$ and $h \le h'$, and observed rewards $R_{j,h'}$ for $j < k$ and for $j = k$ and $h < h'$. 
\end{definition}

\xhdr{Idealized statistics.} For this filtration, we define the idealized sequence $\xstoch\kh,\rstoch\kh$ as follows: 
\begin{align*}
\xstoch\khpl  &:= \begin{cases} x\khpl & \text{ episode } k \text{ is not corrupted} \\
\sim \pst(x\kh,a\kh) \mid \calF_{k;h} & \text{ episode } k \text{ is corrupted}
\end{cases}\,.\\
\rstoch\kh &:= \begin{cases} r\kh & \text{ episode } k \text{ is not corrupted} \\
 \rst(x\kh,a\kh) & \text{ episode } k \text{ is corrupted}
\end{cases}\,.
\end{align*}
We also set $\xstoch\kone = x\kone$. Note that $\xstoch\kh,\rstoch\kh$ are always distributed according to transitions and rewards from the nominal MDP, i.e., regardless of episode $k$: a) $\xstoch\khpl \mid \calF\kh \sim p(x\kh,a\kh)$ is a stochastic draw from the nomimal transition distribution and b) $\Exp[\rstoch\kh \mid \calF_{k,h}] = \rst\xa$.

For the analysis, we consider estimates defined using the idealized stochastic data above, which are amenable to analysis from the purely stochastic RL literature.
\vspace{0.1in}
\begin{definition}[Idealized stochastic model estimates]\label{defn:stochastic_model}\hspace{0.1em} The global and subsampled stochastic model estimates are denoted by $\modelstoch\kgl = (\rhatstoch\kgl,\phatstoch\kgl, \bonus\kgl)$ and $\modelstoch\klsb = (\rhatstoch\klsb,\phatstoch\klsb, \bonus\klsb)$ respectively, where $\rhatstoch\kgl,\phatstoch\kgl, \rhatstoch\klsb,\phatstoch\klsb$ are constructed analogously to the estimates in Eq.~\eqref{eq:subsampled_estimates} with data coming from the $\rstoch$ and $\xstoch$ sequences.
\end{definition}

\xhdr{Deviations from idealized statistics due to corruptions.} The model estimates $\model\kgl$ and $\model\klsb$ constructed from observed data (Definition~\ref{defn:true_model}) only differ from the idealized stochastic ones (Definition~\ref{defn:stochastic_model}) due to data from corrupted stages $\datacorrupt\kgl$ and $\datacorrupt\klsb$. 

For the global model estimates, since there are at most $C$ corrupted episodes each of which has horizon $H$, the size of $\datacorrupt$ can be immediately bounded as follows:
\vspace{0.1in}
\begin{lemma}[Global model deviations] \hspace{0.1em}\label{lem:global_corruption}
Let the set of global corrupted stages be
$\datacorrupt\kgl
= \{(j,h): j <k, ~(x\jhpl,r\jh) \ne (\xstoch\jhpl,\rstoch\jh) \}$.
Then, $|\datacorrupt\kgl| \le CH$.
\end{lemma}
We now focus on the subsampled model estimates for robust base learners ($\ell: 2^{\ell}>C$). By Claim~\ref{claim:sched_ub}, the probability of charging an episode's data to $\ell$ is at most $2^{-\ell}$. As a result, robust base learners experience an expected amount of corruption of less than $1$ and with high probability less than a logarithm. Modifying this argument from \cite{LykourisMiPa18}, we formalize the above claim in the following lemma whose proof is provided in Appendix~\ref{app:proof_subsampling}.
\vspace{0.1in}
\begin{lemma}[Robust base learners' subsampled model deviations]\label{lem:local_subsampled_corruptions}\hspace{0.1em}
Define the subsampled corrupted stages 
$\datacorrupt\klsb= \{(j,h): j < k, ~ f(j,H) = \ell, (x\jhpl,r\jh) \ne (\xstoch\jhpl,\rstoch\jhpl)\}$.
Then, with probability $1-\delta/4$, it holds simultaneously for all $\ell: 2^\ell \ge C$ and $k\in[K]$ that:
$|\datacorrupt\klsb| \le 2H\ln(16\ell^2/\delta)$.
We denote by $\eventsubsmp$ the event that the above inequalities hold. Moreover, from Lemma~\ref{lem:global_corruption}, we also have $|\datacorrupt\klsb| \le 2^\ell H$.
\end{lemma}
Having established the idealized stochastic model estimates and the model deviations due to corruptions, we now explain a simple recipe on how to combine them with the Theorem~\ref{thm:final_reg_guarantee} in order to derive the final regret guarantees. In the next sections, we instantiate this recipe for tabular and linear MDPs.
\begin{enumerate}
	\item We first establish \emph{\textbf{concentration bounds for the idealized stochastic model estimates}} $|\rhatstoch\kh - \rst\xa|$ and $|(\phatstoch\xa - \pst\xa)^\top V|$ with respect to value functions $V$. We always need uniform bounds with respect to any value function but in tabular settings it also helps to obtain enhanced bounds with respect a fixed value function, in particular $\vst$. These bounds follow standard concentration arguments essential for the uncorrupted analyses.
	\item We then reason about the \emph{\textbf{sensitivity of model estimates to corruptions}} and obtain bounds on  $|\rhat\kh - \rst\xa|$ and $|(\phat\xa - \pst\xa)^\top V|$ by combining the idealized stochastic concentration bounds with model deviations calculated in Lemmas~\ref{lem:global_corruption}-\ref{lem:local_subsampled_corruptions}.
	\item The next step is to prove that both \emph{\textbf{global and subsampled bonuses are valid}} for the Q-supervisors of all robust base learners. This is established directly by the above sensitivity bounds and is the only assumption needed for Theorem~\ref{thm:final_reg_guarantee}.
	\item The last step is to provide a \emph{\textbf{confidence sum argument to bound the sums of bonuses}} that appear in the RHS of Theorem~\ref{thm:final_reg_guarantee}. This is again similar to the uncorrupted analyses. Combined with the uniform corruption-sensitive bounds on $\max_{V \in \dots } |(\phat\xa - \pst\xa)^\top V|$ (established in step 2), this automatically provides the final regret guarantee.  
\end{enumerate}

\section{Instantiation to tabular MDPs}\label{ssec:tabular_initial}
In this section, we instantiate the aforementioned recipe to the tabular setting and derive an initial result. For clarity, we apply the recipe in a modular way, without any refined techniques. Our result
is loose compared to the final regret guarantee of Theorem~\ref{thm:main_tabular} (and the optimal results for $C=0$ from \cite{AzarOsMu17,SimchowitzJamieson19}). 
Indeed, we only provide worst-case guarantees, and the dependence on $S$ is linear rather than $\sqrt{S}$. The result can also be improved upon by simpler algorithms (see Appendix~\ref{app:discussion}). 
The refinement to Theorem~\ref{thm:main_tabular} invokes several techniques from prior work on tabular RL without adversarial corruptions, but no new ideas; we defer it to Appendix~\ref{app:log_regret}.

\vspace{0.1in}
\begin{theorem}\label{thm:tabular_initial}
\hspace{0.1em} In the tabular setting, with probability at least $1-\delta$, the regret is at most:
\begin{align*}\Regret\lesssim CH^{3}S\sqrt{AT}\polylog(S,A,T,1/\delta)+C^2H^4 \polylog(T,1/\delta),\end{align*}
where $\lesssim$ hides constant terms and $\polylog$ denotes polylogarithmic dependence on arguments.
\end{theorem}

\xhdr{Concentration bounds for idealized stochastic model estimates.} The following two lemmas bound the global and subsampled idealized stochastic model estimates. Their proof follows standard concentration and uniform convergence arguments via a covering technique and is provided in Appendix~\ref{app:idealized_stociastic_tabular} for completeness.
\vspace{0.1in}
\begin{lemma}\label{lem:global_idealized_stochastic_tabular} \hspace{0.1em} With probability $1 - \frac{\delta}{8}$, simultaneously for all $\xa\in\states\times \actions$, $k\in [K]$, it holds:
\begin{align*}
|\rhatstoch\kgl\xa - \rst \xa + (\phatstoch\kgl - \pst(x,a))^\top \Vsthpl| & \le 2H\sqrt{\frac{ 2\ln (64 SAHT^2/\delta)}{N\kgl\xa}}.\\
\max_{V \in [0,H]^{\states}}|(\phatstoch\kgl - \pst(x,a))^\top V| & \le 8H\sqrt{\frac{S \ln (64 SAT^2/\delta)}{N\kgl\xa}}.
\end{align*}
We denote by $\eventglest$ the event that the above inequalities occur.
\end{lemma}

\vspace{0.1in}
\begin{lemma}\label{lem:local_concentration} \hspace{0.1em} With probability $1 - \frac{\delta}{8}$, simultaneously for all $\xa$, $k \in [K]$,$\ell \in [ \lmax]$, it holds : 
\begin{align*}
|\rhatstoch\klsb\xa - \rst \xa + (\phatstoch\klsb - \pst(x,a))^\top \Vsthpl| & \le 2H\sqrt{\frac{ 2\ln (64 SAHT^3/\delta)}{N\klsb\xa}}.\\
\max_{V \in [0,H]^{\states}}|(\phatstoch\klsb - \pst(x,a))^\top V| & \le 8H\sqrt{\frac{S \ln (64 SAT^3/\delta)}{N\klsb\xa}}.
\end{align*}
We denote by $\eventsubest$ the event that the above inequalities occur.
\end{lemma}

\xhdr{Sensitivity of model estimates to corruptions.}
Next we reason about the sensitivity of global and subsampled model estimates to corruptions via the next two lemmas.
\vspace{0.1in}
\begin{lemma}\label{lem:global_concentration} \hspace{0.1em} On the event $\eventglest$, simultaneously for all $k \in [K]$ and $\xa \in \states \times \actions$, it holds: 
\begin{align*}
|\rhat\kgl\xa - \rst \xa + (\phat\kgl - \pst(x,a))^\top \Vsthpl| & \le H \wedge \left( 2H\sqrt{\frac{ 2\ln (64 SAHT^2/\delta)}{N\kgl\xa}}  + \frac{CH^2}{N\kgl\xa}\right)  \\
\max_{V \in [0,H]^{\states}}|(\phat\kgl - \pst(x,a))^\top V| & \le H \wedge \left( 8H\sqrt{\frac{S \ln (64 SAT^2/\delta)}{N\kgl\xa}} + \frac{CH^2}{N\kgl\xa}\right)
\end{align*}
where we use the shorthand $y \wedge z=\min(y,z)$.
\end{lemma}
\begin{proof} 
We prove the first inequality; proving the second is analogous. The LHS is at most $H$ as $\rhat\kgl\xa \in [0,1],\rst\xa \in [0,1]$, $\phat\kl,\pst$ lie on the simplex, and $\Vsthpl \in [0,H-1]^{\calX}$ for $h \ge 1$.
Consider the episodes $\calK_{k,gl} := \{(j,h): j < k\}$ $\{j,h\}$. We have that
\begin{align*}
&\rhat\kgl\xa - \rst \xa + (\phat\kgl - \pst(x,a))^\top \Vsthpl\\ 
&\qquad =\rhatstoch\kgl\xa - \rst \xa + (\phatstoch\kgl - \pst(x,a))^\top \Vsthpl\\
&\qquad\quad+ \frac{1}{N\kgl\xa} \sum_{(j,h) \in \calS_{k,gl}}(\rstoch_{j,h} - r_{j,h}) + (\bolddel_{\xstoch\khpl} - \bolddel_{x\khpl})^\top \Vsthpl
\end{align*}
where $\bolddel_{x} $ denotes the indicator vector on $x \in \states$. The first term is bounded by Lemma~\ref{lem:global_idealized_stochastic_tabular}. To bound the term on the second line, we observe that each term in the sum is at most $H$, since $\rstoch_{j,h} - r_{j,h} \leq 1$, and since for all $h \ge 1$, $\Vsthpl \in [0,H-1]^{\calX}$. 

Moreover, the summands are only nonzero for terms in the difference set $\datacorrupt\kgl(x,a)$ defined in Lemma~\ref{lem:global_corruption}, which has size $|\datacorrupt\kgl| \le CH$. Hence, the term on the second line of the above dislay is at most $\frac{CH^2}{N\kgl\xa}$. Combining with Lemma~\ref{lem:global_idealized_stochastic_tabular} concludes the proof.
\end{proof}

\begin{lemma}\label{lem:local_concentration_tabular} \hspace{0.1em} On event  $\eventsubest \cap \eventsubsmp$, it holds simultaneously for all $\xa$, $k \in [K]$, and $\ell:2^{\ell}\geq C$: 
\begin{align*}
|\rhat\klsb\xa - \rst \xa + (\phat\klsb - \pst(x,a))^\top \Vsthpl| & \le  H \wedge \left( 2H\sqrt{\frac{ 2\ln (64 SAHT^3/\delta)}{N\klsb\xa}}  + \frac{2H^2   \ln\frac{16\ell^2}{\delta}}{N\klsb\xa}\right)\\
\max_{V \in [0,H]^{\states}}|(\phat\klsb - \pst(x,a))^\top V| & \le  H \wedge \left( 8H\sqrt{\frac{ S\ln (64 SAT^3/\delta)}{N\klsb\xa}}  + \frac{2H^2 \ln\frac{16\ell^2}{\delta} }{N\klsb\xa}\right)
\end{align*}
where we use the shorthand $y \wedge z=\min(y,z)$.
\end{lemma}
\begin{proof}
The difference from the proof of  Lemma~\ref{lem:global_concentration} is that the number of corrupted data is controlled by the set $\datacorrupt\klsb$ from Lemma~\ref{lem:local_subsampled_corruptions} for all robust base learners $\ell$ with $2^{\ell}\geq C$. On the event of $\eventsubsmp$ of that lemma, its cardinality is thus at most $2H \ln\frac{16\ell^2}{\delta}$. \end{proof}

\xhdr{Valid global and subsampled bonuses.} Recall from Eq.~ \eqref{eq:subsample_bonus} and~\eqref{eq:global_bonus} that the global and subsampled bonuses are defined as:
\begin{align*}
\bonus\klgl\xa&=\min\left\{H,\left( 2H\sqrt{\frac{ 2\ln (64 SAHT^2/\delta)}{N\kgl\xa}}  + \frac{2^{\ell}H^2}{N\kgl\xa}\right)\right\} \\    
    \bonus\klsb\xa&=\min\left\{H,\left( 2H\sqrt{\frac{ 2\ln (64 SAHT^3/\delta)}{N\klsb\xa}}  + \frac{2H^2   \ln\frac{16\ell^2}{\delta}}{N\klsb\xa}\right)\right\}
\end{align*}
We show the model estimates with these bonuses are valid (Def.~\ref{defn:valid}) for robust base learners.
\vspace{0.1in}
\begin{lemma}\label{lem:valid_bonus_tabular}
\hspace{0.1em}
On event $\eventglest\cap \eventsubsmp\cap \eventsubest$, it holds simultaneously for all $(x,a)$, $k\in [K]$, and $\ell: 2^{\ell}\geq C$, that model-estimates $(\rhat\kgl,\phat\kgl,\bonus\klgl)$ and $(\rhat\klsb,\phat\klsb,\bonus\klsb)$ are both valid. 
\end{lemma}
\begin{proof}
Since $C\leq 2^{\ell}$, the global and subsampled bonuses upper bound the RHS of the first inequalities in Lemmas~\ref{lem:global_concentration} and \ref{lem:local_concentration_tabular} respectively. Moreover, the global and subsampled transition estimates are computed as averages of empirical transitions and are therefore nonnegative. As a result, we satisfy condition (a) of Def.~\ref{defn:valid} and the bonuses are valid. \end{proof}

\xhdr{Bounding the confidence sum.} Finally, we bound the terms that appear in the RHS of Theorem~\ref{thm:final_reg_guarantee}. The following lemma bounds the per-episode contribution of the RHS as a \emph{confidence term} by combining the definition of the bonuses and Lemmas~\ref{lem:global_concentration} and \ref{lem:local_concentration_tabular} and using that $2^{\lst}\leq 2C$. Combined with integration arguments, this leads to the final guarantee.
\vspace{0.1in}
\begin{lemma}\label{lem:confidence_term_tabular}\hspace{0.1em}
Let $\alpha_{1,\le \lst} := 196H^2S\ln (64 SAHT^3/\delta)$ and $\alpha_{2,\le \lst}:=5CH^2$. Then:
\begin{align*}
2\bonus_{k,gl,\lst}\xa+ \max_{V\in\Vset_{k,\lst;h+1}}\left(\phat_{k,gl}\xa-p^{\star}\xa\right)^{\top}V \le \min\left\{3H,\, \sqrt{\frac{\alpha_{1,\le \lst}}{N\kgl\xa}} + \frac{\alpha_{2,\le \lst}}{N\kgl\xa}  \right\}.
\end{align*}
Moreover, for $\alpha_{1,\ell} := 196H^2S\ln (64 SAHT^3/\delta)$ and $\alpha_{2,\ell}:=6H^2   \ln\frac{16\ell^2}{\delta}$, it holds that
\begin{align*}
2\bonus_{k,sb,\ell}\xa+ \max_{V\in\Vset_{k,\ell;h+1}}\left(\phat_{k,gl}\xa-p^{\star}\xa\right)^{\top}V \leq \min\left\{3H,\, \sqrt{\frac{\alpha_{1,\ell }}{N\klsb\xa}} + \frac{\alpha_{2,\ell}}{N\klsb\xa}  \right\}.
\end{align*}
\end{lemma}

\begin{proof}[Proof of Theorem~\ref{thm:tabular_initial}]
We focus on event $\eventglest\cap\eventsubest\cap \eventsubsmp\cap\eventsample$; the failure probability is at most $\delta$.
By Lemma~\ref{lem:valid_bonus_tabular}, under these events, both global and subsampled model-estimates are valid for all $\ell\geq \lst$ and episodes $k$, and thus satisfy the assumption of Theorem~\ref{thm:final_reg_guarantee}. Combining with Lemma~\ref{lem:confidence_term_tabular}, it holds that:
\begin{align*}
&\Regret\leq CH+ 4eH^2 \sum_{\ell>\lst}\ell q_{\ell}
    \cdot\sum_k\Exp^{\boldpi_{k,\ell}^{\textsc{master}}}\left[\sum_{h=1}^H\min\left\{3H,\, \sqrt{\frac{\alpha_{1,\ell }}{N\klsb\xhah}} + \frac{\alpha_{2,\ell}}{N\klsb\xhah}  \right\}\right]\\
&\quad+ q_{\le \lst} \cdot 8 e CH^2\log (2C) \sum_k \Exp^{\boldpi_{k,\lst}^{\textsc{master}}}\left[\sum_{h=1}^H\min\left\{3H,\, \frac{\alpha_{1,\le \lst}}{\sqrt{N\kgl\xa}} + \frac{\alpha_{2,\le \lst}}{N\kgl\xa}  \right\}\right]\,.
\end{align*} 
where  $\alpha_{1,\le \lst} := 196H^2S\ln (64 SAHT^3/\delta)$, $\alpha_{2,\le \lst}:=5CH^2$, $\alpha_{1,\lst} := 196H^2S\ln (64 SAHT^3/\delta)$ and $\alpha_{2,\lst}:=6H^2   \ln\frac{16\lst^2}{\delta}$. By integration arguments detailed in Appendix~\ref{app:confidence_integration}, this is at most:
\begin{align*}
    \lesssim CH &+ C H^2 \log T \left(H\ln \frac{HSA\lmax}{\delta}+CH\right) \\&+CH^2\log T\left(\sqrt{SAT\cdot H^2S\ln(SAHT/\delta)}+H^2\log(\lst/\delta)\ln T+CH^2\ln T\right)
\end{align*}
which is at most order of $CH^{3}S\sqrt{AT}\polylog(S,A,T,1/\delta)+C^2H^4\polylog(T,1/\delta)$.
\end{proof}

\section{Instantiation to linear MDPs}
\label{ssec:regret_guarantee_linear_mdp}
In this section, we introduce the linear MDP model as well as the linear-specific adaptation to our algorithm. We then show how our framework proves Theorem~\ref{thm:main_linear} for linear MDPs via the recipe described in the end of Section~\ref{sec:our_framework_RL}.

\subsection{Model for linear MDP}\label{ssec:linear_model}
In linear MDP,  for each state-action pair $\xa$, we have oracle access to a \emph{feature embedding} $\phi\xa \in \R^d$. It further posits that there exist two operators $\mu^{\star} \in \R^{\calX \times d}$ and $\theta^{\star} \in \R^{1 \times d}$ such that, for all $\xa \in \states \times \actions$, $\pst\xa = \mu^\star\phi\xa$, and $\rst\xa = \theta^\star\cdot \phi\xa$. To ensure efficiency of our algorithm, we assume that $A = |\actions|$ is finite. To avoid measure-theoretic complications, we often assume that $|\calX| < \infty$, but $|\calX|$ appears in neither the regret guarantee nor algorithm runtime.  We further assume $\sup_{x,a}\|\phi\xa\|_2 \leq 1$, $\| v^{\top}\mu^\star\|_2 \leq \sqrt{d}$ for $\|v\|_{\infty}\leq 1$, and $\|\theta^\star\|_2 \leq \sqrt{d}$. The linear setting encompasses the tabular setting by letting $d = |\calX|\cdot|\calA|  = SA$, and assigning $\phi\xa$ to a distinct canonical basis vector of each $\xa$ pair; however, this reduction yields suboptimal dependents on $S,A$. 

\subsection{Algorithm for linear MDPs}\label{ssec:linear_algorithm}
On a high level, our algorithm is based on UCBVI which is a model-based algorithm. Our algorithm conceptually is different from the model-free Least-square Value Iteration (LSVI) algorithm from \cite{jin2019provably}. However, for linear MDP, least square value iteration on history data is equivalent to model-based value iteration (i.e., first fitting a non-parametric model on the history data and then perform value iteration on the fitted model) in the sense that the two procedures generate the same Q functions. 

Below we first show how to obtain the non-parametric model estimate  $\hat{p}$ and the parametric reward estimate $\hat{r}$ from both global and subsampled history data. We then discuss how to set reward bonus for different layers. With the instantiation of $\hat{p}, \hat{r}, b_k$, we can run algorithm~\ref{alg:corruption_robust_rl} for linear MDP. 

\xhdr{Model estimates.} Given the dataset $\{x_{\tau,h}, a_{\tau,h}, r_{\tau_h}, x_{\tau,h+1}\}$ with $\tau \in [1,k-1]$ and $h\in [H]$, we construct global estimates of  $\mu^\star$ and $\theta^\star$ via ridge linear regression:
\begin{align}
    &\muhat\kgl = \arg\min_{\mu} \sum_{\tau=1}^{k-1} \sum_{h=1}^H \| \mu\phi(x_{\tau,h},a_{\tau,h}) - \delta_{x_{\tau,h+1}} \|_2^2 + \lambda \|\mu\|_F^2, \label{eq:muhat_global_def}\\
    & \thetahat\kgl = \arg\min_{\theta} \sum_{\tau=1}^{k-1}\sum_{h=1}^H \left(\theta\cdot \phi(x_{\tau,h},a_{\tau,h}) - R_{\tau,h}\right)^2 + \lambda \|\theta\|_2^2,\label{eq:thetahat_global_def}
\end{align} 
where $\lambda\in\mathbb{R}^+$ is a regularization parameter, and $\delta_{x}$ represents the one-hot encoding vector with zeros in all entries except that the entry corresponding to $x$ is one. Similarly, the local estimates are computed as
\begin{align}
&\muhat\klsb = \arg\min_{\mu} \sum_{\tau \in \epSet_{k,\ell}} \sum_{h=1}^H \| \mu\phi(x_{\tau,h},a_{\tau,h}) - \delta_{x_{\tau,h+1}} \|_2^2 + \lambda \|\mu\|_F^2, \label{eq:muhat_local_def}\\
    & \thetahat\klsb = \arg\min_{\theta} \sum_{\tau\in \epSet_{k,\ell}}^{k-1}\sum_{h=1}^H \left(\theta\cdot \phi(x_{\tau,h},a_{\tau,h}) - R_{\tau,h}\right)^2 + \lambda \|\theta\|_2^2,\label{eq:thetahat_local_def}
\end{align}
Note that we can obtain closed-form solutions for the above via the solution to ridge linear regression; below, we explain  how to represent $\muhat_k$ in a non-parametric fashion without ever explicitly storing the (poentially infinite) matrix $\muhat\kgl$ or $\muhat\klsb$.

With $\muhat\kgl$ and $\thetahat\kgl$, our model estimtaes are defined via
\begin{align*}
\phat\kgl\xa &= \muhat\kgl \phi\xa, \quad \text{and } \quad \rhat\kgl\xa = \thetahat\kgl^{\top}\phi\xa
\end{align*}
Similarly, the local estimates are
\begin{align*}
\phat\klsb\xa &= \muhat\klsb \phi\xa, \quad \text{and } \quad \rhat\klsb\xa = \thetahat\klsb^{\top}\phi\xa
\end{align*}

\xhdr{Confidence bonuses:} Denote the global covariance matrix
\begin{align*}
\Lambda_{k;gl} := \sum_{\tau=1}^{k-1} \sum_{h=1}^H \phi(x_{\tau,h},a_{\tau,h})\phi(x_{\tau,h},a_{\tau,h})^{\top} + \lambda I,
\end{align*} with $\lambda\in\mathbb{R}^+$ and $I$ being the identity matrix, 
and the covariance matrix of base learner $\ell$ as 
\begin{align*}
\Lambda\klsb &:= \sum_{\tau\in \epSet\kl} \sum_{h=1}^H \phi(x_{\tau,h},a_{\tau,h})\phi(x_{\tau,h},a_{\tau,h})^{\top} + \lambda I,
\end{align*}
The bonuses for the linear instantiation of the algorithm are:
\begin{align}
    &\bonus\klglb\xa = \beta (d + \sqrt{A}) H \|\phi\xa\|_{\Lambda_{k,gl}^{-1}} + 4H^2\overline{C}_{\ell;gl} \|\Lambda_{k;gl}^{-1}\phi\xa\|_{2},,\label{eq:linear_global_bonus}\\
     & \bonus\klsb\xa = \beta (d + \sqrt{A}) H \|\phi\xa\|_{\Lambda^{-1}_{k,sb,\ell}} + 4H^2\overline{C}_{\ell;sb} \| \Lambda_{k,sb,\ell}^{-1}\phi\xa \|_2, \quad \label{eq:linear_local_bonus},
\end{align} 
where we denote $\|x\|_{A} := \sqrt{x^{\top}A x}$ for some positive definite matrix $A$, and we define
\begin{align*}
\overline{C}_{\ell;gl} = 2^{\ell} \quad \text{and} \quad \overline{C}_{\ell;sb} = \min\{2^{\ell},2\ln\tfrac{16\ell^2}{\delta}\},
\end{align*}
and the algorithm parameter $\beta$
\begin{align}
 \beta := 14\sqrt{ 30 \ln \tfrac{  A dT^2 H }{\delta}}.
\label{eq:linear_beta_definition}
\end{align}

\xhdr{Efficient implementation.} Note that $\widehat\mu\kgl \in \mathbb{R}^{|\states|\times d}$. For state space with extremely large or infinitely many states, we cannot store $\widehat\mu\kgl$ explicitly. Instead, we can represent $\widehat\mu\kgl$ in a non-parametric way by storing all previous data that is used to estimate $\widehat\mu_k$. Specifically, we can compute the closed-form solution  $\widehat\mu\kgl$ of Equation~\eqref{eq:muhat_global_def}:
\begin{align*}
\widehat{\mu}\kgl := \sum_{i=1}^{k}\sum_{h=1}^H \delta_{x_{i,h+1}} \phi_{i,h}^{\top} \Lambda\kgl^{-1}.
\end{align*} In value iteration, given some value function $V':\states\to [0,H]$ and a state-action pair $\xa$, computing $(\widehat\mu\kgl\phi\xa) \cdot V'$ can be performed as
\begin{align*}
(\widehat\mu\kgl\phi\xa) \cdot V' = \sum_{i=1}^k \sum_{h=1}^H V'(x_{i,h+1}) \phi_{i,h}^{\top} \Lambda\kgl^{-1} \phi\xa, 
\end{align*}  of which the computation complexity is $O( d^3 + kHd^2 )$ with $d^3$ computing from computing the matrix inverse. To support such operation, we store all data which require space in $O(kH)$. This matches to the computational complexity and space complexity from the state-of-art corruption-free algorithm for linear MDPs from \cite{jin2019provably}.

\subsection{Regret guarantee for linear MDPs}
\newcommand{\lelst}{\leq \lst}
\newcommand{\muhatstoch}{\muhat^{\,\mathrm{stoch}}}
\newcommand{\thetahatstoch}{\thetahat^{\,\mathrm{stoch}}}
\newcommand{\eventsubestlinear}{\mathcal{E}^{\text{sub,est,lin}}}
\newcommand{\eventglestlinear}{\mathcal{E}^{\text{gl,est,lin}}}
\newcommand{\calV}{\mathcal{V}}
\newcommand{\klsbhpl}{_{k,l;sb;h+1}}
\newcommand{\klgbhpl}{_{k,l;gl;h+1}}
\newcommand{\ih}{_{i;h}}
\newcommand{\Rih}{R_{i;h}}
\newcommand{\signsig}{\boldsymbol{\sigma}}
\newcommand{\Nbeta}{\calN(\beta)}
For the purpose of our analysis, we define the corresponding idealized stochastic model estimates for the global setting are given by
\begin{align*}
\phatstoch\kgl\xa = \muhatstoch\kgl\phi\xa \quad \text{and} \quad \rhatstoch\kgl\xa = (\thetahatstoch\kgl)^{\top} \phi\xa,
\end{align*}
where $\muhatstoch$ and $\thetahatstoch$ are obtained by applying ridge regression on the stochastic samples, namely
\begin{align}
    &\muhatstoch_{k,gl} = \arg\min_{\mu} \sum_{\tau=1}^{k-1} \sum_{h=1}^H \| \mu\phi(x_{\tau,h},a_{\tau,h}) - \delta_{\xstoch_{\tau,h+1}} \|_2^2 + \lambda \|\mu\|_F^2, \label{eq:muhat_global_def_stoch}\\
    & \thetahatstoch_{k,gl} = \arg\min_{\theta} \sum_{\tau=1}^{k-1}\sum_{h=1}^H \left(\theta\cdot \phi(x_{\tau,h},a_{\tau,h}) - \rstoch_{\tau,h}\right)^2 + \lambda \|\theta\|_2^2,\label{eq:thetahat_global_def_stoch},
\end{align} 
where $\delta_{\xstoch_{\tau,h+1}}$ is the indicator vector at state $\xstoch_{\tau,h+1}$.
The stochastic local estimates $\phatstoch\klsb$ and $\rhatstoch\klsb$ are defined similarly.

\xhdr{Concentration bounds for idealized stochastic model estimates.} Before bounding the global and subsampled estimates using the idealized stochastic data, we bound how rich the class of value functions are that arise in the algorithm. Formally:
\newcommand{\calVst}{\calV_{\star}}
\vspace{0.1in}
\begin{lemma}[Cardinality of the value function set]
\label{lem:value_function_set} \hspace{0.1em} There exists a set of functions $\calVst =  \{f: \calX \to [0,H]\}$ such that the following holds:
\begin{enumerate}
	\item Deterministically, for all episodes $k$, stages $h$, and learners $\ell$, it holds that all value functions
	 $\vsthpl,\vup\klgbhpl,\vup\klsbhpl,\vlow\klgbhpl\vlow\klsbhpl \in \calVst$.
	\item For $\epsilon \le 1$, $\lambda \le 1$ and $\beta \ge 1$, there exists an $\epsilon$-net in the $\ell_{\infty}$ metric $\dist(V,V') = \max_{x \in \states} |V(x) - V(x')|$ which has cardinality bounded as:
	\begin{align*}
	\ln\calN_{\epsilon,\infty}(\calVst) \le A + 12 d^2 \ln\left( 1+ \frac{16 d^2 T H A  \beta}{\lambda \epsilon}  \right).
	\end{align*}
\end{enumerate}
Moreover, setting $\epsilon = 1/T$ and $\lambda = 1$, the above is at most 
\begin{align}
\ln\calN_{\epsilon,\infty}(\calVst) \le  \ln \Nbeta := A + 12 d^2 \ln\left( 1+16 d^2 T^2 H A  \beta   \right) \label{eq:Nbeta}.
\end{align}
\end{lemma}
\begin{proof}[Proof sketch]
The above lemma is proven in Appendix~\ref{ssec:proof:lem:value_function_set}. We sketch the main steps here. The sorts of Q-functions that arise from value iteration have the following form: some $\signsig \in \{-1,1\}$, $a \in \actions$,  bounded vector $w \in \R^d$, bounded parameters $\alpha_1,\alpha_2$, and $\Lambda \succeq \lambda I$, $V(x)$ has the form $x\mapsto w^\top\phi\xa  + \signsig(\alpha_1 \|\Lambda^{-1/2}\phi\xa\| + \alpha_2(\|\Lambda^{-1/2}\phi\xa\|)$. Here, the bounds on $\alpha_1,\alpha_2$ and $w$ depend only on the magnitude of the bonuses, which can be bounded uniformly across rounds. 

The arising Q-functions are \emph{composite} Q-functions $\max_{a \in \activeset} \min\{H,Q_1,Q_2\}$, or $\max\{0,Q_1,Q_2\}$, where $Q_1,Q_2$ have the parameteric form described above.  Modifying \cite{jin2019provably}, we show that an appropriate parametric covering (Def.~\ref{defn:par_cover}) ensures that these composite Q functions which give rise to $\qup\klh,\qlow\klh$ are covered as well. This gives rise to the $d^2$-term in the log covering number. 

Finally, the value functions that arise are of the form $\max_{a \in \activeset} \qup\kl\xa$ or $\max_{a \in \activeset} \qlow\kl\xa$. Thus, we take a separate cover of the Q-functions for each possible $\activeset \subseteq \actions$. This gives rise to $A$ term in the log covering number.
\end{proof}
We next present two lemmas that bound the global and subsampled idealized model estimates defined above. \vspace{0.1in}
\begin{lemma}
[Global concentration for linear RL]
\label{lemma:linear_global_stoch}\hspace{0.1em} Take $\lambda = 1$.  The following bounds hold simultaneously for all $k\in [K]$, $V \in \calVst$, and $\xa\in \states\times\actions$ with probability at least $1-\delta/8$:
\begin{align*}
|\rhatstoch\kgl\xa - \rst \xa + (\phatstoch\kgl - \pst(x,a))^\top V| & \le 7H\|\phi\xa\|_{\Lambda\kgl^{-1}} \cdot  \sqrt{  \ln\left(\Nbeta\cdot \tfrac{16}{\delta} \right)} .\\
|(\phatstoch\kgl - \pst(x,a))^\top V| & \le 4H\|\phi\xa\|_{\Lambda\kgl^{-1}} \cdot  \sqrt{  \ln\left(\Nbeta\cdot \tfrac{16}{\delta} \right)},
\end{align*}
where we recall $\Nbeta$ from Eq.~\ref{eq:Nbeta}. Denote this event as $\eventglestlinear$.
\end{lemma}
\vspace{0.1in}
\begin{lemma}[Subsampled concentration for linear RL]
\label{lem:local_concentration_linear} \hspace{0.1em} Take $\lambda = 1$. The following bounds hold simultaneously for all $k\in [K]$, $V \in \calVst$, and $\xa\in \states\times\actions$ with probability at least $1-\delta/8$:
\begin{align*}
|\rhatstoch\klsb\xa - \rst \xa + (\phatstoch\klsb - \pst(x,a))^\top V| & \le 7H\|\phi\xa\|_{\Lambda\klsb^{-1}} \cdot  \sqrt{\ln\left(\Nbeta\cdot \tfrac{16\lmax}{\delta} \right)} \\
|(\phatstoch\klsb - \pst(x,a))^\top V| & \le 4H\|\phi\xa\|_{\Lambda\klsb^{-1}} \cdot  \sqrt{ \ln\left(\Nbeta\cdot \tfrac{16\lmax}{\delta} \right)},
\end{align*}
where we recall $\Nbeta$ from Eq.~\ref{eq:Nbeta}. Denote this event $\eventsubestlinear$.
\label{lemma:linear_local_stoch}
\end{lemma}
The rigorous proofs of the above lemmas are detailed in Appendix~\ref{app:proof_linear_rl_concentration_stochastic}. Because the above consider \emph{stochastic} transition and reward data, the arguments are near-identical to those of Jin, Yang, Wang, and Jordan \cite{jin2019provably}. We present a sketch below:
\begin{proof}[Proof Sketch of Lemma~\ref{lemma:linear_global_stoch} and Lemma~\ref{lemma:linear_local_stoch}]
Let us focus on the global estimates; the local ones are analogous, but involve a union bound over layers $\ell \in[\lmax]$. Since the parameter estimates $\muhatstoch\kgl$ and $\thetahatstoch\kgl$ are the solution to a ridge regression problem, they admit a closed form solution. Moreover, since they are determined by purely stochastic transitions, this enables a transparent bound on the betwee error $\muhatstoch\kgl$ and $\mustar$ and $\thetahatstoch\kgl$ and $\thetast$. Using this close form for $(\muhatstoch\kl -\mustar)$, we can bound the next-stage value difference for any state-action pair, i.e., $(\muhatstoch\phi\xa - \mu^\star\phi\xa)\cdot V$ for any fixed $V:\states\in [0,H]$, by applying the self-normalized bounds for vector-valued martingales (Lemma~\ref{lemma:self_normalized}), initially introduced for the purpose of improving regret bounds for linear bandits \cite{abbasi2011improved}. We then union bound over the value functions $V \in \calVst$, using a bound on an $\epsilon$-net constructed in Lemma~\ref{lem:value_function_set}.
\end{proof}

\xhdr{Sensitivity of model estimates to corruptions.} 
Next, we bound the sensitivity to corruptions, yielding our next two lemmas:
\vspace{0.1in}
\begin{lemma}[Global concentration for linear RL w/ corruption]\label{lemma:linear_global_corrupt} \hspace{0.1em} Take $\lambda = 1$. On $\eventglestlinear$, the following bounds hold simultaneously for all $k\in [K]$, $\xa\in \states\times\actions$, and $V \in \calVst$: 
\begin{align*}
|\rhat\kgl\xa - \rst \xa + (\phat\kgl - \pst(x,a))^\top V| & \le 7H\|\phi\xa\|_{\Lambda\kgl^{-1}}  \sqrt{  \ln \tfrac{\Nbeta16}{\delta} } + 4 CH^2 \|\phi\xa \|_{\Lambda\kgl^{-2}}\\
\text{and } |(\phat\kgl - \pst(x,a))^\top V| & \le 4H\|\phi\xa\|_{\Lambda\kgl^{-1}}  \sqrt{  \ln \tfrac{\Nbeta16}{\delta} } + 2 CH^2 \|\phi\xa \|_{\Lambda\kgl^{-2}}.
\end{align*} 
\end{lemma}

\vspace{0.1in}
\begin{lemma}[Subsampled concentration for linear RL w/ corruption] Take $\lambda = 1$. On $\eventsubestlinear \cap \eventsubsmp$ holds, the following bounds hold simultaneously for all $k\in [K]$, $\ell\geq \lst$ and $\xa\in \states\times\actions$, and $V \in \calVst$: 
\begin{align*}
&|\rhat\klsb\xa - \rst \xa + (\phat\klsb - \pst(x,a))^\top V| \\
& \qquad\le 7H\|\phi\xa\|_{\Lambda\klsb^{-1}}  \sqrt{  \ln \tfrac{\Nbeta \lmax 16}{\delta} } + 4 H\overline{C}_{\ell;sb} \|\phi\xa \|_{\Lambda\klsb^{-2}}.
\end{align*}
and
\begin{align*}
|(\phat\klsb - \pst(x,a))^\top V| & \le 4H\|\phi\xa\|_{\Lambda\klsb^{-1}}  \sqrt{  \ln \tfrac{\Nbeta \lmax 16}{\delta} } + 2 H \overline{C}_{\ell;sb} \|\phi\xa \|_{\Lambda\klsb^{-2}}.
\end{align*} 
\label{lemma:linear_sub_corrupt}
\end{lemma}
\begin{proof}[Proof sketch of Lemma~\ref{lemma:linear_global_corrupt} and Lemma~\ref{lemma:linear_sub_corrupt}]
We provide a simple, generic sensitivity bound to worst-case perturbations of data in a general ridge-regression setting (see Lemma~\ref{lemma:linear_regression_corrupt} for details). This yields the following general guarantee: if there are at most $CH $ many corruptions, then we can bound the next-stage value prediction difference $| (\muhat_k^{\stoch}\phi\xa - \muhat_k\phi\xa) \cdot V | \leq CH^2 \|\Lambda_{k;gl}^{-1} \phi\xa\|_2$ with arbitrary value function $V:\states\to[0,H]$. For the local estimates, we have only logarithmically many corruptions to account for, in view of Lemma~\ref{lem:local_subsampled_corruptions}. The full proof of the two lemmas is provided in Appendix~\ref{app:proof_linear_rl_concentration_corrupt}. 
\end{proof}

\xhdr{Valid global and subsampled bonuses.} 
Next, we need to verify that global and subsampled bonus dominate the concentration inequalities defined above. This requires some subtley, because the concentration inequalities themselves depend on the confidence parameter $\beta$ through the covering parameter $\Nbeta$. We begin with the following technical lemma, proved in Appendix~\ref{app:validbonuses}
\vspace{0.1in}
\begin{lemma}[Correctness of $\beta$]
\label{lem:correct_beta} \hspace{0.1em} Recall from Eq.~\eqref{eq:linear_beta_definition} the definition of term $\beta$ and recall from Lemma~\ref{lem:value_function_set} the quantitity $\ln \Nbeta := A + 12 d^2 \ln\left( 1+16 d^2 T^2 H A  \beta   \right) $. The following holds:
\begin{align*} 
7H\sqrt{  \ln \tfrac{\Nbeta \lmax 16}{\delta} } \le \beta \cdot H(d+\sqrt{A})
\end{align*}
\end{lemma}

As a direct consequence of the above lemma, we verify that validity of the bonuses, which we first recall from Eq.~ \eqref{eq:linear_global_bonus} and \eqref{eq:linear_local_bonus}:
\begin{align*}
    &\bonus\klglb\xa = (d+\sqrt{A})H\beta \|\phi\xa\|_{\Lambda_{k,gl}^{-1}} + 4H^2\overline{C}_{\ell;gl} \|\Lambda\kgl^{-1}\phi\xa\|_{2}, \quad \overline{C}_{\ell;gl} = 2^{\ell}\\
     & \bonus\ksbell\xa = (d+\sqrt{A})H\beta  \|\phi\xa\|_{\Lambda^{-1}\ksbell} + 4H^2\overline{C}_{\ell;sb}  \| \Lambda\ksbell^{-1}\phi\xa \|_2, \quad \overline{C}_{\ell;sb} = \min\{2^{\ell},2\ln\tfrac{16\ell^2}{\delta}\} 
\end{align*} 

\begin{lemma}\label{lem:valid_bonus_linear} \hspace{0.1em}
For linear MDPs, on event $\eventglestlinear\cap \eventsubsmp\cap \eventsubestlinear$, it holds simultaneously for all $(x,a)$, $k\in [K]$, and $\ell \geq \lst$, that model-estimates $(\rhat_{k;gl}, \phat_{k;gl}, \bonus\klgl)$ and $( \rhat\ksbell,\phat\ksbell, \bonus\klsb)$ are both valid.
\end{lemma}
\begin{proof}
Let us prove validity for the global bonuses. The validity of the subsampled estimates follows analogously.  Recall Definition~\ref{defn:valid}. Since $\phat\xa$ is not necessarily positive in function approximation setting, we aim to proof the second setting in Definition~\ref{defn:valid}. From first  inequality in Lemma~\ref{lemma:linear_sub_corrupt}, followed by Lemma~\ref{lem:correct_beta} and the bound $2^\ell \ge C$, we have that for any $V \in \calVst$,
\begin{align*}
&\left\lvert \rhat\kgl\xa - \rst \xa  + (\phat\kgl - \pst(x,a))^\top V \right\rvert \\
&\qquad \le7H\|\phi\xa\|_{\Lambda\kgl^{-1}}  \sqrt{  \ln \tfrac{\Nbeta16}{\delta} } + 4 CH^2 \|\phi\xa \|_{\Lambda\kgl^{-2}}\\
&\qquad \le (d+\sqrt{A})H\beta  \|\phi\xa\|_{\Lambda\kgl^{-1}} + 4 H 2^{\ell} \|\phi\xa \|_{\Lambda\kgl^{-2}} := \bonus_{k,\ell;gl}\xa.
\end{align*}
Moreover, for any $h \in [H]$, $\{V^\star_{h+1}, \vlow\klhpl, \vup\klhpl\} \subset \calVst$ by Lemma~\ref{lem:value_function_set}. This implies that  $(\phat\kgl, \rhat\kgl, \bonus\kgl)$ satisfies the second notion of valid in Definition~\ref{defn:valid}.
\end{proof}

\xhdr{Bounding the confidence sum.} We now bound the terms that appear in the RHS of Theorem~\ref{thm:final_reg_guarantee}. The following lemma bounds the per-episode contribution of the RHS as a \emph{confidence term} by combining the definition of the bonuses and Lemma~\ref{lemma:linear_global_corrupt} and Lemma~\ref{lemma:linear_sub_corrupt} and using that $2^{\lst}\leq 2C$.
\vspace{0.1in}
\begin{lemma}\label{lem:confidence_term_linear}\hspace{0.1em}
Recall $\Vset_{k,\ell;h+1}:= \{\vup\klhpl -
    \vsthpl,\vup\khpl -
    \vlow\klhpl\}$ as defined in Theorem~\ref{thm:final_reg_guarantee}. On the event $\eventglestlinear \cap  \eventsubsmp\cap \eventsubestlinear$, 
for base learner $\ell^\star$, we have:
\begin{align*}
2\bonus_{k,\lst;gl}\xa+ \max_{V\in\Vset_{k,\lst;h+1}}\left({\phat_{k,gl}\xa-p^{\star}\xa}\right)^{\top}V \le 6 \bonus_{k,\lst;gl}\xa
\end{align*}
Moreover, for $\ell > \lst$, it holds that
\begin{align*}
2\bonus\ksbell\xa+ \max_{V\in\Vset_{k,\ell;h+1}}\left({\phat\ksbell\xa-p^{\star}\xa}\right)^{\top}V \leq 6 \bonus\ksbell\xa
\end{align*}
\end{lemma}
\begin{proof}
For any $V \in \Vset_{k,\ell;h+1}$, Lemma~\ref{lem:value_function_set} implies that we can express $V = V_1 - V_2$ where $V_1, V_2 \in \calVst$. The now follows from the same steps as Lemma~\ref{lem:valid_bonus_linear}, using concentration inequalities Lemma~\ref{lemma:linear_global_corrupt} and Lemma~\ref{lemma:linear_sub_corrupt} and Lemma~\ref{lem:correct_beta}. 
\end{proof}

\begin{proof}[Proof of Theorem~~\ref{thm:main_linear}]
We focus on policies $\pimaster\klelst$ and $\pimaster\kl$ for $\ell \geq \lst$. For policy $\pimaster\klelst$, every episode $k$ we have probability $\qlelst$ to sample it to generate the corresponding trajectory and for $\pimaster\kl$, every episode $k$ we have probability $\ql$ to sample it to generate the corresponding trajectory.  
Using an integration lemma for the linear setting (detailed in Appendix~\ref{app:proof_linear_rl_confidence_sum}), we can show that with probability at least $1-\delta/ 8$,
\begin{align*}
&q_{\le \lst} \sum_k \Exp^{\pimaster_{k,\leq \lst}}\left[{\sum_{h=1}^H \min\left\{ H, 6\bonus_{k,gl,\lst}\xhah \right\}}\right] \\
&\quad\lesssim  H( (d+\sqrt{A})H\beta + CH^2)
\sqrt{K \left(\ln\frac{1}{\delta} + d\ln(1+T)\right)}
\end{align*}
Similarly, for a fixed $\ell \in [\lmax]$, and using $\overline{C}_{\ell;sb} \lesssim \ln(\lmax/\delta)$ for $\delta \in (0,1/2)$,  the following holds with probability $1 - \delta/8\lmax$:
\begin{align*}
&\ql \sum_k\Exp^{\boldpi_{k,\ell}^{\textsc{master}}}\left[{\sum_{h=1}^H  \min\left\{H, 6 \bonus_{k,sb,\ell}\xhah \right\}}\right] \\
& \lesssim CH \ql + H( (d+\sqrt{A})H\beta + H^2 \log \frac{\lmax}{\delta} )
\sqrt{K \left(\ln\frac{\lmax}{\delta} + d\ln(1+T)\right)}.
\end{align*} 
By combining the above two results and simplifying terms, we prove Theorem~\ref{thm:main_linear}.
\end{proof}

\section{Lower bound on action-eliminating algorithms}
\label{sec:lower_bound_active_arm_elimination}
In the special case of multi-armed bandits, near-optimal regret can be achieved by sampling actions uniformly from the plausible sets.\footnote{Even for multi-armed bandits, it does not suffice to choose actions \emph{arbitrarily} from the plausible sets.} This approach is at the heart of the first corruption-robust bandit result \cite{LykourisMiPa18} but does not extend to corruption-robust episodic RL.

Consider the stochastic case (i.e., $\calM_k = \calM$ for all episodes $k$). We focus on algorithms that produce a Q-supervisor
    $\tuplek = (\qupk,\qlowk,\activesetk,\pigreedk)$
before each episode $k$, and at each stage $h$ of this episode select action $a\kh$ independently and uniformly at random from the corresponding plausible set
    $\plauset\kh\ofx$; we refer to these algorithms as \emph{uniformly action-eliminating}. Algorithms' Q-supervisors need to be related to the MDP; in fact, we need this for any problem instance $\calM = (p_0,p,\Drew)$ on a given MDP environment $\envir = (\states,\actions,H)$. Formally, we say that an algorithm is \emph{uniformly admissible} for $\envir$ if, for any problem instance with environment $\envir$,
the Q-supervisor $\tuplek$ it produces is confidence-admissible with probability at least $\tfrac{1}{2}$. We show that any 
algorithm that is uniformly action-eliminating and uniformly admissible suffers linear regret in the first exponentially many episodes.

We emphasize that the below construction applies to \emph{any} uniformly admissible construction of the Q-supervisor; this even applies for constructions that would yield efficient regret if, rather than selecting uniform actions for elimination, we selected actions according to $\pigreedk$ (as in our algorithm).  The lower bound construction follows the so called \emph{combination lock instance}~\footnote{A similar construction was suggested in the context of meritocratic fairness for RL in \cite{JabbariJosKeaMorRot2017}.}, which makes the algorithm follow the optimal trajectory with exponentially small probability for exponentially many episodes.\footnote{In multi-armed bandits, however, uniform action-elimination visits each action with probability at least~$\frac{1}{A}$.} On the other hand, the UCB-policy $\pigreedk$ selects the optimal trajectory in most episodes, under a proper construction of $\qupk$, yields efficient regret \cite{AzarOsMu17}.

\vspace{0.1in}
\begin{theorem}
\label{thm:uniform_elim_lower_bound}
\hspace{0.1em} For every $H \ge 1$ and $A \ge 2$, there exists an MDP environment
    $\envir = (\calX,\calA,H)$
with
    state space $|\calX| = H+1$ and action space $|\calA| = A$,
and a stochastic problem instance with MDP $\calM$ on this environment such that the following holds.
Consider any uniformly action-eliminating algorithm which is uniformly admissible for $\envir$. If the number of episodes is $K \le A^H/4$, then the algorithm suffers expected regret at least $K/8$.
\end{theorem}

\begin{proof}[Proof sketch.] The idea behind the proof is to consider two combination lock instances. A combination lock instance is a chain-MDP with $H$ states $x^{(0)}, \dots x^{(H)}$,  where at every state $x^{(h)}$, there is an optimal action that takes the agent deterministically to $x^{(h+1)}$, while other non-optimal actions take agent back to $x^{(0)}$. A combination lock often puts reward only at $x^{(H)}$, thus making sure that in order to hit any reward signal, the agent needs to select the correct sequence of actions in every step --- thus the name combination lock. Our proof consists of constructing two combination lock instances both of which share the exact same environment and transition, but differ at the reward structure: one has reward $0$ everywhere and the other one has reward $1$ at $x^{(H)}$. The intuition is that these two instances are indistinguishable to the agent until the agent hits $x^{(H)}$ at the second instance. However until then, no actions will be eliminated since all actions are optimal actions in the first instance, and thus any uniformly action-eliminating algorithm will have exponentially small probability (i.e., $1/A^H$) to hit $x^{(H)}$. We formalize this idea in Appendix~\ref{app:aae_lower_bound}.
\end{proof}

\section{Conclusions}
\label{sec:conclusions}
In this paper, we introduced the study of episodic reinforcement learning when the resulting MDP can be adversarially corrupted in some episodes. We provided a modular way to provide corruption-robust guarantees via combining ideas from the classical ``optimism under uncertainty'' paradigm with ``active sets'' and we instantiated these guarantees to the tabular and linear MDP variants. We hope that our framework can be useful in other RL settings where active sets are useful even beyond corruption-robustness.

Our work opens up a number of interesting questions. First, in the tabular case, all works that achieve gap-dependent guarantees in corruption-robustness can only handle a corruption of at most $C<\sqrt{T}$. This is because of the presence of the additive term of $C^2$ in the bounds which appears in both our work and in the follow-up work by Chen, Du, and Jamieson~\cite{ChenDuJamieson21}. The very recent work of Wei and Luo can provide sublinear guarantees for any $C=o(T)$ but does not provide gap-dependent guarantees. It would be interesting to be have an algorithm that has gap-dependent guarantees while also having sublinear worst-case guarantees for any $C=o(T)$; this is indeed the case in multi-armed bandits \cite{gupta2019better,ZimmertSeldin21}.\footnote{Recent work by Wei, Dann, and Zimmert \cite{WeiDannZimmert22} has resolved the above two open questions.} Another interesting direction is to extend our guarantees beyond the episodic RL setting, \eg to the infinite-horizon setting or the stochastic shortest path problem. Finally, there has recently been a prolific line of work with respect to best-of-both worlds guarantees that interpolate between i.i.d. and adversarial rewards while assuming that the transitions are uncorrupted \cite{JinLuo20,JinHuangLuo21}. It would be nice to extend these results beyond the tabular setting and also to settings where transitions can also be corrupted. Finally, establishing lower bounds on how regret scales as a function of the corruption level $C$ is another nice direction that can help characterize the statistical complexity of corruption-robust reinforcement learning.

\subsection*{Acknowledgements}
The authors would like to thank Christina Lee Yu, Sean Sinclair, and \'Eva Tardos for useful discussions that helped improve the presentation of this paper as well as the anonymous review teams at \emph{COLT 2021} and at \emph{Mathematics of Operations Research} for their valuable feedback
\bibliographystyle{alpha}
\bibliography{bibliog,extras}

\newcommand{\etalchar}[1]{$^{#1}$}
\begin{thebibliography}{AYBK{\etalchar{+}}13}

\bibitem[AAK{\etalchar{+}}20]{AmirAttiasKorenLivniMansour2020predictioncorrupted}
Idan Amir, Idan Attias, Tomer Koren, Roi Livni, and Yishay Mansour.
\newblock Prediction with corrupted expert advice.
\newblock In {\em Advances in Neural Information Processing Systems (NeurIPS)},
  2020.

\bibitem[AAP21]{AgarwalAgarwalPatil21}
Arpit Agarwal, Shivani Agarwal, and Prathamesh Patil.
\newblock Stochastic dueling bandits with adversarial corruption.
\newblock In Vitaly Feldman, Katrina Ligett, and Sivan Sabato, editors, {\em
  Proceedings of the 32nd International Conference on Algorithmic Learning
  Theory}, volume 132, pages 217--248. PMLR, 2021.

\bibitem[AC16]{Auer-colt16}
Peter Auer and Chao{-}Kai Chiang.
\newblock An algorithm with nearly optimal pseudo-regret for both stochastic
  and adversarial bandits.
\newblock In {\em 29th Conf. on Learning Theory (COLT)}, 2016.

\bibitem[ACBF02]{auer2002finite}
Peter Auer, Nicol\`o Cesa-Bianchi, and Paul Fisher.
\newblock Finite-time regret bounds for the multi-armed bandit problems.
\newblock {\em Machine Learning}, 47:2--3, 2002.

\bibitem[ACBFS02]{auer2002nonstochastic}
Peter Auer, Nicolo Cesa-Bianchi, Yoav Freund, and Robert~E Schapire.
\newblock The nonstochastic multiarmed bandit problem.
\newblock {\em SIAM journal on computing}, 32(1):48--77, 2002.

\bibitem[AMS09]{audibert2009exploration}
Jean-Yves Audibert, R{\'e}mi Munos, and Csaba Szepesv{\'a}ri.
\newblock Exploration--exploitation tradeoff using variance estimates in
  multi-armed bandits.
\newblock {\em Theoretical Computer Science}, 410(19):1876--1902, 2009.

\bibitem[AOM17]{AzarOsMu17}
Mohammad~Gheshlaghi Azar, Ian Osband, and R{\'e}mi Munos.
\newblock Minimax regret bounds for reinforcement learning.
\newblock In {\em Proceedings of the 34th International Conference on Machine
  Learning (ICML)}, 2017.

\bibitem[AYBK{\etalchar{+}}13]{NIPS2013_4975}
Yasin Abbasi-Yadkori, Peter~L Bartlett, Varun Kanade, Yevgeny Seldin, and Csaba
  Szepesvari.
\newblock Online learning in markov decision processes with adversarially
  chosen transition probability distributions.
\newblock In {\em Advances in Neural Information Processing Systems}. 2013.

\bibitem[AYPS11]{abbasi2011improved}
Yasin Abbasi-Yadkori, D{\'a}vid P{\'a}l, and Csaba Szepesv{\'a}ri.
\newblock Improved algorithms for linear stochastic bandits.
\newblock In {\em Advances in Neural Information Processing Systems}, pages
  2312--2320, 2011.

\bibitem[BKS20]{bogunovic}
Ilya Bogunovic, Andreas Krause, and Jonathan Scarlett.
\newblock Corruption-tolerant gaussian process bandit optimization.
\newblock In {\em International Conference on Artificial Intelligence and
  Statistics (AISTATS)}, 2020.

\bibitem[BS12]{BestofBoth-colt12}
S\'{e}bastien Bubeck and Aleksandrs Slivkins.
\newblock The best of both worlds: stochastic and adversarial bandits.
\newblock In {\em 25th Conf. on Learning Theory (COLT)}, 2012.

\bibitem[CDJ21]{ChenDuJamieson21}
Yifang Chen, Simon~S. Du, and Kevin Jamieson.
\newblock Improved corruption robust algorithms for episodic reinforcement
  learning.
\newblock In {\em Proceedings of the 38th International Conference on Machine
  Learning (ICML)}, 2021.

\bibitem[CKW19]{ChenKrishWang2019}
Xi~Chen, Akshay Krishnamurthy, and Yining Wang.
\newblock Robust dynamic assortment optimization in the presence of outlier
  customers.
\newblock {\em Operations Research (forthcoming)}, 2019.

\bibitem[CSLZ23]{cheung2023nonstationary}
Wang~Chi Cheung, David Simchi-Levi, and Ruihao Zhu.
\newblock Nonstationary reinforcement learning: The blessing of (more)
  optimism.
\newblock {\em Management Science}, 69(10):5722–5739, 2023.

\bibitem[CW23]{Chen2020RobustDP}
Xi~Chen and Yining Wang.
\newblock Robust dynamic pricing with demand learning in the presence of
  outlier customers.
\newblock {\em Oper. Res.}, 71(4):1362--1386, 2023.

\bibitem[DB15]{dann2015sample}
Christoph Dann and Emma Brunskill.
\newblock Sample complexity of episodic fixed-horizon reinforcement learning.
\newblock In {\em Advances in Neural Information Processing Systems}, pages
  2818--2826, 2015.

\bibitem[DLB17]{dann2017unifying}
Christoph Dann, Tor Lattimore, and Emma Brunskill.
\newblock Unifying pac and regret: Uniform pac bounds for episodic
  reinforcement learning.
\newblock In {\em Advances in Neural Information Processing Systems}, pages
  5713--5723, 2017.

\bibitem[DLWZ19]{du2019provably}
Simon~S. Du, Yuping Luo, Ruosong Wang, and Hanrui Zhang.
\newblock Provably efficient q-learning with function approximation via
  distribution shift error checking oracle.
\newblock In Hanna~M. Wallach, Hugo Larochelle, Alina Beygelzimer, Florence
  d'Alch{\'{e}}{-}Buc, Emily~B. Fox, and Roman Garnett, editors, {\em Neural
  Information Processing Systems}, pages 8058--8068, 2019.

\bibitem[DMP{\etalchar{+}}21]{domingues2020kernel}
Omar~Darwiche Domingues, Pierre M{\'{e}}nard, Matteo Pirotta, Emilie Kaufmann,
  and Michal Valko.
\newblock A kernel-based approach to non-stationary reinforcement learning in
  metric spaces.
\newblock In Arindam Banerjee and Kenji Fukumizu, editors, {\em The 24th
  International Conference on Artificial Intelligence and Statistics
  (AISTATS)}, volume 130 of {\em Proceedings of Machine Learning Research},
  pages 3538--3546. {PMLR}, 2021.

\bibitem[EDKM09]{even2009online}
Eyal Even-Dar, Sham~M Kakade, and Yishay Mansour.
\newblock Online markov decision processes.
\newblock {\em Mathematics of Operations Research}, 34(3):726--736, 2009.

\bibitem[EMM06]{Even-DarManMan06}
Eyal Even{-}Dar, Shie Mannor, and Yishay Mansour.
\newblock Action elimination and stopping conditions for the multi-armed bandit
  and reinforcement learning problems.
\newblock {\em Journal of Machine Learning Research (JMLR)}, 7:1079--1105,
  2006.

\bibitem[GKT19]{gupta2019better}
Anupam Gupta, Tomer Koren, and Kunal Talwar.
\newblock Better algorithms for stochastic bandits with adversarial
  corruptions.
\newblock In {\em Proceedings of the 32nd Annual Conference on Learning Theory
  (COLT)}, 2019.

\bibitem[GMSS23]{GolrezaeiManSchSek21}
Negin Golrezaei, Vahideh~H. Manshadi, Jon Schneider, and Shreyas Sekar.
\newblock Learning product rankings robust to fake users.
\newblock {\em Oper. Res.}, 71(4):1171--1196, 2023.

\bibitem[HZG21]{he2020logarithmic}
Jiafan He, Dongruo Zhou, and Quanquan Gu.
\newblock Logarithmic regret for reinforcement learning with linear function
  approximation.
\newblock In Marina Meila and Tong Zhang, editors, {\em Proceedings of the 38th
  International Conference on Machine Learning (ICML) 2021}, volume 139, pages
  4171--4180. {PMLR}, 2021.

\bibitem[JAZBJ18]{jin2018q}
Chi Jin, Zeyuan Allen-Zhu, Sebastien Bubeck, and Michael~I Jordan.
\newblock Is q-learning provably efficient?
\newblock In {\em Advances in Neural Information Processing Systems}, pages
  4863--4873, 2018.

\bibitem[JHL21]{JinHuangLuo21}
Tiancheng Jin, Longbo Huang, and Haipeng Luo.
\newblock The best of both worlds: Stochastic and adversarial episodic mdps
  with unknown transition.
\newblock {\em working paper}, 2021.

\bibitem[JJK{\etalchar{+}}17]{JabbariJosKeaMorRot2017}
Shahin Jabbari, Matthew Joseph, Michael Kearns, Jamie Morgenstern, and Aaron
  Roth.
\newblock Fairness in reinforcement learning.
\newblock In {\em Proceedings of the 34th International Conference on Machine
  Learning - Volume (ICML)}, 2017.

\bibitem[JJL{\etalchar{+}}20]{jin2019learning}
Chi Jin, Tiancheng Jin, Haipeng Luo, Suvrit Sra, and Tiancheng Yu.
\newblock Learning adversarial mdps with bandit feedback and unknown
  transition.
\newblock In {\em Proceedings of the 33rd Annual Conference on Learning Theory
  (COLT)}, 2020.

\bibitem[JL20]{JinLuo20}
Tiancheng Jin and Haipeng Luo.
\newblock Simultaneously learning stochastic and adversarial episodic mdps with
  known transition.
\newblock In {\em Advances in Neural Information Processing Systems (NeurIPS)},
  2020.

\bibitem[JOA10]{jaksch2010near}
Thomas Jaksch, Ronald Ortner, and Peter Auer.
\newblock Near-optimal regret bounds for reinforcement learning.
\newblock {\em Journal of Machine Learning Research}, 11(Apr):1563--1600, 2010.

\bibitem[JYWJ20]{jin2019provably}
Chi Jin, Zhuoran Yang, Zhaoran Wang, and Michael~I Jordan.
\newblock Provably efficient reinforcement learning with linear function
  approximation.
\newblock In {\em Proceedings of the 33rd Annual Conference on Learning Theory
  (COLT)}, 2020.

\bibitem[Kak03]{kakade2003sample}
Sham~Machandranath Kakade.
\newblock {\em On the sample complexity of reinforcement learning}.
\newblock PhD thesis, 2003.

\bibitem[KLPS23]{KrishnamurthyLykPod2020corruptedbinarysearch}
Akshay Krishnamurthy, Thodoris Lykouris, Chara Podimata, and Robert~E.
  Schapire.
\newblock Contextual search in the presence of adversarial corruptions.
\newblock {\em Oper. Res.}, 71(4):1120--1135, 2023.

\bibitem[LLS19]{LiLouShan2019}
Yingkai Li, Edmund~Y Lou, and Liren Shan.
\newblock Stochastic linear optimization with adversarial corruption.
\newblock {\em arXiv preprint arXiv:1909.02109}, 2019.

\bibitem[LMPL18]{LykourisMiPa18}
Thodoris Lykouris, Vahab Mirrokni, and Renato Paes~Leme.
\newblock Stochastic bandits robust to adversarial corruptions.
\newblock In {\em Proceedings of the 50th ACM Annual Symposium on Theory of
  Computing (STOC)}, 2018.

\bibitem[NAGS10]{neu2010online}
Gergely Neu, Andras Antos, Andr{\'a}s Gy{\"o}rgy, and Csaba Szepesv{\'a}ri.
\newblock Online markov decision processes under bandit feedback.
\newblock In {\em Advances in Neural Information Processing Systems}, pages
  1804--1812, 2010.

\bibitem[NGS12]{neu2012adversarial}
Gergely Neu, Andras Gyorgy, and Csaba Szepesv{\'a}ri.
\newblock The adversarial stochastic shortest path problem with unknown
  transition probabilities.
\newblock In {\em Artificial Intelligence and Statistics}, pages 805--813.
  PMLR, 2012.

\bibitem[NO21]{neu_olkhovskaya_2020}
Gergely Neu and Julia Olkhovskaya.
\newblock Online learning in mdps with linear function approximation and bandit
  feedback.
\newblock In Marc'Aurelio Ranzato, Alina Beygelzimer, Yann~N. Dauphin, Percy
  Liang, and Jennifer~Wortman Vaughan, editors, {\em Advances in Neural
  Information Processing Systems (NeurIPS)}, pages 10407--10417, 2021.

\bibitem[NPB20]{NeuPikeBurke20}
Gergely Neu and Ciara Pike-Burke.
\newblock A unifying view of optimism in episodic reinforcement learning.
\newblock In {\em Advances in Neural Information Processing Systems (NeurIPS)},
  2020.

\bibitem[RM19]{rosenberg2019online}
Aviv Rosenberg and Yishay Mansour.
\newblock Online convex optimization in adversarial markov decision processes.
\newblock In {\em International Conference on Machine Learning}, pages
  5478--5486, 2019.

\bibitem[Rus19]{russo2019worst}
Daniel Russo.
\newblock Worst-case regret bounds for exploration via randomized value
  functions.
\newblock In Hanna~M. Wallach, Hugo Larochelle, Alina Beygelzimer, Florence
  d'Alch{\'{e}}{-}Buc, Emily~B. Fox, and Roman Garnett, editors, {\em Advances
  in Neural Information Processing Systems (NeurIPS)}, pages 14410--14420,
  2019.

\bibitem[SJ19]{SimchowitzJamieson19}
Max Simchowitz and Kevin Jamieson.
\newblock Non-asymptotic gap-dependent regret bounds for tabular mdps.
\newblock In {\em Advances in Neural Information Processing Systems (NeurIPS)},
  2019.

\bibitem[SL17]{Seldin-colt17}
Yevgeny Seldin and G{\'{a}}bor Lugosi.
\newblock An improved parametrization and analysis of the {EXP3++} algorithm
  for stochastic and adversarial bandits.
\newblock In {\em 30th Conf. on Learning Theory (COLT)}, 2017.

\bibitem[SS14]{BestofBoth-icml14}
Yevgeny Seldin and Aleksandrs Slivkins.
\newblock One practical algorithm for both stochastic and adversarial bandits.
\newblock In {\em 31th Intl. Conf. on Machine Learning (ICML)}, 2014.

\bibitem[WDZ22]{WeiDannZimmert22}
Chen{-}Yu Wei, Christoph Dann, and Julian Zimmert.
\newblock A model selection approach for corruption robust reinforcement
  learning.
\newblock In Sanjoy Dasgupta and Nika Haghtalab, editors, {\em International
  Conference on Algorithmic Learning Theory (ALT)}, volume 167, pages
  1043--1096. {PMLR}, 2022.

\bibitem[WL18]{WeiL18}
Chen{-}Yu Wei and Haipeng Luo.
\newblock More adaptive algorithms for adversarial bandits.
\newblock In {\em Conference On Learning Theory, {COLT} 2018, Stockholm,
  Sweden, 6-9 July 2018.}, pages 1263--1291, 2018.

\bibitem[WL21]{WeiLuoCOLT21}
Chen-Yu Wei and Haipeng Luo.
\newblock Non-stationary reinforcement learning without prior knowledge: An
  optimal black-box approach.
\newblock In {\em Proceedings of the 34th Annual Conference on Learning Theory
  (COLT)}, 2021.

\bibitem[WWDK21]{wang2019optimism}
Yining Wang, Ruosong Wang, Simon~Shaolei Du, and Akshay Krishnamurthy.
\newblock Optimism in reinforcement learning with generalized linear function
  approximation.
\newblock In {\em 9th International Conference on Learning Representations,
  (ICLR)}, 2021.

\bibitem[YW20]{YangWang2019kernel}
Lin Yang and Mengdi Wang.
\newblock Reinforcement learning in feature space: Matrix bandit, kernels, and
  regret bound.
\newblock In {\em Proceedings of the 37th International Conference on Machine
  Learning (ICML)}, volume 119, pages 10746--10756. {PMLR}, 2020.

\bibitem[ZB19]{zanette2019tighter}
Andrea Zanette and Emma Brunskill.
\newblock Tighter problem-dependent regret bounds in reinforcement learning
  without domain knowledge using value function bounds.
\newblock In Kamalika Chaudhuri and Ruslan Salakhutdinov, editors, {\em
  Proceedings of the 36th International Conference on Machine Learning (ICML)},
  volume~97, pages 7304--7312. {PMLR}, 2019.

\bibitem[ZCZS21]{zhang2021robust}
Xuezhou Zhang, Yiding Chen, Xiaojin Zhu, and Wen Sun.
\newblock Robust policy gradient against strong data corruption.
\newblock In {\em International Conference on Machine Learning}, pages
  12391--12401. PMLR, 2021.

\bibitem[ZLW19]{Haipeng-BoB019}
Julian Zimmert, Haipeng Luo, and Chen{-}Yu Wei.
\newblock Beating stochastic and adversarial semi-bandits optimally and
  simultaneously.
\newblock In {\em 36th Intl. Conf. on Machine Learning (ICML)}, pages
  7683--7692, 2019.

\bibitem[ZN13]{zimin_neu_2013online}
Alexander Zimin and Gergely Neu.
\newblock Online learning in episodic markovian decision processes by relative
  entropy policy search.
\newblock In {\em Neural Information Processing Systems (NeurIPS)}, 2013.

\bibitem[ZS19]{zimmert2018tsallisinf}
Julian Zimmert and Yevgeny Seldin.
\newblock Tsallis-inf: An optimal algorithm for stochastic and adversarial
  bandits, 2019.

\bibitem[ZS21]{ZimmertSeldin21}
Julian Zimmert and Yevgeny Seldin.
\newblock Tsallis-inf: An optimal algorithm for stochastic and adversarial
  bandits.
\newblock {\em Journal of Machine Learning Research (JMLR)}, 2021.

\end{thebibliography}

\newpage
\appendix

\section{Lower bound for action eliminating algorithms (Theorem~\ref{thm:uniform_elim_lower_bound})}
\label{app:aae_lower_bound}

\newcommand{\supervisedunif}{\textsc{SuperVIsed.Unif}}
\begin{proof}[Proof of Theorem~\ref{thm:uniform_elim_lower_bound}]
Let us define two MDP instances $\calM$ and $\calM_0$, defined on the environment $\envir = (\states,\actions,H)$ with  $\states = \{x^{(0)},\dots,x^{(H)}\}$ and action space $[A]$. The transition probabilities under $\calM$ and $\calM_0$ are identical and defined as follows: for $h \in [H-1]$, selecting action $a = 1$ at state $x^{(h)}$ deterministically transitions to $x^{(h+1)}$. Otherwise, selecting $a > 1$ transitions to the \emph{sink state} $x^{(0)}$. All actions in the sink state transition to the sink state with probability 1. 

Under $\calM_0$, all the rewards are defined to be $0$. Therefore, all actions are optimal. On the other hand, under $\calM$, the action $a = 1$ at state $x^{(H)}$ yields reward $1$, while all other actions yield reward zero. This is known as \emph{combination-lock} instance, because to obtain reward, one needs to play the right ``combination'' $a = 1$ for all time steps $h$.

Let $\calF_k$ denote the filtration induced by all observed trajectories and algorithm randomness until the end of episode $k$. We can now make the following observations:
\newcommand{\eventcomb}{\mathcal{E}^{\mathrm{combi}}}
\newcommand{\eventnoelim}{\mathcal{E}^{\mathrm{no\,elim}}}

\begin{enumerate}
	
	\item Let $\eventcomb_k$ denote the event that the learner plays the right ``combination'' of actions, namely $(a_{k;1},\dots,a_{k;H}) = (1,\dots,1)$ at episode $k$. Then, for any event $\calB \in \calF_{k}$, it holds that
	\begin{align}\label{eq:indistinguishable}
	\Pr^{\calM}\left[{\calB \cap \bigcap_{k=1}^{K}(\eventcomb_k)^c}\right] = 	\Pr^{\calM_0}\left[{\calB \cap \bigcap_{k=1}^{K}(\eventcomb_k)^c}\right].
	\end{align}

	In other words, until the learner has seen the sequence $(a_{k;1},\dots,a_{k;H}) = (1,\dots,1)$ once, $\calM$ and $\calM_0$ are indistinguishable.
	\item Let $\eventnoelim_k$ refer to the event that no action is eliminated at episode $k$, i.e.,
	\begin{align*}
	\eventnoelim_k:=  & \{\activeset\kh\ofx = \actions,\, \forall x \in \states, \forall h \in [H]\} \\
	& \quad \cap \{ \forall x\in\states, h\in [H],\, a\in \actions: \qup\kh(x,a) \ge \max_{a'\in \actions} \qlow\kh(x,a')\}
	\end{align*}
	In words, when $\eventnoelim_k$ holds, at episode $k$, the active set $\activeset\kh\ofx$ contains all actions for any $x,h$.  Then, on $\eventnoelim_k$, the uniform-at-random policy $\boldpik$ selects each action $a_h$ independently and uniformly at random.  Hence,  the correct combination is selected with probability $A^{-H}$. That is
	\begin{align}
	\Pr^{\calM}[\eventcomb_k \mid \eventnoelim_k] = A^{-H},\quad \Pr^{\calM_0}[\eventcomb_k \mid \eventnoelim_k] = A^{-H}.
	\end{align}
	Moreover, when $\eventnoelim_k$ holds, $\vst - V^{\calM,\boldpik} = 1 - A^{-H}$, since the probability of selecting the correct combination is $A^{-H}$.
	\item Lastly, let $\eventadm_k$ denote the event that the Q-supervisor $(\qupk,\qlowk,\activesetk)$ is confidence-admissible with respect to the optimal actions and Q-functions of $\calM_0$. Since all actions are optimal under $\calM_0$, it follows that if $\eventadm_k$ holds under $\calM_0$, no actions are eliminated, so that $\eventnoelim_k$ holds as well. 
\end{enumerate}
We are now ready to conclude the proof. For a given $K$, the regret is 
\begin{align*}
\Regret_K^{\calM} &= \sum_{k=1}^K \vst - V^{\calM,\boldpik} \ge (1-A^{-H}) \sum_{k=1}^K \I\{\eventnoelim_k\}
\end{align*}
by point (2) above. Thus 
\begin{align*}
\Exp^{\calM}\left[\Regret^{\calM}_K\right] &\ge (1-A^{-H}) \sum_{k=1}^K \Pr[\eventnoelim_k]\\
&\ge (1-A^{-H})K\Pr^{\calM}\left[\bigcap_{k=1}^K \eventnoelim_k\right] \\
&\overset{(i)}{\ge} (1-A^{-H})K\Pr^{\calM}\left[\bigcap_{k=1}^{K}(\eventcomb_k)^c \cap \bigcap_{k=1}^K \eventnoelim_k\right]\\
&\overset{(ii)}{=} (1-A^{-H})K\Pr^{\calM_0}\left[\bigcap_{k=1}^{K}(\eventcomb_k)^c \cap \bigcap_{k=1}^K \eventnoelim_k\right]\\
&\overset{(iii)}{\ge} (1-A^{-H})K\Pr^{\calM_0}\left[\bigcap_{k=1}^{K}(\eventcomb_k)^c \cap \bigcap_{k=1}^K \eventadm_k\right],
\end{align*}
where in $(i)$ we intersect with the event that the combination $(1,\dots,1)$ is never selected up to episode $K$, in $(ii)$ we use the indistinguishability of $\calM$ and $\calM_0$ from~\eqref{eq:indistinguishable}, and in $(iii)$ we use that admissibility under $\calM_0$ implies no elimination via $\eventnoelim_k \supseteq \eventadm_k$. Continuing, we have Boole's law and a union bound, 
\begin{align*}
\Pr^{\calM_0}\left[\bigcap_{k=1}^{K}(\eventcomb_k)^c \cap \bigcap_{k=1}^K \eventadm_k\right] &= \Pr[\bigcap_{k=1}^K \eventadm_k] - \Pr^{\calM_0}\left[\bigcup_{k=1}^{K}\eventcomb_k \cap \bigcap_{k=1}^K \eventadm_k\right]\\
&\ge \Pr[\bigcap_{k=1}^K \eventadm_k] - \sum_{k=1}^K\Pr^{\calM_0}\left[\eventcomb_k \cap \bigcap_{k'=1}^K \eventadm_{k'}\right]\\
&\ge \Pr[\bigcap_{k=1}^K \eventadm_k] - \sum_{k=1}^K\Pr^{\calM_0}\left[\eventcomb_k \cap \eventadm_k\right]\\
&\ge \Pr[\bigcap_{k=1}^K \eventadm_k] - \sum_{k=1}^K\Pr^{\calM_0}\left[\eventcomb_k \mid \eventadm_k\right].
\end{align*}
By assumption,we have $\Pr[\bigcap_{k=1}^K \eventadm_k] \ge \frac{1}{2}$ for  $K \le A^H/4$. Moreover, as shown above, the probability of selecting the combination if no actions are eliminated is $A^{-H}$. Then, we can lower bound the above sequence of inequalities by $\frac{1}{2} - KA^{-H}$, which is at least $1/4$ for $K \le A^H/4$. By inequality $(iii)$ above, this implies
\begin{align*}
\Exp^{\calM}[\Regret_K] \ge (1-A^{-H})K \cdot \frac{1}{4} \ge K/8,
\end{align*}
where we use $A \ge 2$. 
\end{proof}

\newpage

\section{Simpler algorithms for known corruption levels}
\label{app:discussion}

As discussed in the main body, when the total number of corruption $C$ is known or we have an accurate estimate of the upper bound of $C$, i.e., $\overline{c} \geq C$, we can run the standard UCBVI algorithm with an enlarged bonus. Specifically, we can run UCBVI with the following bonuses for tabular and linear MDP respectively:
\begin{align*}
&\bonus_{k}(x,a)  = \min\left\{H,\left( 2H\sqrt{\frac{ 2\ln (64 SAHT^2/\delta)}{N_k\xa}}  + \frac{\overline{c} H^2}{N_k\xa}\right)\right\}; \\
&\bonus_k\xa =\min\left\{H, \beta d H \|\phi\xa\|_{\Lambda_{k}^{-1}} + 4H^2\overline{c} \|\Lambda_k^{-1}\phi\xa\|_{2}\right\}.
\end{align*}
With the additional $\overline{c}H^2 / N_k(x,a)$ (for tabular MDP), we can ensure that value iteration under $\widehat{p}_k, \widehat{r}_k$ and bonus $b_k$, computes an optimistic policy (the proof is almost identical to that of Lemma~\ref{lem:valid_bonus_tabular}). Similarly, for linear MDP, we can show that $(\widehat{p}_k, \widehat{r}_k, b_k)$ is valid (see proof of Lemma~\ref{lem:valid_bonus_linear}).  

For tabular MDP, we can show that this algorithm achieves the following regret if the true amount of corruption $C \leq \overline{c}$ (otherwise the regret bound does not hold):
\begin{align}
\label{eq:tabular_known_c}
\Regret \lesssim \poly(H) \left[  \min\left\{\sqrt{SA T}, \gapcomplexity \right\} +  S^2 A \ln(T) + \overline{c} SA \ln(T)\right].
\end{align}
Similarly, for linear MDP, using the UCBVI analysis and Lemma~\ref{lemma:potential_function}: 
\begin{align*}
\sum_{k=1}^K \sum_{h=1}^H \| \Lambda_k^{-1} \phi(x_{k;h},a_{k;h})\|_{2}  \lesssim \poly(H)\sqrt{  T d\log(1+ T/\lambda)},
\end{align*} we can show the regret of UCBVI with enlarged reward bonus is:
\begin{align}
\Regret \lesssim \poly(H)\left[ d^{1.5} \sqrt{T} +\overline{c} \sqrt{dT} \ln(T)\right].
\label{eq:linear_known_c}
\end{align}
We now compare the above regret bounds to our main results that are agnostic to $C$.

\subsection{Comparison to our results on tabular MDPs}
The above simple algorithm requires to know the upper bound $\overline{c}$ of $C$, thus the regret bound of the simple algorithm does not improve when the true $C$ is indeed small, i.e., it does not achieve the $\min\{\sqrt{SAT}, \gapcomplexity\}$ when $C$ is a constant. \footnote{To achieve $\min\{\sqrt{SAT}, \gapcomplexity\}$, this simple algorithm has to set $\overline{c} \leq \ln(T)$, which makes it impossible to tolerate any $C = T^{\alpha}$ for $\alpha >0$.} In contrast, our algorithm adapts to the true value of $C$ without relying to guess it correctly. In particular, when $C$ is a constant, our algorithm achieves $\min\{\sqrt{SAT},\gapcomplexity\}$.

That said, if we knew that $\gapcomplexity$ is extremely large (e.g., one cares about the $\sqrt{T}$ regime), the above simple algorithm indeed has an advantage. We can set $\overline{c} = \sqrt{T / (SA)}$, which  gives the following regret bound for the simple algorithm:
\begin{align*}
\Regret_{\overline{c} \leftarrow \sqrt{T/(SA)}} \lesssim \poly(H)\sqrt{SAT}.
\end{align*}
Note that the above regret holds when $C \leq \sqrt{T/(SA)}$. When $C > \sqrt{T / (SA)}$, the above regret is violated, but our algorithm's regret becomes vacuous as it scales linearly with respect to $T$. Thus, in the worst-case regime, our algorithm does not have advantage whenever $C = T^{\alpha}$ for some $0\leq \alpha \leq 1/2$, since the leading-order term in our regret bound is $T^{0.5 + \alpha}$, while the simple algorithm's regret is always in the order of $\sqrt{T}$.

\subsection{Comparison to our results on linear MDPs}
For linear MDPs, in general, the lower term of the regret bound of the UCBVI algorithm with enlarged bonus scales in the order of $\overline{c} \sqrt{dT}$.  To use the simple algorithm in the unknown $C$ setting and compare to our regret bound, one can simply set $\overline{c} = T^{1/4}$, as our algorithm can only tolerate $C = T^{1/4}$.
With $\overline{c} = T^{1/4}$, the simple algorithm admits the following regret:
\begin{align*}
\Regret_{\overline{c}\leftarrow T^{1/4}} \lesssim \poly(H) d^{1.5} \sqrt{T} + \poly(H) \sqrt{d} T^{3/4} \ln(T),
\end{align*} which holds whenever $C \leq T^{1/4}$ and is vacuous whenever $C > T^{1/4}$ (recall our regret bound is vacuous when $C > T^{1/4}$ as well).

For simplicity, let us focus on the setting where $T$ is much larger than $d$ and $H$. In this case, as we can see that when $ 0 \leq C < T^{1/4 - \alpha}$ with $0<\alpha < 1/4$, our regret scales in the order of $T^{3/4-\alpha}$, while the simple algorithm's regret scales in the order of $T^{3/4}$ regardless how small $C$ is. Thus our algorithm has a strict advantage whenever $C < T^{1/4 - \alpha} $ with $\alpha$ being a positive real number in $(0,1/4)$. In the extreme case where $C$ is an absolute constant, our regret scales as $\sqrt{T}$, while the simple algorithm suffers sub-optimal regret $T^{3/4}$. 

We note that one could set $\overline{c}$ to be smaller, e.g., $\overline{c}=T^\beta$ for some $\beta<1/4$; in this case, the regret will be linear for $C>T^{\beta}$ and the same argument about non-adaptivity of the bound of the simple algorithm to the actual $C$ still holds for smaller values of $C$.

\subsection{Building block towards our main algorithm}
Having discussed the difference in the regret from our main algorithm, we now briefly explain why this simpler algorithm serves as a building block for our main algorithm. Note that the robust base learners in our approach need to be able to handle a logarithmic amount of corruption based on their own (subsampled) data. As a result, the choice of the subsampled bonus of base learner $\ell$ (Eq.~\ref{eq:subsample_bonus}) is essentially the simpler algorithm with $\overline{c}\approx H^2 \ln\left(\ell^2 / \delta\right)$.
To handle the under-robust base learners, we also use bonuses based on global data; for that, the choice of the global bonus of base learner $\ell$ (Eq.~\ref{eq:global_bonus}) is essentially the simpler algorithm with $\overline{c}= H^2 2^{\ell}$. 

Applying our framework, the analysis of our algorithm decomposes to bounding the regret of the over-robust base learners (which mimicks the analysis of the simpler algorithm with their subsampled bonus) and the regret of under-robust base learners (which mimicks the analysis of the simpler algorithm with the global bonus of the critically tolerant base learner $\lst=\lceil C\rceil$.

\newpage

\section{Improved guarantee for tabular MDPs (Theorem~\ref{thm:main_tabular})}
\label{app:log_regret}
\newcommand{\eventbern}{\calE^{\mathrm{bern}}}

\newcommand{\gapcliph}{\gapclip_h}
\newcommand{\errlogsqclip}{\widetilde{\errlog}}
\newcommand{\errlogsqclipk}{\errlogsqclip_k}

\newcommand{\errlogsqclipi}{\tilde{\errlog}^{(i)}}
\newcommand{\errlogclipi}{\errlogclip^{(i)}}
\newcommand{\gapcliplow}{\check{\gap}_{\min,\mathrm{low}}}
\newcommand{\Rclip}{R^{\mathrm{clip}}}
\newcommand{\eventcrossclip}{\mathcal{E}^{\mathrm{cross,clip}}}
\newcommand{\errlog}{\mathbf{Err}}
\newcommand{\errlogk}{\errlog_k}
\newcommand{\errlogbar}{\overline{\errlog}}
\newcommand{\errlogbark}{\errlogbar_k}
\newcommand{\errlogclip}{\check{\errlog}}

\newcommand{\gapcheck}{\check{\gap}}
\newcommand{\gapcheckh}{\gapcheck_h}
\newcommand{\condref}[1]{Condition~\ref{#1}}

In this section, we introduce \supervisedtab{}, a refinement of the \supervised{} framework of  Section~\ref{sec:our_framework_RL}, tailored to the tabular setting. This enables the following two improvements:
\begin{enumerate}
	\item  A worst-case regret bound of $\BigOhTil{\poly(H) \left(C \sqrt{SAT} +C^2 S^2 A\right)}$,	improving the leading term from Theorem~\ref{thm:tabular_initial} by a factor of $\sqrt{S}$.
	\item Gap-dependent logarithmic regret with an optimal $S$ dependence on the leading term. 
\end{enumerate}
To avoid redundancy, we prove a decomposition from which both the $\sqrt{SAT}$-regret and logarithmic gap-dependent regret are derived. We stress that the same unmodified\supervisedc{} algorithm achieves these bounds and that the analysis still relies on the three major principles of \supervised{}. 

The remainder of the section is organized as follows. In Appendix~\ref{ssec:background_log}, we initially provide the technical background that is essential for the refined regret. Appendix~\ref{ssec:supervised_tab_frame} formally introduces the \supervisedtab{} framework, and provides two meta-theorems: a generic regret decomposition for policies which satisfy the frameworks conditions (Theorem~\ref{thm:clipped_cat_framework}), and an ``automated'' regret bound which includes the relevant confidence-sum arguments to achieve concrete regret guarantees (Theorem~\ref{thm:logT_catvi}). The latter theorem is instantiated in Section~\ref{ssec:proof_refined_supervisedc_tabular} in order to prove Theorem~\ref{thm:main_tabular}. Finally, Section~\ref{ssec:proof_of_clip_meta_theorem} proves the the regret decomposition (Theorem~\ref{thm:clipped_cat_framework}) referred to above.

\subsection{Background for refined regret}
\label{ssec:background_log}
Here, we review two analytic techniques from the tabular literature which emnable refined regret.  

\xhdr{Optimal $S$-dependence via value decomposition.}
When bounding the Bellman errors, we considered terms  terms of the form
\begin{align*}
\rhat\kl\xa - \rst\xa + (\phat\kl\xa - \pst\xa)^\top \vup\klhpl,
\end{align*}
where $\rhat\kl$ and $\phat\kl$ were either the global or subsampled estimates of $\rst$ and $\pst$, depending on the layer $\ell$ under consideration (analogous terms were also considered with respect to $\vlow\klhpl$). Since $\phat\kl\xa$ and $\vup\klhpl$ are statistically dependent, Appendix~\ref{app:analysis_framework} applies a naive uniform bound
\begin{align*}
|(\phat\kl\xa - \pst\xa)^\top \vup\klhpl| \le \max_{V \in \mathscr{V}}|(\phat\kl\xa - \pst\xa)^\top V|
\end{align*}
where $\mathscr{V}$ is an appropriate set of possible value functions we may encounter. In the tabular setting, $\mathscr{V} = [0,H]^S$, this yields an additional factor of $\sqrt{S}$ in the above error term due to a union bound. 

A more careful decomposition due to Azar, Osband, and Munos \cite{AzarOsMu17} instead writes
\begin{align*}
|(\phat\kl\xa - \pst\xa)^\top \vup\klhpl| &\le |(\phat\kl\xa - \pst\xa)^\top \vst\khpl| \\
&\quad+  |(\phat\kl\xa - \pst\xa)^\top (\vup\klhpl - \vst\khpl)|.
\end{align*}
For the first term, $\vst\khpl$ is now deterministic, so $(\phat\kl\xa - \pst\xa)^\top \vst\khpl$ can be analyzed without uniform concentration. On the other hand, $(\phat\kl\xa - \pst\xa)^\top (\vup\klhpl - \vst\khpl)$ is roughly second order, since $\phat\kl\xa$ is converging to $\pst\xa$ and, ideally $\vup\khpl$ converses to the true value function $\vst$. Thus, the leading term corresponds to $|(\phat\kl\xa - \pst\xa)^\top \vst\khpl|$, which no longer suffers an extraneous $\sqrt{S}$ factor. Suppressing lower order terms, one can use Bernstein's inequality and AM-GM inequality to roughly bound:
\begin{align*}
|(\phat\kl\xa - \pst\xa)^\top (\vup\klhpl - \vst\khpl)| \lessapprox \BigOhTil{\frac{S \cdot \poly(H)}{N\kl\xa}} + \frac{1}{H}\, \pst\xa^\top (\vup\klhpl - \vst\khpl).
\end{align*}
The first term ultimately contributes $\BigOhTil{ S^2 A \poly(H)}$ to the regret, while the lower order term is recursively folded into future Bellman errors, incurring a multiplicative factor of $(1 + \frac{1}{H})^H \le e$ in the final regret bound.

\xhdr{Gap-dependent guarantees via clipping.}
We start by defining the clipping operator:
\begin{align*}
\clip{\epsilon}{x} := x\I(x \ge \epsilon). 
\end{align*} 
The intuition behind the clipping operator is that, for any $\gap > 0$, we have $\sum_{s=1}^k \frac{1}{\sqrt{s}} \lesssim \sqrt{k}$, but 
\begin{align*}
\sum_{s=1}^k \clip{\gap}{\frac{1}{\sqrt{s}}}\lesssim \frac{1}{\gap}.
\end{align*}
More precisely, we shall use the following lemma:
\begin{lemma}[Clipped Integral]\label{lem:clipped_integral} 
\vspace{0.1in}
Let $f(u) = \clip{\gap}{\sqrt{c/u}}$. Then,
\begin{align*}
\int_0^{N} f(u)\rmd u  \le 2\min\left\{\sqrt{c N},\,\frac{c}{\gap}\right\}.
\end{align*}
\end{lemma}
\begin{proof} First, $\int_0^{N} f(u)\rmd u  \le \int_0^N \sqrt{c/u} \rmd u = 2\sqrt{cN}$. To use the gap dependent bound, set $N_0 = c/\gap^2$. Then, for $u > N_0$, we have $\sqrt{c/u} > \gap$, so $\clip{\gap}{\sqrt{c/u}} = 0$. Thus, $\int_0^{N} f(u)\rmd u \le \int_0^{N_0} f(u)\rmd u = 2\sqrt{cN_0} = \frac{c}{\gap}$.
\end{proof}
In Theorem~\ref{thm:clipped_cat_framework}, we therefore establish a regret guarantee which replaces the Bellman errors with clipped analogous. This means, that, when we integrate these errors in a confidence-sum argument (Theorem~\ref{thm:logT_catvi}), we convert  $\sqrt{T}$-bounds into analogues depending on $\frac{1}{\gap}$.

\subsection{The \supervisedtab{} framework}\label{ssec:supervised_tab_frame}
To avoid the complication associated with handling the differences between multiple base learners, we state two meta-theorems of an abstract setup in the same spirit of  Propositions~\ref{prop:lucb_regret}and \ref{prop:lucb_regret_rho}. These \supervisedtab{} theorems have more confidences than their \supervised{} analogous, but also automate much of the regret analysis. We adopt this slightly different organization because the quantities which arise in the refined analysis are somewhat more intricate, and aim to avoid the complication of specifying them for various base learners until the very end.

We begin with the three major components of the \supervised{} framework:
\begin{enumerate}
	\item At each episode $k$, we have a confidence-admissible Q-supervisor $\tuple_k=\left(\qup_k,\qlow_k,\activeset_k,\pigreed_k\right)$, where $\qup_k$ denotes the upper Q-estimate, $\qlow_k$ the lower estimate, $\activeset_k$ a collection of sets of candidate actions $\activeset\kh\ofx$ each containing the optimal $\argmax_{a} \qsth(x,a)$, and $\pigreed\kh\ofx \in \argmax \qup\kh\xa$; see Definition~\ref{defn:conf_admissible} for details.
	\item At each episode $k$, we produce a randomized Markovian policy $\boldpik$ which is Q-supervised by $\tuplek$; that is $\boldpik$ only selects actions $a$  at pairs $(x,h)$ with $a \in \activeset\kh\ofx$; see Definition~\ref{defn:compatible}.
	\item Finally, $\boldpik$ has $\rho$-bounded visitation ratios  with respect to $\pigreedk$; see Definition~\ref{defn:visitation}.
\end{enumerate} 
Let us summarize the above three points into a single condition:
\vspace{0.1in}
\begin{condition}\label{cond:first_three_conditions} \hspace{0.1em} The tuple $\tuplek$ and policies $\boldpik$ satisfy the above three numberd points. 
\end{condition}
Unlike Section~\ref{sec:our_framework_RL}, we also place additional conditions which mimic the optimistic value iteration performed by base learners in \supervisedc. Specifically,
\vspace{0.1in}
\begin{condition}[VI supervisors]\label{cond:VI_triple_log} \hspace{0.1em} Consider Q-supervisor $\tuple_k=(\qup_k,\qlow_k,\activeset_k,\pigreed_k)$. At every round $k$, there exists a model-estimate $(\rtil,\ptil,\bonusk)$,
for which the upper and lower Q-estimates $\qup\kh\xa$ and $\qlow\kh$ satisfy the following, value-iteration-like inequalities:
\begin{align}
    \qup\kh\xa &\le \min\{H,\rhatk\xa + \phatk\xa^\top \vup\khpl + \bonusk\xa\}
    \label{eq:CAT-triples-up}
    \\
    \qlow\kh\xa &\ge \max\{0,\rhatk\xa + \phatk\xa^\top \vlow\khpl -\bonusk\xa\},
    \label{eq:CAT-triples-low}
\end{align}
where
    $\vup\kHpl\ofx = \vlow\kHpl\ofx = 0$,
    $\vup\kh\ofx = \qup\kh(x,\pigreed\kh(x))$,
and
    $\vlow\kh\ofx =\max_{a\in\activeset_{k;h}\ofx}\qlow_k\xa
    $.
\end{condition} 
Next, we state a condition which captures the validity of the bonuses used in the value iteration.
\vspace{0.1in}
\begin{condition}[VI-bonuses]\label{cond:log_vi_bonus}
\hspace{0.1em}
For all $k\in [K]$, $h\in[H]$, $x\in\states$, $a\in\actions$,  $(\rhat_k,\phat_k,\bonus_k)$ satisfy:
\begin{align}
|\rhat_k\xa - \rst\xa + (\phat_k\xa - \pst\xa)^\top \vsthpl| \le \bonusk\xa.
\end{align}
\end{condition}
For our present discussion, \condref{cond:log_vi_bonus} will be used to establish our schematic regret bound, and \emph{not}, as in Theorem~\ref{thm:final_reg_guarantee}, to verify admissibility.

\xhdr{Error Bounds for Refined Regret.}
In order to state our refined regret meta-theorem, we introduce several definitions which facilitate upper bounds on the Bellman errors that arrive. First, we define
\begin{align}
\errlogk\xa := 2\bonusk\xa + \max_{V \in \R^{\states}: \|V\|_{\infty} \le H}|(\phatk\xa - \pst\xa)^\top V|   - \frac{\|V\|_{2,\pst\xa}^2}{H^2}, \label{eq:def_clipped_error_terms}
\end{align}
for all episodes $k \in [K]$ and triples 
$(x,a,h)\in \states\times \actions\times [H]$. 
Here
\begin{align*}
\|V\|_{2,\pst\xa}^2 &:= \sum_{x'}\pst(x'\mid x,a)\; V(x')^2.
\end{align*}
Further, we define several more ``nuanced" error terms (all determined by $\errlogk$).
We set 
\[ \errlogbark\xa := {\min\{H,\errlogk\xa\}}.\] 
Define \emph{effective-gap terms}:
\begin{align}
\gapcheck_h\xa := \max\left\{\frac{\gapmin}{8H^2},\,\frac{\gaph\xa}{8H}\right\}, \quad \gapcliplow := \frac{\gapmin}{64SAH^3} \label{eq:final_clipped_gaps}
\end{align}
Now, we introduce two new error terms by clipping  $\errlogbar$ and $\errlogbar^2$ by the terms from \eqref{eq:final_clipped_gaps}:
\begin{align}
\label{eq:clipped_bounds}
& \errlogclip\kh\xa := \clip{\gapcheckh\xa/4}{\errlogbark\xa},\\
 &\errlogsqclipk\xa:=  \clip{\gapcliplow}{\errlogbark\xa^2}. 
\label{eq:sq_clipped_bounds}
\end{align}
Typically, $\errlogbark\xa = \BigOhTil{1/\sqrt{\nk\xa}}$, so that $\errlogbark\xa^2 = \BigOhTil{1/\nk\xa}$. Hence, even though $\gapcliplow$ scales like $1/SAH$, and might be quite small, this only factors \emph{logarithmically} into the regret contribution of the terms $\errlogsqclip\supi\kh\xa$  bound due to the computation
\begin{align*}
\sum_{s=1}^k \clip{\gap} {\frac{1}{s}}\lesssim \log(\frac{1}{\gap}).
\end{align*}  

\xhdr{Our meta-theorems.} We now state our first meta-theorem, which gives a generic regret bound for sequences of policies $\boldpik$ satisfying the above conditions:
\onespace
\begin{theorem}[Refined Regret Decomposition for Tabular \supervisedtab]\label{thm:clipped_cat_framework} \hspace{0.1em} Let $\Alg$ be an algorithm which at each round $k$, produces randomized policies $\boldpik$ and Q-tuples satisfying Conditions \ref{cond:first_three_conditions},\ref{cond:VI_triple_log},\ref{cond:log_vi_bonus}. Then, for  $\errlogclip$ and $\errlogsqclipk(x_h,a_h)$ defined in \eqref{eq:clipped_bounds} and \eqref{eq:sq_clipped_bounds} above,
\begin{align*}
 \sum_{k=1}^K \vst - V^{\boldpik} &\lesssim H \rho \sum_{k=1}^K\Exp^{\boldpik}\left[\sum_{h=1}^H  \errlogclip\kh(x_h,a_h)  + \errlogsqclipk(x_h,a_h)\right],
\end{align*}
\end{theorem}
The above theorem is proven in Section~\ref{ssec:proof_of_clip_meta_theorem}. For our purposes, we can use the functional form of the error terms $\errlogk\xa$ (in terms of which $\errlogsqclip$ and $\errlogclip$ are defined) to automate the process of establishing a regret bound:  for certain, problem dependence constants $\constc_{1},\constc_{2} \ge 0$, the error terms satisfy the following inequality:
\begin{align}
\errlogk\xa &:= 2\bonusk\xa + \max_{V \in \R^{\states}: \|V\|_{\infty} \le H}|(\phatk\xa - \pst\xa)^\top V|   - \frac{\|V\|_{2,\pst\xa}^2}{H^2}\nonumber\\
&\le \sqrt{\frac{\constc_{1}}{\Nksp\xa}} + \frac{\constc_{2}}{\Nksp\xa}\label{eq:log_term_bound}
\end{align}
Finally, we we recalll the defined the gap-dependent term from Eq. \eqref{eq:gap_complexity}: 
\begin{align*}
\gapcomplexity := \sum_{\xa \in \subacts} \frac{H}{\gap\xa} + \frac{H^2|\optacts|}{\gapmin}
\end{align*} 
We can now state the main theorem, proved in Section~\ref{ssec:prop:logT_catvi}:
\vspace{0.1in}
\begin{theorem}[Automated Regret Bound for Tabular \supervisedtab ]
\label{thm:logT_catvi} \hspace{0.1em} Let $\psample \in (0,1]$,  $\Nsp \ge 1$, $\constc_1,\constc_2 \ge 1$. Let $\calE$ denote the event on which Conditions~\ref{cond:first_three_conditions}, \ref{cond:VI_triple_log},  and  \ref{cond:log_vi_bonus} hold,  that $\Nksp\xa$ are $(\psample,\Nsp)$-subsampled counts (Def.~\ref{defn:subsasmple_count}), and that the upper bound \eqref{eq:log_term_bound} holds for constants $\constc_1,\constc_2$. Then, following regret bound holds:
\begin{align*}
\psample\left(\sum_{k=1}^K \vst - V^{\boldpik}\right) &\lesssim SA \rho  \cdot (H^3\Nsp + H\constc_1 \ln (T) + \constc_2  (H^3 + H\ln T) ) \\
&\quad + \rho H \min\left\{ \sqrt{\constc_1 \psample\, SA T   }, \,\constc_1 \,\gapcomplexity\right\},
\end{align*}
\end{theorem}

\subsection{Proof of refined tabular regret (Theorem \ref{thm:main_tabular}) \label{ssec:proof_refined_supervisedc_tabular}}
In light of Theorem~\ref{thm:logT_catvi}, our regret bound follows by appropriately upper bounding the terms:
\begin{align*}
\errlog_{k,\leq\lst;gl}\xa &:= 2\bonus\klstgl\xa + \max_{V \in \R^{\states}: \|V\|_{\infty} \le H}|(\phat\kgl\xa - \pst\xa)^\top V|   - \frac{\|V\|_{2,\pst\xa}^2}{H^2}\\
\errlog\klsb\xa &:= 2\bonus\klsb\xa + \max_{V \in \R^{\states}: \|V\|_{\infty} \le H}|(\phat\klsb\xa - \pst\xa)^\top V|   - \frac{\|V\|_{2,\pst\xa}^2}{H^2}
\end{align*}
We begin with the following bound using the ``stochastic'' model estimates:
\vspace{0.1in}
\begin{lemma}\label{lem:stoch_uniform_conc_refined}\hspace{0.1em}
The following bound holds with probability $1 - \delta/8$ simultaneously, for all $k \in [K]$, $\xa \in \states \times \actions$, learners $\ell \in [\lmax]$, and value functions $V \in \R^{\states}$ with $\|V\|_{\infty} \le H$:
\begin{align*}
\left|\left(\phatstoch\kgl\xa - \pst\xa\right)^\top V\right| 
    &\le \frac{\|V\|_{2,p}^2}{H^2} + \frac{2SH^2 \ln(64S^2 A T^2/\delta)}{N\kgl\xa}.\\
\left|\left(\phatstoch\klsb\xa - \pst\xa\right)^\top V\right| &\le \frac{\|V\|_{2,p}^2}{H^2} + \frac{2SH^2 \ln(64S^2 A T^3/\delta)}{N\klsb\xa},
\end{align*}
Denote this event $\eventbern$.
\end{lemma}
\begin{proof}
The proof uses a careful application of the Azuma-Bernstein inequality (Lemma~\ref{lem:azuma_bernstein}), specialized to multinomial-martingale sequences (Lemma~\ref{lem:uncorrupted_multinomial_concentration}). A full proof is in Section~\ref{app:stoch_uniform_conc_refined_proof}.\end{proof}
Recall the events $\eventsubsmp$ (from Lemma~\ref{lem:local_subsampled_corruptions}), $\eventglest$ (from Lemma~\ref{lem:global_idealized_stochastic_tabular}), $\eventsubest$ (from Lemma~\ref{lem:local_concentration}), and  $\eventsample$ (from Lemma~\ref{lem:good_subsampling}). By relating the stochastic estimates to the subsampled ones, we have
\vspace{0.1in}
\begin{lemma}
\label{lem:uniform_conc_refined} \hspace{0.1em}
On $\eventbern \cap \eventsubsmp$, the following bound holds for all $k \in [K]$, $\xa \in \states \times \actions$, base learners $\ell \ge \lst$, and value-shaped functions $V \in \R^{\states}$ with $\|V\|_{\infty} \le H$:
\begin{align*}
\left|\left(\phat\kgl\xa - \pst\xa\right)^\top V\right| 
    &\le \frac{\|V\|_{2,p}^2}{H^2} + \frac{2SH^2 \ln(64S^2 A T^2/\delta)}{N\kgl\xa}+  \frac{CH^2  }{N\kgl\xa}\\
\left|\left(\phat\klsb\xa - \pst\xa\right)^\top V\right| &\le \frac{\|V\|_{2,p}^2}{H^2} + \frac{2SH^2 \ln(64S^2 A T^3/\delta)}{N\klsb\xa} + \frac{2H^2 \ln\frac{16\ell^2}{\delta} }{N\klsb\xa}.
\end{align*}\end{lemma}
\begin{proof}[Proof sketch] 
One controls the difference between the partially-corrupted estimates and the purely stochastic estimates analogously to the proofs of Lemma~\ref{lem:local_concentration} and Lemma~\ref{lem:global_concentration}.\end{proof}

\begin{corollary} On $\eventbern \cap \eventsubsmp$, for all $\ell > \lst$, we have:
\begin{align*}
\errlog_{k,\leq\lst;gl} &\le  4 H\sqrt{\frac{2\ln (64 SAHT^2/\delta)}{N\kgl\xa}}+ \frac{ 3CH^2 +  2SH^2 \ln(64S^2 A T^2/\delta)}{N\kgl\xa},\\
\errlog\klsb\xa &\le 4 H\sqrt{\frac{ 2\ln (64 SAHT^3/\delta)}{N\klsb\xa}}+ \frac{6H^2 \ln\frac{16\ell^2}{\delta} + 2SH^2 \ln(64S^2 A T^3/\delta)}{N\klsb\xa},
\end{align*}
\end{corollary}
\begin{proof}
Recall that for $\errlog_{k,\leq\lst;gl}\xa$, we have:
\begin{align*}
\errlog_{k,\leq\lst;gl}\xa &:= 2\bonus_{k,\lst;gl}\xa + \max_{V \in \R^{\states}: \|V\|_{\infty} \le H}|(\phat\kgl\xa - \pst\xa)^\top V|   - \frac{\|V\|_{2,\pst\xa}^2}{H^2}\\
& \leq 2\bonus_{k,\lst;gl}\xa + \frac{2SH^2 \ln(64S^2 A T^2/\delta)}{N\kgl\xa}+  \frac{CH^2  }{N\kgl\xa},
\end{align*} where in the first inequality we use Lemma~\ref{lem:stoch_uniform_conc_refined}. 
Substitute the form of  the global bonus:
\begin{align*}
\bonus\klgl\xa&=\min\left\{H,\left( 2H\sqrt{\frac{ 2\ln (64 SAHT^2/\delta)}{N\kgl\xa}}  + \frac{2^{\ell}H^2}{N\kgl\xa}\right)\right\} 
\end{align*} with $2^\lst = C$ concludes the first part. 
Using the second inequality in Lemma~\ref{lem:stoch_uniform_conc_refined} and the subsampled bonus definition concludes second part of the corollary. 
\end{proof}

\begin{proof}[Proof of Theorem~\ref{thm:main_tabular}]
We condition on the event $\eventsample \cap \eventbern \cap \eventsubsmp \cap \eventglest \cap \eventsubest$ below. The failure probability of the event is bounded by $\delta$.
By Lemma~\ref{lem:policy_decomp}, we have 
\begin{align*}
 \sum_{k} \left(\vst - \valf^{\pimasterk}\right)=  \qlelst \sum_{k }\left(\vst-\valf^{\pimaster\klelst}\right)+ \sum_{\ell>\lst}\ql \sum_k \left(\vst-\valf^{\pimaster\kl}\right),
\end{align*}
We first bound the regret corresponding to $\pimaster\klelst$. For for $\pimaster\klelst$, its associated error term is
\begin{align*}
\errlog_{k,\leq\lst;gl} \le  4 H\sqrt{\frac{2\ln (64 SAHT^2/\delta)}{N\kgl\xa}}+ \frac{ 3CH^2 +  2SH^2 \ln(64S^2 A T^2/\delta)}{N\kgl\xa}
\end{align*} Visitation ratios $(\rho\klelst, \rho\kl)$ are upper bounded by Lemma~\ref{lem:visitation_bound}. Since, Conditions \ref{cond:first_three_conditions},\ref{cond:VI_triple_log},\ref{cond:log_vi_bonus} in the event $ \eventglest \cap \eventsubest$, by Theorem~\ref{thm:logT_catvi} with 
$\alpha_1 = 32 H^2 \ln(64SAHT^2/\delta)$, $\alpha_2 = 3CH^2 +  2SH^2 \ln(64S^2 A T^2/\delta)$,  $(\psample,\Nsp)= \left(q_{\le \lst}, 4H\ln\frac{32 eHSA \lmax}{\delta}+ 2CH\right)$, $\rho \klelst \le 4 e C H \log (2C)$:
\begin{align*}
&\sum_{k=1}^K  \qlelst \sum_{k }\left(\vst-\valf^{\pimaster\klelst}\right) \\
& \lesssim SAH C \ln(K) \left( H^4 \ln(SAHT/\delta) +  H^3\ln(T)(CH^2+SH^2\ln(SAT/\delta))  \right) \\
& \qquad + \left(  CH^2\ln(K) \min\left\{ \sqrt{H^2  \ln (SAHT/\delta) SAT}, H^2 \ln (SAHT/\delta)\gapcomplexity   \right\} \right)  \\
& \lesssim \min\left\{ CH^3  \sqrt{ SAT\ln (SAHT/\delta) }\ln(K), \quad CH^4 \cdot \gapcomplexity \cdot  \ln^2(SAHT/\delta)\ln(K)   \right\}  \\
& \qquad + H^5SA C \ln^2(HSAT/\delta) + H^6SA C^2 \ln^2 (T) + H^6S^2A C\ln^3(SAT/\delta)
\end{align*}
Similarly, for $\pimaster\kl$, we use its associated error term $\errlog\klsb\xa$ instead. Applying again Theorem~\ref{thm:logT_catvi} with:
$\alpha_1 = 32H^2 \ln(64SAHT^3/\delta)$, $\quad \alpha_2 = 6H^2 \ln\frac{16\ell^2}{\delta} + 2SH^2 \ln(64S^2 A T^3/\delta)$, $(\psample,\Nsp) = (\ql, 4H \ln\frac{32 eHSA \lmax}{\delta}+ 2CH q_{\ell})$, $\rho\kl = 4eH\ell$,  we have:
\begin{align*} 
&\ql \sum_k \left(\vst-\valf^{\pimaster\kl}\right) \\ 
& \lesssim SAH \ln(K) \left( H^4 \ln\frac{ HSA }{\delta}+ CH^4 + H^3\ln^2(SAHT/\delta) + H^3\log(T) (H^2\ln(\ln(K)^2/\delta) + SH^2\ln(SAT/\delta)) \right) \\
& \qquad + H^2\ln(K)\min\left\{ \sqrt{ H^2\ln(SAHT/\delta) SAT },  H^2\ln(SAHT/\delta)  \gapcomplexity \right\} \\
& \lesssim \min\left\{  H^3\sqrt{SAT \ln(SAHT/\delta)}\ln(K),\quad H^4 \cdot \gapcomplexity \cdot  \ln^2(SAHT/\delta)  \right\} \\
& \qquad + H^5 SA \ln^2(SAHT/\delta) + H^5SA C\ln(T) + H^4SA \ln^3(SAHT/\delta) + H^6S^2A\ln^3(SAT/\delta)
\end{align*}
Adding the regret bound for different policies together with the fact that there are at most $\ln(T)$ many base learners, we get:
\begin{align*}
& \sum_{k} \left(\vst - \valf^{\pimasterk}\right) \\
 & \leq \min\left\{ CH^3  \sqrt{ SAT\ln (SAHT/\delta) }\ln(T)^2, \quad CH^4 \cdot \gapcomplexity \cdot  \ln^3(SAHT/\delta)   \right\} \\
 & \qquad  + H^5SA C \ln^2(HSAT/\delta) + H^6SA C^2 \ln^2 (T) + H^6S^2A C\ln^3(SAT/\delta)\\
 & \qquad + H^5 SA \ln^2(SAHT/\delta) + H^5SA C\ln(T) + H^4SA \ln^3(SAHT/\delta) + H^6S^2A\ln^3(SAT/\delta)\\
 & \lesssim \min\left\{ CH^3  \sqrt{ SAT\ln (SAHT/\delta) }\ln^2(T), \quad CH^4 \cdot \gapcomplexity \cdot  \ln^3(SAHT/\delta)   \right\} \\
 & \qquad +  H^6S^2 A C \ln^3(HSAT/\delta)   +  H^6 SAC^2 \ln^2(T), 
\end{align*} 
Substituing in $\deltaeff :=  \delta/HSAT$ concludes gives the following bound, which concludes the proof
\begin{align*}
 &\lesssim \min\left\{ CH^3  \sqrt{ SAT\ln \tfrac{1}{\deltaeff} }\ln^2(T), \quad CH^4 \cdot \gapcomplexity \cdot  \ln^3 \tfrac{1}{\deltaeff}   \right\} \\
 & \qquad +  H^6 A \left(S^2 C \ln^3 \tfrac{1}{\deltaeff} + SC^2\ln^2(T)\right).
 \end{align*}
\end{proof}

\subsection{Proof of the meta-theorem for logarithmic regret (Theorem~\ref{thm:clipped_cat_framework})}
\label{ssec:proof_of_clip_meta_theorem}

\xhdr{Clipping the Bellman Errors.}
	Recall the gap definitions from Eq.~\ref{eq:final_clipped_gaps}
	\begin{align*}
	\gapcheck_h\xa = \max\left\{\frac{\gapmin}{8H^2},\,\frac{\gaph\xa}{8H}\right\}, \quad \gapcliplow = \frac{\gapmin}{64SAH^3}.
	\end{align*}
	Moreover, recall the definition of the Bellman Error
	\begin{align*}
	\bellmanup\kh\xa
	    &=\qup\kh\xa-\left(r^{\star}\xa+p^{\star}\xa\cdot
	        \vup_{k;h+1}\right),
	    &\vup\kh(x)=\qup\kh(x,\pigreed\kh(x)),
	    \\
	\bellmanlow\kh\xa
	    &=\qlow\kh\xa-\left(r^{\star}\xa+p^{\star}\xa\cdot
	       \vlow_{k;h+1}\right).
	&\vlow\kh(x)=\max_{a\in\activeset\kh\ofx}\qlow\kh(x,a)
\end{align*}
For simplicity, we also introduce 
\begin{align*}
\bellmanlu\kh\xa:= \bellmanup\kh\xa - \bellmanlow\kh\xa.
\end{align*}
 In this section, we develop bounds in terms of their clipped analogues. 	We define the associated \emph{clipped} Bellman errors as follows:
	\begin{align}
	&\advfullclip\kh\xa = \clip{ \gapcheck_h\xa }{\adv\kh\xa},~ \advlucbfullclip\kh\xa = \clip{ \gapcheck_h\xa }{ \advlucb\kh\xa }. \label{eq:def_clipped_bellman}
	\end{align}
	We now have the following lemma that shows that the per-episode regret can be upper bounded by the clipped Bellman errors.

\begin{lemma}\label{lem:clipped_regret} \hspace{0.1em} Suppose that at a round $k$, $\tuplek := (\qup_k,\qlow_k,\activeset_k,\pigreed_k)$ is a confidence-admissible Q-supervisor, and that $\boldpi_k$ is Q-supervised by $\tuplek$. Then, 
\begin{align*}
\vst - \valf^{\boldpi_k} &\le 2\sum_{h=1}^H \Exp^{\boldpi_k}\left[\advfullclip\kh(x_{h},a_{h})\right] + 20 e{\sum_{\hbar=1}^H \Exp^{\boldpi_k\oplus_{\hbar}\pigreed_k}\left[\sum_{h = \hbar}^{H}\advlucbfullclip\kh(x_h,a_h) \right]}.
\end{align*}
\label{lem:surplus_log}
\end{lemma}
\begin{proof}[Proof Sketch]
The proof of the above lemma is based on the same fundamental ideas as in the proof of Proposition~\ref{prop:lucb_regret}, augmented by the techniques from \cite{SimchowitzJamieson19}. The key insight is that a policy $\boldpik$ can only select a suboptimal action $a \notin\pisth(x)$ when $\qup\kh\xa - \qlow\kh\xa$ is considerably larger than the gap $\gaph\xa$. As in \cite{SimchowitzJamieson19}, we demonstrate that this occurs \emph{either} if the current Bellman error is larger than $\Omega(\gaph\xa)$, or if future Bellman errors \emph{under rollouts selecting the bonus-greedy policy $\pigreedk$} are large. 

When the Bellman error is large, it can be clipped, and otherwise its contribution to the regret can be subsumed by later terms. Again, the mismatch between $\boldpik$ and the bonus-greedy policy introduces the need for rollouts of fictitious sequences under $\boldpik\oplus_h \pigreedk$. Clipping the Bellman errors associated with optimal actions $a \notin \pisth\ofx$ selected by  Q-supervised policies requires considerably more technical effort. A full proof of the lemma is given in Appendix~\ref{ssec:proof_of_reg_decompose_clip}. \end{proof}

\xhdr{Bounding the clipped Bellman Errors.}
 Recall $\errlogbark\xa := \min\{H,\errlogk\xa\}$ and $\errlogk\xa$ from \eqref{eq:def_clipped_error_terms}. It holds that:\vspace{0.1in}
	\begin{lemma}\label{lem:refined_lucb_bellman_bound} \hspace{0.1em} When the bonuses are valid, we have the inequality 
	\begin{align*}
0 \vee \bellmanup\kh\xa \vee \bellmanlu\kh\xa \le  \errlogbark\xa  + \frac{\|\vup\khpl - \vlow\khpl\|_{2,\pst\xa}^2}{H^2}.
\end{align*}
where $(y\vee z)=\max(y,z)$.
	\end{lemma}
	\begin{proof}
	From the VI construction of the Q-supervisions:
	\begin{align*}
	\bellmanup\kh\xa
	    &=
	    \qup\kh\xa-\left(r^{\star}\xa+p^{\star}\xa\cdot
	        \vup_{k;h+1}\right),
	    \\
	    &\le \left( \rhatk\xa + \bonusk\xa + \phatk\xa^\top \vup\khpl\right) +\bonusk\xa -\left(r^{\star}\xa+p^{\star}\xa\cdot
	        \vup_{k;h+1}\right),\\
	    &=  \rhatk \xa - \rst\xa + (\phatk\xa - \pst \xa)^\top  \vup_{k;h+1} + \bonus \kh\xa\\
	    &=  \rhatk \xa - \rst\xa + (\phatk\xa - \pst \xa)^\top  \vst_{k;h+1} + \bonus \kh\xa \\
	    &\qquad+ (\phatk\xa - \pst \xa)^\top  (\vup\khpl - \vst_{k;h+1}).
\end{align*}
From the validity of the bonuses, $\rhatk \xa - \rst\xa + (\phatk\xa - \pst \xa)^\top  \vst_{k;h+1}  \le 2\bonusk\xa$. Moreover, since $0 \le \vst\khpl \le \vup\khpl \le H$, we have that $\| \vup \khpl- \vst \khpl\|_{\infty} \le H$, so that  we obtain the following from the definition of $ \errlogk\xa$ from \eqref{eq:def_clipped_error_terms}:
\begin{align*}
	\bellmanup\kh\xa &\le  2\bonus \kh\xa + (\phatk\xa - \pst \xa)^\top  (\vup\khpl - \vst_{k;h+1})\\
	&\le  \errlogk\xa +\frac{\|\vup\khpl - \vst_{k;h+1}|_{2,\pst\xa}^2}{H^2}\\
	&\le  \errlogk\xa +\frac{\|\vup\khpl - \vlow\khpl\|_{2,\pst\xa}^2}{H^2}
\end{align*}
 and in the last step, we use that under confidence-admissibility, $ \vlow\khpl \le \vst\khpl \le \vup\khpl $. Similarly, we may verify that
\begin{align*}
\bellmanlu\kh\xa  \le \errlogk\xa +\frac{\|\vup\khpl - \vlow\khpl\|_{2,\pst\xa}^2}{H^2}.
\end{align*}
Finally, one can check that from the VI condition $\bellman\kh\xa$ and $\bellmanlu\kh\xa  \le H$, so that recalling $\errlogbark\xa := \min\{H,\errlogk\xa\}$, we obtain that
\begin{align*}
\bellmanup\kh\xa \vee \bellmanlu\kh\xa \le  \errlogbark\xa  + \frac{\|\vup\khpl - \vlow\khpl\|_{2,\pst\xa}^2}{H^2}.
\end{align*}
\end{proof}
Next, using the upper bound derive above, we argue that that we can bound $\vup\khpl(x) - \vlow\khpl(x)$ in terms of an expectation of the error terms $\errlogbar$ over a function fictitious trajectory. This will serve us in developing on the Bellman errors from Lemma~\ref{lem:refined_lucb_bellman_bound}\vspace{0.1in}
\begin{lemma}[Gap between upper and lower value estimates]\label{lem:refined_vup_vlow_gap_bellman} \hspace{0.1em} When  Q-supervisors $\tuplek$ are admissible and bonuses are valid, the gap between upper and lower value functions can be bounded as follows:
	\begin{align*}
	\vup\kh(x) - \vlow\kh(x) &\le \errlogbark\xa + e\Exp^{\pigreedk}\left[\sum_{h'=h+1}^H \errlogbark(x_{h'},a_{h'}) \mid x_h = x,a_h = a\right].
	\end{align*} 
\end{lemma}
\begin{proof}
We begin with a simple claim:
	\vspace{.1in}
	\begin{claim}[Bound on $\|\vup\khpl - \vlow\khpl\|_{2,\pst\xa}^2$ for admissible values]\label{claim:admissable_value_2p_bound} \hspace{0.1em} If $\tuplek$ is admissible, it holds that $\frac{1}{H^2}\|\vup\khpl - \vlow\khpl\|_{2,\pst\xa}^2 \le \frac{1}{H}p\xa^\top (\vup\khpl - \vlow\khpl)$.
	\end{claim}
	\begin{proof}
	Due to the value iteration condition, $\vup\khpl \le H$, $\vlow\khpl \ge 0$. Moreover, if admissibility holds, we have $\vup\khpl\ofx \geq \vst\ofx \geq \vlow\khpl\ofx$ for all $h$ and $x$. 
	Then $(\vup\khpl(x') - \vlow\khpl(x'))^2 \le H|\vup\khpl(x') - \vlow\khpl(x')| = H(\vup\khpl(x') - \vlow\khpl(x'))$. 
Therefore, 
	\begin{align*}
	\|\vup\khpl - \vlow\khpl\|_{2,\pst\xa}^2  &= \sum_{x'}p(x'\mid x,a)(\vup\khpl(x') - \vlow\khpl(x'))^2  \\
	&\le H \sum_{x'}p(x'\mid x,a) (\vup\khpl(x') - \vlow\khpl(x')) \\
	&=  Hp\xa^\top (\vup\khpl - \vlow\khpl).
	\end{align*}
\end{proof}
	Hence, we find that 
	\begin{align*}
	\max\{0,\bellmanlu\kh\xa\} &\le  \errlogbark\xa  + \frac{1}{H} \pst\xa^\top(\vup\khpl - \vlow\khpl).\\
	&=  \errlogbark\xa  + \frac{1}{H} \Exp^{\pigreedk}\left[\vup\khpl(x_{h+1}) - \vlow\khpl(x_{h+1})\right | x_h = x,a_h = a].
	\end{align*}
	Hence, from a variant of the regret decomposition  in Lemma~\ref{eq:lm:LUCB-Bellman},
	\begin{align*}
	\vup\kh\xa - \vlow\kh\xa &= \bellmanlu\kh\xa + \Exp^{\pigreedk}\left[\sum_{h'=h+1}^H\bellmanlu\khpr(x_{h'},a_{h'}) \mid x_h = x,a_h = a\right]\\
	&\le \errlogbark\xa  + (1+\frac{1}{H})\Exp^{\pigreedk}\left[\sum_{h'=h+1}^H \max\{0,\bellmanlu\khpr(x_{h'},a_{h'}) \}\mid x_h = x,a_h = a\right].
	\end{align*}
	Repeating the argument recursively, we find
	\begin{align*}
	\vup\kh\xa - \vlow\kh\xa &\le \errlogbark\xa + \Exp^{\pigreedk}\left[\sum_{h'=h+1}^H (1+\frac{1}{H})^{h'-h}\errlogbark\xa \mid x_h = x,a_h = a\right]\\
	&\le \errlogbark\xa + e\Exp^{\pigreedk}\left[\sum_{h'=h+1}^H \errlogbark(x_{h'},a_{h'}) \mid x_h = x,a_h = a\right].
	\end{align*}
\end{proof}

Finally, let us substitute the bound from Lemma~\ref{lem:refined_vup_vlow_gap_bellman} into ~\ref{lem:refined_lucb_bellman_bound} to obtain a useful bound on the Bellman errors, which is ammenable to the gap-dependent, clipping analysis:

\begin{lemma}[Clipping-friendly bound on Bellman errors]\label{lem:clipping_friendly_bellman_bound} \hspace{0.1em} Suppose $\tuplek$ is confidence-admissible and the bonuses are valid. Then, 
	\begin{align*}
	\advlucb\kh\xa \vee \adv\kh\xa  \le \errlogbark\xa + \frac{e^2}{H} \Exp^{\pigreedk}\left[{\sum_{\tau=h+1}^H \errlogbar\ktau(x_\tau,a_\tau)^2 \mid (x_h,a_h)=(x,a)}\right].
	\end{align*}
	\end{lemma} 
\begin{proof} 
	Substituting Lemma~\ref{lem:refined_vup_vlow_gap_bellman} into ~\ref{lem:refined_lucb_bellman_bound}, $\max\{\advlucb\kh\xa,\adv\kh\xa\}$ is at most:
	\begin{align*}
	&\errlogbark\xa + \frac{1}{H^2} \Exp_{x' \sim \pst\xa}\left[\left( e\Exp^{\pigreedk}\left[\sum_{\tau=h+1}^H \errlogbar\ktau(x_\tau,a_\tau) \mid x_{h+1}=x'\right] \right)^2\right]\\
	&\quad  \overset{(i)}{\le} \errlogbark\xa + \frac{e^2}{H^2} \Exp_{x' \sim \pst\xa}\left[\Exp^{\pigreedk}\left[\left( \sum_{\tau=h+1}^H \errlogbar\ktau(x_\tau,a_\tau)\right)^2 \mid x_{h+1}=x'\right] \right]\\
	&\quad\overset{(ii)}{\le} \errlogbark\xa + \frac{e^2}{H} \Exp_{x' \sim \pst\xa}\left[\Exp^{\pigreedk}\left[\left( \sum_{\tau=h+1}^H \errlogbar\ktau(x_\tau,a_\tau)^2\right) \mid x_{h+1}=x'\right]\right]\\
	&\quad\overset{(iii)}{=} \errlogbark\xa + \frac{e^2}{H} \Exp^{\pigreedk}\left[\left( \sum_{\tau=h+1}^H \errlogbar\ktau(x_\tau,a_\tau)^2\right) \mid \xhah = \xa\right],
	\end{align*}
	where in inequality $(i)$ we use Jensen's inequality, in $(ii)$ we note that $(\sum_{i=1}^m x_i)^2 \le m\sum_{i=1}^m x_i^2$, and in $(iii)$ we use law of total expectation. This concludes the proof of the lemma.\end{proof}

\xhdr{Clipping the Bellman Error Bound.}
In Lemma~\ref{lem:clipping_friendly_bellman_bound}, we obtained a bound on the maximum of the Bellman errors $\max\{\advlucb\kh\xa,\adv\kh\xa\} $. To apply the clipped decomposition~\ref{lem:clipped_regret}, we need to convert this into a bound on $\max\{\advlucbfullclip\kh\xa,\advfullclip\kh\xa\}$: 
\vspace{.1in}
\begin{lemma}[Bound on the clipped Bellman errors via fictitious trajectories]\label{lem:Bellman_error_bound_clip_final} If the bonuses $\bonusk\xa$ are valid and Q-advisor $\tuplek\xa$ admissible, for all $x,a,h$, it holds:
\begin{align*}
\max\{\advlucbfullclip\kh\xa,\advfullclip\kh\xa\} &\le 2\clip{\frac{\gapcheck_h\xa}{4}}{\errlogbark\xa} \\
&+ \frac{4e^2}{H} \Exp^{\pigreedk}\left[\left( \sum_{\tau=h+1}^H \clip{\gapcliplow }{\errlogbar_{k;\tau}(x_\tau,a_\tau)^2}\right)\right].
\end{align*}
\end{lemma}
\begin{proof}[Proof sketch]
Since $\clip{\epsilon}{x}$ is non-decreasing in $x$,
\begin{align*}
&\max\{\advlucbfullclip\kh\xa,\advfullclip\kh\xa\}  \\&\quad= \clip{\gapcheck_h\xa}{\max\{\advlucb\kh\xa,\adv\kh\xa\} } \\
&\quad\le \clip{\gapcheck_h\xa}{\errlogbark\xa + \frac{e^2}{H} \Exp^{\pigreedk}\left[\left( \sum_{\tau=h+1}^H \errlogbar\ktau(x_\tau,a_\tau)^2\right)\right]}
\end{align*}
We then distribute the clipping operations across the summands  above by the next lemma:
\begin{claim}[Lemma B.5 in \cite{SimchowitzJamieson19}]\label{claim:distribution_clipping}\hspace{0.1em} Let $m \ge 2$, $a_1,\dots,a_m \ge 0$. Then 
\begin{align*}
\clip{\epsilon}{\sum_{i=1}^m a_i} \le 2 \sum_{i=1}^m  \clip{\frac{\epsilon}{2m}}{a_i}
\end{align*}
\end{claim}
The rest of the proof is identical to the proof of Lemma B.6 in \cite{SimchowitzJamieson19} and omitted in the interest of brevity. Here we note that the $\gapclipmin/8SAH$ term that arises in clipping $\errlogbark(x_\tau,a_\tau)^2$ comes from expressing the expectation as a weighted sum with at most $SAH$ terms. We use the crude lower bound $\gapmin/(8H^2)$ on $\gapcheck_h\xa$ because it involves an expectation over future state/action pairs other than $\xa$, and we wish to avoid the complication of, for example, clipping $\errlogbark(x',a')^2$ at $\gapcliph\xa$ for $(x',a') \ne \xa$. 
\end{proof}

\xhdr{Concluding the proof of Theorem~\ref{thm:clipped_cat_framework}.}
With Lemma~\ref{lem:Bellman_error_bound_clip_final} and Lemma~\ref{lem:clipped_regret}, we can upper bound the per-episode regret via terms $\errlogclip$ and $\errlogsqclip$.
\vspace{0.1in}
\begin{lemma}[Clipped localized regret bound in terms of fictitious trajectories]\label{lem:clipped_localized_fictious_regret} Recall the clipped error bounds from 
\eqref{eq:clipped_bounds} and \eqref{eq:sq_clipped_bounds}. Then, if the Q-adversors $\tuplek$ is confidence-admissible and the bonuses $\bonusk\xa$ are valid,
\begin{align*}
    \vst - \valf^{\boldpi_k} &\le 4\sum_{h=1}^H \Exp^{\boldpik}\left[\errlogclip\kh(x,a)\right] + 40 e\sum_{h=1}^H\Exp^{\boldpi_k\oplus_{h}\pigreed_k}\left[\sum_{\tau= h}^{H}\errlogclip\ktau\xtauatau\right]\\
    &\quad + 81 e^3\sum_{h=1}^H\Exp^{\boldpi_k\oplus_{h}\pigreed_k}\left[\sum_{\tau= h}^{H}\errlogsqclip\ktau\xtauatau\right].
    \end{align*}
\end{lemma}
Using the $\rho$-bounded visitation ratios, and the fact that the terms $\errlogclip\kh$ and $\errlogsqclipk$ are non-negative:
\begin{align*}
\sum_{h=1}^H\Exp^{\boldpi_k\oplus_{h}\pigreed_k}\left[{\sum_{\tau= h}^{H}\errlogclip\kh\xtauatau}\right] &\leq \rho \sum_{h=1}^H\Exp^{\boldpi_k}\left[\sum_{\tau= h}^{H}\errlogclip\ktau\xtauatau\right]\\ &\leq \rho H  \Exp^{\boldpi_k}\left[\sum_{h=1}^H  \errlogclip\kh(x_h,a_h) \right]; \\
\sum_{h=1}^H\Exp^{\boldpi_k\oplus_{h}\pigreed_k}\left[{\sum_{\tau= h}^{H}\errlogsqclip\ktau\xtauatau}\right] &\leq \rho \sum_{h=1}^H\Exp^{\boldpi_k}\left[\sum_{\tau= h}^{H}\errlogsqclip\ktau\xtauatau\right] \\&\leq  \rho H \Exp^{\boldpi_k}  \left[{\sum_{h= 1}^{H}\errlogsqclipk(x_h,a_h)}\right].
\end{align*}
Combining the above two inequalities with Lemma~\ref{lem:clipped_localized_fictious_regret}, and suming over episodes $k$, we conclude the proof of Theorem~\ref{thm:clipped_cat_framework} 

\newpage

\section{Omitted proofs for unified analytical framework}
\label{app:analysis_framework}
\subsection{Per-episode Bellman-error regret decomposition (Proposition~\ref{prop:lucb_regret})}
\label{app:bellman_error_decomposition_supervised}
Before providing the proof for the Bellman error regret decomposition for Q-supervised policies, we present the classical analysis for UCB policies which we extend.

\xhdr{Bellman-error regret decomposition for UCB policies.} The Bellman errors allow to bound the value of the UCB policy of a Q-supervisor; Lemma~\ref{lm:LUCB-Bellman} typically constitutes the first step in the analysis of UCB episodic RL algorithms (e.g., Lemma E.15 in \cite{dann2017unifying} for tabular MDPs).
\vspace{0.1in}
\begin{lemma}\label{lm:LUCB-Bellman}
\hspace{0.1em} For episode $k$, if Q-supervisor $\tuplek=(\qup_k,\qlow_k,\activeset_k,\pigreed_k)$ is confidence-admissible then:
\begin{align}\label{eq:lm:LUCB-Bellman}
\vst -  V^{\pigreed_k} \le \Exp^{\pigreed_k}\left[\sum_{h=1}^H\bellmanup\kh\xhah\right].
\end{align}
\end{lemma}
\begin{proof}
Rearranging the definition of Bellman error (Definition~\ref{defn:Bellman_error}), it holds that:
\begin{align}
\qup\kh\xa &=\left(r^{\star}\xa + \bellmanup\kh\xa\right)+ p^{\star}(x,a)^{\top} \vup\khpl
&\forall x\in\states, a\in\actions, \nonumber \\
\vup\kh\ofx &= \qup\kh(x, \pigreed_k(x))
&\forall x\in\states.
\end{align}
This defines a Bellman equation of a new MDP with the same  transition probabilities and under the same policy $\pigreed_k$ as the one used by $V^{\pigreed_k}$; the difference is that the reward at a state-action pair $\xa$ is $r^{\star}\xa+\bellmanup\kh\xa$. Hence, $\vup_k = \vup_{k;1}(x_1)= \Exp^{\pigreed_k}[\sum_{h=1}^H r^{\star}\xa + \bellmanup\kh(x_h,a_h)]$.
Since $\tuplek$ is confidence-admissible, it holds that $\vupk\geq \vst$. Since  $V^{\pigreed_k}= \Exp^{\pigreed_k}[\sum_{h=1}^H r^{\star}\xa]$, the difference $\vst-V^{\pi}$ is upper bounded by the lemma's right hand side.
\end{proof}

\xhdr{Bellman error-regret decomposition for Q-supervised policies.} We now extend the above decomposition to Q-supervised policies that may differ from the UCB policy. Proposition~\ref{prop:lucb_regret} differs from decompositions for UCB policies in the expectations $\Exp^{\boldpi_k\oplus_{h}\pigreed_k}$ over \emph{fictitious trajectories} arising from following $\boldpi_k$ until a stage $h$, and switching to $\pigreed_k$.

\begin{proof}[Proof of Proposition~\ref{prop:lucb_regret}]
If the policy $\boldpi_k$ coincides with $\pigreed_k$ (i.e. greedily optimizes w.r.t. $\qup$) then the proposition holds without the second term as in Lemma~\ref{lm:LUCB-Bellman}. To address the mismatch between $\boldpi_k$ and $\pigreed_k$, we relate the performance of policy $\boldpi_k$ to that of a fictitious policy that follows $\boldpi_k$ until some pre-defined step and then switches to $\pigreed_k$. The second term in the proposition essentially captures the expected confidence intervals of future cumulative rewards that such a policy would have had, balancing the aforementioned mismatch. The proof then  follows similar arguments as the ones of optimistic policies by working recursively across steps and is formalized below. 

Recall that $\Exp^{\boldpi} = \Exp^{\calM,\boldpi}$ denotes expectation over a random realization $(x_1,a_1),\dots,(x_H,a_H)$ under $\calM,\boldpi$. To take advantage of the per-stage values, we assume without loss of generality that $\boldpi$ is a deterministic Markovian policy denoted as $\pi$. Randomized Markovian policies $\boldpi$ can be handled by taking expectation of the deterministic atoms $\pi$ which constitute it. For simplicity, we abuse notation by letting $\pist_h(x)$ denote an arbitrary element in the set of optimal actions at $x$ and $h$.
    
Since the Q-supervisor $\tuple_k$ is confidence-admissible, it holds that:
\begin{align*}
\vst_1(x_1)= \qst\kone(x_1, \pist_1(x_1)) \leq  \qup_{k;1} (x_1, \pist_1(x_1)) \leq \max_{a\in\activeset_{k;1}(x_1)}\qup_{k;1}(x_1,a )= \vup_{k;1}(x_1).
    	\end{align*}
    	Taking expectation over $x_1$, the LHS is at most the following quantity (to be bounded recursively):
    	\begin{align*}
    \vst-\valf^{\pi_k}=\Exp^{\pik}\left[\vst_1(x_{1})- \valf_1^{\pi_k}(x_{1})\right]\leq \Exp^{\pik}\left[\vup_{k;1}(x_1)-\valf_1^{\pi_k}(x_{1})\right].
    	\end{align*}
As in Lemma~\ref{lm:LUCB-Bellman}, when $\pik$ coincides with the greedy policy  $\pi_k = \pigreed_k$, $\vup\kh$ is the value function of $\pik$ on an MDP with transition $p^{\star}\xa$ and reward $r_h\xa = r^{\star}\xa+ \adv\kh\xa$. Since $V^{\pik}$ is the value function of the same policy $\pi_k$ under an MDP with the same transition $p^{\star}$ but with the original reward $r^{\star}$, then $\vup_{k;1}(x_1) - V^{\pi_k}_1(x_1) = \Exp^{\pi_k} \left[\sum_{h=1}^H \adv\kh(x_h,a_h) \right]$ for any $x_1\in\states$. 

Extending this to Q-supervised policies $\pik$ which may deviate from $\pigreedk$, let $a_1 = \pi\kone(x_1)$. Since $\pi_k$ is Q-supervised by $\tuple_k$, it holds that $\qup_{k;1}(x_1, a_1)  \geq \max_{a'\in\activeset_{k;1}(x)} \qlow_{k;1}(x,a') = \vlow_{k;1}(x_1)$. Combining with the definition of $\bellmanup_{k;1}(x_1,a_1)$ and the fact that $\valf_1^{\pi_k}(x_1) = Q_1^{\pi_k}(x_1,a_1)=r^{\star}(x_1,a_1) + p^{\star}(x_1,a_1)^{\top}\valf^{\pi_k}_{2}$, we obtain:
\begin{align}
\Exp^{\pik}\left[\vup_{k;1} (x_1) - \valf^{\pi_k}_1(x_1)\right] &\leq \Exp^{\pik}\left[\qup_{k;1}(x_1, a_1) - \vlow_{k;1}(x_1)+ \vup_{k;1} (x_1)- \valf^{\pik}_1(x_1)\label{eq:decomp_reduction_step}\right]\\
    	&\leq \Exp^{\pik}\left[\bellmanup_{k;1}(x_1,a_1)+p^{\star}(x_1,a_1)^{\top}\left(\vup_{k;2}-\valf_2^{\pi_k}\right) +  \vup_{k;1} (x_1)- \vlow_{k;1}(x_1)\right]	\nonumber \\
    	&= \Exp^{\pi_k}\left[\bellmanup_{k;1}(x_1,a_1) + \vup_{k;2}(x_2) - \valf_2^{\pi_k}(x_2)  +  \vup_{k;1} (x_1)- \vlow_{k;1}(x_1)\right]\nonumber
    	\end{align}
We now focus on the last term of \eqref{eq:decomp_reduction_step}. For any $h\in[H]$, by the definitions of $\vup\kh$ and $\vlow\kh$ and the definition of Bellman error, letting $a_h=\pigreed_k(x_h)$, it holds that:   	
    	\begin{align*}
    	    &\vup\kh(x_h) - \vlow\kh(x_h) \leq \qup\kh(x_h, \pigreedk(x_h)) - \qlow\kh(x_h, \pigreedk(x_h)) \\
    	    &\quad =\rhat_k(x_h,a_h) + \phat_k(x_h,a_h)^{\top} \vup_{k;h+1} - \rhat_k(x_h,a_h) - \phat_{k}(x_h,a_h)^{\top} \vlow_{k;h+1} + 2\bonus\kh(x_h,a_h) \\
    	    &\quad = (\phat_k(x_h,a_h) - p^\star(x_h,a_h))^{\top} \left( \vup_{k;h+1} - \vlow_{k;h+1} \right) + 2\bonus\kh(x_h,a_h) + p^\star(x_h,a_h)\left(\vup_{k;h+1} - \vlow_{k;h+1}\right) \\
    	    &\quad = \bellmanup_{k;h}(x_{h},a_{h})-\bellmanlow_{k;h}(x_{h},a_{h}) + p^\star(x_h,a_h)\left(\vup_{k;h+1} - \vlow_{k;h+1}\right),
    	\end{align*} where the last equality follows from the definition of Bellman error via the derivations below:
    	\begin{align*}
    	    &\bellmanup_{k;h}(x_{h},a_{h})-\bellmanlow_{k;h}(x_{h},a_{h}) \\
    	    & = \qup\kh(x_h,a_h) - r^\star(x_h,a_h) -p^\star(x_h,a_h) \vup_{h+1} - \qlow\kh(x_h,a_h) - r^\star(x_h,a_h) + p^\star(x_h,a_h) \vlow_{h+1} \\
    	    & = 2\bonus\kh(x_h,a_h) + (\phat_k(x_h,a_h) - p^\star(x_h,a_h))^{\top}\left(\vup_{k;h+1} - \vlow_{k;h+1}\right).
    	\end{align*} Recursively applying the derivations to $p^\star(x_h,a_h)(\vup_{k;h+1} - \vlow_{k;h+1})$, the final term in \eqref{eq:decomp_reduction_step} is:
    	\begin{align*}
    	   \vup\kh(x_h) - \vlow\kh(x_h) \leq  \Exp^{\pigreed_k}\left[\sum_{\tau=h}^H \bellmanup_{k;\tau}(x_{\tau},a_{\tau})-\bellmanlow_{k;\tau}(x_{\tau},a_{\tau})\right],\forall h\in [H]
    	\end{align*}
    Recursively applying the above argument for $\Exp^{\pik}\left[\vup_{k;2}(x_{k;2})-\valf^{\pi_k}_{k;2}(x_{k;2})\right]$ bounds the second term in \eqref{eq:decomp_reduction_step} and completes the proposition.
	\end{proof}

\subsection{Confidence admissibility of $Q$-supervisors (Lemma~\ref{lem:valid_implies_admissibility})}\label{app:valid_implies_admissibility}

\begin{proof}[Proof of Lemma~\ref{lem:valid_implies_admissibility}]
We apply double induction, an outer on base learners $\ell=\lmax,\ldots,\lst$ and an inner on steps $h=H,\ldots,1$. Recall the two-estimate Value Iteration in \textsc{UCBVI-BASE} (Algorithm~\ref{alg:UCBVI_base}, lines 6-12).

The induction hypothesis for the outer induction is that $\activeset_{k,\ell;h}(x)$ contains all optimal actions for every state $x$ and step $h$; this is the first condition of confidence-admissibility for $\tuple_{k,\ell}$. The base of the induction is satisfied as $\activeset_{k,\lmax}(x)=\actions$ (Algorithm~\ref{alg:corruption_robust_rl})

The induction hypothesis for the inner induction is that $\vup_{k,\ell;h+1}(x)\geq \vst_{h+1}(x)\geq \vlow_{k,\ell;h+1}(x)$ for all states~$x$. The base of induction is satisfied as $\vup_{k,\ell;H+1}(x)=\vlow_{k,\ell;H+1}(x)=\vst_{H+1}(x)=0$. 

To prove the induction step for the inner induction, we start from the inductive hypothesis $\vup_{k,\ell;h+1}\geq\vst_{h+1}$ and show that 
$\qup\klh-\qst_h=\min\big{\{}H, \rhat_{k,gl}+\phat_{k,gl}^{\top}\vup\klhpl +  \bonus_{k,gl,\ell},  \rhat_{k,sb,\ell}+\phat_{k,sb,\ell}^{\top}\vup\klhpl + \bonus_{k,sb,\ell}\big{\}}-\qst_h$ is always nonnegative. $H-\qst_h\geq 0$ as all rewards are in $[0,1]$. Since $\mathfrak{m}_{k,gl,\ell}$ is valid with respect to $\tuple_{k,\ell}$, it holds that $\rhat_{k,gl}+\phat_{k,gl}^{\top}\vup\klhpl-\qup_k\geq 0$. If condition (b) of Definition~\ref{defn:valid} holds, the above follows directly by the definition. If condition (a) holds, if follows by the inductive hypothesis and the fact that $\phat_{k,gl}\geq 0$. Similarly, since $\mathfrak{m}_{k,sb,\ell}$ is valid with respect to $\tuple_{k,\ell}$, it also holds  that $\rhat_{k,sb}+\phat_{k,sb}^{\top}\vup\klhpl-\qup_k\geq 0$.
Hence, for all $x\in\states$, $\qup\klh(x,\pigreed_{k,\ell})\geq \qup\klh(x,\pigreed_{k,\ell})$ which proves the inductive step $\vup\klh(x)\geq \vst\klh(x)$. The proof of $\vlow\klh(x)\leq \vst\klh(x)$ is essentially the same using the policy maximizing $\qlow\klh$ instead of $\pigreedk$.

We next prove the induction step for the outer induction. By the induction hypothesis, $\activeset_{k,\ell;h}$ contains all optimal actions. We now show that this continues to hold for base learner $\ell-1$ as well. We have already shown that, for all states-actions-steps $(x,a,h)$ and base learners $\ell\geq \lst$, it holds that $\qlow\klh(x,a)\leq \qst_h(x,a)\leq \qup\klh(x,a)$. Since $\pist_h(x)\in\activeset\klh(x)$, for any $a'\in\activeset\klh\ofx$:
\begin{align*}
    \qlow\klh(x,a')\leq \qst_h(x,a')\leq \qst_h(x,\pist_h(x))\leq \qup\klh(x,\pist_h(x)),
\end{align*}
which implies that $\pist_h(x)\in\activeset_{k,\ell-1;h}$ the way $\activeset_{k,\ell-1;h}$ is created (Algorithm~\ref{alg:UCBVI_base}). As a result, both criteria of confidence-admissibility are satisfied for $\tuple_{k,\ell}$ which concludes the proof.
\end{proof}

\subsection{Robust base learners' subsampled model deviations{ (Lemma~\ref{lem:local_subsampled_corruptions}})}\label{app:proof_subsampling}

\begin{proof}[Proof of Lemma~\ref{lem:local_subsampled_corruptions}]
    Recall that $\ck$ is the indicator that episode $k\in[K]$ is corrupted. Let $Z_{k,\ell}$ be the indicator of whether episode $k$ is charged to base learner $\ell$. Then the number of corrupted datapoints is \begin{align}\label{eq:bound_total_corruption}|\datacorrupt\klsb|\leq H\cdot\sum_{k=1}^K Z_{k,\ell}\ck.
        \end{align}

        Next, let $\{\mathcal{F}_{k}\}_{k \ge 1}$ denote the filtration generated by the set of observed trajectories up to and including episode $k$, the adversary's decision to corrupt, $c_{k+1}$, and the learner's \emph{selection} of a randomized policy (i.e. a distribution over deterministic policies), but \emph{not} the draws of the learner's random coins. Since the adversary is oblivious within episodes -- that is, the adversary chooses whether to corrupt before seeing the learner's random coins -- this is indeed a filtration.  We therefore have that $c_{k}$ is $\calF_{k-1}$-measurable, but $Z_{k,\ell}$ is $\calF_{k}$-measurable, and 
        \begin{align*}
        \Exp[(Z_{k,\ell}\ck)^2 \mid \calF_{k-1}] = \ck\Exp[Z_{k,\ell}^2 \mid \calF_{k-1}] \leq \ck\Exp[Z_{k,\ell} \mid \calF_{k-1}] \le \ck 2^{-\ell}
        \end{align*}
        where the last step follows from the bound on the probability of charging the episode $k$ to base learner $\ell$ which is equal to the probability that the episode schedule ends at base learner $\ell$ (Claim~\ref{claim:sched_ub}). 
        
       Applying the Azuma-Bernstein  inequality (Lemma~\ref{lem:azuma_bernstein}) with  $X_{k,\ell} = Z_{k,\ell}\ck$, $b = 1$, and $\sigma^2 = \frac{1}{K}\sum_{k=1}^K \Exp[(Z_{k,\ell}\ck)^2 \mid \calF_{k-1}] \le \frac{1}{K}\sum_{k=1}^k \ck 2^{-\ell} = 2^{-\ell}C/K$, for any fixed $\ell$,  with probability $1-\delta$:
        \begin{align*}
        \left\lvert \frac{1}{K}\left( \sum_{k=1}^K X_{k,\ell} - \Exp[X_{k,\ell} | \calF_{k-1}] \right) \right\rvert \leq \sqrt{\frac{2 \cdot (2^{-\ell} C)\ln(2/\delta)}{K^2}} + \frac{2 \ln(2/\delta)}{3K}.
        \end{align*}
     Using the fact that $\Exp[X_{k,\ell} | \calF_{k-1}] \leq \didcorrupt_k 2^{-\ell}$ and thus $\sum_{k=1}^K \Exp[X_{k,\ell}|\calF_{k-1}] \leq \sum_{k=1}^K \didcorrupt_k 2^{-\ell} = C2^{-\ell}$, we can conclude that:
        \begin{align*}
         \sum_{k=1}^K X_{k,\ell} \leq \sqrt{2 \cdot (2^{-\ell} C)\ln(2/\delta)} + \frac{2 \ln(2/\delta)}{3} + 2^{-\ell}C.
        \end{align*}
        In particularly, union bounding over all learners $\ell$ by invoking the above inequality with $\delta \leftarrow \delta/8\ell^2$ for learner $\ell$, we have that with probability $1 - \delta/4$, the following holds uniformly for all $\ell \in \lmax$, 
        \begin{align*}
        \sum_{k=1}^K X_{k,\ell} \le \sqrt{2 \cdot (2^{-\ell} C)\ln(16\ell^2/\delta)} + \frac{2 \ln(16\ell^2/\delta)}{3} + 2^{-\ell}C.
        \end{align*}
        In particular for all $\ell \ge \lst$, $2^{-\ell} C \le 1$, and thus
        \begin{align}\label{eq:high_prob_bound_corruption}
        \sum_{k=1}^K Z_{k,\ell}\ck \leq (\sqrt{2\ln(16\ell^2/\delta)} + \frac{2}{3}\ln(16\ell^2/\delta) + 1 \le 2\ln(16\ell^2/\delta).
        \end{align}
    Combining Eq.~\eqref{eq:bound_total_corruption}~and~\eqref{eq:high_prob_bound_corruption} concludes the lemma.\end{proof}

\newpage

\section{Omitted proofs on tabular MDPs}
\label{app:aux_tabular}

\subsection{Concentration for idealized model estimates (Lemmas~\ref{lem:global_idealized_stochastic_tabular}~and~\ref{lem:local_concentration})} \label{app:idealized_stociastic_tabular}

\begin{proof}[Proof of Lemma~\ref{lem:global_idealized_stochastic_tabular}]
Let $\bolddel_{x}$ is be indicator vector on $x \in \states$. We begin with the first inequality for which it holds:
    \begin{align*}
    &\rhatstoch\kgl\xa - \rst \xa + (\phatstoch\kgl - \pst(x,a))^\top \Vsthpl \\
    &\quad=  \frac{1}{N\kgl}\sum_{(j,h) \in \calS_{k,gl}} (\rstoch_{j,h} - r(x\jh,a\jh)) + (\bolddel_{\xstoch_{j,h+1}} - \pst(x\jh,a\jh) )^\top\Vsthpl
    \end{align*}
    Note that $\rstoch_{j,h},r(x\jh,a\jh)) \in [0,1]$,  $\bolddel_{\xstoch_{j,h+1}} - \pst(x\jh,a\jh)$ lie on the simplex, and $\Vsthpl \in [0,H-1]^{\calX}$ for $h \ge 1$. 
As a result, each summand is at most $H$. 

Moreover, since  $\Exp[\rstoch_{j,h}  \mid \calF_{j;h}] = r(x\jh,a\jh)$ and $\xstoch_{j,h+1} \sim \pst(x\jh,a\jh) \mid \calF_{j;h}$, the summands form a martingale difference sequence with respect to an appropriate filtration (see e.g. Section B.4 in \cite{AzarOsMu17} for formal constructions of said filtration). By the anytime Azuma-Hoeffiding inequality (Lemma~\ref{lem:azuma_hoeffding_anytime}), with probability $1 - \delta$, for any fixed $\xa$ and for all $k \in [K]$:
    \begin{align*}
    |\rhatstoch\kgl\xa - \rst \xa + (\phatstoch\kgl - \pst(x,a))^\top \Vsthpl| &\le H\sqrt{\frac{2\ln(4 N\kgl\xa^2/\delta)}{N\kgl\xa}} \\ &\le H\sqrt{\frac{2 \ln (4 T^2/\delta)}{N\kgl\xa}}.
    \end{align*}
    Union bounding over all $\xa$, over $h \in [H]$ for each $\Vsthpl$, and shrinking $\delta$ by a factor of $1/16SA$ provides the guarantee of the present lemma, with failure probability at most $\delta/16$. 

    Let us now turn to the second inequality, where we will instead consider $V \in [-H,H]^{\states}$. While a more refined argument is developed in Section~\ref{app:stoch_uniform_conc_refined_proof}, we use a simple argument based on a uniform covering. To begin, the same argument applied in the previous part of the lemma shows that for any \emph{fixed} $V \in [-H,H]^{\states}$ and \emph{fixed} $\xa$ pair, we have that with probability $1-\delta$.
    \begin{align*}
    \forall k \in [K], \quad |(\phatstoch\kgl - \pst(x,a))^\top V| \le 2H\sqrt{\frac{2 \ln (4 T^2/\delta)}{N\kgl\xa}}.
    \end{align*}
    Next, let $\calN$ denote an $H/2$-net of $[-H,H]^{\states}$ in the $\ell_{\infty}$ norm; we may take $|\calN| \le 4^{S}$. Then, with probability $1-\delta$, we have
    \begin{align*}
    \max_{V' \in \calN} |(\phatstoch\kgl - \pst(x,a))^\top V' | \le 2H\sqrt{\frac{2 \ln (4 T^2/\delta) + 2S \ln(4)}{N\kgl\xa}} \le 4H\sqrt{\frac{S \ln (4 T^2/\delta)}{N\kgl\xa}} 
    \end{align*}
    Moreover, since for all $V \in [-H,H]^{\calX}$ and $V' \in \calN$, we have that $V - V' \in \frac{1}{2}[-H,H]^{\calX}$, we have 
    \begin{align*}
    \max_{V \in [-H,H]^{\calX}} |(\phatstoch\kgl - \pst(x,a))^\top V| \le 2\max_{V' \in \calN}|(\phatstoch\kgl - \pst(x,a))^\top V| \le 8H\sqrt{\frac{S \ln (4 T^2/\delta)}{N\kgl\xa}}.
    \end{align*}
    Union bounding over all $\xa$ and shrinking $\delta$ by a factor of $1/(16SA)$ provides the bound in the lemma with failure probability at most $\delta/16$. Together the failure probability is at most $1 -\delta/8$.
 \end{proof}

\begin{proof}[Proof of Lemma~\ref{lem:local_concentration}] The proof is identical to Lemma~\ref{lem:global_idealized_stochastic_tabular}, with the exception that we union bound over base learners $\ell \in [\lmax]$. Since $\lmax = \lceil{\log T}\rceil \le T$, we inflate the logarithm by at most a factor of $T$.
\end{proof}

\subsection{Auxiliary details on bounding confidence sums (Theorem~\ref{thm:tabular_initial})}\label{app:confidence_integration}
We now provide a general approach to bound the sum of the confidence terms that arise in Lemma~\ref{lem:confidence_term_tabular}. We reuse and slightly adjust this approach for the more refined bounds in Appendix~\ref{app:log_regret}. To begin, let us consider a generic sequence of policies $\boldpi_1,\dots,\boldpi_K$. The main ingredient is a notion which relates an abstract sampled count $\Nksp\xa$ to visitations under the measures induced by $(\boldpi_k)$. This allows us to establish an upper bound on the expected number of samples $\Nsp$ needed in order for concentration bounds to kick in and enable us to relate the two quantities. 
\vspace{0.1in}
\begin{definition}[Subsampled Count]\label{defn:subsasmple_count} \hspace{0.1em}
Let $\boldpi_1,\dots,\boldpi_K$ denote a sequence of policies. Define counts
\begin{align*}
\wk\xa = \Exp^{\boldpik}\left[\sum_{h=1}^H \I(\xhah = \xa) \right] \qquad \text{ and }\qquad
\nbark\xa = \sum_{j=1}^k \wkh\xa.
\end{align*}
Given $\psample > 0$, $\Nsp \ge 1$, we say that $\Nksp$ is a $(\psample,\Nsp)$-subsampled count with respect to $\boldpi_1,\dots,\boldpi_K$ if for all $\xa$ and all $k$ such that $\psample\nbark\xa \ge \Nsp \ge 1$, 
\begin{align*}
\Nksp\xa \ge \frac{1}{4} \psample\nbark\xa. 
\end{align*}
\end{definition}
The following lemma verifies that local and global counts $N\kgl\xa$ and $N\klsb\xa$ satisfy the above definition with $\Nsp$ linear in $C$ and logarithmic in the other quantities. Recall the mixture probabilities $q_{\ell} := \Pr[f(k,H) = \ell]$ and $q_{\le \lst} := \Pr[f(k,H) \le \lst]$ (Definition~\ref{defn:policy_decomp}).
\vspace{0.1in}
\begin{lemma}\label{lem:good_subsampling} \hspace{0.1em} With probability $1 - \delta/8$, for all $\ell,\lst \in [\lmax]$, it holds simultaneously that the counts $N\klsb\xa$ are a $(q_{\ell},4H \ln\frac{32 eHSA \lmax}{\delta}+ 2CH q_{\ell})$-subsampled count for the policies $(\pimaster\kl)$, and  $N\klsb\xa$ are a $(q_{\le \lst}, 4H\ln\frac{32 eHSA \lmax}{\delta}+ 2CH)$ for the policies $(\pimaster\klelst)$.  We denote this good event by $\eventsample$.
\end{lemma}
\begin{proof} 
		\newcommand{\weightil}{\widetilde{\weight}}
		\newcommand{\Ntil}{\widetilde{N}}
		Similar to the definition of subsampled counts (Definition~\ref{defn:subsasmple_count}), we define an analogue for the observed data that include corrupted data:		\begin{align*}
		\weightil\klelsth\xa := \Pr^{\pimaster_{k,\le \lst}}_{{\calM}_k}((x_h,a_h) = (x,a)), \quad \weightil\klh\xa := \Pr^{\pimaster\kl}_{{\calM}_k}((x_h,a_h) = (x,a)),
		\end{align*}
		 where ${\calM_k}$ is the MDP the learner actually experiences at episode $k$ (i.e., it is chosen by the adversary at episode $k$ if $\ck = 1$ and otherwise is the nominal MDP $\calM$). We also define the aggregates
		 \begin{align*}
		\weightil\klelst\xa &:= \sum_{h=1}^H\weightil\klelsth\xa, \quad \weightil\kl\xa := \sum_{h=1}^H\weightil\klh\xa, \text{ and}\\
		\Ntil\klelst\xa &:= \sum_{h=1}^H\weightil\klelsth\xa, \quad \Ntil\kl\xa := \sum_{h=1}^H\weightil\klh\xa,
		\end{align*}
		 Namely $\widetilde{\weight}\klelst\xa$ is the expected total number of times $(x,a)$ is visited during the whole episode if one executes $\pimaster\klelst$ under the MDP ${\calM}_k$; analogously of $\weightil\kl\xa$ is the corresponding expected number if one executes $\pimaster\kl$ under $\calM_k$. 
		 We define the weights under the nominal MDP:  
		 \begin{align*}
	\weight\klelsth\xa &:= \Pr^{\pimaster_{k,\le \lst}}_{{\calM}}((x_h,a_h) = (x,a)), \quad \weight\klh\xa := \Pr^{\pimaster\kl}_{{\calM}}((x_h,a_h) = (x,a))\\
		\weight\klelst\xa &:= \sum_{h=1}^H\weight\klelsth\xa,\quad
		\weight\kl\xa := \sum_{h=1}^H\weight\klh\xa\\
		\nbar\klelst\xa &:= \sum_{h=1}^H\weight\klelsth\xa,\quad
		\nbar\kl\xa := \sum_{h=1}^H\weight\klh\xa,
		\end{align*}
	Using the definition of $C$, one can bound the  difference between $\widetilde{N}_k\xa$ and $\nbar_k\xa$. For any $k \in [K]$ and $(x,a)\in\states\times\actions$, we have
\begin{align*}
&\lvert \Ntil\kl\xa - \nbar\kl\xa  \rvert  \leq \sum_{i=1}^k \lvert \weight_{i,\ell}\xa - \weightil_{i,\ell}\xa \rvert = \sum_{i=1}^k \I(\didcorrupt_i = 1)  \lvert \weightil_{i,\ell}\xa - \weight_{i,\ell}\xa \rvert \leq CH,
\end{align*} 
since $\max\{\weightil_{i,\ell}\xa, \weight_{i,\ell}\xa\} \leq H$. Similarly, 
\begin{align*}
&\lvert \Ntil\klelst\xa - \nbar\klelst\xa  \rvert  \leq  \sum_{i=1}^k \I(\didcorrupt_i = 1)  \lvert \weightil_{i,\le\lst}\xa - \weight_{i,\le\lst}\xa \rvert \leq CH,
\end{align*}
Focusing on $\weightil\kl\xa$ and $\widetilde{N}\kl\xa$, we can verify that $\Exp_k\left[ N_{k+1,\ell;sb}\xa \right] = q_{\ell}\weightil\kl\xa + N\klsb\xa$, where $\Exp_k$ is the conditional expectation conditioned on all events up to and including episode $k$, and where $q_{\ell} := \Pr[f(k,H) = \ell]$ for any given episode $k$. From Lemma E.1 in \cite{dann2017unifying}, we have that with probability at least $1-\delta/2$, we have:
\begin{align*}
\{\forall k, x,a: \nklsb\xa \geq \frac{q_{\ell}}{2} \Ntil_{k-1}\xa - H\ln\frac{2HSA}{\delta}\}.
\end{align*} This means that that the following event holds with probability at least $1-\delta/2$:
\begin{align*}
\{\forall (x,a,k), \nklsb\xa \geq \frac{q_{\ell}}{2} \nbar_{k-1,\ell} - H \ln\frac{2HSA}{\delta} - \frac{ q_{\ell} CH}{2}\}
\end{align*} 
Use the same argument from Lemma B.7 in \cite{SimchowitzJamieson19}, we note that $\nbar_{k,\ell;sb}\xa \leq \nbar_{k-1,\ell;sb}\xa + H$. The above event implies that with probability $1 - \delta$, we have
\begin{align*}
N_{k,\ell;sb}\xa \xa &\geq \frac{q_{\ell}}{2}\nbar\kl\xa - H \ln\frac{2HSA}{\delta} - \frac{q_{\ell} CH}{2} - \frac{q_{\ell}H}{2}  \\
&\geq \frac{q_{\ell}}{2}\nbar\kl\xa - H\ln\frac{2eHSA}{\delta} - \frac{q_{\ell} CH }{2},
\end{align*} 
where we absorb the $q_{\ell}H/2$ by placing an $e$ into the logarithim above. Hence when 
$$q_{\ell}\nbar\kl\xa \ge 4H\ln\frac{2eHSA}{\delta} + 2 CH q_{\ell},$$
it holds that
$N\klsb\xa \ge \frac{q_{\ell}\nbar\kl\xa}{4}$. 
Similarly, using that $N\kgl$ is greater than the cumulative counts from episodees $\{f(k,h) \le \lst\}$ we can show that with probability $1 - \delta$, for any fixed $\lst$ and $q_{\le \lst} := \Pr[f(k,H) \le \lst]$, it holds
\begin{align*}
q_{\le \lst}\nbar\kl\xa \ge 4H\ln\frac{2eHSA}{\delta} + 2 CH , 
\end{align*}
and therefore $N\kgl\xa  \ge \frac{\nbar\klelst\xa}{4}$.

Rescaling $\delta \leftarrow \delta/(16\lmax)$ and union bounding over $\ell$ (and possible values of $\lmax$), we find that with probability $1 - \delta/8$, $N\klsb\xa$ is an $(q_{\ell},4H \ln\frac{32 eHSA \lmax}{\delta} + 2CH q_{\ell})$-subsampled count for the policies $(\pimaster\kl)$ and $N\klsb\xa$ is an $(q_{\le \lst},\Nsample, 4H \ln\frac{32 eHSA \lmax}{\delta}+ 2CH)$-subsampled count for the policies $(\pimaster\klelst)$. This
concludes the bound.
\end{proof}
Finally, we bound confidence sums with subsampled counts.

\begin{lemma}[Integration]\label{lem:simple-integration} \hspace{0.1em} Consider a sequence of policies $\boldpi_1,\dots,\boldpi_K$, and let $\Nksp$ be $(\psample,\Nsp)$-subsampled counts with respect to this sequence. Using $\lesssim$ to hide constants in the bounds, it holds:
\begin{align*}
&\psample\sum_{k=1}^K\Exp^{\boldpik}\left[\sum_{h=1}^H \min\left\{3H, \sqrt{\frac{\alpha_{1 }}{\Nksp\xhah}} + \frac{\alpha_{2}}{\Nksp\xhah} \right\} \right] \\
&\quad\lesssim  SA (H\Nsp+H^2) + \left(\sqrt{SA  \,\psample \alpha_1 \,T} + \alpha_2 \ln T\right),
\end{align*}
\end{lemma}
\begin{proof}
We define the episode when statistics of state-action pair $\xa$ have sufficiently converged:
\begin{align*}
k_0\xa :=  \min\{K+1 , \inf\{k: \psample\nbark\xa \ge \Nsp\}\}
\end{align*}
Since $\Nksp$ is $(\psample,\Nsp)$-subsampled (Definition~\ref{defn:subsasmple_count}), after this episode, it holds that:
\begin{align*}
\Nksp\xa \ge \frac{\psample}{4} \cdot \nbark\xa, \quad \forall k \ge k_0\xa
\end{align*}

We now bound the main term of the LHS in the lemma:
\begin{align*}
&\sum_{k=1}^K\Exp^{\boldpik}\left[\sum_{h=1}^H \min\left\{3H, \sqrt{\frac{\alpha_{1 }}{\Nksp\xhah}} + \frac{\alpha_{2}}{\Nksp\xhah} \right\} \right]\\
&=\sum_{x,a}\sum_{k=1}^K \wk\xa \min\left\{3H, \sqrt{\frac{\alpha_{1 }}{\Nksp\xa}} + \frac{\alpha_{2}}{\Nksp\xa} \right\} \\
&\le \sum_{x,a}\sum_{k=1}^{k_0-1}\wk\xa \cdot 3H + \sum_{x,a}\sum_{k = k_0\xa}^K \wk\xa \min\left\{3H, \sqrt{\frac{\alpha_{1 }}{\Nksp\xa}} + \frac{\alpha_{2}}{\Nksp\xa} \right\} \\
\end{align*}
Using this, we can bound the first term (for the lemma's LHS including $\psample$) as:
\begin{align*}
\psample\sum_{x,a}\sum_{k=1}^{k_0-1}\wk\xa\cdot 3H = 3\psample H \sum_{x,a}\nbar_{k_0}\xa \le 3SA H\Nsp
\end{align*}
Regaarding the second term (for the lemma's LHS including $\psample$), we  let:
\begin{align*}
f(u) = \min\left\{3H, \sqrt{\frac{\alpha_{1 }}{u}} + \frac{\alpha_{2}}{u} \right\}.
\end{align*}
Then, for $k \ge k_0$, since $\Nksp\xa \ge \nbark\xa/4$, we have that
\begin{align*}
\min\left\{3H, \sqrt{\frac{\alpha_{1 }}{\Nksp\xa}} + \frac{\alpha_{2}}{\Nksp\xa} \right\} \le f(\psample \nbark\xa /4)
\end{align*}
From a modification of Lemma C.1 in \cite{SimchowitzJamieson19} and the fact that $\nbar_{k_0} \xa \ge \Nsp \ge 1$, we have (assuming $\nbarK\xa \ge 1$, for otherwise the following term is zero):
\begin{align*}
\sum_{k = k_0\xa}^K \wk\xa  f(\psample \nbark\xa/4) &\le 3H f(\psample/4) + \int_{1}^{\nbarK\xa} f(\psample u/4)\rmd u \\
&\le 12H^2  + \int_{1}^{\nbarK\xa} \sqrt{\frac{4\alpha_1}{\psample u}} + \frac{4\alpha_2}{\psample u}\rmd u   \\
&= 12H^2 + \frac{1}{\psample}\left(4\sqrt{\alpha_1 \psample \nbarK\xa} + 4\alpha_2\ln (\nbarK\xa) \right)\\
&\le 12H^2 + \frac{1}{\psample}\left(4\sqrt{\alpha_1 \psample \nbarK\xa} + 4\alpha_2 \ln T\right),
\end{align*}
where in the last line we use $\nbarK\xa \le KH = T$. Thus, by Cauchy-Schwartz,
\begin{align*}
\sum_{x,a}\sum_{k = k_0\xa}^K \wk\xa  f(\Nksp\xa) &\le 12SA H^2 + \frac{4}{\psample}\left(\sqrt{SA \psample \alpha_1 \sum_{x,a} \nbarK\xa} + \alpha_2 \ln T\right)\\
&= 12SA H^2 + \frac{4}{\psample}\left(\sqrt{SA  \psample \alpha_1 T} + \alpha_2 \ln T\right),
\end{align*}
where we use that $\sum_{x,a}\nbarK\xa = KH = T$. Hence, putting things together,  we have
\begin{align*}
&\psample \sum_{k=1}^K\Exp^{\boldpik}\left[\sum_{h=1}^H \min\left\{3H, \sqrt{\frac{\alpha_{1 }}{\Nksp\xhah}} + \frac{\alpha_{2}}{\Nksp\xhah} \right\} \right] \\
&\quad\le SA (3H\Nsp+12\psample H^2) + 4\left(\sqrt{SA  \,\psample \alpha_1 \,T} + \alpha_2 \ln T\right).
\end{align*}
\end{proof}

 \subsection{Logarithmic regret for $\supervisedtab$ (Theorem~\ref{thm:logT_catvi}\label{ssec:prop:logT_catvi})} 
\begin{proof}[Proof of Theorem~\ref{thm:logT_catvi}]
As in the proof of Lemma~\ref{lem:simple-integration}, we define 
\begin{align*}
k_0\xa :=  K+1 \wedge \inf\{k: \nbark\xa \ge \psample\Nsp\}
\end{align*}
which again implies:
\begin{align*}
\Nksp\xa \ge \frac{\psample}{4} \cdot \nbark\xa, \quad \forall k \ge k_0\xa
\end{align*}
Starting Theorem~\ref{thm:clipped_cat_framework}, we have
\begin{align}
 \sum_{k=1}^K \vst - V^{\boldpik} &\lesssim H \rho \sum_{k=1}^K\Exp^{\boldpi\supi_k}\left[\sum_{h=1}^H  \errlogclip\kh(x_h,a_h)  + \errlogsqclip\kh(x_h,a_h)\right],\nonumber\\
 &\lesssim H \rho \sum_{k=1}^K \sum_{x,a}  \wk\xa \left(\errlogclip\kh(\xa)  + \errlogsqclip\kh(\xa)\right)\nonumber\\
 &= H \rho \sum_{x,a} \sum_{k=1}^{k_0\xa-1} \wk\xa + H \rho   \sum_{x,a} \sum_{k = k_0\xa}^K \wk\xa \left(\errlogclip\kh(\xa)  + \errlogsqclip\kh(\xa)\right)\nonumber\\
 &\lesssim \rho H \left(\frac{SA H^2  \Nsp}{\psample} + H \rho   \sum_{k = k_0\xa}^K \wk\xa \left(\errlogclip\kh(\xa)  + \errlogsqclip\kh(\xa)\right)\right),\label{eq:log_clip_reg_bound}
\end{align}
where in the last line, we use that $\left(\errlogclip\kh(\xa)  + \errlogsqclip\kh(\xa)\right) \le H^2$ by definition, and the identity $\sum_{k > k_0\xa} \wk\xa = \nbar_{k_0}\xa \le \Nsp/\psample$. Let us now bound the sum corresponding to $k \ge k_0\xa$.

\paragraph{Bounding the terms $k \ge k_0\xa$} Recall the definition
\begin{align*}
\gapcheck_h\xa = \max\left\{\frac{\gapmin}{32H^2},\,\frac{\gaph\xa}{32H}\right\}
\end{align*}
define $\gapcheck\xa := \min_h \gapcheck_h\xa/4$, and introduce the functions
\begin{align*}
f_{1}(u) &:= \clip{\gapcheck\xa}{\sqrt{\frac{4\constc_{1}} {u \,\psample}}}\\
f_{2}(u) &:= \frac{ \constc_{2} }{u \,\psample}\\
f_{3}(u) &:= \frac{\constc_{1}}{u \,\psample}\\
f_{4}(u) &:= \min\left\{H^2, \frac{\constc_{2}^2 }{(u \,\psample)^2} \right\}\\
\end{align*}
Recall that we assume $\errlog\kh\supi\xa$ is upper bounded as:
\begin{align*}
\errlog\kh\xa \le \sqrt{\frac{\constc_{1}}{\Nksp\xa}} + \frac{\constc_{2}}{\Nksp\xa}.
\end{align*}

Using Claim~\ref{claim:distribution_clipping}, we have:
\begin{align*}
&\errlogclip\kh\xa =  \clip{\gapcheckh\xa}{\min\{H, \errlog\kh\xa\}} \\&\le \clip{\gapcheckh\xa}{\errlog\kh\xa}\\
& \lesssim  \clip{\gapcheckh\xa/4}{\sqrt{ \frac{\alpha_{1}}{\Nksp\xa}} }  + \frac{\alpha_{2}}{\Nksp\xa}\\
& \le  \clip{\gapcheck\xa}{\sqrt{ \frac{\alpha_{1}}{\Nksp\xa}} }  + \frac{\alpha_{2}}{\Nksp\xa}\\
& \lesssim \clip{\gapcheck\xa}{ \sqrt{ \frac{4\alpha_{1}}{\psample\nbark\xa}} }  + \frac{4\alpha_{2}}{\psample\nbark\xa}\\
&\lesssim f_{1}(\nbark\xa)\} + f_{2}(\nbark\xa)
\end{align*}
Similarly, for $\errlogsqclip\kh\xa$, using Claim~\ref{claim:distribution_clipping} and the inequality that $(a+b)^2 \lesssim a^2 + b^2$, we have:
\begin{align*}
&\errlogsqclip\kh\xa \\
&= \clip{\gapcliplow}{\left( \min\{H, \errlog\kh\xa\} \right)^2}\\
&\le \left( \min\{H, \errlog\kh\xa\} \right)^2\\
&= \min\left\{ H^2, (\errlog\kh\xa)^2 \right\} \\
& \le \min\left\{  H^2,  2\frac{\alpha_{1}}{\Nksp\xa} + 2\left(\frac{\alpha_{2}}{\Nksp\xa}\right)^2 \right\} \\
& \le \min\left\{  H^2,  2\frac{\alpha_{1}}{\psample \cdot \nbark\xa} + 2\left(\frac{4\alpha_{2}}{\psample \cdot \nbark\xa}\right)^2 \right\} \\
&\lesssim f_{3}(\nbark\xa)\} + \min\{H^2,f_{4}(\nbark\xa)\}
\end{align*}
Moreover, using the fact that $\errlogclip\kh\xa \le H$, and $\errlogsqclip\kh\xa \le H^2$, we can bound
\begin{align}
&\sum_{x,a}\sum_{k = k_0\xa}^K \wk\xa \left(\errlogclip\kh(\xa)  + \errlogsqclip\kh(\xa)\right)\le \sum_{x,a}\sum_{i=1}^4\sum_{k = k_0\xa}^K \wk\xa  \min\{H^2,f_{i}(\nbark\xa)\}. \label{eq:errlog_to_f}
\end{align}
Using the fact $\nbar_{k_0\xa} \xa \ge \Nsp\xa \ge 1$,   a modification of Lemma C.1 in \cite{SimchowitzJamieson19} yields
\begin{align*}
\sum_{k = k_0\xa}^K \wk\xa  \max\{H^2,f_{i}(\nbark\xa) \le H^3 + \int_{1}^{\nbarK\xa} f_i(u)\rmd u.
\end{align*}
Let us evaluate each of these integrals in succession. We have 
\begin{align*}
\int_{1}^{\nbarK\xa} f_1(u)\rmd u &= \int_{1}^{\nbark\xa}  \min\left\{H,\clip{\gapcheck\xa}{\sqrt{\frac{4\constc_{1}} {u \,\psample}}}\right\} \rmd u \\
&= \frac{1}{\psample} \int_{\psample}^{\psample \nbark\xa}  \min\left\{H,\clip{\gapcheck\xa}{\sqrt{\frac{4\constc_{1}}{s}}} \right\}\rmd s\\
&\lesssim \frac{1}{\psample} \min\left\{\sqrt{\constc_{1}\psample \nbarK\xa}, \frac{\constc_1}{\gapcheck\xa} \right\},
\end{align*}
where the last step uses Lemma~\ref{lem:clipped_integral}. Moreover, 
\begin{align*}
\int_{1}^{\nbark\xa} (f_2(u)+f_3(u))\rmd u &= \int_{1}^{\nbarK\xa}  \frac{\constc_1 + \constc_2}{\psample \cdot u} du \le \ln(\nbark\xa) \le \frac{1}{\psample} \cdot \left(\constc_1 + \constc_2\right)\ln(T),
\end{align*}
where $\nbark\xa \le KH = T$. Next, with $a\vee b = \max\{a,b\}$,
 \begin{align*}
\int_{1}^{\nbark\xa} f_4(u) \rmd u &= \int_{1}^{\nbarK\xa}\min\left\{H^2, \frac{\constc_{2}^2 }{(u \,\psample)^2}\right\} \rmd u \\
&\le \frac{1}{\psample}  \int_{\psample}^{\psample\nbarK\xa}\min\left\{H^2, \frac{\constc_{2}^2 }{s^2}\right\}\rmd s\\
&\le \frac{1}{\psample} \left(H^2 + \int_{1}^{\constc_2 \vee 1} H^2 ds + \int_{\constc_2^2 \vee 1}^{\psample\nbarK\xa \vee 1 } \frac{\constc_2^2 }{s^2}\right) \rmd s \\
&\le \frac{1}{\psample} \left(H^2 + H^2 \constc_2 + \int_{\constc^2 \vee 1}^{\psample\nbarK\xa \vee 1 } \frac{\constc_2 }{s} \right)\rmd s \\
&\le \frac{1}{\psample} \left(H^2 + H^2 \constc_2 + \int_{\constc^2 \vee 1}^{\psample\nbarK\xa \vee 1 } \frac{\constc_{2} }{s} \right) \rmd s\\
&\le \frac{1}{\psample} \left(H^2 + \constc_2 (H^2+\ln T) \right) \rmd s\\
\end{align*}
Hence,
\begin{align*}
\sum_{x,a}\sum_{i=1}^4\sum_{k = k_0\xa}^K \wk\xa  f_{i}(\nbark\xa) &\lesssim \frac{SA}{\psample} \left(\constc_1 + H^2 + \constc_2( H^2 + \ln(T)) \right) \\
&\qquad+ \frac{1}{\psample} \sum_{x,a} \min\left\{\sqrt{\constc_{1}\psample \nbarK\xa}, \frac{\constc_1}{\gapcheck\xa} \right\},
\end{align*}
Thus, combining with \eqref{eq:log_clip_reg_bound} to \eqref{eq:errlog_to_f}, we have
\begin{align}
\psample\left(\sum_{k=1}^K \vst - V^{\boldpik}\right) &\lesssim SA H \rho  \cdot (H^2 \psample  \Nsp + \constc_1 + \constc_2  + H^2 + \constc_2 (H^2+\ln T) )\nonumber\\
&\quad+ \rho \sum_{x,a} \min\left\{\sqrt{\constc_{1}\psample \nbarK\xa},\, \frac{\constc_1}{\gapcheck\xa} \right\}. \label{eq:log_integral_second_to_last}
\end{align}
Using $\constc_2 \ge 1$,  and $\Nsp \ge 1$, we can simplify
\begin{align}
&SA H\rho \ln(T) \cdot (H^2 \Nsp + \constc_1 + \constc_2 + H^2 + \constc_2 (H^2 + \ln T)) \nonumber\\
&\quad\lesssim SA \rho  \cdot (H^3\Nsp + H\constc_1 \ln (T) + \constc_2  (H^3 + H\ln T) ). \label{eq:simplification_Log}
\end{align}
Finally, we develop two bounds on the term on the final line. First, a gap-free bound of 
\begin{align*}
 \sum_{x,a} \sqrt{\constc_{1}\psample \nbark\xa} \le  \sqrt{\constc_1 SA \psample  \sum_{x,a}\constc_{1}\psample \nbarK\xa} = \sqrt{\constc_1 \psample\, SA T   }.
\end{align*}
Second, the gap dependent bound. Recall $\gapcheckh \xa := \max\left\{\frac{\gapmin}{32H^2},\,\frac{\gaph\xa}{32H}\right\}$ and $\gapcheckh = \min_{h}\gapcheck\xa/4$. Moreover, recall $\optacts := \{\xa: \gaph\xa = 0 \text{ for some } h\}$, and $\subacts = \states \times \actions - \optacts$, and $\gap\xa = \min_h \gaph\xa$. Then, we have for $(x,a) \in \optacts$, we have
\begin{align*}
\gapcheck\xa \gtrsim \begin{cases} \gap\xa/H & \xa \in \subacts\\
\gapmin\xa/H^2 & \xa \in \optacts 
\end{cases}
\end{align*}
Hence, recalling 
\begin{align*}
\gapcomplexity := \sum_{\xa \in \subacts} \frac{H}{\gap\xa} + \frac{H^2|\optacts|}{\gapmin},
\end{align*} 
we have
\begin{align*}
 \sum_{x,a} \frac{\constc_1}{\gapcheck\xa}  \le \constc_1 \gapcomplexity. 
\end{align*}
Taking the better of the above two bounds,
\begin{align*}
 \rho \sum_{x,a} \min\left\{\sqrt{\constc_{1}\psample \nbarK\xa},\, \frac{\constc_1}{\gapcheck\xa} \right\} \le \rho\min\left\{ \sqrt{\constc_1 \psample\, SA T   }, \constc_1 \gapcomplexity\right\}.
\end{align*}
Combining with \eqref{eq:log_integral_second_to_last} and \eqref{eq:simplification_Log}, we conclude
\begin{align*}
\psample\left(\sum_{k=1}^K \vst - V^{\boldpik}\right) &\lesssim SA \rho  \cdot (H^3\Nsp + H\constc_1 \ln (T) + \constc_2  (H^3 + H\ln T) ) \\
&\quad H\rho \min\left\{ \sqrt{\constc_1 \psample\, SA T   }, \constc_1 \gapcomplexity\right\},
\end{align*}
as needed.\end{proof}

\subsection{Refined concentration bounds (Lemma~\ref{lem:uniform_conc_refined})}\label{app:stoch_uniform_conc_refined_proof}

To prove the lemma, we use an auxiliary lemma stated below.
\vspace{0.1in}
 \begin{lemma}\label{lem:uncorrupted_multinomial_concentration} \hspace{0.1em} Let $\{x_i\}_{i =1}^{\infty} $ be a sequence of states in $\states$ adapted to a martingale $\calF_i$, such that $x_i \mid \calF_{i-1}$ is drawn from a multinomial distribution with parameter $p\in\Delta(S)$ for all $i \ge 1$. Denoting by $\delta_x$ the indicator vector of state $x$, let $\phat_n := \frac{1}{n}\sum_{i=1}^n \bolddel_{x_i}$. Then, with probability $1-\delta$, the following holds for all vectors $V \in \R^{\states}$, all $u > 0$, and all $n\in\mathbb{N}^+$  simultaneously:
    \begin{align*}
    \left|\left(\phat_n - p\right)^\top V\right| \le \frac{\|V\|_{2,p}^2}{2u} + \frac{S(3u+2\|V\|_{\infty}) \ln(4Sn^2/\delta)}{3n},
    \end{align*}
    where $\|V\|_{2,p}^2 = \sum_{x}p_xV(x)^2$, and $\|V\|_{\infty} = \max_x|V(x)|$.
    \end{lemma}
    \begin{proof}
        Consider a fixed $n\in\mathbb{N}^+$.  By decomposing the LHS, we obtain:
    \begin{align*}
    n\left|\left(\phat_n - p\right)^\top V\right|= \left|\sum_{i=1}^n \left(\bolddel_{x_i} - p\right)^\top V\right| &=\|\sum_{i=1}^n\sum_x V(x) \left({\I_{x_i=x}-p_x}\right)\|\le \sum_{x \in \states} V(x) \left|\sum_{i=1}^n \I_{x_i = x} - p_{x}\right|.
    \end{align*}
    Since the variance of a random variable $\I_{x_i=x}-p_x$ is upper bounded by $p_x(1-p_x)\leq p_x$, for all $x\in \states$, the sequence $\{\I_{x_i=x}-p_x\}_{i\geq 1}$ satisfies the conditions of Azuma-Bernstein
    (Lemma~\ref{lem:azuma_bernstein}) with $(b,\sigma^2) = (1,p_x)$. Thus, with probability $1-\delta_n$ with $\delta_n:=\frac{\delta}{2n^2}$, for any \emph{fixed} state $x\in\states$, it holds that
    \begin{align*}
    \left|\sum_{i=1}^n \I_{x_i = x} - p_{x}\right| \le \sqrt{\frac{2 p_x\ln(2/\delta_n)}{n}} + \frac{2 \ln(2/\delta_n)}{3n}.
    \end{align*}
    Union bounding over all states $x$, with probability $1-\delta_n$, the following holds for all $V \in \R^\states$:
    \begin{align*}
     n\left|\left(\phat_n - p\right)^\top V\right| &\le \sum_{x \in \states} V(x) \left(\sqrt{\frac{2 p_x\ln(2S/\delta_n)}{n}} + \frac{2 \ln(2S/\delta_n)}{3N}\right)\\
    &\le \left(\sum_{x \in \states} V(x) \sqrt{\frac{2 p_x\ln(2S/\delta_n)}{n}}\right) +  \frac{2 S\|V\|_{\infty} \ln(2S/\delta_n)}{3n}\\
    &\le  \sqrt{\frac{2 S (\sum_x p_x V(x)^2)\ln(2S/\delta_n)}{n}} +  \frac{2 S\|V\|_{\infty} \ln(2S/\delta_n)}{3n}\\
    &:=  \sqrt{\frac{2 S \|V\|_{2,p}^2 \ln(2S/\delta_n)}{n}} +  \frac{2 S\|V\|_{\infty} \ln(2S/\delta_n)}{3n}.
    \end{align*}
    Finally, using the inequality $ab \le \frac{a^2}{2u} + \frac{u}{2}b^2$, we have that for all $V \in \R^\states$ and all $u > 0$ simultaneously,
    \begin{align*}
    &\sqrt{\frac{2 S \|V\|_{2,p}^2 \ln(2S/\delta_n)}{n}} +  \frac{2 S\|V\|_{\infty} \ln(2S/\delta_n)}{3n} \\
    &\le \frac{\|V\|_{2,p}^2}{2u} + \frac{Su\ln(2S/\delta_n) }{n} + \frac{2 S\|V\|_{\infty} \ln(2S/\delta_n)}{3n}\\
    &= \frac{\|V\|_{2,p}^2}{2u} + \frac{S(3u+2\|V\|_{\infty}) \ln(2S/\delta_n)}{3n}, 
    \end{align*}
    The final result follows by union bounding over $n \ge 1$ and the fact that $\sum_{n\ge 1} \delta/(2n^2)\leq 2$. \end{proof}
    \begin{proof}[Proof of Lemma~\ref{lem:uniform_conc_refined}]
    In particular, by representing $(\phatstoch\kgl\xa - \pst\xa)^\top V$ as a multi-nomial martingale sequence as in the above lemma, setting $u = H^2/2$ we obtain that for any fixed $\xa$ and $\delta$, and $V:\|V\|_{\infty} \le H$,  we obtain that with probability $1-\delta$, we have for all rounds $k$ that
    \begin{align*}
    \left|\left(\phatstoch\kgl - p\right)^\top V\right| &\le \frac{\|V\|_{2,p}^2}{H^2} + \frac{S(3H^2/2+2H) \ln(4SN\kgl^2/\delta)}{3N\kgl\xa},\\
    &\le \frac{\|V\|_{2,p}^2}{H^2} + \frac{2SH^2 \ln(4ST^2/\delta)}{N\kgl\xa}.
    \end{align*}
    Thus, with probability at least $\delta / 16$, we have for all $x,a,k$ simultaenously that for all $V:\|V\|_{\infty} \le H$,
    \begin{align*}
    \left|\left(\phatstoch\kgl - p\right)^\top V\right| 
    &\le \frac{\|V\|_{2,p}^2}{H^2} + \frac{2SH^2 \ln(64S^2 A T^2/\delta)}{N\kgl\xa}.
    \end{align*}
    Similarly, with probability at least $\delta/16$, we have for all $x,a,k$ simultaenously that for all $V:\|V\|_{\infty} \le H$,
     \begin{align*}
    \left|\left(\phatstoch\klsb - p\right)^\top V\right| 
    &\le \frac{\|V\|_{2,p}^2}{H^2} + \frac{2SH^2 \ln(64S^2 A T^3/\delta)}{N\klsb\xa},
    \end{align*}
    where we pick up an extra factor of $T$ inside the log from a union bound over $\ell \in [\lmax] \subset [T]$. Altogether, the total failure probability is at most $\delta/8$.
    
Recall Lemma~\ref{lem:global_corruption} and \ref{lem:local_subsampled_corruptions} state the upper bounds on the number of corruptions for base learners. 
    To bound $\left\lvert (\phat\kgl - p^\star)^{\top} V \right\rvert$ and $\left\lvert (\phat\klsb - p^\star)^{\top} V \right\rvert$, we use the same derivation that we had in the proof of Lemma~\ref{lem:global_concentration} and Lemma~\ref{lem:local_concentration_tabular}.    \end{proof}

\newcommand{\IAh}{\I_{\Ah}}
\newcommand{\IAi}{\I_{\calA_i}}

\subsection{Clipped Bellman error decomposition (Lemma~\ref{lem:surplus_log})}
\label{ssec:proof_of_reg_decompose_clip}
We consider a \emph{deterministic} policy $\pi$ which is $\tuplek$-supervised. Again, the result extends to randomized policies $\boldpi$ by taking an expectation over its deterministic atoms.  

In order to lighten notation, we introduce the following shorthand for the conditioning operator. Given functions $f_1,\dots,f_H:= \states \times \actions \to \R$, a stage $h \in [H]$, and a pair $(x_h,a_h) \in \states \times \actions$, we let
\begin{align*}
\Exp^{\pi}\left[ \sum_{h'=h}^H f_{h'}(x_{h'},a_{h'}) \mid x_h,a_h \right]
\end{align*}
To denote expectation over a sequence $(\widetilde{x}_{h'},\widetilde{a}_{h'})$ drawn from the law induced by $\pi$ and the nominal MDP $\calM$, starting at $\widetilde{x}_h,\widetilde{a}_h = (x_h,a_h)$ in stage $h$. Similarly, 
\begin{align*}
\Exp^{\pi}\left[ \sum_{h'=h}^H f_{h'}(x_{h'},a_{h'}) \mid x_h \right]
\end{align*}
denotes the analogous expectation, beginning at $\widetilde{x}_{h} = x_h$, and $\widetilde{a}_h = \pi(x_h)$.

\paragraph{Half Clipping}
Following~\cite{SimchowitzJamieson19}, we begin an argument called ``half-clipping'',  which allows us to clip even actions $a \in \pisth(x)$ corresponding which are optimal at $(x,h)$:
\vspace{0.1in}
\begin{lemma}[Half-Clipping]\label{lemma:half_clipping} \hspace{0.1em}
Let $\pi$ denote any  deterministic policy which is Q-supervised by the Q-supervisor $\tuple_k = \{\qup_k, \qlow_k, \activeset_k, \pigreed_k\}$. Then, 
\begin{align*}
\vst - \valf^{\pi} \le 2\sum_{h=1}^H \Exp^{\pi}\left[\advhalfclip\kh(x_{h},a_{h})\right] + 2\sum_{\hbar=1}^H \Exp^{\pi\oplus_{\hbar}\pigreed_k}\left[\sum_{h = \hbar}^{H}\advlucbhalfclip\kh(x_h,a_h)\right],
\end{align*}
where we define $\advlucbhalfclip\kh(x,a) := \clip{\frac{\gapmin}{8H^2}}{\advlucb\kh(x,a)}$ and $\advhalfclip\kh(x,a) := \clip{\frac{\gapmin}{8H}}{\adv\kh(x,a)}$.
\end{lemma}
Here, the single $\dot{\matE}$ denotes a ``half-clipping'', where we clip Bellman errors only by the minimal non-zero gap $\gapmin$. The ``full-clipping'' allows us to clip with $\gap_h\xa$ at suboptimal $a \notin \pisth$, which may, in general, be much larger. The argument that follows is perhaps less intuitive than the ``full-clipping'' argument. We therefore suggest that the reader may want to skim this section on a first read, and focus on understanding the arguments in the subsequent section instead.

\newcommand{\Ah}{\calA_h}

Before proving the above lemma, we first present the following claim. Fix an episode $k$, and define the events measurable with respect to a trajectory $(x_1,a_1),\dots,(x_H,a_H)$: 
\begin{align*}
\calE_h := \{a_h \notin \pist(x_h)\},\quad \calE_h^{c}:= \{a_h \in \pist(x_h)\},\quad \text{and}\quad\Ah := \calE_h \cap \bigcap_{h'=1}^{h-1}\calE^{c}_{h'}\,.
\end{align*}
 In words, $\Ah$ is the event, on a given trajectory, the algorithm has played optimally at stages $h' = 1,\dots,h-1$, but suboptimally at stage $h$. We begin with the following claim:
\begin{claim} \hspace{0.1em} For any policy $\pi$, it holds that 
\begin{align*}
\vst - \valf^{\pi} = \Exp^{\pi}\left[\sum_{h=1}^H\IAh\left(\gap_h(x_h,a_h) + \qst_h(x_h,a_h) - \qf^{\pi}(x_h,a_h) \right)\right]\,.
\end{align*}
\end{claim}
\begin{proof} 
Let us start from the first time step below. For notation simplicity, we denote $\{x_h,a_h\}_h$ as a realization of a trajectory from $\pi$. 
\begin{align*}
&\vst - \valf^{\pi} = \Exp^{\pi}\left[ \I\left(\calE_1\right) (\qst_1(x_1, \pist(x_1)) -\qf^{\pi}_1(x_1,a_1))  + \I\left( \calE^c_1 \right)(\qst_1(x_1, \pist(x_1)) -\qf^{\pi}_1(x_1,a_1))  \right]\\
& = \Exp^{\pi}\left[ \I\left(\calE_1\right) (\gap_1(x_1,a_1) + \qst_1(x_1,a_1) -\qf^{\pi}_1(x_1,a_1))  + \I\left( \calE^c_1 \right)(p(x_1,a_1)^{\top} (\vst_1  - \valf^{\pi}_1)  \right].
\end{align*} We can repeat the above process for $\Exp^{\pi}\left[ \vst_1(x_1) - \valf^{\pi}_1(x_1) \right]$ to get:
\begin{align*}
&\Exp^{\pi}\left[ \vst_1(x_1) - \valf^{\pi}_1(x_1) \right] \\
& = \Exp^{\pi}\left[ \I\left( \calE_2 \right) (\gap_2(x_2,a_2)  + \qst_2(x_2,a_2) - \qf^{\pi}_2(x_2,a_2) ) + \I\left(  \calE^c_2 \right) (\vst_3(x_3) - \valf^{\pi}_3(x_3)) \right] .
\end{align*}
Combining the above and repeating the same procedure till $H$, the proof follows.
\end{proof}

Now we prove Lemma~\ref{lemma:half_clipping} below.
\begin{proof}[Proof of Lemma~\ref{lemma:half_clipping}]
Note that on $\Ah$, we have that $a_h \notin \pist(x_h)$. Hence for $$\gapmin := \min_{x,a,h}\{\gaph\xa > 0\},$$ we have $\gaph{(x_h,a_h)} \le 2(\gaph{(x_h,a_h)} - \frac{\gapmin}{2})$. Since $\qst_h(x_h,a_h) - \qf^{\pi}(x_h,a_h) \geq 0$ by optimality of $\qst$, we can bound
\begin{align*}
\vst - \valf^{\pi} &\le 2\Exp^{\pi}\left[\sum_{h=1}^H\IAh\left(\gap_h(x_h,a_h) + \qst_h(x_h,a_h) - \qf_h^{\pi}(x_h,a_h) - \frac{\gapmin}{2} \right)\right]
\end{align*}
Moreover, we have  $\qst_h(x_h,a_h) + \gaph(x_h,a_h) = \vst_h(x_h) \le  \vup\kh(x_h)$, which leads to 
\begin{align*}
\vst - \valf^{\pi} &\le 2\Exp^{\pi}\left[\sum_{h=1}^H\IAh\left(\vup\kh(x_h) - \qf^{\pi}_h(x_h,a_h) - \frac{\gapmin}{2} \right)\right]\\
&= 2\Exp^{\pi}\left[\sum_{h=1}^H\IAh\left((\vup\kh(x_h) - \bar{\qf}^{\pi}_h(x_h,a_h)) + (\bar{\qf}^{\pi}_h(x_h,a_h) - \qf^{\pi}_h(x_h,a_h)) - \frac{\gapmin}{2} \right)\right],
\end{align*} where we define $\qup^{\pi}_h\xa$ as the state-action value function of policy $\pi$ under the MDP with transitions $p(x,a)$ and \emph{non-stationary} rewards $r_h(x,a) + \adv\kh(x,a)$.  Note that $\vup\kh$ can be understood as the value function of $\pigreed_k$ on the same MDP. Hence, by the now-standard Performance Difference Lemma (see Lemma 5.2.1 in \cite{kakade2003sample}):
\begin{align*}
\vup\kh(x_h) - \bar{\valf}^{\pi}_h(x_h) = \sum_{h'\ge h} \Exp^{\pi}\left[  \vup\khpr(x_{h'}) - \qup\khpr(x_{h'},a_{h'})  \vert x_h\right],
\end{align*} where $a_h'$ is from $\pi$.

By definition of $\qup^{\pi}_h$ and $\qf_h^{\pi}$, we have:
\begin{align*}
\qup^{\pi}_h(x_h,a_h) - \qf^{\pi}_h(x_h,a_h) = \sum_{h'\ge h} \Exp^{\pi}\left[ \adv\kh(x_{h'},a_{h'}) \vert x_h,a_h \right]
\end{align*}

Combine the above two results together, we have:
\begin{align*}
&(\vup\kh(x_h) - \qup^{\pi}_h(x_h,a_h)) + (\qup^{\pi}_h(x_h,a_h) - \qf_h^{\pi}(x_h,a_h)) \\
&\quad= \sum_{h' \ge h}\Exp^{\pi}\left[(\vup\khpr(x_{h'}) - \qup\khpr(x_h',a_{h'}) + \adv\kh(x_{h'},a_{h'}) \mid x_h,a_h\right],
\end{align*} where $a_h$ is from $\pi(x_h)$.

Hence,
\begin{align*}
\vst - \valf^{\pi} &\le 2\Exp^{\pi}\left[\sum_{h=1}^H\IAh \left(\vup\kh(x_h) - \qf^{\pi}(x_h,a_h) - \frac{\gapmin}{2} \right)\right]\\
&= 2\Exp^{\pi}\left[\sum_{h=1}^H \IAh\left( \sum_{h'\ge h}\vup\khpr(x_{h'}) - \qup\khpr(x_{h'},a_{h'}) + \adv\kh(x_{h'},a_{h'}) - \frac{\gapmin}{2} \right)\right]\\
& \leq 2\Exp^{\pi}\left[\sum_{h=1}^H \IAh\left( \sum_{h'\ge h}\left(\vup\khpr(x_{h'}) - \qup\khpr(x_{h'},a_{h'}) + \adv\kh(x_{h'},a_{h'}) - \frac{\gapmin}{2H} \right)\right)\right]\\
&\le 2\Exp^{\pi}\left[\sum_{h'=1}^H\left(\sum_{h \le h'}\IAh\right)\left( \vup\khpr(x_{h'}) - \qup\khpr(x_{h'},a_{h'}) + \adv\kh(x_{h'},a_{h'}) - \frac{\gapmin}{2H} \right)\right].
\end{align*}

Since $\pi$ is a Q-supervised policy (i.e., $\qup\kh(x_h,a_h) \geq \vlow\kh(x_h)$), we can further bound the above by (relabeling $h' \leftarrow h$ and $h \leftarrow i$)
\begin{align*}
&\vst - \valf^{\pi} \\
&\le 2\Exp^{\pi}\left[\sum_{h=1}^H\left(\sum_{i \le h}\IAi\right)\left( \vup\kh(x_{h}) - \vlow\kh(x_{h}) + \adv\kh(x_{h},a_{h}) - \frac{\gapmin}{2H} \right)\right]\\
&\le 2\Exp^{\pi}\left[\sum_{h=1}^H\left(\sum_{i \le h}\IAi\right) \left(\Exp^{\pigreed_k}\left[\sum_{h' = h}^H \advlucb\khpr(x'_{h'},a'_{h'}) \mid x'_{h} = x_h\right]  + \adv\kh(x_{h},a_{h}) - \frac{\gapmin}{2H}\right) \right],
\end{align*}
where the inner sequence is taken from an expectation under $\pigreed_k$, and where $\advlucb\kh = \adv\kh - \advlow\kh$. Note that $\{\Ah\}_h$ are mutually exclusive, hence $\sum_{i\leq h} \IAi \leq 1$.
Using the fact that $\clip{\epsilon}{x} \ge x - \epsilon$, with some algebra we can bound 
\begin{align*}
&\Exp^{\pigreed_k}\left[\sum_{h' = h}^H \advlucb\khpr(x'_{h'},a'_{h'}) \mid x'_{h} = x_h\right]  + \adv\kh(x_{h},a_{h}) - \frac{\gapmin}{2H}  \\
&\le \Exp^{\pigreed}\left[\sum_{h' = h}^H \clip{\frac{\gapmin}{4H^2}}{\advlucb\khpr(x'_{h'},a'_{h'})} \mid x'_{h} = x_h\right]  + \clip{\frac{\gapmin}{4H}}{\adv\kh(x_{h},a_{h})}\\
&\le \Exp^{\pigreed}\left[\sum_{h' = h}^H \advlucbhalfclip\khpr(x'_{h'},a'_{h'})\mid x'_{h} = x_h\right]  + \advhalfclip\kh(x_{h},a_{h}),
\end{align*}
where we recall $\advlucbhalfclip\khpr(x,a) := \clip{\frac{\gapmin}{8H^2}}{\advlucb\khpr(x,a)}$ and $\advhalfclip\khpr(x,a) := \clip{\frac{\gapmin}{8H}}{\adv\kh(x,a)}$,

Now we can bound
\begin{align*}
\vst - \valf^{\pi} &\le2\sum_{h=1}^H\Exp^{\pi}\left[\Exp^{\pigreed}\left[\sum_{h' = h}^H \advlucbhalfclip\khpr(x'_{h'},a'_{h'})\mid x'_{h} = x_h\right]  + \advhalfclip\kh(x_{h},a_{h})\right] \\
&= 2\sum_{h=1}^H \Exp^{\pi}\left[\advhalfclip\kh(x_{h},a_{h})\right] + 2\sum_{\hbar=1}^H \sum_{h = \hbar}^{H}\Exp^{\pi\oplus_{\hbar}\pigreed}[\advlucbhalfclip\kh(x_h,a_h)].
\end{align*}
\end{proof}

\xhdr{Full Clipping} Now we proceed to obtain the result for fully clipped Bellman errors. Recall that the definitions of clipped Bellman errors from Eq.~\eqref{eq:def_clipped_bellman}:
\begin{align*}
\advfullclip\kh\xa :=\clip{ \gapcheck_h\xa  }{ \adv\kh\xa }, \quad \advlucbfullclip\kh\xa :=\clip{ \gapcheck_h\xa  }{ \advlucb\kh\xa },
\end{align*} 
 where we recall again $\gapcheck_h\xa = \max\left\{\frac{\gapmin}{8H^2},\,\frac{\gaph\xa}{8H}\right\}$.
We begin by observing the following fact for optimal actions
\begin{align}
\advlucbhalfclip\kh\xa\le \advlucbfullclip\kh\xa, \quad \advhalfclip\kh\xa\le \advfullclip\kh\xa, \quad \forall x,h, a \in \pisth(x),  \label{eq:optimal_action_clipping_bound}
\end{align} as $\gap_h\xa = 0$ for $a\in \pisth\ofx$ by definition.

Hence, to replace the half-clipped terms with the full clipped terms, our arguments will focus on suboptimal actions $a \notin \pisth(x)$, that is, $a$ for which $\gaph\xa > 0$. We shall rely heavily on the following claim, which states that if a Q-supervised policy $\pi$ selects an action $a$ with $\gaph\xa > 0$, then either an associated Bellman errors $\adv\kh\xa$ or $\advlucb\kh\xa$ is large, or other related quantities in terms of $\vup$ and $\vlow$ are large:
\vspace{.1in} 
\begin{claim}[Fundamental Gap Inequalities]\label{claim:fundamental_gap_inequality}
\hspace{0.1em} Consider a confidence-admissible Q-supervisor $\tuple_k$ and a triple $(x,a,h)$ where $\qup\kh\xa \geq \vlow_k\ofx$. We have:
\begin{align*}
\gaph\xa &\leq  \adv\kh\xa + p\xa^{\top}( \vup\khpl - \vlow\khpl) + \vup\kh\ofx - \vlow\kh\ofx.
\end{align*}
Moreover, if $a = \pigreed\kh\xa$, then
\begin{align*}
\gap_h(x,a)  \le \advlucb\kh(x,a) + p\xa^{\top}\left( \vup\khpl - \vlow\khpl \right).
\end{align*}
\end{claim}
\begin{proof}
Let's begin with the first inequality. By definition of $\gap_h\xa$, 
\begin{align*}
\gap_h\xa & = \vst_h\ofx - \qst_h\xa = \vlow\kh\ofx - \qst_h\xa + \vst_h\ofx - \vlow\kh\ofx \\
& \leq \qup\kh\xa - \qst_h\xa + \vst_h\ofx - \vlow\kh\ofx\\
& = \adv\kh\xa + p\xa^{\top}( \vup\khpl - \vst_{h+1}) + \vst_h\ofx - \vlow\kh\ofx\\
&\leq \adv\kh\xa + p\xa^{\top}( \vup\khpl - \vlow\khpl) + \vup\kh\ofx - \vlow\kh\ofx,
\end{align*} where in the last inequality, we use the fact that
$\vlow\kh(x) \leq \vst_{h}(x) \leq \vup\kh(x)$.

For the second inequality, suppose $a= \pigreed\kh(x)$. Then, 
\begin{align*}
\gap_h(x,a) & = \vst_h\ofx - \qst_h(x,a) \leq \vup\kh\ofx - \qlow\kh\xa = \qup\kh(x,a) - \qlow\kh(x,a) \\
&= \advlucb\kh(x,a) + p\xa^{\top}\left( \vup\khpl - \vlow\khpl \right).
\end{align*}
\end{proof}
Now we further clip $\advhalfclip\kh$, at the expense of incurring additional terms for $\advlucbhalfclip\kh$:
\begin{lemma}[Clipping under Q-supervised $\pi$]\label{lemma:clip_under_lucb_compatible} \hspace{0.1em}
Given a confidence-admissible Q-supervisor $\tuple_k := (\qup_k,\qlow_k,\activeset_k,\pigreed_k)$ and any deterministic Q-supervised  policy $\pi$ with respect to $\tuple_k$, we have:
\begin{align*}
\sum_{h=1}^H \Exp^{\pi}\left[\advhalfclip\kh(x_{h},a_{h})\right] &\le \sum_{h=1}^H \Exp^{\pi}\left[\advfullclip\kh(x_{h},a_{h})\right] + 4\sum_{\hbar=1}^H \Exp^{\pi\oplus_{\hbar}\pigreed_k}\left[\sum_{h = \hbar}^{H}\advlucbhalfclip\kh(x_h,a_h) \right].
\end{align*}
\end{lemma}
\begin{proof}
It suffices to establish the bound that, for any triple $(x_h,a_h,h)$, we have
\begin{align*}
&\advhalfclip\kh(x_h,a_h)  \leq \advfullclip\kh(x_h,a_h) + 4\Exp^{\pigreed_k}\left[ \sum_{h'=h}^{H} \advlucbhalfclip\khpr(x_{h'},a_{h'}) \mid  \xhah  \right].
\end{align*}
since summing this bound from $h = 1,\dots,H$ and taking an expectation under $\Exp^{\pi}$ concludes the proof. Moreover, it suffices to consider the case where $a_h\notin \pist_h(x_h)$, since for $a_h \in \pist_h(x_h)$, we in fact have that $\advhalfclip\kh(x_{h},a_{h}) \le \advfullclip\kh(x_{h},a_{h})$ (see Eq~\eqref{eq:optimal_action_clipping_bound}), and the additional terms arising from $\advlucbhalfclip\kh(x_h,a_h)$ are nonnegative.

In this case, by definition, we have $\gap_h(x_h,a_h) \geq \gap_{\min}$.  From the first inequality in Claim~\ref{claim:fundamental_gap_inequality}, we have:
\begin{align*}
&\frac{\gap_h(x_h,a_h)}{2} \\
&\quad\leq \adv\kh(x_h,a_h) + p(x_h,a_h)^{\top}(\vup\khpl - \vlow\khpl) + (\vup\kh(x_h) - \vlow\kh(x_h)) - \frac{\gap_{\min}}{2} \\
&\quad \leq \adv\kh(x_h,a_h) + \Exp^{\pigreed_k}\left[\sum_{h'=h+1}^H \advlucb\khpr(x_{h'},a_{h'}) \mid \xhah\right] + \Exp^{\pigreed_k}\left[ \sum_{h'=h}^{H} \advlucb\khpr(x_{h'},a_{h'})  \mid \xhah\right] \\
&\qquad- \frac{\gap_{\min}}{2} \\
& \quad\leq  \advhalfclip\kh(x_h,a_h) + \Exp^{\pigreed_k}\left[\sum_{h'=h+1}^H \advlucbhalfclip\khpr(x_{h'},a_{h'}) \mid \xhah\right] +  \Exp^{\pigreed_k}\left[ \sum_{h'=h}^{H} \advlucbhalfclip\khpr(x_{h'},a_{h'}) \mid  \xhah  \right]\\
& \quad\leq  \advhalfclip\kh(x_h,a_h) + 2\Exp^{\pigreed_k}\left[ \sum_{h'=h}^{H} \advlucbhalfclip\khpr(x_{h'},a_{h'}) \mid  \xhah  \right],
\end{align*} 
where in the second-to-last inequality, we absorb $\gapmin$ into each half-clipped Bellman errors, as in the proof of Lemma~\ref{lemma:half_clipping}, and in the final inequality, appeal to non-negativity of the clipped terms. 

We now consider two cases. First, if 
 \begin{align*}
 \advhalfclip\kh(x_h,a_h) \geq \frac{1}{4}\gap_h(x_h,a_h),
 \end{align*}
then we have 
\begin{align*}
\advhalfclip\kh(x_h,a_h) = \clip{\frac{1}{4}\gap_h(x_h,a_h)}{  \advhalfclip\kh(x_h,a_h) } \leq \advfullclip\kh(x_h,a_h),
\end{align*} where we recall 
$\advfullclip\kh\xa :=\clip{ \gapcheck_h\xa  }{ \adv\kh\xa }$ with $\gapcheck_h\xa = \max\left\{\frac{\gapmin}{8H^2},\,\frac{\gaph\xa}{8H}\right\}$, since $\gap_h\xa / 4 \ge \gap_h\xa / (8H)$ for any $H\ge 1$.

 On other other hand, if 
 \begin{align*}
 \advhalfclip\kh(x_h,a_h) < \frac{1}{4}\gap_h(x_h,a_h),
 \end{align*} then we must have:
\begin{align*}
\advhalfclip\kh(x_h,a_h) &< \frac{1}{4}\gap_h(x_h,a_h) ~\le 4\Exp^{\pigreed_k}\left[ \sum_{h'=h}^{H} \advlucbhalfclip\khpr(x_{h'},a_{h'}) \mid  \xhah  \right].
\end{align*}
This implies that:
\begin{align*}
&\advhalfclip\kh(x_h,a_h) < \max\left\{\advfullclip\kh\xa,  ~4\Exp^{\pigreed_k}\left[ \sum_{h'=h}^{H} \advlucbhalfclip\khpr(x_{h'},a_{h'}) \mid  \xhah  \right].\right\} \\
& \leq \advfullclip\kh(x_h,a_h) +  4\Exp^{\pigreed_k}\left[ \sum_{h'=h}^{H} \advlucbhalfclip\khpr(x_{h'},a_{h'}) \mid  \xhah  \right].
\end{align*}
 \end{proof}

We now apply ``full clipping'' to the surpluses $\advlucbhalfclip\kh$. Note that these surpluses only arise under rollouts where $a_h$ is selected according to $\pigreed\kh\ofx$. For comparison, we needed to consider clipping terms $\advhalfclip\kh\xa$ where $a$ could be selected according to \emph{any} deterministic Q-supervised  policy. 
Since we follow the policy $\pigreed$, we can employ the proof strategy adopted for optimistic policies in \cite{SimchowitzJamieson19}.
\vspace{.1in}
\begin{lemma}[Clipping under $\pigreed_k$]\label{lemma:clip_under_pigreed}
Consider a confidence-admissible Q-supervisor $\tuple_k$. Consider $(x,a,h)$ where $a = \pigreed\kh(x)$. We have:
\begin{align*}
\advlucbhalfclip\kh\xa \leq \advlucbfullclip\kh\xa  + \frac{e}{H} \Exp^{\pigreed_k}\left[ \sum_{h'=h+1}^H \advlucbfullclip\khpr(x_{h'},a_{h'}) \right].
\end{align*}
\end{lemma}
\begin{proof}
Throughout, consider a pair $\xa$ where $a = \pigreed\kh\ofx$. Again, in light of ~\eqref{eq:optimal_action_clipping_bound}, we may assume that $a \notin \pisth\ofx$ is suboptimal. From the second inequality in Claim~\ref{claim:fundamental_gap_inequality}, we have
\begin{align*}
\gaph\xa \le \advlucb\kh\xa + p\xa^{\top}\left( \vup\khpl - \vlow\khpl \right).
\end{align*}
Assuming $a\notin \pist_h(x)$, $\gaph\xa > \gapmin$, which implies that:
\begin{align*}
&\frac{1}{2}\gap_h\xa \leq \advlucb\kh\xa + p\xa^{\top}\left( \vup\khpl - \vlow\khpl \right) - \frac{1}{2}\gapmin\\
& \le \advlucb\kh\xa + \Exp^{\pigreed_k}\left[ \sum_{h'=h+1}^H \advlucb\khpr(x_{h'},a_{h'})  \right] - \frac{1}{2}\gapmin \\
& \leq \advlucbhalfclip\kh\xa  + \Exp^{\pigreed_k}\left[ \sum_{h'=h+1}^H \advlucbhalfclip\khpr(x_{h'},a_{h'})  \right],
\end{align*} 
where again we decompose $\gap_{\min}$ into $H$ $\gap_{\min}/H$ and absorb them into the half-clipped Bellman errors, as in the proof of Lemma~\ref{lemma:half_clipping}.

As in the proof of Lemma~\ref{lemma:clip_under_lucb_compatible}, we consider two cases. First, if $\advlucbhalfclip\kh\xa \geq \frac{1}{2(H+1)}\gap_h\xa $, then, we have:
\begin{align*}
\advlucbhalfclip\kh\xa = \clip{  \frac{1}{2(H+1)}\gap_h\xa  }{ \advlucbhalfclip\kh\xa } \leq \advlucbfullclip\kh\xa,
\end{align*}  where recall $\advlucbfullclip\kh\xa :=\clip{ \gapcheck_h\xa  }{ \advlucb\kh\xa } =  \clip{ \max\{\frac{\gap_{\min}}{8H^2}, \frac{\gap_h\xa}{8H} \}  }{ \advlucb\kh\xa }$ and the fact that $2(H+1) \leq 8H$ for any $H\geq 1$.

On the other hand, if $\advlucbhalfclip\kh\xa < \frac{1}{2(H+1)}\,\gap_h\xa $, 
then, we must have that
\begin{align*}
\Exp^{\pigreed_k}\left[ \sum_{h'=h+1}^H \advlucbhalfclip\khpr(x_{h'},a_{h'})  \right] &\ge \left(1 - \frac{1}{H+1}\right)\frac{\gap_h\xa}{2} \\
&\ge \left(1 - \frac{1}{H+1}\right)\cdot (H+1)\advlucbhalfclip\kh\xa = H\advlucbhalfclip\kh\xa.
\end{align*}
Rearranging, 
\begin{align*}
\advlucbhalfclip\kh\xa  &\le \frac{1}{H}  \Exp^{\pigreed_k}\left[ \sum_{h'=h+1}^H \advlucbhalfclip\khpr(x_{h'},a_{h'})  \right]
\end{align*}  
Thus 
since $\advlucbfullclip$ is non-negative and $\advlucbhalfclip$ is also non-negative,  combine the above two cases, we will have:
\begin{align*}
\advlucbhalfclip\kh\xa &\leq \max\left\{\advlucbfullclip\kh\xa ,\frac{1}{H} \Exp^{\pigreed_k}\left[ \sum_{h'=h+1}^H \advlucbhalfclip\khpr(x_{h'},a_{h'})  \right]\right\}\\
&\leq \advlucbfullclip\kh\xa  + \frac{1}{H}  \Exp^{\pigreed_k}\left[ \sum_{h'=h+1}^H \advlucbhalfclip\khpr(x_{h'},a_{h'})  \right]
\end{align*} regardless of the value $\advlucbhalfclip\kh$,
where we use the fact that the two terms on the right hand side are nonnegative.

Recursively expanding the above inequality and using that $(1+1/H)^H \leq e$, we obtain
\begin{align*}
\advlucbhalfclip\kh\xa \leq \advlucbfullclip\kh\xa  + \frac{e}{H} \Exp^{\pigreed_k}\left[ \sum_{h'=h+1}^H \advlucbfullclip\khpr(x_{h'},a_{h'}) \right],
\end{align*} which concludes the proof. 
\end{proof}

Now replacing half-clipped Bellman errors in Lemma~\ref{lemma:half_clipping} via fully clipped ones via Lemmas~\ref{lemma:clip_under_lucb_compatible} and~\ref{lemma:clip_under_pigreed}, we obtain the special case of Lemma~\ref{lem:clipped_regret} for any deterministic $\tuplek$-supervised policy:
\vspace{.1in}
\begin{lemma}[Full-clipping]
Given a confidence-admissible Q-supervisor $\tuple_k$ and any deterministic,policy $\pi$ which is $\tuplek$-supervised, we have:
\label{lemma:full_clipping}
\begin{align*}
\vst - V^{\pi} \leq 2\sum_{h=1}^H \Exp^{\pi}\left[\advfullclip\kh(x_{h},a_{h})\right]  + 20e  \sum_{\tau=1}^H \Exp^{\pi\oplus_{\tau} \pigreed_k}\left[ \sum_{h=\tau}^H \advlucbfullclip\kh(x_h,a_h)  \right].
\end{align*}
\end{lemma}
\begin{proof}
From Lemma~\ref{lemma:half_clipping}, we have:
\begin{align*}
\vst - \valf^{\pi} &\le 2\sum_{h=1}^H \Exp^{\pi}\left[\advhalfclip\kh(x_{h},a_{h})\right] + 2\sum_{\hbar=1}^H \Exp^{\pi\oplus_{\hbar}\pigreed_k}\left[\sum_{h = \hbar}^{H}\advlucbhalfclip\kh(x_h,a_h)\right]\\
& \leq  2\sum_{h=1}^H \Exp^{\pi}\left[\advfullclip\kh(x_{h},a_{h})\right] + 10\sum_{\tau=1}^H \Exp^{\pi\oplus_\tau\pigreed_k}\left[ \sum_{h=\tau}^H \advlucbhalfclip\kh\xa  \right]\\
& \leq  2\sum_{h=1}^H \Exp^{\pi}\left[\advfullclip\kh(x_{h},a_{h})\right]  \\
&\qquad+ 10 \sum_{\tau=1}^{H} \Exp^{\pi\oplus_\tau \pigreed_k} \left[ \sum_{h=\tau}^H \left( \advlucbfullclip\kh(x_h,a_h)  + \frac{e}{H}\Exp^{\pigreed_k}\left[\sum_{h'=h+1}^{H} \advlucbfullclip\khpr(x_{h'},a_{h'})  \right) \right]\right] \\
& \leq 2\sum_{h=1}^H \Exp^{\pi}\left[\advfullclip\kh(x_{h},a_{h})\right]  + 10(e+1) \sum_{\tau=1}^H \Exp^{\pi\oplus_{\tau} \pigreed_k}\left[ \sum_{h=\tau}^H \advlucbfullclip\kh(x_h,a_h)  \right],
\end{align*} which concludes the proof. 
\end{proof}

As described above, the fully general Lemma~\ref{lem:clipped_regret} follows by representing any randomized Markov $\boldpi$ as a distribution of deterministic policies $\{\pi^{(i)}\}$ which constitutes its atoms. Applying the above lemma to each $\{\pi^{(i)}\}$, and taking an expectation over the randomness in $\boldpi$ concludes the proof of the lemma. 

\subsection{Clipped regret via fictitious trajectories (Lemma~\ref{lem:clipped_localized_fictious_regret}\label{sssec:lem:clipped_localized_fictious_regret})}
Recall the shorthand
\begin{align*}
\errlogclip\kh\xa &:= \clip{\frac{\gapcheck\xa}{4}}{\errlogbark\xa} \\
\errlogsqclipk\xa&:=  \clip{\gapcliplow}{\errlogbark\xa^2}.
\end{align*}
By Lemma~\ref{lem:Bellman_error_bound_clip_final}, we have
\begin{align}\label{eq:shorthand_clipped_bellman}
\max\{\advlucbfullclip\kh\xa,\advfullclip\kh\xa\} &\le 2\errlogclip\kh\xa  + \frac{4e^2}{H} \Exp^{\pigreedk}\left[{\left( \sum_{\tau=h+1}^H \errlogsqclipk\xtauatau\right)}\right].
\end{align}
Let us substitute the above into Lemma~\ref{lem:clipped_regret}, which we recall states
\begin{align*}
\vst - \valf^{\boldpi_k} &\le 2\sum_{h=1}^H \Exp^{\boldpi_k}\left[\advfullclip\kh(x_{h},a_{h})\right] + 20 e{\sum_{\hbar=1}^H \Exp^{\boldpi_k\oplus_{\hbar}\pigreed_k}\left[\sum_{h = \hbar}^{H}\advlucbfullclip\kh(x_h,a_h) \right]}.
\end{align*}
For the first term, using the bound from Lemma~\ref{lem:Bellman_error_bound_clip_final}
    \begin{align*}
    &\sum_{h=1}^H \Exp^{\boldpi_k}\left[\advfullclip\kh(x_h,a_h)\right] \\
    &\le \sum_{h=1}^H \Exp^{\boldpik}\left[2\errlogclip\kh(x,a) +\frac{ 4e^2}{H}\Exp^{\pigreedk}\left[{\sum_{\tau=h+1}^H \errlogsqclipk\xtauatau \mid  x_h=x,a_h=a}\right]\right]\\
    &= 2\sum_{h=1}^H \Exp^{\boldpik}\left[\errlogclip\kh(x,a)\right] + \frac{4e^2}{H}\sum_{h=1}^{H}\Exp^{\boldpi_k \oplus \pigreedk}\left[{\sum_{\tau=h+1}^H \errlogsqclipk\xtauatau \mid  x_h=x,a_h=a}\right],\\
    &\le 2\sum_{h=1}^H \Exp^{\boldpik}\left[\errlogclip\kh(x,a)\right] + \frac{4 e^2}{H}\sum_{h=1}^{H}\Exp^{\boldpi_k \oplus \pigreedk}\left[{\sum_{\tau=h}^H \errlogsqclip_{k;\tau}(x_\tau,a_\tau)\mid  x_h=x,a_h=a}\right].
    \end{align*}
    In the first equality, we note that the fictitious rollouts under $\pigreedk$ that arise in \eqref{eq:shorthand_clipped_bellman} can be viewed as rollouts under $\boldpi_k \oplus \pigreedk$. In the second inequality, we use that $\errlogbark\xa \ge 0$. Similarly, 
    \begin{align*}
    &\Exp^{\boldpi_k\oplus_{h}\pigreed_k}\left[{\sum_{\tau= h}^{H}\advlucbfullclip\ktau(x_{\tau},a_{\tau})}\right] \\
    &\quad\le \Exp^{\boldpi_k\oplus_{h}\pigreed_k}\left[{\sum_{\tau= h}^{H}2\errlogclip\kh\xa + \frac{4e^2}{H}\Exp^{\pigreed_k}\left[\sum_{\tau' = \tau+1}^{H} \errlogsqclipk(x_{\tau'},a_{\tau'}) \mid x_{\tau}=x,a_{\tau}=a\right]}\right]\\
    &\quad= \Exp^{\boldpi_k\oplus_{h}\pigreed_k}\left[{\sum_{\tau= h}^{H} \left(2\errlogclip\kh\xa + \frac{4e^2}{H} \sum_{\tau'=\tau+1}^{H}  \errlogsqclipk(x_{\tau'},a_{\tau'}\right) }\right]\\
    &\quad\le \Exp^{\boldpi_k\oplus_{h}\pigreed_k}\left[{\sum_{\tau= h}^{H}2\errlogclip\kh\xtauatau + 4e^2\errlogsqclipk\xtauatau  }\right].
    \end{align*}
    Therefore,
    \begin{align*}
    \vst - \valf^{\boldpi_k} &\le 2\sum_{h=1}^H \Exp^{\boldpik}\left[2\errlogclip\kh(x,a)\right] + (20 e \cdot 2)\sum_{h=1}^H\Exp^{\boldpi_k\oplus_{h}\pigreed_k}\left[{\sum_{\tau= h}^{H}\errlogclip\kh\xtauatau}\right]\\
    &\quad + (20 e \cdot 4e^2 + \tfrac{2\cdot 4e^2}{H})\sum_{h=1}^H\Exp^{\boldpi_k\oplus_{h}\pigreed_k}\left[{\sum_{\tau= h}^{H}\errlogsqclipk\xtauatau}\right].
    \end{align*}
    We can lastly bound $20 e \cdot 4e^2 + \tfrac{2\cdot 4e^2}{H} \le 80e^3 + 8e^2/8 \le 81 e^3$, concluding the proof. 

\newpage

\section{Omitted proofs on linear MDPs}
\label{app:aux_linear}

\newcommand{\calU}{\mathcal{U}}

\newcommand{\Ahat}{\widehat{A}}
\newcommand{\Ridge}{\mathrm{Ridge}}
\newcommand{\scrN}{\mathscr{N}}
\newcommand{\Nmax}{N_{\max}}
\newcommand{\calQ}{\mathcal{Q}}
\newcommand{\kglh}{_{k;gl;h}}
\newcommand{\Nepsinf}{\calN_{\epsilon,\infty}}
\newcommand{\Pigrd}{\Pi_{\mathrm{grd}}}

\newcommand{\scrF}{\mathscr{F}}

In this section, we first focus on proving that the bonuses we designed for linear MDPs induces a valid model-estimate (recall Definition~\ref{defn:valid}). The proofs regarding concentration results in this section is built on the general results of multi-variate Ridge linear regression with adversarial corruptions that we establish in Appendix~\ref{app:proof_linear_rl_concentration_corrupt}. We also provide upper bounds on the Bellman error and prove that summing over the confidence bounds leads to a sub-linear regret.

We recall the bonuses first:
\begin{align*}
& b\kglell\xa  = \beta (d+\sqrt{A})H \|\phi\xa \|_{\Lambda\kgl^{-1}} + 4\overline{C}_{\ell;gl} H^2 \|\Lambda_{k;gl}^{-1 }\phi\xa\|_2 , \quad \text{with } \overline{C}_{\ell;gl} = 2^\ell,\\
& b\ksbell\xa = \beta (a+\sqrt{A})H \|\phi\xa\|_{\Lambda^{-1}\ksbell} + 4\overline{C}_{\ell;sb} H^2 \| \Lambda\ksbell^{-1}\phi\xa \|_2 , \quad \text{with }\overline{C}_{\ell;sb} = 2\ln(16 \ell^2/\delta).
\end{align*}

\subsection{Computing the set of value functions (Lemma~\ref{lem:value_function_set})\label{ssec:proof:lem:value_function_set}}

\newcommand{\quptil}{\tilde{\qup}}
\newcommand{\qlowtil}{\tilde{\qlow}}
\newcommand{\vuptil}{\tilde{\vup}}
\newcommand{\vlowtil}{\tilde{\vlow}}

Define the Q-functions
\begin{align*}
\qup_{k,\ell;gl,h}\xa &=  \bonus\kglell\xa + \rhat\kgl\xa +  \phat\kgl\xa^{\top} \vup\klhpl\\
\qup_{k,\ell;sb,h}\xa &=  \bonus\klsb\xa + \rhat\klsb\xa +  \phat\klsb\xa^{\top} \vup\klhpl\\
\qlow_{k,\ell;gl,h}\xa &=  -\bonus\kglell\xa + \rhat\kgl\xa +  \phat\kgl\xa^{\top} \vlow\klhpl\\
\qlow_{k,\ell;sb,h}\xa &=  -\bonus\klsb\xa + \rhat\klsb\xa +  \phat\klsb\xa^{\top} \vlow\klhpl.
\end{align*}
Define as well
\begin{align*}
&\qup\klh\xa = \min\{H,\qup_{k,\ell;gl,h}\xa,\qup_{k,\ell;sb,h}\xa\},  \qlow\kl\xa = \min\{H,\qlow_{k,\ell;gl,h}\xa,\qlow_{k,\ell;sb,h}\xa\}.
\end{align*}
Lastly, recall then that
\begin{align}
\vup\klh\ofx = \max_{a \in \activeset\khl} \qup\klh\xa, \quad \vlow\klh\ofx = \max_{a \in \activeset\khl} \qlow\klh\xa \label{eq:linear_value_funcs}
\end{align}

\xhdr{A parametric form for the Q-functions.} In this section, we show that all Q-functions which arise take the following form:\vspace{0.1in}
\begin{definition}[Q-function Sets]\label{defn:Q_functions} Given $\signsig \in \{-,+\}$, define the set of Q-functions
\begin{align}
\calQ_{\signsig} := \left\{Q\xa =  w^\top \phi\xa+  \signsig\left(\alpha_1 \|\phi\xa\|_{\Lambda^{-1}} + \alpha_2 \|\phi\xa\|_{\Lambda^{-2}} \right)\right\},\label{eq:Q_a_pls}.
\end{align}
where the following hold
\begin{itemize}
\item $\|w\|_2 \le L:= 2H\frac{T}{\lambda}$ and $\Lambda \succeq \lambda I $, $\signsig \in \{-1,1\}$
\item $0 \le \alpha_1 \le B_1 := H (\sqrt{A} + d) \beta  $ and $0 \le \alpha_2 \le H^2\max_{\ell} \{\overline{C}_{\ell;gl}, \overline{C}_{\ell;sb}\} \le H^2\max_{\ell} 2^{\ell} \le B_2 := 2 TH$.
\end{itemize}
\end{definition}

The following lemma establishes that the Q-functions considered admit the form above:\vspace{0.1in}
.\begin{lemma}\label{lem:Q_set} \hspace{0.1em}
For $h \in [H+1]$, $\qup_{k,\ell;gl,h},\qup_{k,\ell;sb,h} \in \calQ_{+}$, and $\qlow_{k,\ell;gl,h},\qlow_{k,\ell;sb,h} \in \calQ_{-}$
\end{lemma}
\begin{proof}
Let us prove the inclusion for $\qup_{k,\ell;gl,h} \in \calQ^+$; the other inclusions follow similarly. We elucidate the forms of the Q-functions, starting schematically with $\qup_{k,\ell;gl,h}$:
\begin{align*}
\qup_{k,\ell;gl,h}\xa &=  b\kglell\xa + \rhat_k\xa +  \phat_k\xa^{\top} \vup_{k,\ell;h+1}\xa \\
&=  b\kglell\xa + \left( \thetahat\kgl +  \int_{x'} \muhat\kgl(x') \vup_{k,\ell;h+1}(x') \right)^{\top} \phi\xa \\
&=  w\kgl(\vup_{k,\ell;h+1})^{\top} \phi\xa + \beta(d+\sqrt{A})H \|\phi\xa \|_{\Lambda\kgl^{-1}} + \overline{C}_{\ell;gl} H^2 \|\Lambda\kgl^{-1 }\phi\xa\|_2,
\end{align*}
where we have defined $w\kgl(V):= \thetahat\kgl + \int_{x'} \muhat\kgl(x') V(x')$. The proof then follows from the definition of the bonuses in Eq. \eqref{eq:linear_local_bonus} and \eqref{eq:linear_global_bonus},  combined with Claim~\ref{claim:bound_on_w_linear} stated below. \end{proof}
\begin{claim}\label{claim:bound_on_w_linear} \hspace{0.1em} For any $V \in [0,H]$, we have that $\|w\kgl(V)\|_2 \leq \frac{2 TH}{\lambda}$ for all $k \in[K]$.
\end{claim}
\begin{proof}
Recall the closed-form solution of $\widehat\theta_k$ and $\widehat\mu_k$ from ridge linear regression (see Appendix~\ref{app:proof_linear_rl_concentration_stochastic} for general ridge linear regression):
\begin{align*}
\muhat\kgl(x,a) := \sum_{i=1}^{k - 1 }\sum_{h=1}^H \delta_{x_{i;h+1}} \phi(x\ih,a\ih)^{\top} \Lambda\kgl^{-1}, \quad \thetahat\kgl :=  \sum_{i=1}^{k-1} \sum_{h=1}^H \Lambda\kgl^{-1} \phi(x\ih,a\ih) \Rih,
\end{align*} with $\delta_{x}$ being the one-hot encoding vector. Using the bound that can bound $\Lambda\kgl \succeq \lambda I$ and $\Rih \le 1$, we bound $w\kgl$ as follows
\begin{align*}
\|w\kgl(V) \|_2  \leq  \|\thetahat\kgl \|_2 + \| \sum_{x'} \widehat\mu_k(x') V(x') \|_2 &\leq \frac{kH}{\lambda} + \left\| \sum_{i=1}^k \sum_{h=1}^H V(x_{i;h+1})\phi^{\top}_{i;h}(x\ih,a\ih) \Lambda_k^{-1}   \right\|_2.
\end{align*} 
Moreover, using the fact that $\|\phi^{\top}_{i;h}(x\ih,a\ih)\| \le 1$, $\Lambda_k \succeq \lambda I $, and $V(\cdot) \in[0,H]$, the second term in the last inequality is at most $\frac{kH\cdot H}{\lambda}$. Thus, the above is at most $\frac{2k H^2}{\lambda} \le \frac{2 TH}{\lambda}$.
\end{proof}

\xhdr{Reduction to Parameter Cover.} As a direct consequence of Lemma~\ref{lem:Q_set},  have the following:
\newcommand{\calQbar}{\overline{\calQ}}
\begin{lemma}\label{lem:Vset} \hspace{0.1em} Let $\calQbar_{+} := \{Q\xa = \min\{H,Q_1\xa,Q_2\xa\}: Q_1,Q_2 \in \calQ_+\}$, and $\calQbar_{-} := \{Q\xa = \max\{0,Q_1\xa,Q_2\xa\}: Q_1,Q_2 \in \calQ_-\}$. Then $\qup\klh \in \calQbar_+$ and $\qlow\klh \in \calQbar_-$. Thus, for all $k,\ell,h$, we have that $\vsthpl,\vlow\klhpl,\vup\klhpl$ lies in the set
\begin{align*}
\calVst := \{\vsthpl\}_{h \in [H]} \cup \bigcup_{\signsig \in \{+,-\}}\bigcup_{\activeset \subset \actions}\{V:\,V(x) = \max_{a \in \activeset} Q\xa, \, Q \in \calQbar_{\signsig}\}
\end{align*}
\end{lemma}
We can now directly bound
\begin{claim}[Covering Number of $\calVst$]\label{claim:covering_number} The covering number of $\calVst$ can be bounded as
\begin{align*}
\ln \Nepsinf(\calVst) \le \ln(H) + A + 2\max_{\signsig \in \{+,-\}}\max_{\activeset\subseteq\actions}\ln\Nepsinf(\{V: V(x) = \max_{a \in \activeset} Q\xa, \, Q \in \calQbar_{\signsig}\})
\end{align*}
\end{claim}
\begin{proof} 
Since there are at  $2^A$ subsets $\activeset \subset \actions$, we have 
\begin{align*}
\ln \Nepsinf(\calVst) \le \ln (H) + \ln(2\cdot 2^A) + \max_{\signsig \in \{+,-\}}\max_{\activeset\subseteq\actions}\ln\Nepsinf(\{V: V(x) = \max_{a \in \activeset} Q\xa, \, Q \in \calQbar_{\signsig}\})
\end{align*}
Using $\ln(2^{A+1}) \le A$ for $A \ge 1$ concludes.
\end{proof}

It remains to compute a covering number of the set $\calQbar_{\signsig}$. To do so, we introduce the notion of a parameter cover:
\begin{definition}[Parameter Covering]\label{defn:par_cover} \hspace{0.1em} Given $\Lambda \succ 0, \alpha_1 \in R^d,\alpha_2 \in \R,w \in \R^d$,  define the function $F_{w,\Lambda,\alpha_1,\alpha_2,\signsig}: \R^d \to \R$ via
\begin{align*}
F_{w,A_1,A_2,\signsig}(v) := w^\top v + \signsig\left( \|v\|_{A_1} + \|v\|_{A_2}\right),
\end{align*}
where for positive definite matrices, we define $\|v\|_{A} := \sqrt{|v^\top A v|}$.
Moreover, define $B_1 := H (\sqrt{A} + d)$, $B_2 := 2 TH, L = 2TH/\lambda$. We say that $\calU = \{(A_{1,i},,A_{2,i},w_i)\}$ is an $\epsilon$-parameter cover if, for all $\alpha_1 \in [0,B_1]$ and $\alpha_2 \in [0,B_2]$,  all $\|w\| \le L$ and all $\Lambda \succeq \lambda I$, and $\signsig \in \{+1,1\}$,
\begin{align*}
\inf_{i} \sup_{v \in\R^d:\|v\| \le 1} \left|\{F_{w, \alpha_1\Lambda^{-1},\alpha_2\Lambda^{-2},\signsig}(v) - F_{w_i,A_{1,i},A_{2,i},\signsig}(v) \right| \le \epsilon \}
\end{align*}
\end{definition}

\newcommand{\Npareps}{\calN_{\mathrm{par},\epsilon}}

\begin{claim}\label{claim:Npareps_cover} \hspace{0.1em} Let $\Npareps$ denote the cardinality of an $\epsilon$-parameter cover $\calU$. Moreover, let $\Nepsinf(\calQbar_{\signsig})$ denote the covering number of $\calQbar_{\signsig}$ in the distance $\dist(Q,Q') = \max_{x,a}|Q\xa - Q'\xa|$. Then, for $\signsig \in \{-1,1\}$, 
\begin{align*}
\max_{\signsig \in \{+,-\}}\max_{\activeset\subseteq\actions}\ln\Nepsinf(\{V: V(x) = \max_{a \in \activeset} Q\xa, \, Q \in \calQbar_{\signsig}\})) \le 2 \ln \Npareps.
\end{align*}
Hence,
\begin{align*}
\ln \Nepsinf(\calVst) \le \ln(H) + A + 2\ln \Npareps
\end{align*}
\end{claim}
\begin{proof}
Let $\calU = \{(\alpha_{1,i},\alpha_{2,i},\Lambda_i,w_i)\}$ denote a minimal $\epsilon$-parameter cover. One can cover $\calQbar_+$ with functions of the form $\qup_{(i,j)} \xa = \min\{H,F_{w_i,A_{1,i},A_{2,i},+},F_{w_j,A_{1,j},A_{2,j},+}\}$ since taking minima is a contraction; this requires at most $|\calU|^2 = \Npareps^2$ functions. Similarly, one can cover $\calQbar_-$ with functions of the form $\qlow_{(i,j)} \xa = \max\{0,F_{w_i,A_{1,i},A_{2,i},-},F_{w_j,A_{1,j},A_{2,j},-}\}$. By the same argument, we can verify that the functions $\max_{a \in \activeset} \qup_{(i,j)}(\cdot,x)$ and $\max_{a \in \activeset} \qlow_{(i,j)}(\cdot,x)$ cover the set $\{V: V(x) = \max_{a \in \activeset} Q\xa, \, Q \in \calQbar_{\signsig}\}$. The bound follows.
\end{proof}

\xhdr{Bounding the size of the parameter cover.}
In light of Claim~\ref{claim:covering_number} and Lemma~\ref{lem:Vset}, it remains to bound the covering number of the sets of the form $\calV_{\signsig,\activeset}$. 
\begin{lemma}[Covering number of quadratic value functions] \label{lemma:epsilon_cover_lemma}  For $L, B_1,B_2$ defined in Def~\ref{defn:par_cover}, the cardinality $\Npareps$ of the minimal $\epsilon$-parameter covering is bounded by  
\begin{align*}
\ln|\Npareps| \leq d\ln(1 + 4L/\epsilon) + d^2 \ln\left( 1+ 32 d^{1/2} B_1^2 / (\lambda \epsilon^2)  \right) + d^2 \ln\left( 1+ 32 d^{1/2} B_2^2 / (\lambda^2 \epsilon^2)  \right).
\end{align*}
\end{lemma}
\begin{proof}
The proof of this lemma mainly follows from the proof of Lemma D.6 in \cite{jin2019provably}. The only difference is that we need to consider an extra term $\sqrt{\phi^{\top}\Lambda^{-2} \phi}$.

In what follows, let us assume $\signsig = 1$ and suppress dependence on this paramater. The $\signsig = -1$ case is similar. Next, consider two sets of parameters $(w,A_1,A_2)$ and $(w',A_1',A_2')$. Let $F = F_{w,A_1,A_2}$ and $F' = F_{w',A_1',2_2'}$.  We can then we can 
\begin{align*}
\max_{v :\|v\| \le 1} \left\lvert F(v) - F'(v) \right\rvert \leq \| w - w' \|_2 + \sqrt{ \| A_1 - A'_1 \|_F } + \sqrt{ \|A_2 - A'_2 \|_F} .
\end{align*}
In order to cover functions of the form $F_{w, \alpha_1^2\Lambda^{-1},\alpha_2^2\Lambda^{-2}}$,  we select our cover to be the product of sets $\{w_i\} \times \{\Lambda_{1,i}\} \times \{\Lambda_{2,i}\}$ where
\begin{enumerate}
	\item $\{w_i\}$ is an $\epsilon/2$ cover $\{w:\|w\| \le L\}$ in $\ell_2$
	\item $\{A_{1,i}\}$ is an is an $\epsilon^2/16$ cover of $\{A\in\mathbb{R}^{d\times d} : \|A \|_F \leq \sqrt{d} B_1^2 /\lambda \}$
	\item $\{A_{2,i}\}$ is an is an $\epsilon^2/16$ cover of $\{A\in\mathbb{R}^{d\times d} : \|A \|_F \leq \sqrt{d} B_2^2 /\lambda^2 \}$
\end{enumerate}
To bound the size of such a covering, we use the following classical covering result (Lemma~\ref{lem:covering_number_linear}, stated below). From Lemma~\ref{lemma:cover_ball}, we know that $|\mathcal{C}_{w} |\leq (1 + 4L/\epsilon)^d$, $|\mathcal{C}_{A_1}| \leq (1 + 32 \sqrt{d}B_1^2 / (\lambda\epsilon))^{d^2}$, and $|\mathcal{C}_{A_2}| \leq (1 + 32 \sqrt{d}B_2^2/(\lambda^2\epsilon))^{d^2}$. Then the size of $\epsilon$-net of the quadratic function class is $|\mathcal{C}_w| |\mathcal{C}_{A_1}| |\mathcal{C}_{A_2}|$. Taking logarithms on both sides concludes the proof.
\end{proof}
\begin{lemma}[Covering number]\label{lem:covering_number_linear} For any $\epsilon > 0$, the $\epsilon$-covering number of the ball $\{ \theta\in\mathbb{R}^d: \|\theta\|_2\leq R \}$ is upper bounded by $(1 + 2R/\epsilon)^d$.
\label{lemma:cover_ball}
\end{lemma}

\xhdr{Concluding the proof of Lemma~\ref{lem:value_function_set}.}
\begin{proof}[Proof of Lemma~\ref{lem:value_function_set}] Let us substitute the following quantities in Lemma~\ref{lemma:epsilon_cover_lemma}:
\begin{enumerate}
	\item $L = 2T h/\lambda$, 
	\item For $\beta \ge 1$, $\max\{B_1^2,B_2^2\} \le \max\{4T^2 H^2, 4d^2 H^2\beta^2 A\} \le 4 d^2 T^2 H^2 \beta^2 A$. 
\end{enumerate}
Together with $\epsilon \le 1$ and $\lambda \le 1$, we have
\begin{align*}
\ln(H) + \ln \Npareps &\le \ln(H) + d\ln(1 + 8TdH/\lambda\epsilon) + 2 d^2 \ln\left( 1+ 128 d^{3} T^2 H^2 \beta^2 A / \lambda \epsilon^2)  \right)\\
&\le 3 d^2 \ln\left( 1+ 128 d^{3} T^2 H^2 \beta^2  A/ \lambda^2 \epsilon^2)  \right)\\
&\le 6 d^2 \ln\left( 1+ \frac{16 d^2 T H A  \beta}{\lambda \epsilon}  \right).
\end{align*}
Hence, from Lemma~\ref{lemma:epsilon_cover_lemma},
\begin{align*}
\ln\calN_{\epsilon,\infty}(\calVst) &\le A + \ln(H) + 2\max_{\signsig,\activeset}\ln \calN_{\epsilon,\infty}(\calV_{\signsig,\activeset})\\
&\le A + 12 d^2 \ln\left( 1+ \frac{16 d^2 T H A  \beta}{\lambda \epsilon}  \right).
\end{align*}
as needed.\end{proof}

\subsection{Concentration for idealized model estimates (Lemma~\ref{lemma:linear_global_stoch} and Lemma~\ref{lemma:linear_local_stoch})}
\label{app:proof_linear_rl_concentration_stochastic}

	For the stochastic data sequence, our analysis follows from a generalized analysis of ``well-specified'' ridge regression. Formally:
	\begin{definition}[Ridge Operator]\label{defn:ridge_operate} \hspace{0.1em} Consider dataset $\{x_i, y_i\}_{i=1}^N$ with $x_i\in\mathbb{R}^{d_1}$ and $y_i\in\mathbb{R}^{d_2}$. We define
	\begin{align*}
	\Ridge_{\lambda}(\{x_i,y_i\}_{i=1}^N ) := \arg\min_{A\in\mathbb{R}^{d_2\times d_1}} \sum_{i=1}^N \| A x_i - y_i \|_2^2 + \lambda \|A \|^2_{F}.
	\end{align*}
	Analytically, we have 
	\begin{align*}
	\Ridge_{\lambda}(\{x_i,y_i\}_{i=1}^N) =  \sum_{i=1}^N y_i x_i^{\top} V_N^{-1}, \quad \text{where } V_N = \sum_{i=1}^N x_ix_i^{\top} + \lambda I.
	\end{align*}
	\end{definition}
	Our stochastic data satisfies the following ``well-specified'' assumption, formalized as follows:
\vspace{0.1in}	\begin{definition}[Well Specified Regression]\label{defn:well_specified_ridge} \hspace{0.1em} We say that a data-set $\{x_i, y_i\}_{i=1}^{\infty}$ is $\sigma$-well specified if there exists a filtration $(\calF_i)$ such that $\Exp[y_i \mid \calF_i] = \Ast x_i$, for some $\Ast$, and $\|y_i - \Ast x_i\|_{\ell_1} \le \sigma$ almost surely. We call $\Ast$ the \emph{Bayes optimal} solution.
	\end{definition}
	For such well-specified data, we stipulate the following guarantee, which is proven below.
	\vspace{0.1in}
	\begin{lemma}[Sample complexity of ridge linear regression]
	\label{lemma:linear_regression} \hspace{0.1em}
	Let $\{x_i, y_i\}_{i=1}^{\infty}$ denote a date set of $\sigma$-well specified, with Bayes optimal matrix $\Ast$. Further, assume that $\|\Ast_i\|_{\ell_1} \leq 1$ for all $i$ where $\Ast_i$ is the $i$-th column of $A$, and that $\|x_i\|_2\leq 1$ for all $i$. 

	Consider a class of functions $\scrF := \{ f \in [0,H]^{d_2} \} \subset \mathbb{R}^{d_2}$, and assume that its covering number in the $\infty$ norm $\dist(f_1,f_2) :=\|f_1- f_2\|_{\infty}$ is at most $\calN_{\epsilon,\infty}$. Then the following holds with probability $1 - \delta$ for all $x \in \R^{d_2}$, $f \in \scrF$, and $N \ge 1$:
	\begin{align*}
	\left\lvert (\Ahat_N) \cdot f - (\Ast x) \cdot f \right\rvert  \leq  \|x\|_{V_N^{-1}} \cdot \left( \sigma H \sqrt{2} \ln^{1/2}\left( \calN_{\epsilon,\infty}\left(1 + \tfrac{N}{\lambda}\right)^{\frac{d}{2}} /\delta \right) +  \frac{N \sigma \epsilon}{\sqrt{\lambda}} + \sqrt{\lambda d} H \right),
	\end{align*}
	where  $\Ahat_N := \Ridge_{\lambda}(\{x_i,y_i\}_{i=1}^N)$ and $V_N :=  \sum_{i=1}^N x_ix_i^\top + \lambda I $. 

	In particular, fix an $\Nmax$. Then, if $\lambda = 1$, $\ln \calN_{\epsilon,\infty} \ge d  \ln( 1 + H \Nmax)$, and  $\epsilon \le \frac{ 1}{\Nmax\sigma}$, then we have the following simplified expression for all $N \le \Nmax$:
	\begin{align*}
	\left\lvert \Ahat_N \cdot f - (\Ast x) \cdot f \right\rvert  \leq  4\sigma H\|x\|_{V_N^{-1}} \cdot \sqrt{\ln\left(\calN_{\epsilon,\infty} \cdot \frac{1}{\delta} \right)}.
	\end{align*}
	\end{lemma}

	We set $\lambda = 1$ specifically as this is what is used in the algorithm.  We now prove Lemma~\ref{lemma:linear_global_stoch}; the proof of Lemma~\ref{lemma:linear_local_stoch} is similar but accounts for a union bound over $\ell \in [\lmax]$.
\begin{proof}[Proof of Lemma~\ref{lemma:linear_global_stoch}] 
	For the stochastic sequece, we satisfy the ``well-specified'' asssumption described above.  Applying Lemma~\ref{lemma:linear_regression} with the function class value function $\scrF = \calVst$, $x_i = \phi(x_{k,h},a_{k,h})$ for an appropriate transformation of indices, $\sigma = 1$ (recall that the noise corresponds to transition probabilities), and $\epsilon = 1/T$, $\lambda = 1$, and $\Nmax = T$ and using that $\ln\calN_{\epsilon,\infty}(\calVst) \le \ln \Nbeta$ satisfies the conditions for the simplified theorem, we have the following bound for all $V \in \calVst$ and all $k \in [K]$ with probability $1- \delta/16$:
	\begin{align*}
	\left\lvert (\widehat\mu^{\stoch}_{k} \phi\xa - \mu^\star\phi\xa)\cdot  V \right\rvert \leq \|\phi\xa\|_{\Lambda^{-1}\kgl} \le 4H\|\phi\xa\|_{\Lambda\kgl^{-1}} \cdot \sqrt{ \ln\left(\Nbeta\cdot \frac{16}{\delta} \right)}.
	\end{align*} 
	Above, we recall that Lemma~\ref{lemma:linear_regression} is in fact an anytime bound. 

	For the reward function, Lemma~\ref{lemma:linear_regression} with $d_2 = 1$, $\scrF = \{1\}$, $\sigma = 1$, $\lambda = 1$ and $\sigma, H = 1$ and  $\epsilon = 0$ (since we cover the singleton function class exactly), we find that the following holds simultaenously for all $a \in \actions$ with probability $1 - \delta/16$ and $k \in [K]$
	\begin{align*}
	\left\lvert (\widehat\theta_k^{\stoch} - \theta^\star)\cdot \phi\xa \right\rvert \leq \|\phi(\xa)\|_{\Lambda\kgl^{-1}} \cdot \left( \sqrt{2} \ln^{1/2}\left( \left(1 + \tfrac{N}{\lambda}\right)^{\frac{d}{2}} 16 /\delta \right) +  \sqrt{ d}  \right),
	\end{align*}
	Crudely aborbing into the error on transition estimates, we have the following bound holds for all $a \in \actions$, $k \in [K]$ and $V \in \calVst$:
	\begin{align*}
	\left\lvert (\widehat\mu^{\stoch}_{k} \phi\xa - \mu^\star\phi\xa)\cdot  V \right\rvert + \left\lvert (\widehat\theta_k^{\stoch} - \theta^\star)\cdot \phi\xa \right\rvert  \le 7H\|\phi\xa\|_{\Lambda\kgl^{-1}} \cdot  \sqrt{ \ln\left(\Nbeta\cdot \frac{16}{\delta} \right)} .
	\end{align*} 
	Recalling that $\phatstoch\xa = \widehat\mu^{\stoch}_{k} \phi\xa,\pst\xa = \mu^\star\phi\xa$, and similarly for the rewards, the bound follows.	\end{proof}

\begin{proof}[Proof of Lemma~\ref{lemma:linear_regression}]
	Let us explore $y_i = \varepsilon_i +  \Exp[y | x_i] = \varepsilon_i + \Ast x_i$. Recalling $V_N := \sum_{i=1}^N x_ix_i^\top + \lambda I$, and writing $\sum_{i=1}^N x_ix_i^\top = XX^\top$ in matrix form, the  closed-form solution for  $\Ahat_N$ (Def~\ref{defn:ridge_operate}) gives
	\begin{align*}
	&\Ahat_N - \Ast = \sum_{i=1}^N y_i x_i^{\top}V_N^{-1} - \Ast \\
	& = \sum_{i=1}^N ( \varepsilon_i + \Ast x_i) x_i^{\top} V_N^{-1} -  \Ast \\
	&= \sum_{i = 1}^{N} \varepsilon_i x_i^{\top} V_N^{-1} + \Ast XX^{\top}(XX^{\top}+\lambda I)^{-1} - A \\
	&=   \sum_{i = 1}^{N} \varepsilon_i x_i^{\top}V_n^{-1} - \lambda \Ast (XX^{\top}+\lambda \mathbf{I})^{-1}\\
	&=   \sum_{i = 1}^{N} \varepsilon_i x_i^{\top}V_N^{-1} - \lambda \Ast V_N^{-1}.
	\end{align*} 
	Hence, for any $x$ and $f$,
	\begin{align*}
	&(\Ahat_N - \Ast)x) \cdot f = \underbrace{\sum_{i=1}^N   x_i^{\top} V_N^{-1} x_i \varepsilon_i^{\top} f}_{\text{(Unbiased Term)}[f,x]} - \underbrace{\lambda x^{\top} V_N^{-1} (\Ast)^{\top} f}_{\text{(Bias Term)}[f,x]}
	\end{align*}
	Let us bound the first term on the RHS of the above equality. Using Cauchy-Schwarz, we have:
	\begin{align*}
	\text{(Unbiased Term)}[f,x] \le \left\lvert \sum_{i=1}^N   x^{\top} V_N^{-1} x_i \varepsilon_i^{\top} f \right\rvert \leq \| x \|_{V_N^{-1}}   \left\| \sum_{i=1}^N x_i (\varepsilon_i \cdot f) \right \|_{V_N^{-1}}.
	\end{align*} 
	Using a standard technique for bounding self-normalized martingale sums (Lemma~\ref{lemma:self_normalized}) and the fact that $ \left\lvert\varepsilon_i \cdot f \right\vert \leq  \|\varepsilon_1\|_1 \| f \|_{\infty} \leq \sigma H$ (Holder's inequality), we have that with probability at least $1-\delta$,  the following holds for all $N$ simultaneously:
	\begin{align*}
	 \left\| \sum_{i=1}^N x_i (\varepsilon_i \cdot f) \right \|_{V_N^{-1}}^2 &\leq 2\sigma^2H^2 \ln\left( \det(V_N)^{1/2} \det(\lambda I)^{-1/2} /\delta \right)\\
	&\leq 2\sigma^2H^2 \ln\left( \left(1 + \frac{N}{\lambda}\right)^{d/2} /\delta \right),
	\end{align*}
	where we use the fact that $\|x_i\| \le 1$ to bound $\det(V_N) \le \|V_N\|^d \le (\lambda+ N)^d$. Let $\scrN$ denote a covering of $\scrF$ of cardinality $\calN_{\epsilon,\infty}$. By a union, the following then holds simultaneously for all $f \in \scrN$ and $N \ge 1$: 
	\begin{align*}
	 \left\| \sum_{i=1}^N x_i (\varepsilon_i \cdot f) \right \|_{V_N^{-1}}^2 &\leq 2\sigma^2H^2 \ln\left( \det(V_N)^{1/2} \det(\lambda\mathbf{I})^{-1/2} |\mathcal{N}_{\epsilon}| /\delta \right) \\
	&\le 2\sigma^2H^2 \ln\left( \calN_{\epsilon,\infty}\left(1 + \tfrac{N}{\lambda}\right)^{\frac{d}{2}} /\delta \right),
	\end{align*}

	Now, for $f\in\mathcal{F}$, by the definition of $\epsilon$-net, there must exist an $f'\in\mathcal{N}_{\epsilon}$ such that $\| f - f' \|_{\infty} \leq \epsilon$. Hence, we have with probability $1 - \delta$ that for all $N$,
	\begin{align*}
	 \left\| \sum_{i=1}^N x_i (\varepsilon_i \cdot f) \right \|_{V_N^{-1}} &\leq \left\| \sum_{i=1}^N x_i (\varepsilon_i \cdot f) \right \|_{V_N^{-1}} + \left\| \sum_{i=1}^N x_i (\varepsilon_i \cdot (f - f')) \right \|_{V_N^{-1}} \\
	& \leq \sqrt{2 }\sigma H \ln^{1/2}\left( \calN_{\epsilon,\infty}\left(1 + \tfrac{N}{\lambda}\right)^{\frac{d}{2}} /\delta \right) + \sqrt{ \frac{1}{\lambda}  N \sum_{i=1}^N \| x_i (\varepsilon_i \cdot (f - f')) \|_2^2} \\
	&\le \sqrt{2 }\sigma H \ln^{1/2}\left( \calN_{\epsilon,\infty}\left(1 + \tfrac{N}{\lambda}\right)^{\frac{d}{2}} /\delta \right) +  \sqrt{ \frac{N^2 \sigma^2 \epsilon^2}{\lambda}},
	\end{align*} where the second inequality uses the fact that the maximum eigenvalue of $V_N^{-1}$ is at most $1/\lambda$, and the last inequality uses that fact that $\|x_i (\varepsilon_i \cdot (f- f'))\|_2^2 \leq \lvert \varepsilon_i \cdot (f- f')\rvert^2 \| x_i\|_2^2 \leq \|\varepsilon_i \|^2_1\| f -f'\|^2_{\infty} \leq \sigma^2 \epsilon^2$. Simplifying inside the square root yields that the following holds with probability at least $1 - \delta$ for all $f \in \scrF$ and all $N \ge 1$ simultaneously:
	\begin{align*}
	\text{(Unbiased Term)}[f,x] \le \|x\|_{V_N^{-1}} \cdot \left( \sqrt{2 }\sigma H \ln^{1/2}\left( \calN_{\epsilon,\infty}\left(1 + \tfrac{N}{\lambda}\right)^{\frac{d}{2}} /\delta \right) +  \frac{N \sigma \epsilon}{\sqrt{\lambda}}\right)
	\end{align*}

	For the term $\text{(Bias Term)}[f,x] = \lambda x^{\top} V_N^{-1} (\Ast)^{\top} f$, we have that:
	\begin{align*}
	\left(\lambda x^{\top} V_N^{-1} (\Ast)^{\top} f\right)^2 \leq \lambda^2 \| x \|_{V_N^{-1}}^2 \| (\Ast)^{\top} f \|_{V_N^{-1}}^2 \leq \lambda^2 \|x\|_{V_N^{-1}} \frac{1}{\lambda} d H^2 = \lambda dH^2 \|x\|^2_{V^{-1}},
	\end{align*} where we again use the fact that $V_N^{-1}$ has the largest eigen-value being upper bounded by $1/\lambda$, and $\| A^{\top} f\|^2_2 \leq   \sum_{i=1}^{d_1} (A_i\cdot f)^2 \leq d_1 \|A_i\|^2_1 \|f\|^2_{\infty} \leq  dH^2$. Note that the above inequality holds for any $f$ from $\mathcal{F}$ as this is a deterministic inequality. 

	Combine the above two results, we reach the conclusion that for any $f'\in\mathcal{F}$, we have:
	\begin{align*}
	& \left\lvert (\hat{A}_Nx) \cdot f' - (A x) \cdot f' \right\rvert \leq \| x \|_{V_N^{-1}}   \left\| \sum_{i=1}^N x_i (\varepsilon_i \cdot f') \right \|_{V_N^{-1}} + \left\lvert \lambda x^{\top} V_N^{-1} A^{\top} f'  \right\rvert \\
	& \leq  \|x\|_{V_N^{-1}} \cdot \left( \sqrt{2 }\sigma H \ln^{1/2}\left( \calN_{\epsilon,\infty}\left(1 + \tfrac{N}{\lambda}\right)^{\frac{d}{2}} /\delta \right) +  \frac{N \sigma \epsilon}{\sqrt{\lambda}} + \sqrt{\lambda d} H \right)
	\end{align*} 
	which concludes the proof of the first statement.

	In particular, fix an $\Nmax$. Then, if $\calN_{\epsilon,\infty} \ge \left(1 + \tfrac{N}{\lambda}\right)^{\frac{d}{2}}$ and $\epsilon \le H\lambda \sqrt{d}/\sigma \Nmax$, the  then the above is at most the following for all $N \le \Nmax$:
	\begin{align*}
	2\sigma \|x\|_{V_N^{-1}} \cdot \left( H\sqrt{ \ln\left(\calN_{\epsilon,\infty} \cdot \frac{1}{\delta} \right)} + \sqrt{\lambda d} H \right).
	\end{align*}
	In particular, if we set $\lambda = 1$ and enforce $\sigma \ge 1$, then we require $\epsilon \le 1/\sigma \Nmax \le \sqrt{d}/\sigma \Nmax$, and see that the last term on the RHS of the first display is dominated by the first. This gives that the above is at most
	\begin{align*}
	4\sigma H \|x\|_{V_N^{-1}} \sqrt{ \ln\left(\calN_{\epsilon,\infty} \cdot \frac{1}{\delta} \right)}\,.
	\end{align*}\end{proof}

\subsection{Sensitivity to corruption (Lemma~\ref{lemma:linear_global_corrupt} and Lemma~\ref{lemma:linear_sub_corrupt}) \label{app:proof_linear_rl_concentration_corrupt} }

First, we develop a generic bound on the sensitivity of ridge regression to corrupted data.
\newcommand{\ystoch}{y^{\mathrm{stoch}}}
\vspace{0.1in}
\begin{definition}[$C$-corrupted linear regression]  We say that a data-set $\{x_i, y_i\}_{i=1}^{\infty}$ is a $C$-corruption  of a data set $\{x_i, \ystoch_i\}_{i=1}^{\infty}$ if $|\{i: y_i \ne \ystoch_i\}| \le C$.
\end{definition} 
\newcommand{\Ahatstoch}{\widehat{A}^{\mathrm{stoch}}}
\vspace{0.1in}
\begin{lemma}[Linear regression with corrupted data]\label{lemma:linear_regression_corrupt} Let $\{x_i, y_i\}_{i=1}^{\infty}$ be a $C$-currption of a data set $\{x_i, \ystoch_i\}_{i=1}^{\infty}$. Let $\scrF$ denote a family of functions $f: \R^{d_1} \to [0,H]$, and let $\Ahat_N := \Ridge_{\lambda}(\{x_i, y_i\}_{i=1}^{\infty})$ and  $\Ahatstoch_N := \Ridge_{\lambda}(\{x_i, \ystoch_i\}_{i=1}^{\infty})$ denote the ridge regression estimates on the respective datasets of size $N$. Further assume $\|x_i\| \le 1$, $\| y_i \|_1 \leq \alpha$ and $\|\ystoch_i \|_1\leq \alpha$ for all $i$ and $\sigma\in\mathbb{R}^+$.  For any $x$ and any $f\in\scrF$, we have:
\begin{align*}
\left\lvert \left( \Ahat_N x - \Ahatstoch_N x\right)\cdot f \right\rvert \leq 2CH\alpha \| V_N^{-1} x\|_{2}.
\end{align*}
\end{lemma}
\begin{proof}
Let $\calI_{N} := \{i \le N: \ystoch_i \ne y_i\}$, and observe $|\calI_{N}| \le C$ due to the bound on corruptions. Using the closed-form solutions of $\Ahat_N$ and $\Ahatstoch_N$, we have:
\begin{align*}
\Ahat_N x - \Ahatstoch_N x  = \sum_{i=1}^N (y_i - \ystoch_i)x_i^{\top} V_N^{-1} = \sum_{i \in \calI_{N}}\left( y_i - \ystoch_i \right) x_i^{\top} V_N^{-1}.
\end{align*}
Hence, we have:
\begin{align*}
\left\lvert \left( \Ahat_N x - \Ahatstoch_N x\right) \cdot f\right\rvert  &=\left\lvert \sum_{i \in \calI_N} \left( {y}_i - \ystoch_i x_i^{\top} V_N^{-1} x \right)\cdot f \right\rvert  \leq \sum_{i=1}^C \left\lvert x^{\top} V_N^{-1} x_i ({y}_i - \ystoch_i)^{\top} f \right\rvert \\
&\leq C \| V_N^{-1} x \|_2\, \max_{i \in \calI_C} \|x_i ({y}_i - \ystoch_i)^{\top} f\|_2,
\end{align*}
where we have used Cauchy-Schwartz in the last inequality. Again, by Cauchy Schwartz  and Holders inequality $\|x_i ({y}_i - \ystoch_i)^{\top} f\|_2 \le \|x_i\| \|y_i - \ystoch_i\|_1 \|f\|_{\infty} \le 2 \alpha H$.  This  concludes the proof. 
\end{proof}

We then prove Lemma~\ref{lemma:linear_global_corrupt} and Lemma~\ref{lemma:linear_sub_corrupt} which explicitly consider corruptions.

\begin{proof}[Proof of Lemma~\ref{lemma:linear_global_corrupt}]
From Lemma~\ref{lemma:linear_regression_corrupt}, and the fact that there are at most $CH$ corruptions introduced to the data-set, the following holds for all $V$ with $\|V\|_{\infty}\leq H$:
\begin{align*}
& \left\lvert(\phat\kgl\xa - \phatstoch\kgl\xa)^\top V \right\rvert  = \left\lvert (\muhat_{k;gl} - \muhat^{\stoch}_{k;gl}) \phi\xa \cdot V \right\rvert \leq 2CH^2 \|\Lambda_{k;gl}^{-1} \phi\xa \|_2, \\
&\left\lvert \rhat\xa - \rhatstoch\xa  \right\rvert =  \left\lvert (\thetahat_{k;gl} - \thetahat_{k;gl}^{\stoch})\cdot \phi\xa \right\rvert \leq 2CH \| \Lambda_{k;gl}^{-1} \phi\xa \|_{2},
\end{align*} where we use the fact that $r\in [0,1]$, $\|V\|_{\infty} \in [0,H]$ and $\|\phi(x,a)\| \le 1$. The bound now follows from the triangle inequality and the definition of the stochastic estimation event $\eventglestlinear$.
\end{proof}

\begin{proof}[Proof of Lemma~\ref{lemma:linear_sub_corrupt}]
The proof is similar to the proof of Lemma~\ref{lemma:linear_global_stoch} except that here we need to focus on the local dataset associated with base learner $\ell$.  On the event $\eventsubsmp$ in Lemma~\ref{lem:local_subsampled_corruptions} holds, the total number of corruptions in the local dataset associated with base learner $\ell$ is at most $\overline{C}\lsb$. Together this concludes the proof.\end{proof}

We now provide the proof of Lemma~\ref{lemma:integration_policy} which was used above. In order to do so, we first introduce an auxiliary lemma  upper bounding the sum of potential functions $\|\phi\xa\|_{\Lambda_k^{-1}}$ and $\|\Lambda_k^{-1} \phi\xa \|_2$, whose proof we provide in the end of this section.

\begin{lemma}\label{lemma:potential_function}
\hspace{0.1em}
Denote $K$ many arbitrary trajectories as $\{ x\kh, a\kh \}_{h\in [H], k\in [K]}$. Denote \\$\Lambda_k = \sum_{t=1}^{k-1}\sum_{h=1}^H \phi(x_{i;h},a_{i;h})\phi^{\top}(x_{i;h},a_{i;h}) + \lambda I$ with $\lambda \geq 1$. Then we have:
\begin{align*}
& \sum_{k=1}^K \sum_{h=1}^H \sqrt{\phi(x\kh,a\kh)^{\top} \Lambda_k^{-2} \phi(x\kh,a\kh)} \leq \sum_{k=1}^K\sum_{h=1}^H \sqrt{\phi(x\kh,a\kh)^{\top} \Lambda_k^{-1} \phi(x\kh,a\kh)} \\
& \quad \leq \sqrt{ 2H^2 K d\log(1+ KH/\lambda)}.
\end{align*}
\end{lemma}

\begin{proof}[Proof of Lemma~\ref{lemma:integration_policy}]
Denote the random variable $v_k = \mathbf{1}\{c_k = 1\} \left( \sum_{h=1}^H \sqrt{\phi\kh^{\top}\Lambda_k^{-1}\phi\kh} \right)$. Note that $v_k \in [0,H]$ since $\Lambda_k$ has eigenvalues no smaller than $1$. We also have that
\begin{align*}
\Exp_k\left[ v_k \right] = q \Exp_{\pi_k, M_k} \left[ \sqrt{\phi\kh^{\top}  \Lambda_k^{-1}\phi\kh} \right].
\end{align*}  Denote $q \Exp_{\pi_k, M_k} \left[  \sqrt{\phi\kh^{\top}\Lambda_k^{-1}\phi\kh} \right] := \bar{v}_k$.
Hence $\{ \bar{v}_k - v_k\}$ forms a sequence of martinagle differences. Using Azuma-Hoeffding (Lemma~\ref{lem:azuma_hoeffding}), we get that with probability at least $1-\delta/16$, 
\begin{align*}
\left\lvert \sum_{k=1}^K \bar{v}_k - \sum_{k=1}^K v_k \right\rvert \leq {H} \sqrt{ 2\ln(32 /\delta) K }.
\end{align*}
On the other hand, for $\sum_{\tau \in \{k\in [K]: c_k = 1\}} \sum_{h=1}^H \sqrt{\phi\kh^{\top}\Lambda_k^{-1}\phi\kh}$, using Lemma~\ref{lemma:potential_function}, we have:
\begin{align*}
\sum_{\tau \in \{k\in [K]: c_k = 1\}} \sum_{h=1}^H \sqrt{\phi\kh^{\top}\Lambda_k^{-1}\phi\kh} \leq \sqrt{ 2H^2 K d\log(1 + KH)},
\end{align*} where we use the fact that $| \{k\in [K]: c_k = 1\} | \leq K$. This implies that:
\begin{align*}
q\sum_{k=1}^K \Exp_{\pi_k,M_k}\left[ \alpha_1 \sqrt{\phi\kh^{\top}\Lambda_{k}^{-1}\phi\kh }\right] \leq \alpha_1 H \left( \sqrt{2\ln(32/\delta)K} + \sqrt{2Kd\log(1+KH)} \right).
\end{align*}

We can perform the similar analysis for $q\sum_{k}\sum_h \|\Lambda_k^{-1} \phi\kh\|_{2}$, which gives us that with probability at least $1-\delta/16$, 
\begin{align*}
p\sum_{k=1}^K \Exp^{\boldpi_k,\mathcal{M}_k} \left[\alpha_2 \|\Lambda_k^{-1}\phi\kh\|_2\right] \leq \alpha_2 H\left(  \sqrt{2\ln(32/\delta)K} + \sqrt{2Kd\log(1+KH)} \right).
\end{align*}
Combine the above inequalities, we can conclude that with probability at least $1-\delta/8$, 
\begin{align*}
&p\sum_{k=1}^K \Exp^{\boldpi_k,\mathcal{M}_k} \left[ \alpha_1 \|\phi\kh \|_{\Lambda_k^{-1}}  + \alpha_2 \|\Lambda_k^{-1} \phi\kh\|_2\right] \leq H(\alpha_1 + \alpha_2) \left(  \sqrt{2\ln(32/\delta)K} + \sqrt{2Kd\log(1+KH)} \right)\\
& \qquad \leq 2 H(\alpha_1 + \alpha_2) \sqrt{ K \left( \ln(32/\delta) + d\log(1+T) \right)}.
\end{align*}
Now we need to link the  LHS of the above inequality to the nominal model $\mathcal{M}$. Since there are at most $C$ many corrupted episodes, we must have:
\begin{align*}
\sum_{k=1}^K \Exp^{\boldpi_k}  \left[ \alpha_1 \|\phi\kh \|_{\Lambda_k^{-1}}  + \alpha_2 \|\Lambda_k^{-1} \phi\kh\|_2\right]  - \sum_{k=1}^K \Exp^{\boldpi_k,\mathcal{M}_k}  \left[ \alpha_1 \|\phi\kh \|_{\Lambda_k^{-1}}  + \alpha_2 \|\Lambda_k^{-1} \phi\kh\|_2\right]  \leq C (\alpha_1 + \alpha_2).
\end{align*} where we use the fact that $\Lambda^{-1}_k$ has eigenvalues upper bounded by $1$. This concludes the proof. 
\end{proof}

\begin{proof}[Proof of Lemma~\ref{lemma:potential_function}]
For notation simplicity in the proof, we use $\phi\kh$ to represent $\phi(x\kh,a\kh)$.
Since $\lambda \geq 1$, we have that $\Lambda_k$'s minimum eigenvalue is no smaller than 1. Since $\|\phi\xa\|_2 \leq 1$ for any $\xa$ by assumption, we have:
\begin{align*}
\phi\xa^{\top} \Lambda_k^{-1}\phi\xa \leq 1.
\end{align*} From episode $k-1$ to $k$, we have $\Lambda_{k+1} = \Lambda_{k} + \sum_{h=1}^H \phi\kh\phi\kh^{\top}$. By matrix inverse lemma, we have:
\begin{align*}
\det( \Lambda_{k+1} ) = \det(\Lambda_{k} + \sum_{h=1}^H \phi\kh\phi\kh^{\top} ) = \det(\Lambda_{k}) \det\left( \mathbf{I} + \Lambda_{k}^{-1/2} \left( \sum_{h=1}^H \phi\kh\phi\kh^{\top}\right) \Lambda_{k}^{-1/2} \right).
\end{align*} Note that we have:
\begin{align*}
 \det\left( \mathbf{I} + \Lambda_{k}^{-1/2} \left( \sum_{h=1}^H \phi\kh\phi\kh^{\top}\right) \Lambda_{k}^{-1/2} \right) = 1 + \sum_{h=1}^H \phi\kh \Lambda_{k}^{-1} \phi\kh.
\end{align*} 
Also note that for any $x\in [0,1]$, we have that $\log(1+x) \leq x \leq 2\log(1+x)$, which implies that:
\begin{align*}
\log\left(1 + \sum_{h=1}^H \phi\kh\Lambda_{k}^{-1} \phi\kh / H \right) \leq \sum_{h=1}^H \phi\kh \Lambda_{k}^{-1} \phi\kh / H \leq 2\log\left(1 + \sum_{h=1}^H \phi\kh\Lambda_{k}^{-1} \phi\kh / H \right).
\end{align*} So, we have:
\begin{align*}
\sum_{k=1}^K \sum_{h=1}^H \phi\kh^{\top}\Lambda_{k}^{-1} \phi\kh / H \leq 2 \sum_{k=1}^K \log\left(1 + \sum_{h=1}^H \phi\kh\Lambda_{k}^{-1} \phi\kh / H \right) = \log\det(  \Lambda_{k+1}) - \log\det(\Lambda_0).
\end{align*} Now use the fact that the maximum eigenvalue of $\Lambda_k$ is at most $\lambda + kH$, we have:
\begin{align*}
\sum_{k=1}^K \sum_{h=1}^H \phi\kh^{\top}\Lambda_{k}^{-1}\phi\kh \leq 2H\left( \log\det(\Lambda_{K}) - \log\det(\lambda\mathbf{I})  \right) \leq 2H d\log\left( 1 + KH/\lambda \right).
\end{align*}Now apply Cauchy-Schwartz, we have:
\begin{align*}
\sum_{k=1}^K \sum_{h=1}^H \sqrt{ \phi\kh\Lambda_{k}^{-1}\phi\kh} \leq \sqrt{KH} \sqrt{\sum_{k=1}^K \sum_{h=1}^H \phi\kh\Lambda_{k}^{-1}\phi\kh}  \leq \sqrt{ 2H^2 K d\log(1+ KH/\lambda)}.
\end{align*} 

On the other hand, we have:
\begin{align*}
\sqrt{\phi\kh^{\top}\Lambda_k^{-2} \phi\kh} \leq \sqrt{\phi\kh^{\top} \Lambda_k^{-1}\phi\kh},
\end{align*} since due to regularization $\lambda \geq 1$, the eigenvalues of $\Lambda_k$ is no smaller than $1$ which implies that the eigenvalues of $\Lambda_k^{-1}$ is no larger than $1$.  This concludes the above proof.\end{proof}

\subsection{Validity of bonuses  (Lemma~\ref{lem:correct_beta}) \label{app:validbonuses}}
\begin{proof}[Proof of Lemma~\ref{lem:correct_beta}]
We have that for $\beta \ge 1$, $\Nbeta =  A + 12 d^2 \ln( 1 + 16 \beta (A d T H)^2 ) \le A + 24 d^2\ln(1 + 4Ad T H ) + 12 d^2 \ln(\beta)$. With some maniuplations, this gives
\begin{align}
\frac{7}{d+\sqrt{A}}\sqrt{  \ln \tfrac{\Nbeta \lmax 16}{\delta} }  &\le   7\sqrt{ \underbrace{30 \ln \tfrac{  A dT^2 H }{\delta}}_{:= \beta_0^2} + 12 \ln(\beta)}. \label{eq:beta_not}
 \end{align}
 We therefore want to show that $7\sqrt{\beta_0^2 + 12 \ln(\beta)} \le \beta$, or that $\sqrt{\beta^2 - 7^2\beta_0^2} \le 7\sqrt{6 \ln(\beta)}$. By suppossing $\beta^2 \ge 2^2\cdot 7^2 \cdot \beta_0^2$, we need to ensure $\beta \sqrt{3/4} \ge \sqrt{12\ln(\beta)}$, or simply $\beta \ge 7\sqrt{16\ln(\beta)}$. 

 By computing derivatives, we verify that $\beta/\sqrt{\ln(\beta)}$ is non-decreasing. Hence, it suffices to take $\beta = \max\{14\beta_0,\beta_1\}$, where $\beta_1$ is any number satisfying $\beta_1 \ge 7\sqrt{16\ln(\beta_1)}$. We can verify that numerically that this holds for $\beta_1 = 14\beta_0$. Thus, $\beta = 14\beta_0$ suffices.\end{proof}

\subsection{Auxiliary details on bounding confidence sums (Theorem~\ref{thm:main_linear})}
\label{app:proof_linear_rl_confidence_sum}
In order to finalize the proof of Theorem~\ref{thm:final_reg_guarantee}, we need to bound the following terms:
\begin{align*}
 &\sum_k \Exp^{\boldpi_{k,\leq \lst}^{\textsc{master}}}\left[{\sum_{h=1}^H \min\left\{3H, 6\bonus_{k,gl,\lst}\xhah \right\}}\right]  \\ &\sum_k\Exp^{\boldpi_{k,\ell}^{\textsc{master}}}\left[{\sum_{h=1}^H \min\left\{3H, 6 \bonus_{k,sb,\ell}\xhah\right\}}\right], \forall \ell>\lst.
\end{align*} 

The upper bounds of the above two terms can be upper bounded by leveraging the specific forms of the bonuses. This is achieved with the following ``integration'' lemma:
\begin{lemma}[Integration Bound for Linear Setting]
\label{lemma:integration_policy} \hspace{0.1em} Consider an arbitrary sequence of policies $\{\boldpi_k\}_{k=1}^K$. For each $k$, sample $c_k$ from a Bernoulli distribution with parameter $q$. Then we sample a trajectory $\{x\kh,a\kh,r\kh, x\khpl\}_{h\in [H]}$ from a MDP $\mathcal{M}_k$ with $\boldpi_k$.  Denote $\Lambda_k = \sum_{\tau=1}^k \mathbf{1}\{c_\tau = 1\}\sum_{h=1}^H \phi_{\tau,h}\phi_{\tau,h}^{\top} +  I$, i.e. with $\lambda = 1$. Then, with probability at least $1-\delta/8$, we have:
\begin{align*}
& q\sum_{k=1}^K \Exp^{\boldpi_k}\left[ \min\left\{ \alpha H, \alpha_1 \sqrt{\phi\kh^{\top}\Lambda_{k}^{-1}\phi\kh } + \alpha_2 \| \Lambda_k^{-1} \phi\kh \|_2 \right\}  \right] \\
& \qquad \leq \alpha CH q + (H(\alpha_1 + \alpha_2)) 2\sqrt{K \left(\ln\frac{32}{\delta} + d\ln(1+T)\right)},
\end{align*} for any $\alpha \geq 0$, $\alpha_1 \geq 0$, and $\alpha_2 \geq 0$.
\end{lemma}
The proof of the above lemma is provided below. One key ingredient that we use to prove the above lemma is the following inequality which is  used for upper bounding regret in linear bandits as well \cite{abbasi2011improved}:
\begin{align*}
\sum_{k=1}^K \sum_{h=1}^H \sqrt{\phi\kh^{\top} \Lambda_{k,gl}^{-1} \phi\kh} \leq \sqrt{2H^2 K d\ln(1+KH/\lambda)}.
\end{align*} We formally prove the above inequality in Lemma~\ref{lemma:potential_function}.
\vspace{0.1em}

\begin{proof}[Proof of Theorem~\ref{thm:main_linear}]
We focus on policies $\pimaster\klelst$ and $\pimaster\kl$ for $\ell \geq \lst$. For policy $\pimaster\klelst$, every episode $k$ we have probability $\qlelst$ to sample it to generate the corresponding trajectory and for $\pimaster\kl$, every episode $k$ we have probability $\ql$ to sample it to generate the corresponding trajectory.  

Denote $\Lambda_{k, \lelst} = \sum_{\tau=1}^{k} \mathbf{1}\{  f(k,H)\leq \lst \} \sum_{h=1}^H \phi\kh\phi\kh + I$.  Note that $ \Lambda_{k} -\Lambda_{k, \lelst} $ is positive definite, i.e., for any $v\in\mathbb{R}^d$, we have $v^{\top} ( \Lambda_{k} - \Lambda_{k, \lelst} ) v  > 0$ for $v\neq 0$. By the matrix inverse lemma, this implies that $v^{\top} ( \Lambda^{-1}_{k} - \Lambda^{-1}_{k, \lelst} ) v  < 0$ for $v \neq 0$. Hence, for $\bonus_{k, \lst;gl}$, we have:
\begin{align*}
\bonus_{k,\lst; gl} \leq (d+\sqrt{A})H\beta \|\phi\xa \|_{\Lambda^{-1}_{k,\lelst}} + 4CH^2 \| \Lambda_{k,\lelst}^{-1} \phi\xa \|_2.
\end{align*} 

By Lemma~\ref{lemma:integration_policy}, with probability at least $1 - \frac{\delta}{8\lmax}$ (and using $\lmax \le \log K$), it holds:
\begin{align*}
&q_{\le \ell} \sum_k \Exp^{\boldpi_{k,\leq \ell}^{\textsc{master}}}\left[{\sum_{h=1}^H \min\left\{ H, 6\bonus_{k,gl,\ell}\xhah \right\}}\right] \nonumber \\
& \lesssim  q_{\le \ell} \sum_k \Exp^{\boldpi_{k,\leq \ell}^{\textsc{master}}}\left[{\sum_{h=1}^H  \min\left\{ H, (d+\sqrt{A})H\beta \|\phi\xa \|_{\Lambda^{-1}_{k,\leq\ell}} + 4CH^2 \| \Lambda_{k,\leq\ell}^{-1} \phi\xa \|_2  \right\}  }\right] \nonumber\\
& \lesssim CH q_{\le \ell} + H( (d+\sqrt{A})H\beta + CH^2)
2\sqrt{K \left(\ln\frac{8 \lmax }{\delta} + d\ln(1+T)\right)}\nonumber\\
& \lesssim CH q_{\le \ell} + H( (d+\sqrt{A})H\beta + CH^2)
\sqrt{K \left(\ln\frac{1}{\delta} + d\ln(1+T)\right)}\nonumber,
\end{align*} 
where we absorb $\ln(8\lmax)$ into the larger $d\ln(1+T)$ term. In particular, with probability $1 - \delta/8$, the following holds for $\ell = \lst$ (which is random, and therefore requires the union bound argument):
\begin{align}
&q_{\le \lst} \sum_k \Exp^{\pimaster_{k,\leq \lst}}\left[{\sum_{h=1}^H \min\left\{ H, 6\bonus_{k,gl,\lst}\xhah \right\}}\right]  \nonumber\\
&\quad\lesssim CH q_{\le \lst} + H( (d+\sqrt{A})H\beta + CH^2)
\sqrt{K \left(\ln\frac{1}{\delta} + d\ln(1+T)\right)}\nonumber\\
&\quad\lesssim  H( (d+\sqrt{A})H\beta + CH^2)
\sqrt{K \left(\ln\frac{1}{\delta} + d\ln(1+T)\right)}
 \label{eq:integration_lst},
\end{align}
where we absorb the $CH q_{\le \lst}$ into the larger terms.
Similarly, for a fixed $\ell \in [\lmax]$, and using $\overline{C}_{\ell;sb} \lesssim \ln(\lmax/\delta)$ for $\delta \in (0,1/2)$,  the following holds with probability $1 - \delta/8\lmax$:
\begin{align*}
&\ql \sum_k\Exp^{\boldpi_{k,\ell}^{\textsc{master}}}\left[{\sum_{h=1}^H  \min\left\{H, 6 \bonus_{k,sb,\ell}\xhah \right\}}\right] \\
& \lesssim \ql \sum_{k} \Exp^{\boldpi_{k,\ell}^{\textsc{master}}}\left[ \min\left\{ H,  (d+\sqrt{A})H\beta  \|\phi\xa\|_{\Lambda^{-1}\ksbell} + 4\overline{C}_{\ell;sb} H^2 \| \Lambda\ksbell^{-1}\phi\xa \|_2 \right\} \right]  \\
& \lesssim CH \ql + H( (d+\sqrt{A})H\beta + H^2 \log \frac{\lmax}{\delta} )
\sqrt{K \left(\ln\frac{8\lmax}{\delta} + d\ln(1+T)\right)}\\
& \lesssim CH \ql + H( (d+\sqrt{A})H\beta + H^2 \log \frac{\lmax}{\delta} )
\sqrt{K \left(\ln\frac{\lmax}{\delta} + d\ln(1+T)\right)},
\end{align*} 
again absorbing $\ln\frac{8\lmax}{\delta}$ in the last inequality.
By a union bound, the above holds for the all $\ell  \in [\lmax]$ with probability $1 - \delta/8$. Moreover, for $\ell \ge \lst$, $CH \ql \le 1$, so we can simplify the above to
\begin{align}
\ql \sum_k\Exp^{\boldpi_{k,\ell}^{\textsc{master}}}\left[{\sum_{h=1}^H  \min\left\{H, 6 \bonus_{k,sb,\ell}\xhah \right\}}\right] \lesssim H( (d+\sqrt{A})H\beta + H^2 \log \frac{\lmax}{\delta} )
\sqrt{K \left(\ln\frac{1}{\delta} + d\ln(1+T)\right)} \label{eq:integration_ell}
\end{align}

\newcommand{\Eintegratelin}{\mathcal{E}^{\mathrm{integ,lin}}}
Now, let $\Eintegratelin$ denote the event that Eq.~\eqref{eq:integration_ell} and \eqref{eq:integration_lst}; thus $\Eintegratelin$ holds with probability $1 - \delta/4$. Combining with the events $\eventsubestlinear,\eventglestlinear,\eventsub$, we have that with total failure probability $1 - \delta$, an appeal to Theorem~\ref{thm:final_reg_guarantee} yields the following regret quarantee
\begin{align*}
& \Regret_T -CH \lesssim  H^2 \sum_{\ell > \lst} \ell \ql \sum_{k=1}^K \Exp^{\boldpi_{k,\ell}^{\textsc{master}}}\left[{\sum_{h=1}^H \min\left\{H, \bonus_{k,sb,\ell}\xhah\right\} } \right] \\
& \qquad \qquad + CH^2 \log(2C) \qlelst    \sum_{k=1}^K \Exp^{\boldpi_{k,\lelst}^{\textsc{master}}}\left[{\sum_{h=1}^H \min\left\{H, \bonus_{k,gl,\lst}\xhah \right\}}\right] \\
& \lesssim  H^3 \lmax^2( (d+\sqrt{A})H\beta + H^2 \log \frac{\lmax}{\delta} )
\sqrt{K \left(\ln\frac{1}{\delta} + d\ln(1+T)\right)}\\
&\qquad+CH^3 \log(2C)( (d+\sqrt{A})H\beta + CH^2)\sqrt{K \left(\ln\frac{1}{\delta} + d\ln(1+T)\right)}\\
&\lesssim \underbrace{((d+\sqrt{A})H^4\beta (\ln^2 T + C\log(2C))\sqrt{Kd \ln(\frac{T}{\delta})}}_{R_1} \\
&+ \underbrace{H^5 (C^2 \log(2C) + \ln \frac{\lmax}{\delta}) ) \sqrt{K d \ln \frac{T}{\delta}}}_{R_2}.
\end{align*} 

Finally, recalling the definition of $\beta$, we can bound
\begin{align*}
\beta &\lesssim \sqrt{\ln(\tfrac{T Ad }{\delta})}
\end{align*}
With some further manipulations, including bounding $\log 2C \lesssim \ln (T/\delta)$, we find
\begin{align*}
R_1 &\lesssim \sqrt{(d^3 + dA) K} \cdot CH^4 \cdot \ln^{3}\tfrac{TAd}{\delta}.
\end{align*}
similarly, we can bound $R_2 \lesssim  C^2\sqrt{dK} \cdot H^5 \ln^{2}\frac{T}{\delta} $. Together, these bounds yield the desired regret of
\begin{align*}
 \lesssim \sqrt{(d^3 + dA) K} \cdot CH^4 \cdot \ln^{3}\tfrac{TAd}{\delta} + \sqrt{dK} \cdot C^2H^5 \cdot \ln^{2}\frac{T}{\delta}
\end{align*}
\end{proof}

\newpage

\section{Standard concentration inequalities}
\label{app:concentration}
We first state two standard concentration inequalities for martingale difference sequences.
    \begin{definition}[$b$-bounded martingale difference sequnce]
    \vspace{0.1in}
    Consider an adapted sequence $\left\{X_i,\calF_i\right\}$ where $\left\{\calF_i\right\}_{i\geq 0}$ denotes a filtration. $X_i$ is a $b$-bounded martingale difference sequence with respect to $\calF_i$ if $|X_i|\leq b$, $\Exp\left[|X_i|\right]<\infty$, and $\Exp\left[X_i\mid \calF_i\right]=0$ for all $i\geq 0$.
    \end{definition}

    \begin{lemma}[Azuma-Hoeffding]\label{lem:azuma_hoeffding}
    \vspace{.1in}
    Let $\{X_i\}_{i=1}^N$ be a $b$-bounded martingale difference sequence with respect to $\calF_i$. Then, with probability at least $1-\delta$, it holds that:
        \begin{align*}
            \left\lvert \frac{1}{N}\sum_{i=1}^N X_i \right\rvert \leq b\sqrt{\frac{2\ln(2/\delta)}{N}}.
        \end{align*}
        \end{lemma}

    \begin{lemma}[Anytime version of Azuma-Hoeffding]\label{lem:azuma_hoeffding_anytime}
    \vspace{.1in}
    Let $\{X_i\}_{i=1}^{\infty}$ be a $b$-bounded martingale difference sequence with respect to $\calF_i$. Then, with probability at least $1-\delta$, for any $N\in \mathbb{N}^+$, it holds that:
        \begin{align*}
            \left\lvert \frac{1}{N}\sum_{i=1}^N X_i \right\rvert \leq b\sqrt{\frac{2\ln(4 N^2/\delta)}{N}}.
        \end{align*}
    \end{lemma}
    \begin{proof}
    We first fix $N\in\mathbb{N}^+$ and apply standard Azuma-Hoeffding (Lemma~\ref{lem:azuma_hoeffding}) with a failure probability $\delta/N^2$. Then we apply a union bound over  $\mathbb{N}^+$ and use the fact that $\sum_{N > 0} \frac{\delta}{2N^2} \leq \delta$ to conclude the lemma. 
    \end{proof}
        
        \begin{lemma}[Azuma-Bernstein]
        \label{lem:azuma_bernstein}
        \vspace{.1in} Let $X_i$ be a $b$-bounded martingale difference sequence with respect to $\calF_i$. Further let $\sigma^2 = \frac{1}{N}\sum_{i=1}^N \Exp[X_i^2 \mid \calF_{i-1}]$. Then, with probability at least $1-\delta$, it holds that:
        \begin{align*}
            \left\lvert \frac{1}{N}\sum_{i=1}^N X_i \right\rvert \leq \sqrt{\frac{2 \sigma^2\ln(2/\delta)}{N}} + \frac{2b \ln(2/\delta)}{3N}.
        \end{align*}
        \end{lemma}

	\begin{lemma}[Self-Normalized Bound for Vector-Valued Martingales \cite{abbasi2011improved}] 
	Let $\{\varepsilon_i\}_{i=1}^{\infty}$ be a real-valued stochastic process with corresponding filtration $\{\calF_{i}\}_{i=1}^{\infty}$ such that $\varepsilon_i$ is $\calF_i$ measurable, and $\Exp[ \varepsilon_i | \calF_{i-1} ] = 0$ and $\max_i | \varepsilon_i  | \leq \sigma\in\mathbb{R}^+$. Let $\{X_i\}_{i=1}^{\infty}$ be a stochastic process with $X_i\in\mathbb{R}^d$ and $X_i$ being $\calF_t$ measurable. Assume that $V\in\mathbb{R}^{d\times d}$ is positive definite. For any $t$, define $V_t = V + \sum_{i=1}^{t} X_iX_i^{\top}$. Then for any $t\geq 1$, with probability at least $1-\delta$, we have:
	\begin{align*}
	\left\|  \sum_{i=1}^{t} X_i \varepsilon_i \right\|_{V_t^{-1}} \leq 2\sigma^2 \ln\left(\frac{ \det( V_t)^{1/2} \det(V)^{-1/2} }{ \delta  }\right).
	\end{align*}
	\label{lemma:self_normalized}
	\end{lemma}

\end{document}